\documentclass{article}
\usepackage{PRIMEarxiv}

\usepackage[utf8]{inputenc} 
\usepackage[T1]{fontenc}    
\usepackage{hyperref}       
\usepackage{nicefrac}       
\usepackage{microtype}      
\usepackage{lipsum}
\usepackage{fancyhdr}

\usepackage{graphicx}       
\graphicspath{{media/}}     

\usepackage{times}  
\usepackage{helvet}  
\usepackage{courier}  
\usepackage{natbib}  
\usepackage{caption} 
\usepackage{algorithm}
\usepackage[noend]{algpseudocode}
\usepackage{float}

\floatstyle{ruled}
\restylefloat{algorithm}

\usepackage{newfloat}
\usepackage{listings}
\DeclareCaptionStyle{ruled}{labelfont=normalfont,labelsep=colon,strut=off} 
\floatstyle{ruled}
\newfloat{listing}{tb}{lst}{}
\floatname{listing}{Listing}
\usepackage{booktabs}       
\usepackage{amsfonts}       
\usepackage{nicefrac}       
\usepackage{microtype}      
\usepackage{rotating}
\usepackage{tabularx}
\usepackage{pdflscape}
\usepackage{amsmath,amsthm,amssymb, bm}
\usepackage[capitalize,noabbrev]{cleveref}
\usepackage{blkarray}  

\usepackage{amsmath}

\usepackage{subcaption} 
\usepackage{dsfont}
\usepackage{float}
\usepackage{cancel}
\usepackage{enumerate, cases}
\usepackage{thmtools,thm-restate}
\usepackage{mathtools}
\usepackage[textsize=tiny]{todonotes}
\usepackage[none]{hyphenat}
\usepackage{multirow}
\usepackage{makecell}

\usepackage{dblfloatfix}

\usepackage{array}
\usepackage{enumitem}
\usepackage{tikz}
\usetikzlibrary{decorations.pathreplacing}
\usepackage{pgfkeys}
\usetikzlibrary{intersections}
\usepackage[textsize=tiny]{todonotes}
\usepackage{etoolbox} 
\usepackage{pdfpages}

\DeclareMathOperator*{\argmax}{arg\,max}
\DeclarePairedDelimiter\abs{\lvert}{\rvert}%
\makeatletter
\let\oldabs\abs
\def\abs{\@ifstar{\oldabs}{\oldabs*}}

\usepackage{pifont}

\theoremstyle{plain}
\newtheorem{thm}{Theorem}
\newtheorem{lem}{Lemma}

\newtheorem{cor}{Corollary}

\newtheorem{rem}{Remark}
\newtheorem{assu}{Assumption}

\newtheorem{techlem}{Lemma}

\usepackage{amsmath,amssymb,amsthm}

\makeatletter
\@ifundefined{definition}{}{%

}
\makeatother
\theoremstyle{definition}
\newtheorem{defi}{Definition}

\usepackage{newfloat}
\usepackage{listings}
\DeclareCaptionStyle{ruled}{labelfont=normalfont,labelsep=colon,strut=off} 
\floatstyle{ruled}
\newfloat{listing}{tb}{lst}{}
\floatname{listing}{Listing}
\newcommand{\Algo}{{Robust Multi-Agent Exploration with $\phi$-Divergence Uncertainty Set (RoMEX-$\phi$)}}
\newcommand{\Algoname}{RoMEX-$\phi$}

\newcommand{\AlgonameTV}{RoMEX-\text{TV}}

\title{Interactive Distributionally Robust Multi-Agent Learning with General Function Approximation}

\author{Debamita Ghosh$^1$, George K. Atia$^{1,2}$, Yue Wang$^{1,2}$ \\
$^1$ Department of Electrical and Computer Engineering, University of Central Florida, Orlando, FL 32816, USA \\
$^2$ Department of Computer Science, University of Central Florida, Orlando, FL 32816, USA 
}

\begin{document}

\maketitle

\begin{abstract}
Model misspecification poses a fundamental challenge in multi-agent reinforcement learning, where transition uncertainty can be amplified by strategic interactions among agents. Distributionally robust Markov games (DRMGs) offer a principled approach to addressing such uncertainty, yet existing methods often rely on restrictive assumptions or scale poorly to large state and joint action spaces. In this work, we study online multi-agent reinforcement learning in general-sum DRMGs with general function approximation and $\phi$-divergence uncertainty sets. We propose {\Algoname}, a model-free framework that integrates equilibrium-based exploration with dual fitted learning. By developing a functional dual representation of the robust multi-agent Bellman operator, {\Algoname} enables tractable worst-case value estimation from nominal interaction data through a unified objective based on a centered empirical robust discrepancy. To characterize the learning complexity, we introduce the \emph{robust Multi-Agent Decoupling Coefficient} (robust MADC), which captures the intrinsic exploration complexity arising from strategic interactions and adversarial transition uncertainty. Using this complexity measure, we establish sublinear robust regret guarantees for {\Algoname} that depend explicitly on the robust MADC rather than the state-space and joint action-space sizes, thereby replacing unfavorable tabular dependence with intrinsic function-class complexity and improving scalability to large-scale multi-agent systems. We complement our theoretical results with numerical experiments on a scalable general-sum DRMG under total variation uncertainty, where {\Algoname} exhibits substantially greater resilience to transition shifts than its non-robust counterpart while remaining competitive with an exact tabular robust baseline. Together, these results provide a scalable framework for distributionally robust multi-agent reinforcement learning with general function approximation.

\end{abstract}

\section{Introduction}
\label{sec:Introduction}


Multi-agent reinforcement learning (MARL) \citep{albrecht2024multi} provides a fundamental framework for modeling sequential decision-making with multiple interacting agents. It has enabled significant advances in domains such as strategic games \citep{GO_DavidSilver, Vinyals2019GrandmasterLI}, autonomous systems \citep{shalev2016safe, hua2024multiagentreinforcementlearningconnected}, and multi-robot coordination \citep{lowe2017multi, matignon2012independent}. Despite this progress, a key obstacle to real-world deployment is the \emph{sim-to-real gap} \citep{zhao2020sim, peng2018sim}: policies trained under nominal or simulated dynamics often degrade significantly under even small model mismatches due to unmodeled effects, sensor noise, or environmental variations \citep{padakandla2020reinforcement, rajeswaran2016epopt}. This challenge is significantly exacerbated in multi-agent systems due to the \emph{intrinsic coupling of agent behaviors}. A small perturbation affecting one agent can propagate through strategic interactions, inducing cascading adaptations across agents and resulting in highly non-stationary and unstable dynamics \citep{papoudakis2019dealing, canese2021multi, wong2023deep}. Consequently, even minor model inaccuracies can destabilize the entire system, making robustness a fundamental requirement for MARL.

To address such uncertainty, Distributionally Robust Markov Games (DRMGs) provide a principled framework \citep{zhang2020robust, kardecs2011discounted,NeuRIPS2023_DoublePessimismDROfflineRL_Blanchet}. DRMGs model transition uncertainty via some divergence-based ambiguity set of transition probabilities, and adopt the Principle of Pessimism to consider the worst-case expected returns over the sets. This pessimistic formulation ensures baseline performance guarantees under model misspecification and acts as an implicit regularizer, promoting policies that generalize better across uncertain environments \citep{vinitsky2020robust, abdullah2019wasserstein, liu2025distributionally}.  Although promising, existing learning algorithms for DRMGs are still falling short in practical deployments due to two limitations. Firstly, most methods are limited to idealized data generation settings, assuming either access to a generative model \citep{Arxiv2024_SampleEfficientMARL_Shi,jiao_minimax-optimal_2024} or comprehensive offline datasets \citep{NeuRIPS2023_DoublePessimismDROfflineRL_Blanchet, li2025sample}. In many real-world applications—such as autonomous systems \citep{demontis2022surveyreinforcementlearningsecurity} and healthcare \citep{alaa2023drlhealthcare, lu2020deepreinforcementlearningready}—assumptions of high-fidelity simulator or high-quality dataset can be unrealistic, and agents must instead learn online via directly interactions. However, the online DRMG setting introduces a fundamentally harder problem: agents face an off-dynamic \citep{holla2021off,Arxiv2020_OffDynRL_Eysenbach} setting: they interact and collect data from the nominal/training environment, whereas their goal is to find equilibria under the worst case. Therefore, they must balance exploration and exploitation under worst-case dynamics, which can be significantly more challenging that the previous settings.

The second limitation of existing studies is their poor scalability. In DRMGs, it is recently shown that online learning algorithms inherently suffer from the \textit{curse of multi-agency} \citep{farhat2026sampleefficient,Arxiv2024_SampleEfficientMARL_Shi}, i.e., the sample complexity depends on the joint action space size, which becomes exponentially large in large-scale systems. To break this curse, either additional uncertainty set structures \citep{shi2024breaking,jiao_minimax-optimal_2024} or a worst performance estimation oracle \citep{ma2023decentralized} are required. A more recent work \citep{zheng2026distributionally} attempts to improve the scalability of robust MARL by utilizing the linear function approximation technique. However, due to the restrictive linear function class, the DRMG is assumed to be linear, which can be restrictive and hard to verify in practice. On the other hand, general function approximation aims to approximate the value function through some general function classes (like neural networks), and is well understood in standard RL through intrinsic complexity measures such as Bellman rank and Bellman–Eluder dimension \citep{NeurIPS2021_BellmanEluderDim_Jin,jin2022power}. However, these measures rely on a single Bellman operator, which breaks down in general-sum robust multi-agent systems. Therefore, it is still unclear how to design provable complexity measurements and sample-efficient algorithms for robust MARL. Taken together, these limitations expose a fundamental research gap: 

\textbf{\textit{Can we design a robust MARL algorithm for DRMGs without restrictive assumptions and with a scalable sample complexity that breaks the multi-agency curse?}}

In this paper, we provide an affirmative answer to the above question by developing a unified framework for online robust MARL with general function approximation, together with scalable sample-efficiency guarantees. Our main contributions are summarized as follows.

\noindent
\textbf{Scalable dual robust fitted learning for online robust MARL with general function approximation.}
We propose {\Algoname}, the first algorithm for \emph{purely online} learning in general-sum DRMGs under general function approximation (instead of restrictive linear functions) and $\phi$-divergence uncertainty sets. Our approach integrates robustness and exploration \emph{via a dual reformulation of the robust multi-agent Bellman operator}, enabling policy evaluation directly from nominal data. Specifically, the algorithm computes a functional dual minimizer for each candidate joint policy and plugs it into a fitted robust Bellman recursion, yielding a \emph{dual robust fitted-learning procedure} that converts intractable worst-case updates into tractable empirical objectives. This formulation both approximates the worst-case Bellman operator and enables a \emph{global, function-class-level quantification of uncertainty} via the induced robust Bellman residual. Unlike tabular methods relying on per-state–action bonuses, {\Algoname} employs a global regularization mechanism—through centered empirical robust discrepancy and regularized payoffs—that penalizes poor Bellman consistency and guides coordinated exploration across agents, while remaining computationally efficient in large-scale settings.

\noindent
\textbf{Robust MADC: intrinsic complexity for robust multi-agent learning.}
We introduce the \emph{robust Multi-Agent Decoupling Coefficient (robust MADC)}, $d_{\mathrm{MADC}}^{\mathrm{rob},\max}$, as the intrinsic complexity measure governing exploration and regret in online DRMGs with general function approximation. It quantifies how empirical robust Bellman errors—computed from nominal data—control value prediction under worst-case dynamics across interacting agents and policy-induced distributions. Unlike BE-dimension-based measures \citep{ghosh2025scaling}, which are defined for a single Bellman operator, robust MADC captures the strategic coupling and equilibrium structure in general-sum DRMGs while accounting for the statistical mismatch between nominal learning and adversarial evaluation. This yields a unified characterization of function complexity, extending intrinsic-complexity theory to the robust multi-agent setting.

\noindent
\textbf{Sharp regret guarantees via robust intrinsic complexity.}
We then establish the first sublinear regret guarantee for online general-sum DRMGs with general function approximation, showing that the robust regret of {\Algoname} scales as
\(
\mathcal{O}\!\big(
d_{\mathrm{MADC}}^{\mathrm{rob},\max}\sqrt{K}
+ H d_{\mathrm{MADC}}^{\mathrm{rob},\max}
+ H(B_\phi^{\max})^2
\sqrt{K\log(|\mathcal F||\mathcal G|KHm/\delta)}
+ \varepsilon_{\mathrm{dual}}^{\mathrm{approx}}
\big).
\) Our bound reveals two fundamental sources of difficulty: the \emph{intrinsic exploration complexity}, characterized by the robust MADC $d_{\mathrm{MADC}}^{\mathrm{rob},\max}$, and the \emph{robust Bellman estimation cost} incurred when learning the worst-case Bellman operator from nominal interaction data, captured by $B_\phi^{\max}$ and the dual approximation error $\varepsilon_{\mathrm{dual}}^{\mathrm{approx}}$. Importantly, the regret has no explicit dependence on the state-space size $|\mathcal S|$ or the joint action-space size $|\mathcal A|$, thereby replacing the tabular dependence with intrinsic function-class complexity while explicitly accounting for the additional statistical cost of distributional robustness.

\noindent
\textbf{Numerical validation.}
We complement our theory with experiments on a scalable general-sum DRMG under total variation (TV) uncertainty (Section~\ref{app:Experiments}). {\Algoname} exhibits substantially greater resilience to transition shift than its non-robust counterpart and remains competitive with an exact tabular robust baseline, demonstrating the practical effectiveness of robust learning with function approximation.

\section{Related Work}
\label{app:Related_Work}

We discuss the relevant literature on DRMGs and its connections to function approximation below.

\textbf{Tabular Distributionally Robust Markov Games.}
A substantial body of work studies distributionally robust Markov games (DRMGs) in the tabular setting under different data-access models. Early works \citep{kardecs2011discounted, zhang2020robust,roch2025distributionally} address foundational questions such as equilibrium existence and convergence under uncertainty. In the offline setting, \citep{NeuRIPS2023_DoublePessimismDROfflineRL_Blanchet} proposes a unified framework (P$^2$M$^2$PO) with sample complexity $\mathcal O(H^5 |\mathcal S|^2 |\mathcal A|^2/\varepsilon)$, while subsequent works \citep{Arxiv2024_SampleEfficientMARL_Shi,jiao_minimax-optimal_2024} extend these results to the generative model setting, achieving minimax-optimal guarantees under strong assumptions.

In the more realistic online setting, where agents learn purely from interaction, only a limited number of works are available. In particular, \citep{ma2023decentralized} studies decentralized robust learning but requires the uncertainty level to satisfy $\sigma_i \le \max\{\varepsilon/(|\mathcal S|H^2), p_{\min}/H\}$, which can be restrictive. More recently, \citep{farhat2026sampleefficient} establishes regret guarantees under divergence-based uncertainty sets, highlighting challenges such as support mismatch and dependence on the joint action space. Despite these advances, all existing approaches rely on explicit state–action enumeration and per-state exploration, and thus do not scale to large or continuous domains. In contrast, our work considers \emph{online DRMGs under general function approximation}, where both robustness and exploration must be addressed at the level of function classes, and the robust Bellman operator cannot be directly estimated from data.

\textbf{Single-agent distributionally robust RL with function approximation.}
A large body of work studies sample-efficient single-agent RL with function approximation under nominal dynamics, where learnability is characterized via intrinsic complexity measures such as the Eluder dimension \citep{russo2013eluder}, Bellman rank \citep{jiang2017contextual}, and Bellman–Eluder (BE) dimension \citep{NeurIPS2021_BellmanEluderDim_Jin,jin2022power}. These frameworks show that efficient exploration is possible without explicit dependence on $|\mathcal S|$ and $|\mathcal A|$, provided Bellman prediction errors generalize across policies and state distributions. Subsequent works extend these ideas to structured settings, including linear and bilinear models \citep{sun2019model, foster2021statistical, du2021bilinear}.

Extending these techniques to the robust setting is substantially more challenging. The robust Bellman operator involves worst-case expectations that are not directly observable from nominal data and may violate standard realizability assumptions, making Bellman-error-based generalization fundamentally more delicate due to the mismatch between nominal data and worst-case evaluation. As a result, existing work is largely restricted to structured settings, particularly linear function approximation \citep{Arxiv2022_OfflineDRRLLinearFunctionApprox_Ma, liu2024distributionally, Arxiv2024_UpperLowerDRRL_Liu,NeurIPS2021_OnlineRobustRLModelUncertainty_Wang}. In the offline (or off-dynamics) setting, recent works \citep{liu2024distributionally,Arxiv2024_UpperLowerDRRL_Liu} develop robust least-squares value iteration methods under coverage assumptions, achieving regret (or sample complexity) guarantees of order 
\(
\widetilde{\mathcal O}\!\big(d H \min\{H,1/\sigma\}\sqrt{K}\big)
\),
together with nearly matching lower bounds up to a factor of $\sqrt{H}$. These results highlight how uncertainty-dependent contraction effects (e.g., $\min\{H,1/\sigma\}$ under TV divergence) can be exploited under strong coverage conditions. In contrast, \citep{he2023nearly} studies online learning in linear MDPs under the standard (non-robust) setting, and does not address distributional robustness. Taken together, these works illustrate a fundamental gap: existing robust methods either rely on offline coverage assumptions or are restricted to non-robust online settings, whereas our work addresses fully online learning in general-sum DRMGs under distributional robustness and general function approximation.

Beyond linear models, existing approaches either assume strong coverage (offline or hybrid regimes) \citep{NeurIPS2022_RobustRLOffline_Panaganti, Arxiv2024_ModelFreeRobustRL_Panaganti} or rely on additional structure. Recent work \citep{ghosh2025scaling} extends intrinsic-complexity ideas to the robust setting via a dual reformulation, enabling sample-efficient learning under general function approximation. However, these approaches rely either on structural assumptions or single-agent formulations, and do not extend to multi-agent settings with strategic coupling.

\textbf{Non-robust MARL with function approximation.}
A growing literature studies sample-efficient MARL with function approximation, broadly divided into \emph{zero-sum} and \emph{general-sum} settings. In zero-sum games, intrinsic-complexity frameworks such as BE dimension extend naturally, yielding efficient algorithms under low complexity \citep{huang2021towards, jin2022power, NeurIPS2021_BellmanEluderDim_Jin, du_bilinear_2021, zhong2022posterior}. These results rely on the minimax structure and a single Nash value function, and therefore do not generalize to general-sum games. Beyond zero-sum settings, a line of work studies \emph{online linear Markov games}, including centralized approaches \citep{xie2020learning, pmlr-v167-chen22d, cisneros2023finite} and decentralized approaches \citep{cui2023breaking, wang2023breaking, dai2024refined}. While decentralized methods mitigate the curse of multi-agency via independent learning, they rely on stronger structural assumptions and more complex algorithmic design. Moreover, these works are largely restricted to linear function classes and do not provide a unified intrinsic-complexity characterization that governs exploration under strategic interaction. In \emph{general-sum settings}, each agent has its own value function and equilibrium behavior must be defined via NE, CCE, or CE, with no unified Bellman fixed point. Recent works address this via model-based DEC frameworks \citep{chen2022unified, foster2023complexity}, model-free CCE/CE methods \citep{zhan2022decentralized, wang2023breaking, cui2023breaking}, and hybrid structured approaches \citep{ni2022representation}.

The most relevant development for our work is the MAMEX framework \citep{Xiong2023SampleEfficientMR}, which introduces the \emph{Multi-Agent Decoupling Coefficient (MADC)} as an intrinsic complexity measure. MADC quantifies how in-sample Bellman error translates into out-of-sample value prediction under strategic interaction, explicitly capturing multi-agent coupling. Its low-complexity regime includes extensions of BE-dimension models \citep{NeurIPS2021_BellmanEluderDim_Jin,jin2022power,huang2021towards}, bilinear classes \citep{du2021bilinear}, and witness-rank models \citep{sun2019model, huang2021towards}, making it a natural foundation for scalable general-sum MARL.

\textbf{Online robust MARL with linear function approximation.}
A recent work by \citep{zheng2026distributionally} studies online distributionally robust general-sum Markov games under \emph{linear} function approximation. They propose a least-squares value-iteration-type algorithm for learning approximate robust CCE and establish a sample-complexity bound of order
\(
\widetilde{\mathcal O}\!\left(
H^3 d_{\mathrm{lin}}^2
\big(\min\{H,1/\sigma\}\big)^2
\varepsilon^{-2}
\right)
\)
under TV-divergence uncertainty, together with a nearly matching lower bound of order
\(
\widetilde{\Omega}\!\left(
H^2 d_{\mathrm{lin}}^2
\big(\min\{H,1/\sigma\}\big)^2
\varepsilon^{-2}
\right).
\) 
Their result provides an important step beyond tabular robust MARL, achieving sharp dependence on the feature dimension by exploiting linear structure. However, the analysis is tied to linear realizability and TV-type uncertainty, and does not yield an intrinsic-complexity characterization for broader function classes. In contrast, our framework applies to general function approximation and multiple $\phi$-divergence uncertainty sets, with sample complexity governed by the robust MADC, which captures how Bellman errors under nominal data translate into value prediction under worst-case dynamics.

\textbf{Challenges in extending to robust multi-agent learning.}
Crucially, the guarantees and techniques developed for non-robust MARL with function approximation do \emph{not} extend directly to the robust setting due to the intractability and non-observability of the robust Bellman operator. In DRMGs, the standard Bellman operator is replaced by a \emph{robust multi-agent Bellman operator} involving worst-case expectations over uncertainty sets, which introduces several fundamental challenges. 
(i) \emph{Non-observability of the robust operator:} the worst-case Bellman backup cannot be evaluated from nominal trajectories, as data is collected under the nominal kernel while performance is defined under adversarial dynamics, creating a fundamental statistical mismatch. 
(ii) \emph{Breakdown of standard realizability and concentration:} unlike the nominal setting, where Bellman errors can be controlled via standard generalization arguments, the robust operator is inherently non-linear and may violate realizability, making error propagation significantly more delicate. 
(iii) \emph{Coupled equilibrium structure:} in general-sum games, each agent’s value depends on others’ strategies, and learning requires controlling deviations across agents under worst-case dynamics, which is not captured by single-agent complexity measures or zero-sum analyses. 

These challenges are further amplified in the \emph{model-free} setting considered in this work, where agents must learn solely from interaction without access to a simulator or an offline dataset. In this regime, exploration must simultaneously account for statistical uncertainty and adversarial model shift, while operating at the level of function classes rather than individual state–action pairs. As a result, existing intrinsic-complexity frameworks—whether BE dimension in single-agent RL or MADC in nominal MARL—cannot be directly applied. Taken together, these limitations expose a fundamental gap: \emph{existing works either address robustness in single-agent settings or multi-agent learning under nominal dynamics, but no prior framework simultaneously captures robustness, equilibrium learning, and general function approximation in a purely online, model-free setting.}  Our work addresses this gap by developing a unified framework that combines dual-based robust Bellman estimation, equilibrium-based policy learning, and a robust extension of intrinsic complexity (robust MADC), leading to the first sample-efficient, model-free algorithm for online general-sum DRMGs with general function approximation.

\section{Preliminaries and Problem Formulation}
\label{sec:Problem_Setup}
Robust MARL is generally formulated as a distributionally robust Markov game (DRMG). 
A finite-horizon $n$-player general-sum DRMG is specified as $\mathcal{MG}_{\text{rob}} = \big\{\mathcal{M}, \mathcal{S}, \mathcal{A}, H, P^\star, \{\mathcal{P}_i\}_{i\in\mathcal{M}}, {\bm r}, \rho_0 \big\}$,
where $\mathcal{M}=\{1,...,m\}$ is the set of  $m$ agents, $\mathcal{S}=\{1,2,\dots,S\}$ denotes the finite state space, $\mathcal{A}$ denotes the joint action space for all agents as $\mathcal{A}= \mathcal{A}_1 \times \cdots \times \mathcal{A}_m$, where $\mathcal{A}_i = \{1,2,\dots,A_i\}$ is the action space of agent $i$, $H$ denotes the horizon length, $P^\star$ is the nominal transition kernel, and ${\bm r}= \{r_{i}\}_{1 \leq i \leq m}$ is the reward function where $r_{i} : \mathcal{S} \times \mathcal{A} \mapsto [0, 1]$. 
Agents in a DRMG maintain their own uncertainty sets of transition kernels $\mathcal{P}_i$, to capture the potential environment uncertainties from their perspective. At each step, the environment transitions follow an arbitrary kernel from the uncertainty set.  We follow the standard \textit{agent-wise $(s,\bm{a})$-rectangular uncertainty structure} \citep{Arxiv2024_SampleEfficientMARL_Shi,shi2024breaking,zhang2020robust,farhat2026sampleefficient} due to its computational tractability \citep{INFORM2005_RobusDP_Iyengar, INFORMS2013_RobustMDP_wiesemann, PMLR2021_DROTabularRL_Zhou, NeuRIPS2023_CuriousPriceDRRLGenerativeMdel_Shi}\footnote{\textcolor{black}{Robust MDPs without the rectangularity assumption can be NP-hard to solve \citep{INFORMS2013_RobustMDP_wiesemann}.}}. Specifically, for each agent $i$, $\mathcal{P}_i$ is defined independently across all horizons, states, and joint actions as $\mathcal{P}_i= \bigotimes_{(h, s, {\bf a})\in [H]\times \mathcal{S}\times\mathcal{A}} \mathcal{P}_{i,h,\phi}^{\sigma_i}(s,{\bf a}),$ where $\otimes$ denotes the Cartesian product.
We consider the  $\phi$-divergence uncertainty set standard in literature \citep{Arxiv2023_TowardsMinimaxOptimalityRobustRL_Clavier,NeuRIPS2023_CuriousPriceDRRLGenerativeMdel_Shi,NeurIPS2022_RobustRLOffline_Panaganti,AnnalsStat2022_TheoreticalUnderstandingRMDP_Yang,NeuRIPS2024_UnifiedPessimismOfflineRL_Yue,zhang2025modelfree}, that each local uncertainty set $\mathcal{P}^{\sigma_i}_{i,h, \phi}(s,\mathbf{a})$ is centered around a \textit{nominal transition kernel} $P^\star$, through some $\phi$-divergence, $\phi\big(P,\ P_h^\star(\cdot|s,\bm{a})\big) = \sum\limits_{s' \in \mathcal{S}} \phi\Big(P(s')\big/P_h^\star(s'|s,\bm{a})\Big) P_h^\star(s'|s,\bm{a})$:
\begin{defi}[$\phi$-Divergence Uncertainty Set]
\label{def:f_divergence_uncertainty}
   The \(\phi\)-divergence uncertainty set is defined as $\mathcal{P}^{\sigma_i}_{i,h,\phi}(s,{\bf a}) \triangleq \left\{ P \in \Delta(\mathcal{S}): \phi\Big(P,P_h^\star(\cdot|s,{\bf a})\Big)\leq \sigma_i, \text{Supp}(P)\subseteq \text{Supp}(P_h^\star(\cdot|s,{\bf a})) \right\}$.
\end{defi}
Support inclusion is essential for online robust RL because states outside the nominal support cannot be observed through online interaction, creating an information bottleneck that can lead to exponential regret \citep{Arxiv2024_DRORLwithInteractiveData_Lu,ghosh2025scaling,farhat2026sampleefficient}. We consider TV, $\chi^2$, and KL uncertainty sets, for which support inclusion holds automatically under $\chi^2$ and KL, and under TV given Assumption~\ref{ass:vanisin_minimal}.

\textbf{Robust Value Functions.} Considering a DRMG, we aim for each agent to maximize its performance against the worst-case transition kernels within its respective uncertainty set. Each agent follows its own policy defined as 
\(
\pi_i = \{\pi_{i,h} : \mathcal{S} \to \Delta(\mathcal{A}_i)\}_{h=1}^{H}.
\)
The robust Q function for each agent $Q_{i,h}^{\pi,\sigma_i}$ for an initial state $s$ and joint action ${\bm a}$ is defined as $Q_{i,h}^{\pi,\sigma_i}(s,\bm{a}) \triangleq \inf_{\tilde{P} \in \mathcal{P}_i} \mathbb{E}_{\pi,\mathcal{\tilde{P}}}\big[\sum_{t=h}^{H}r_{i,t}(s_t,\bm{a}_t) \big| s_h = s, \bm{a}_h = \bm{a}\big]$, where the expectation is taken over the randomness of the
joint policy \(
\pi = \{\pi_h:\mathcal{S} \to \Delta(\mathcal{A})\big\}_{h=1}^{H}\), and the kernel $\tilde{P}$ over its own uncertainty set, and one can similarly define the corresponding robust value function $V_{i,h}^{\pi,\sigma_i}(s)$ for an initial state $s$. For any joint policy $\pi$, the robust value functions to satisfy the robust Bellman equations \citep{INFORM2005_RobusDP_Iyengar}
\begin{align}
    Q_{i,h}^{\pi, \sigma_i} (s,\bm{a}) = r_{i,h}(s,\bm{a}) + \mathbb{E}_{\mathcal{P}_{i,h,\phi}^{\sigma_i}}[V_{i,h+1}^{\pi,\sigma_i}]\quad \text{and} \quad  V_{i,h}^{\pi, \sigma_i} (s) = \mathbb{E}_{\bm{a} \sim \pi}[Q_{i,h}^{\pi,\sigma_i}(s,\bm{a})], \label{eq:bellman_Q_V}
\end{align}
where $\mathbb{E}_{\mathcal{P}_{i,h,\phi}^{\sigma_i}(s, {\bm a})}[V_{i,h+1}^{\pi,\sigma_i}]  \triangleq \inf_{P_h \in \mathcal P_i}\mathbb E_{s^{\prime} \sim P_h}\Big[V_{i,h+1}^{\pi,\sigma_i}(s')\Big]$ is the worst-case expectation.

\textbf{Robust best-response policy.} For any given joint policy $\pi$, let $\pi_{-i}$ denote the collection of marginal policies of all agents except agent $i$. The best-response policy of agent $i$ to $\pi_{-i}$, denoted by $\pi^{\dagger,\sigma_i}_i(\pi_{-i})$, is defined as the policy that maximizes its own robust value function at a given step $h$ and state $s$, i.e.,
\(
\pi^{\dagger, \sigma_i}_{i}(\pi_{-i}) \triangleq  \underset{\pi_{i}^{\prime} \in \Delta(\mathcal{A}_i)}\argmax V_{i,h}^{\pi_{-i} \times \pi_{i}^{\prime}, \sigma_i}(s),
\)
and we denote the corresponding robust value function as    $ V_{i,h}^{\dagger, \pi_{-i}, \sigma_i}(s) \triangleq \underset{\pi_{i}^{\prime} \in \Delta(\mathcal{A}_i)}\max V_{i,h}^{\pi_{i}^{\prime} \times \pi_{-i}, \sigma_i} (s).$

\textbf{DRMG Equilibria.} Due to potentially conflicting objectives among agents, the goal in a DRMG is to find an appropriate equilibrium policy under which no agent can improve its robust value through unilateral deviation \citep{fudenberg1991game}. We consider three standard equilibrium notions adapted to the robust setting: {\it robust Nash Equilibrium (NE)}, {\it robust Coarse Correlated Equilibrium (CCE)}, and {\it robust Correlated Equilibrium (CE)}, whose existence has been established in prior work \citep{NeuRIPS2023_DoublePessimismDROfflineRL_Blanchet}. We present the definition of robust NE, robust CCE and CE below . Since computing a Nash equilibrium is generally computationally intractable \citep{daskalakisnash09}, we also consider the more computationally tractable CCE and CE notions.

\begin{defi} [$\epsilon$-Robust NE]
\label{def:gap_NE}
Let $\pi \in \Delta(\mathcal{A}_1) \times \dots \times \Delta(\mathcal{A}_m)$ be a product policy. It is an \textit{$\epsilon$-robust-NE} if $\mathrm{Gap}^{\mathrm{NE}}(\pi, s) \triangleq \max_{i \in \mathcal{M}} \big\{ V_{i,1}^{\dagger, \pi_{-i}, \sigma_i}(s_1) - V_{i,1}^{\pi, \sigma_i}(s_1) \big\} \le \epsilon.$
\end{defi}

Following the formulation of Robust NE, one can relax the independent policy assumption as shown in \citep{pmlr-v139-liu21z} by considering correlated joint policies and define robust CCE as follows:
\begin{defi} [$\epsilon$-Robust CCE]
\label{def:gap_CCE}
Let $\pi \in \Delta(\mathcal{A})$ be a correlated joint policy, and it is an \textit{$\epsilon$-robust-CCE} if $\mathrm{Gap}^{\mathrm{CCE}}(\pi, s) \triangleq \max_{i \in \mathcal{M}} \big\{ V_{i,1}^{\dagger, \pi_{-i}, \sigma_i}(s) - V_{i,1}^{\pi, \sigma_i} \big\}(s) \le \epsilon, \forall s\in \mathcal{S}.$
\end{defi}
We define a strategy modification $\psi_{h,s} \colon \mathcal{A}_i \to \mathcal{A}_i$ where $\psi \in \Psi_i$, the set of all possible strategy modifications for agent $i$. Applying the strategy modification $\psi$ to a joint policy $\pi$ induces a new joint policy $\psi \diamond \pi$, which replaces agent $i$'s action $a_i$ with $\psi_{h,s}(a_i)$.
\begin{defi} [$\epsilon$-Robust CE]
\label{def:gap_CE}
Let $\pi \in \Delta(\mathcal{A})$ be a correlated joint policy. We say $\pi$ is an \textit{$\epsilon$-robust-CE} if for any $s \in \mathcal{S}$, $\mathrm{Gap}^{\mathrm{CE}}(\pi, s) \triangleq \max_{i \in \mathcal{M}} \big\{ \max_{\psi \in \Psi_i}V_{i,1}^{\psi \diamond \pi, \sigma_i}(s) - V_{i,1}^{\pi, \sigma_i} \big\}(s) \le \epsilon.$
\end{defi}

In the online setting, agents need to learn from direct interaction with the nominal environment $P^{\star}$. The interaction will be deployed for $K$ episodes, each starting from $s_1^k$, execute the policy $\pi^k$ to collect new samples, and update the policy. To evaluate learning performance, we utilize the robust regret \citep{farhat2026sampleefficient}, which measures the accumulated equilibrium gap.
\begin{defi} [Robust Regret]
\label{def:robust_regret}
    Let $\pi^{k}$ be a policy executed at episode $k$. The robust regret over $K$ episodes is defined as
        $\mathrm{Reg}^{\{\mathrm{NE, CE, CCE}\}} (K) \triangleq \sum_{k=1}^{K} \mathrm{Gap}^{\{\mathrm{NE, CE, CCE}\}}(\pi^k, s_1^k).$
\end{defi}

\section{Robust MARL with Function
Approximation}
In this subsection, we discuss the structural and computational challenges of online DRMG with general function approximation, and present our approach to overcoming them.

\textbf{Model-Free General Function Approximation.}
In DRMGs, large state–action spaces make tabular MARL methods computationally intractable \citep{farhat2026sampleefficient,ma2023decentralized}. To address this, we adopt a model-free function approximation framework, where each agent’s robust action-value function is represented using a structured hypothesis class  $\mathcal{F} = \bigotimes_{i \in \mathcal{M}} \mathcal{F}_{i}, \quad \text{where}\quad \mathcal F_{i}=\bigotimes_{h=1}^H \mathcal F_{i,h},$

where for each agent \(i\) and stage \(h\), \(\mathcal{F}_{i,h} = \{f_{i,h} : \mathcal{S} \times \mathcal{A} \to \mathbb{R}\}\) consists of candidate functions (e.g., low-dimensional parametric models like neural networks) used to approximate the  robust value functions. For any fixed agent \(i\), we denote a candidate function as \(f_i = \{f_{i,h}\}_{h=1}^H\), and denote the one-step robust Bellman backup under policy \(\pi\) as
\begin{align}
    \Big[\mathcal{T}_{i,h}^{\pi, \phi,\sigma_i} f_{i,h+1}\Big] (s,\bm{a}) \triangleq r_{i,h}(s,\bm{a}) + \mathbb{E}_{\mathcal{P}_{i,h,\phi}^{\sigma_i}(s, {\bm a})}\Big[V^{\pi,f}_{i,h+1}\Big],\label{eq:robust_bellman_op}
\end{align}
where the policy-averaged value function is given by $V_{i,h}^{\pi,f_i}(s')\triangleq\mathbb{E}_{\bm a' \sim \pi_{h}(\cdot|s')}[f_{i,h}(s',{\bm a}')]$.

To enable statistically and computationally efficient learning with function approximation, prior work imposes \emph{representation conditions} that characterize how the function class F interacts with the underlying (robust) MDP dynamics \citep{russo2013eluder,jiang2017contextual,sun2019model,wang_reinforcement_2020,NeurIPS2021_BellmanEluderDim_Jin,NeurIPS2022_RobustRLOffline_Panaganti}. n standard RL, these include \emph{realizability} and \emph{completeness}. However, in the robust setting, realizability of the optimal robust value function is often overly restrictive due to the presence of worst-case expectations. For instance, in the linear function case \citep{zheng2026distributionally}, the underlying DRMG is restrictively assumed to be linear to ensure realizability. Instead, following recent advances in robust RL \citep{Arxiv2022_OfflineDRRLLinearFunctionApprox_Ma,NeurIPS2024_MinimaxOptimalOfflineRL_Liu,Arxiv2024_UpperLowerDRRL_Liu,Arxiv2024_SampleComplexityOfflineLinearDRMDP_Wang}, realizability is effectively relaxed and can be implicitly ensured through appropriate assumptions on the dual function class used to represent the robust Bellman operator.

Therefore, in our setting, we adopt the standard \emph{completeness} condition to guarantee that the function class remains closed under the robust Bellman backup:
\begin{assu}[Robust Completeness]
\label{ass:completeness}
    Let $\pi$ be any joint policy. For any $i \in \mathcal{M}$ and $f_{i,h+1} \in \mathcal{F}_{i,h+1}$, we have 
    \(
    \mathcal{T}_{i,h}^{\pi, \phi,\sigma_i} f_{i,h+1} \in \mathcal{F}_{i,h}.
    \)
\end{assu}
This assumption ensures closure under the robust Bellman operator and directly implies robust realizability, eliminating the need for a separate realizability assumption \citep{ghosh2025scaling}.

\textbf{Dual Formulation of robust Bellman operator and functional optimization.} Using standard duality \citep{yangdual22}, the worst-case expectation in the robust Bellman operator is gven by
\begin{equation}
\mathbb{E}_{\mathcal{P}_{i,h,\phi}^{\sigma_i}(s,\bm a)}[V^{\pi,f}_{i,h+1}]
=
\inf_{\lambda>0,\nu\in\mathbb{R}}
\Big\{
\lambda\sigma_i-\nu
+
\lambda\,\mathbb{E}_{s'\sim P_h^\star}
\big[
\phi^\star\!\big((\nu - V^{\pi,f}_{i,h+1}(s'))/\lambda\big)
\big]
\Big\},
\label{eq:general_phi_dual}
\end{equation}
which yields the dual form of the robust Bellman operator
\begin{equation}
\left[\mathcal{T}^{\pi,\phi,\sigma_i}_{i,h} f_{i,h+1}\right](s,\bm a)
=
r_{i,h}(s,\bm a)
-
\inf_{\lambda,\nu}
\mathbb{E}_{s'\sim P_h^\star}
\big[
\mathrm{Loss}_{i,h,\phi}(f_i; s,\bm a,s';\lambda,\nu)
\big],
\label{eq:dual_rob_bellman_operator}
\end{equation}
where $\mathrm{Loss}_{i,h,\phi}$ denotes the pointwise dual integrand for a single next-state $s^\prime$ and joint action ${\bm a}$ for agent $i \in \mathcal M$: $\mathrm{Loss}_{i,h,\phi}(f_{i};s,\bm{a},s^\prime;\lambda,\nu) \triangleq \lambda \sigma_i  - \nu + \lambda \phi^\star \Big(-\big(V^{\pi,f}_{i, h+1}(s^\prime) - \nu\big)\big/\lambda\Big).$
While the dual reformulation removes the need to optimize over transition kernels, evaluating the resulting operator still involves solving an inner optimization over $(\lambda,\nu)$ for each state–action pair $(s,\bm a)$. Consequently, the computational cost scales with the size of the state space, rendering direct evaluation of the robust Bellman operator impractical in large-scale settings. To address this, we adopt a \emph{functional optimization} perspective \citep{NeurIPS2022_RobustRLOffline_Panaganti,ghosh2025scaling}, replacing pointwise optimization with a single global problem over the state–action space. Let $\mu^\pi_h$ denote the state–joint action visitation distribution induced by policy $\pi$ under the nominal kernel. We define the dual loss
\begin{equation}
\mathrm{LossDual}_{i,h,\phi}(g; f_i)
=
\mathbb{E}_{(s,\bm a)\sim \mu^\pi_h}
\Big[
\mathbb{E}_{s'\sim P_h^\star}
\mathrm{Loss}_{i,h,\phi}(f_i; s,\bm a,s'; g(s,\bm a))
\Big],
\label{eq:dual_loss}
\end{equation}
where $g(s,\bm a) = (g_i^\lambda, g_i^\nu)$ denotes a dual function.\footnote{Here, $g$ ranges over an appropriate function space (e.g., $\mathcal{L}^1(\mu^\pi)$), and $\mathrm{Loss}_{i,h,\phi}$ is the pointwise dual integrand as above. This formulation is equivalent to the original robust expectation under mild conditions; see \citep{NeurIPS2022_RobustRLOffline_Panaganti,ghosh2025scaling} for details.} 
We can further show that minimizing the dual loss in \eqref{eq:dual_loss} over $\mathcal{L}^1(\mu,\mathbb{R}^2)$ exactly recovers the robust Bellman operator in expectation under $\mu_h^\pi$ (refer to Lemma \ref{lem:equiv_loss_dual} in Appendix~\ref{app:dual} for details). Thus, instead of solving an inner optimization for each $(s,\bm a)$, we can optimize for a single dual function that captures worst-case transitions globally.

This observation allows us to replace the intractable pointwise optimization with a single functional optimization of the empirical dual loss of \eqref{eq:dual_loss},  $\widehat{\mathrm{LossDual}_{i,h,\phi}}(g; f_i)$. We further introduce a structured function class $\mathcal{G}=\bigotimes_{i,h}\mathcal{G}_{i,h}$, where each \(
\mathcal G_{i,h} = \{g_i=(g^\lambda_{i},g^\nu_{i}) :\mathcal{S}\times\mathcal{A}\to\Theta_{\phi} \subseteq \mathbb R^2\}
\), to replace $\mathcal{L}^1(\mu,\mathbb{R}^2)$ and to  approximate the dual solution. We adopt the following assumption to quantify the approximation error \citep{NeurIPS2022_RobustRLOffline_Panaganti,Arxiv2024_ModelFreeRobustRL_Panaganti,ghosh2025scaling}. 

\begin{assu}
\label{ass:approx_dual_realizability}
There exists a constant $\varepsilon^{\mathrm{approx}}_{\mathrm{dual}}$, such that for any agent $i \in \mathcal M$, horizon $h \in [H]$, function $f_i \in \mathcal{F}$, and policy $\pi$, 
\(
\inf_{g\in \mathcal{G}_i} \mathrm{LossDual}_{i,h,\phi}(g; f_i)
-
\inf_{g\in \mathcal{L}^1_i(\mu^\pi;\mathbb R^2)} \mathrm{LossDual}_{i,h,\phi}(g; f_i)
\;\leq\; \varepsilon^{\mathrm{approx}}_{\mathrm{dual}}.
\)
\end{assu}

\textbf{Empirical dual minimization and operator approximation.} Based on the discussion above, we can approximate the robust Bellman operator by first minimizing the empirical dual loss over $\mathcal{G}$:
\begin{equation}
\widehat{g_i}^f
=
\arg\min_{g\in\mathcal{G}}
\widehat{\mathrm{LossDual}}_{i,h,\phi}(g; f_i),
\label{eq:dual_minimizer}
\end{equation}
and defining the approximate Bellman update by plugging $\widehat{g_i}_i^f$ into the dual formulation:
\begin{equation}
\big[\widehat{\mathcal{T}}^{\pi,\phi,\sigma_i}_{i,h} f_i\big](s,\bm a)
=
r_{i,h}(s,\bm a)
-
\mathbb{E}_{s'\sim P_h^\star}
\big[
\mathrm{Loss}_{i,h,\phi}(f_i; s,\bm a,s'; \widehat{g_i}_i^f(s,\bm a))
\big].
\label{eq:approx_operator}
\end{equation}

To characterize the approximation of the true robust Bellman operator, we adopt the following assumption, which holds for all considered uncertainty sets \citep{ghosh2025scaling}.

\begin{assu}
\label{ass:multiplier-range-TV}
We assume the existence of a compact set $\Theta_{\phi} \subseteq \mathbb{R}_+ \times \mathbb{R}$ and a finite constant $B_{\phi}(\sigma_i)$ such that the dual loss remains uniformly bounded. In particular, for all $h \in [H]$, $i \in \mathcal{M}$, any $f_{i,h+1} \in \mathcal{F}_{i,h+1}$, all $(\lambda, \nu) \in \Theta_{\phi}$, and every $(s, \bm{a}, s') \in \mathcal{S} \times \mathcal{A} \times \mathcal{S}$, we have
\(
\big|\mathrm{Loss}_{i,h,\phi}(f_i; s, \bm{a}, s'; \lambda, \nu)\big| \le B_{\phi}(\sigma_i).
\)
Additionally, we assume that the dual function class $\mathcal{G}$ is restricted to take values within the set $\Theta_{\phi}$.
\end{assu}
We then develop the following error quantification, whose proof is deferred to Appendix~\ref{subsec:Lemma2_proof}.

\begin{lem}[Uniform dual-operator approximation under adaptive data]
\label{lem:phi_operator_approx}
Fix an agent $i\in\mathcal M$ and a stage $h\in[H]$.
Let $\{\pi^t\}_{t=1}^n$ be a predictable sequence of behavior policies, and let $\mathcal D_h$ be the stage-$h$
dataset collected by executing $\{\pi^t\}_{t=1}^n$ over $n$ consecutive episodes. Define the random historical occupancy mixture
\(
\bar\mu_h^{\,n}
:=\sum_{t=1}^n\mu_h^{\pi^t}/n.
\)
Assume Assumptions \ref{ass:completeness}--\ref{ass:multiplier-range-TV} hold and let the
function classes $\mathcal F_{i}$ and $\mathcal G_i$ be finite. For each
$\pi\in\Pi^{\rm pur}$ and $f_{i,h+1}\in \mathcal F_{i,h+1}$, let
$\widehat{g}_i^{\,f,\pi}$ be an empirical minimizer of the dual loss on
$\mathcal D_h$ as defined in \eqref{eq:dual_minimizer}. Then, for any
$\delta\in(0,1)$, with probability at least $1-\delta$:
\begin{align*}
&\sup_{\substack{f_{i,h+1}\in\mathcal F_{i,h+1}, \pi\in\Pi^{\rm pur}}}
\left\|
\mathcal T_{i,h}^{\pi,\phi,\sigma_i}f_{i,h+1}
-
\mathcal T_{\widehat g_i^{\,f,\pi}}^{\pi,\phi,\sigma_i}f_{i,h+1}
\right\|_{1,\bar\mu_h^{\,n}} \nonumber\\
&\le
C B_\phi(\sigma_i)
\sqrt{\log\!\big(
|\mathcal F_{i,h+1}|
|\mathcal G_{i,h}|
|\Pi^{\rm pur}|/\delta
\big)\big/n
} +
\varepsilon_{\rm dual}^{\rm approx}.
\end{align*}
\end{lem}

Prior single-agent results \citep{NeurIPS2022_RobustRLOffline_Panaganti,Arxiv2024_ModelFreeRobustRL_Panaganti} establish similar guarantees under a fixed offline distribution, whereas our result holds uniformly over all policies and their induced state--joint action distributions in the multi-agent setting. Lemma~\ref{lem:phi_operator_approx} shows that the empirical dual optimization uniformly approximates the robust Bellman operator under the $\mathcal{L}^1_i(\mu^\pi;\mathbb{R}^2)$ norm, with the approximation error controlled globally over the distribution $\mu^\pi$ rather than pointwise over individual $(s,\bm a)$. This global control enables the construction of robust confidence sets and provides a unified error term underlying both our algorithm design and theoretical analysis.

\section{Interactive Robust Multi-Agent Reinforcement Learning}
\label{sec:algorithms}

We then present our \emph{{\Algo}} (\Cref{alg:robust_mamex}). In each episode, we will learn value function approximations via fitted methods, where the fitting error is defined through our empirical robust operator. 
We describe our algorithm in this section, and defer our discussions on algorithmic novelties to Appendix~\ref{app:algorithm_novelties}.

\begin{algorithm}[!htb]
\caption{{\Algo}}
\label{alg:robust_mamex}
\begin{algorithmic}[1]
\State \textbf{Input:} Number of episodes $K$, function classes
$\mathcal F=\otimes_{i=1}^M \mathcal F_i$ and $\mathcal G=\otimes_{i=1}^M \mathcal G_i$, regularization parameter $\beta>0$,
equilibrium oracle $\mathrm{EqOracle}(\cdot;\mathrm{NE/CCE/CE})$
\State \textbf{Initialize: } $\mathcal D_h^0 \gets \emptyset$ for all $h\in[H]$
\For{$k=1,2,\ldots,K$}
    \For{each pure joint policy $\pi \in \Pi^{\mathrm{pur}}$}
            \State Set $f_{i,H+1}^{k,\pi} \equiv 0$ for each agent $i\in \mathcal M$.
            \For{$h=H,H-1,\ldots,1$}
                \State Compute empirical dual minimizer $ \widehat{g}_{i,h}^{k,\pi}$ as in \eqref{eq:emp_dual_min} for each agent $i\in \mathcal M$.

                \State Compute plug-in robust Bellman fit $ f_{i,h}^{k,\pi}$ (by Step 2) for each agent $i\in \mathcal M$.
            \EndFor           
            \State Compute regularised robust payoff $\widehat U_i^k(\pi)$ as given in \eqref{eq:reg_robust_payoff} for $i\in \mathcal M$.
    \EndFor
    \State Construct the normal-form game
    \(
    \widehat{\bm \Gamma}^k := \Big(\{\Pi_i^{\mathrm{pur}}\}_{i\in \mathcal M}, \{\widehat U_i^k\}_{i\in \mathcal M}\Big)
    \) and compute equilibrium policy
    \(
    \pi^k \gets \mathrm{EqOracle}(\widehat{\bm \Gamma}^k;\mathrm{NE/CCE/CE})
    \).
    \State Execute $\pi^k$ in the nominal environment $P^\star$ for one episode and collect
    \(
    \tau^k=\{(s_h^k,{\bm a}_h^k,r_h^k,s_{h+1}^k)\}_{h=1}^H
    \) and update dataset $\mathcal D_h^k \gets \mathcal D_h^{k-1}\cup \{(s_h^k,{\bm a}_h^k,s_{h+1}^k)\}$ for $h \in [H]$.
\EndFor
\State \textbf{Return} $\pi^{\mathrm{out}}\coloneqq \{\pi^k\}_{k=1}^K$ or $\frac{1}{K}\sum_{k=1}^K \pi^k$.
\end{algorithmic}
\end{algorithm}

\textbf{Fitted estimation.}
Let $\Pi_i^{\mathrm{pur}}$ denote the class of pure policies for agent $i$, and let
$\Pi^{\mathrm{pur}} := \Pi_1^{\mathrm{pur}} \times \cdots \times \Pi_m^{\mathrm{pur}}$ be the set of pure joint policies. Fix an episode $k$, a candidate joint policy $\pi \in \Pi^{\mathrm{pur}}$, and an agent $i \in [M]$. The algorithm computes a fitted robust value sequence $\{f_{i,h}^{k,\pi}\}_{h=1}^{H+1}$ via backward dynamic programming on the collected dataset, with terminal condition $f_{i,H+1}^{k,\pi} \equiv 0$ (Line 5). For each stage $h=H,\ldots,1$, two coupled steps are performed.

\emph{Step 1: Empirical dual minimization (Line 7).} Using the stage-$h$ dataset $\mathcal D_h^{k-1}$ collected up to episode $k-1$, we compute the
empirical dual minimizer
\begin{align}
\label{eq:emp_dual_min}
\widehat{g}_{i,h}^{k,\pi}
\in
\arg\min_{g \in \mathcal G_{i,h}}
\widehat{\mathrm{LossDual}}_{i,h,\phi}^{k-1}(g; f_{i,h+1}^{k,\pi}, \pi),
\end{align}
where $\widehat{\mathrm{LossDual}}_{i,h,\phi}^{k-1}$ is the empirical version of the dual loss
defined in \eqref{eq:dual_loss} using $\mathcal D_h^{k-1}$.

\emph{Step 2: Robust Bellman fitting with plug-in residual (Line 8).} Once the dual minimizer $\widehat{g}_{i,h}^{k,\pi}$ has been obtained, the algorithm uses it to define a plug-in empirical approximation of the robust Bellman operator. It then computes the next-stage $h$ fitted value function by minimizing the empirical plug-in robust Bellman residual: $f_{i,h}^{k,\pi}
\in
\arg\min_{f \in \mathcal F_{i,h}}
\sum_{(s,{\bm a},s') \in \mathcal D_h^{k-1}}
\widehat \delta_{i,h}(f_i; \widehat{g}_{i,h}^{k,\pi}; \pi; s,{\bm a},s')^2,$
where $\widehat \delta_{i,h}$ is the plug-in empirical discrepancy defined as
\begin{align}
\label{eq:plug-in_robust_Bellman_res}
  \widehat{\delta}_{i,h}(f_i; g_i; \pi; s,{\bm a},s')
:= f_i(s,{\bm a})-r_{i,h}(s,{\bm a})
+\mathrm{Loss}_{i,h,\phi}(f_i;s,{\bm a},s';g_i,\pi).  
\end{align}

Note that $\widehat \delta_{i,h}$ penalizes hypotheses that do not match the robust Bellman equation well on the historical data. The dual minimizer enters precisely through
this term: the Bellman discrepancy is no longer computed using the nominal transition expectation, but through the dual-induced robust integrand. Intuitively, the dual minimizer $\widehat{g}_{i,h}^{k,\pi}$ converts the
robust Bellman backup into a tractable empirical target, and the fitted value function
$f_{i,h}^{k,\pi}$ best fits this target on the observed data.

\textbf{Regularized robust payoff (Line 10).}
For each episode $k\in[K]$, agent $i\in\mathcal M$, and candidate pure joint policy
$\pi\in\Pi^{\mathrm{pur}}$, we define the revised empirical payoff by
\begin{equation}
\widehat U_i^k(\pi)
:=
f_{i,1}^{k,\pi}(s_1^k)
-
\beta \widehat {\mathcal K}_i^{k-1}\!\Big(f_i^{k,\pi};\widehat{g}_i^{k,\pi};\pi\Big),
\label{eq:reg_robust_payoff}
\end{equation}
where $\beta>0$ is the regularization parameter, $\widehat{g}_{i,h}^{k,\pi}$ is the empirical dual minimizer in \eqref{eq:dual_minimizer} for $f_{i,h+1}^{k,\pi}$ and $\pi$, and $\widehat{\mathcal K}_i^{k-1}(f_i;g_i;\pi)$ denotes the cumulative centered empirical robust discrepancy:
\begin{align}
\sum_{h=1}^H
\Big[
&\sum_{j=1}^{k-1}
\widehat\delta_{i,h}(f_i;g_{i,h};\pi;s_h^j,\bm a_h^j,s_{h+1}^j)^2 -
\inf_{f'_i\in\mathcal F_i}
\sum_{j=1}^{k-1}
\widehat\delta_{i,h}(f'_i;g_{i,h};\pi;s_h^j,\bm a_h^j,s_{h+1}^j)^2
\Big],
\label{eq:robust_empirical_discrepancy_centered}
\end{align}
with $\widehat{\mathcal K}_i^{0}(f_i;g_i;\pi)=0$.
The first term in \eqref{eq:reg_robust_payoff} favors policies with large estimated robust value, while the second penalizes large empirical robust Bellman discrepancy. The resulting objective acts as a conservative surrogate value, trading off value estimation and Bellman consistency rather than introducing an explicit optimistic bonus. Similar to MAMEX~\citep{Xiong2023SampleEfficientMR}, this discrepancy term serves as a global regularizer, and in our setting, it controls the gap between empirical robust Bellman error and robust value prediction via the robust MADC.

\textbf{Policy selection and data collection (Line 12-15).} We then define a
normal-form game
\(
\widehat {\bm \Gamma}^k
:=
\Big(\{\Pi_i^{\mathrm{pur}}\}_{i\in[M]}, \{\widehat U_i^k\}_{i\in[M]}\Big),
\)
whose pure-strategy profile is $\pi \in \Pi^{\mathrm{pur}}$ and whose payoff to agent $i$ at profile $\pi$ is $\widehat U_i^k(\pi)$.
Once the normal-form game $\widehat {\bm \Gamma}^k$ is formed, the algorithm computes the desired
equilibrium through some oracle: $\pi^k \leftarrow \mathrm{EqOracle}(\widehat {\bm \Gamma}^k; \mathrm{NE/CCE/CE}).$ As we mainly focus on statistical complexity, we assume such oracles. Specifically, the NE oracle is generally NP-hard \citep{daskalakis2013complexity}, while CE and CCE can be computed with polynominal complexity \citep{zheng2026distributionally}. 
The policy $\pi^k$ is then executed for one episode in the nominal environment $P^\star$,
producing the trajectory
\(
\tau^k = \{(s_h^k,{\bm a}_h^k,r_h^k,s_{h+1}^k)\}_{h=1}^H.
\)
For each stage $h$, the transition tuple $(s_h^k,{\bm a}_h^k,s_{h+1}^k)$ is appended to the
dataset $\mathcal D_h^{k-1}$ to form $\mathcal D_h^k$, for the following algorithm executions. The final output is either the sequence
$\{\pi^k\}_{k=1}^K$ or the uniform mixture
$\bar\pi = \frac{1}{K}\sum_{k=1}^K \pi^k$, depending on the equilibrium notion under
consideration. 

\section{Theoretical Guarantee}
\label{sec:Theoretical_ORBIT_MAIL}
We now develop our regret analysis. Inspired by the non-robust studies \citep{Xiong2023SampleEfficientMR}, we first introduce a fundamental complexity measure, named \emph{robust multi-agent decoupling coefficient} (Robust MADC), to quantify the inherent difficulty of exploration in DRMGs. This notion will enable us to develop our analysis from an optimization perspective. 

\begin{defi}[Robust Multi-Agent Decoupling Coefficient]
\label{def:robust_MADC}
    For each agent $i \in \mathcal M$, policy $\pi \in \Pi^{\mathrm{pur}}$, and stage $h \in [H]$, let $\mu^\pi_h$ denote the stage-$h$ state-joint action visitation distribution induced by policy $\pi$ under the nominal kernel $P^\star$. Then, for any $i \in \mathcal{M}$, $\mu, c > 0$, $f_i^k \in \mathcal{F}_i$, and joint policy $\{\pi^k\}_{k=1}^{K}$, the robust MADC for agent $i \in \mathcal M$ is the smallest constant $1 \leq d_{i,\mathrm{MADC}}^{\mathrm{rob}} <\infty $ such that:
    \begin{equation}
        \sum_{k=1}^{K}(V_{i,1}^{\pi^k,f^k_i}(s_1^k) - V_{i,1}^{\pi^k,\sigma_i}(s_1^k)) \le \frac{1}{\mu}\sum_{k=1}^K\sum_{s=1}^{k-1} e_{i,1}^s(f_i^k,\pi^k) + \mu\, d_{i,\mathrm{MADC}}^{\mathrm{rob}} + cH d_{i,\mathrm{MADC}}^{\mathrm{rob}},\label{eq:robust_madc_main_corrected}
    \end{equation}
\end{defi}
where $e_{i,1}^{s}(f,\pi)$ is the discrepancy function defined as the mean-squared robust Bellman error:
\begin{equation*}
   e_{i,1}^s(f,\pi) \triangleq \sum_{h=1}^H \mathbb E_{(s_h,{\bm a}_h)\sim \mu_{h}^{\pi^s}}\bigl[\bigl(f_{i,h}(s_h,{\bm a}_h)- [\mathcal T_{i,h}^{\pi,\phi,\sigma_i}f_{i,h+1}](s_h,{\bm a}_h) \bigr)^2\bigr],  \forall s \in [K], \forall f \in \mathcal{F}_i.
\end{equation*}

\textbf{Interpretation of robust MADC.}
Our robust MADC extends the decoupling coefficient in
\citep{dann2021provably,Xiong2023SampleEfficientMR} to robust settings, capturing the joint difficulty of strategic interactions and transition uncertainty. At a high level, $d^{\mathrm{rob}}_{i,\mathrm{MADC}}$ quantifies the decoupling between (i) the \emph{training error}, given by the cumulative empirical robust Bellman residual, and (ii) the \emph{prediction error}, defined as the gap between the learned value and the true robust value under worst-case dynamics. Definition~\ref{def:robust_MADC} ensures that the cumulative prediction error is controlled by the cumulative training error up to a term governed by $d^{\mathrm{rob}}_{i,\mathrm{MADC}}$. Unlike the non-robust setting \citep{Xiong2023SampleEfficientMR}, the prediction error is evaluated under adversarial dynamics while the training error is computed from nominal interaction data. Thus, $d^{\mathrm{rob}}_{i,\mathrm{MADC}}$ characterizes how effectively nominal interaction data controls value prediction under worst-case dynamics, capturing the intrinsic exploration complexity of DRMGs.

\textbf{Structural perspective.}
Beyond the decoupling interpretation, robust MADC also provides a structural view: a smaller $d^{\mathrm{rob}}_{i,\mathrm{MADC}}$ indicates that small empirical robust Bellman error translates into small value prediction error, reflecting better alignment of the function class with the robust Bellman operator. This parallels Bellman-error-based complexity measures such as the BE dimension, while additionally capturing the mismatch between nominal data and adversarial evaluation. This perspective extends naturally to structured function classes, such as bilinear models \citep{du2021bilinear}, where low-dimensional representations enable accurate estimation of the robust Bellman operator from nominal data, thereby controlling the discrepancy terms, yielding a bounded robust MADC, and improving generalization under worst-case dynamics.

In previous RL studies, the (multi-agent) Bellman-Eluder (MABE) dimension is another widely used complexity measure in RL \citep{NeurIPS2021_BellmanEluderDim_Jin,ghosh2025scaling}. In the non-robust setting, a low BE dimension implies a low MADC \citep{Xiong2023SampleEfficientMR}, making MADC a more general complexity measure. In the robust setting, the robust BE dimension was introduced for single-agent RL in \citep{ghosh2025scaling} and can be extended to multi-agent settings. Our robust MADC instead directly quantifies how robust Bellman errors estimated from nominal interactions control value prediction under uncertain transition kernels. Moreover, Theorem~\ref{thm:robust_thm54} shows that a low robust MABE, extending the robust BE dimension of \citep{ghosh2025scaling}, implies a low robust MADC, establishing robust MADC as a more general complexity measure.

Extending results from the non-robust setting \citep{NeurIPS2021_BellmanEluderDim_Jin,jin2022power}, the inherent robust MABE can be substantially smaller than the tabular complexity $S\prod_{i=1}^m A_i$ for structured function classes. By Theorem~\ref{thm:robust_thm54}, the robust MADC can therefore also be much smaller than $S\prod_{i=1}^m A_i$ under function approximation. Since our sample-complexity guarantees are characterized by the robust MADC, this enables improved scalability over tabular approaches.

We now present the theoretical guarantees of {\Algoname}; full details are deferred to Appendix~\ref{app:regret_bound_proof}.

\begin{thm}
\label{thm:main-robust-regret}
Assume Assumptions~\ref{ass:completeness}--\ref{ass:multiplier-range-TV} and the function classes $\mathcal F_i$ and dual function class $\mathcal G_i$ are finite\footnote{This assumption is made for technical simplicity. While extensions to infinite function classes may be possible using covering arguments (e.g., $\epsilon$-nets \citep{NeurIPS2022_RobustRLOffline_Panaganti}), 
establishing such results rigorously in the robust multi-agent setting remains 
non-trivial and is left for future work.}. Let $\{\pi^k\}_{k=1}^K$ be policies produced by \Cref{alg:robust_mamex} with different oracles, and denote $d_{\mathrm{MADC}}^{\mathrm{rob},\max}=\max_{i \in \mathcal M}d_{i,\mathrm{MADC}}^{\mathrm{rob}}$,  $B^{\max}_{\phi}=\max_{i \in \mathcal M}B_{\phi}(\sigma_i)$ and $\iota_{\mathcal F, \mathcal G}^\delta:= \log\!\big(\frac{|\mathcal F||\mathcal G||\Pi^{\mathrm{pur}}|KHm}{\delta}\big)$. Then, with probability at least $1-\delta$, 
the corresponding robust regret satisfies
\begin{align*}
\mathrm{Reg}(K) \le \mathcal O\!\left(
d_{\mathrm{MADC}}^{\mathrm{rob},\max}\sqrt K
+
H d_{\mathrm{MADC}}^{\mathrm{rob},\max}
+
H(B_\phi^{\max})^2\sqrt{K\iota_{\mathcal F, \mathcal G}^\delta}
+
\varepsilon_{\mathrm{dual}}^{\mathrm{approx}}
\right).
\end{align*}
\end{thm}
 
Theorem~\ref{thm:main-robust-regret} shows that {\Algoname} efficiently learns robust equilibria in general-sum DRMGs with general function approximation. The bound decomposes into four terms: 
(i) the \emph{exploration term}
$d_{\mathrm{MADC}}^{\mathrm{rob},\max}\sqrt K$, governed by the robust MADC, which—analogous to \citep{Xiong2023SampleEfficientMR}—links cumulative in-sample robust Bellman discrepancy to out-of-sample robust value prediction error and thus captures the
intrinsic difficulty of exploration. Optimizing the decoupling parameter yields the $\sqrt K$ rate. While small cumulative discrepancy ensures that the
exploration-induced component of regret is sublinear, the overall regret additionally depends on the robust estimation term arising from approximating the worst-case Bellman operator from nominal data;
(ii) a \emph{burn-in horizon-dependent term} \(H\,d_{\mathrm{MADC}}^{\mathrm{rob},\max}\), which reflects horizon-dependent initialization effects and does not
affect the asymptotic rate; 
(iii) a \emph{robust estimation term} $H\,(B_\phi^{\max})^2\sqrt{K\,\iota_{\mathcal F,\mathcal G}^\delta}$,
which is unique to the robust setting. This term captures the statistical
cost of estimating the robust Bellman operator from nominal data via the
dual formulation. Here, $B_\phi^{\max}$ controls the dual loss envelope,
$\iota_{\mathcal F,\mathcal G}^\delta$ reflects the complexity of the primal and dual function classes; and 
(iv) the \emph{dual approximation error} \(\varepsilon_{\mathrm{dual}}^{\mathrm{approx}}\), which captures the representational error of the dual function class and vanishes under realizability of $\mathcal G$. Overall, the bound reveals two core challenges: exploration under strategic interaction, captured by \(d_{\mathrm{MADC}}^{\mathrm{rob},\max}\)and robust Bellman estimation under distributional mismatch, reflected in
by $(B_\phi^{\max})^2\iota_{\mathcal F,\mathcal G}^\delta$. Refer to subsection \ref{app:algorithm_novelties} for a detailed discussion on the technical novelties of {\Algoname}.

As a consequence of our regret guarantees, applying a standard online-to-batch conversion \citep{NeuRIPS2001_OnlineLearningAlgo_Cesa} gives the first sample-complexity guarantees for learning an $\varepsilon$-optimal policy with {\Algoname}. As corollary, we obtain regret bound under total variation, $\chi^2$- and KL-divergence uncertainty sets, presented in Appendix~\ref{app:cor_sample_complexity_phi} and Appendix~\ref{app:TV_regret_bound}.  To our knowledge, these results constitute the first sample-complexity characterization for distributionally robust MARL with general function approximation. Crucially, the bounds are governed by the robust MADC and contain no explicit dependence on the state-space size $|\mathcal S|$ or the joint action-space size $|\mathcal A|$. Thus, compared with tabular robust MARL, our analysis replaces the unfavorable joint-action-space dependence with an intrinsic function-class complexity, improving scalability and mitigating the curse of multi-agency \citep{Arxiv2024_SampleEfficientMARL_Shi,farhat2026sampleefficient}.

To further interpret the sharpness and generality of our results, we specialize the general bound to representative settings. In particular, we consider the tabular case, robust MABE/BE-type characterizations, and linear function approximation. A more detailed and comprehensive comparison with prior work is
provided in subsection \ref{app:comaprison}.

\subsection{Novelties of {\Algoname}}
\label{app:algorithm_novelties}
In this section, we discuss both the algorithm and technical novelties of {\Algoname}.

\textbf{Algorithm novelties.} Our algorithm is conceptually motivated by two complementary lines of work: the MAMEX framework \citep{Xiong2023SampleEfficientMR}, which constructs payoffs using nominal Bellman errors but does not account for distributional uncertainty, and the RFL method \citep{ghosh2025scaling}, which uses
dual-based robust Bellman residuals to enable worst-case value estimation in single-agent robust RL. However, neither approach directly applies to online DRMGs with general function approximation: MAMEX lacks robustness, while RFL does not address strategic multi-agent interactions or equilibrium-based payoff construction. This creates a central algorithmic challenge: how to
construct payoffs that are simultaneously robust, statistically estimable from nominal data, and suitable for equilibrium-based exploration.

\textit{(1) Robust payoff design via dual fitted learning.}
The central novelty of {\Algoname} is the construction of a robust payoff
function that combines worst-case value estimation with exploration-aware
regularization. In nominal MARL, payoff construction can rely on Bellman
consistency under the data-generating dynamics. In contrast, in DRMGs, policy
values are defined through adversarial transition kernels that are not directly
observed from nominal interaction data. To overcome this mismatch, {\Algoname}
uses a dual-based robust fitted-learning procedure: for each candidate joint
policy, a functional dual minimizer is estimated from nominal data and then
plugged into a fitted robust Bellman recursion. This produces a tractable
empirical surrogate for the robust value. The resulting payoff is then
regularized by a centered empirical robust discrepancy, which penalizes the
excess robust Bellman residual relative to the best in-class fit. Thus, the
payoff does not merely subtract a raw residual; it subtracts a centered
function-class-level measure of robust Bellman inconsistency.

\textit{(2) Function-class-level uncertainty quantification under robustness.}
The centered empirical robust discrepancy provides a global uncertainty
quantification mechanism tailored to robust multi-agent learning. Since it is
defined as an excess empirical robust Bellman error relative to the best
hypothesis in the class, it measures how much additional robust Bellman
inconsistency is incurred by the fitted value function beyond the best
available in-class approximation. This is crucial under general function
approximation, where uncertainty cannot be captured by tabular state-action
bonuses. Instead, {\Algoname} uses this centered discrepancy as a global
regularizer in the payoff, and the regret analysis relates this empirical
robust discrepancy to robust value prediction error through the robust MADC.

\textit{(3) Coupling robust evaluation with equilibrium-based learning.} {\Algoname} integrates dual-based robust policy evaluation with equilibrium-based learning by using the regularized robust payoffs to define a surrogate normal-form game over candidate pure joint policies. The equilibrium oracle is then applied to this surrogate game, so the selected policy is guided
not only by estimated robust value but also by the centered robust Bellman discrepancy. This aligns equilibrium computation with the robust learning objective: agents are encouraged to choose policies whose robust values are large and whose fitted robust Bellman recursions are statistically reliable under the accumulated nominal data.

Overall, the novelty of {\Algoname} lies in resolving the incompatibility between MAMEX-style payoff-based exploration and dual-based robust Bellman evaluation. The key algorithmic contribution is a regularized robust payoff that combines dual fitted robust value estimation with a centered empirical
robust discrepancy, enabling purely online equilibrium learning in general-sum DRMGs under general function approximation.

\textbf{Technical novelties.}
The primary technical challenge in our setting arises from the fact that
the robust Bellman operator is defined through a worst-case transition
kernel and is therefore not directly observable from nominal interaction
data. This creates a fundamental mismatch between the data-generating
distribution and the target of optimization, a difficulty that does not
arise in non-robust MARL. To address this, our analysis develops a unified framework that simultaneously controls three sources of error:
(i) approximation error induced by plug-in dual-based robust Bellman
operators,
(ii) statistical generalization error under function approximation, and
(iii) strategic deviation error due to multi-agent interactions.
A key ingredient is the use of a \emph{centered empirical robust discrepancy}, which enables a telescoping control of estimation errors across episodes. This centering is critical: it eliminates systematic bias introduced by
the evolving plug-in operators and prevents linear accumulation of
estimation errors over time.

Building on this, we introduce a robust extension of the MADC, which quantifies how empirical robust Bellman error—computed from nominal data—controls value prediction under adversarial transition dynamics. This leads to a regret bound that explicitly captures the interplay
between function approximation, equilibrium computation, and distributional robustness in a fully online setting. To the best of our knowledge, this provides the first intrinsic-complexity characterization of exploration in general-sum DRMGs under
$\phi$-divergence uncertainty sets.

\subsection{Comparison with Prior Works}
\label{app:comaprison}

\begin{table}[!htb]
\centering
\caption{
Comparison of sample-complexity guarantees across general function approximation,
tabular models, BE/MABE-type characterizations, and linear function approximation
in multi-agent RL. We focus on online model-free methods whenever possible.
For compactness, let $\iota_{\mathcal F,\mathcal G}^{\delta}
:=
\log\!\Big(
|\mathcal F||\mathcal G||\Pi^{\mathrm{pur}}|KHm/\delta
\Big)$, $A_{\mathrm{joint}}:=\prod_{i=1}^m A_i$, and define the divergence-dependent factors
\(
\Psi_{\mathrm{TV}}(\sigma)
:=
\big(\min\{H,1/\sigma_{\min}\}\big)^4,\quad
\Psi_{\chi^2}(\sigma)
:=
(1+\sqrt{\sigma_{\max}})^4,\quad
\Psi_{\mathrm{KL}}(\sigma)
:=
\max\{1,\sigma_{\max}^4\}.
\) For notational simplicity, we further denote $d_{\text{MABE}}^{\text{rob},\max}:= d_{\text{MABE}}^{\text{rob},\max} (\mathcal{F}, \Pi,1/K)$ (robust MABE-dimension) and $d_{\text{MABE}}:= d_{\text{MABE}}(\mathcal{F}, \Pi,1/K)$ (non-robust MABE-dimension).
Our bounds keep the robust-to-nominal concentrability coefficient
\(C_{\mathrm{cov}}\) explicit.
}
\label{tab:comparison}
\small
\setlength{\tabcolsep}{2.2pt}
\renewcommand{\arraystretch}{1.12}
\begin{tabular}{c|c|c|c|p{0.45\textwidth}}
\toprule
\textbf{Setting} & \textbf{Method} & \textbf{Robust} & \textbf{Divergence} & \textbf{Sample complexity} \\
\midrule

\multirow{5}{*}{\textbf{General}}

& \makecell{Online  \citep{Xiong2023SampleEfficientMR}}
& No & --
& \(\widetilde{\mathcal O}\!\left(
m^2H^3d_{\mathrm{MADC}}^2\,\iota_{\mathcal F}^{\delta}\,
\varepsilon^{-2}
\right)\) \\

\cmidrule(l){2-5}

& \multirow{3}{*}{\makecell{\textbf{Online}\\\textbf{{\Algoname} (Ours)}}}
& \multirow{3}{*}{Yes}
& TV
& \(\widetilde{\mathcal O}\!\left(
\max\left\{
\frac{H^2(d_{\mathrm{MADC}}^{\mathrm{rob},\max})^2}{\varepsilon^2},
\frac{C_{\mathrm{cov}}^2H^7\Psi_{\mathrm{TV}}(\sigma)
\,\iota_{\mathcal F,\mathcal G}^{\delta}}{\varepsilon^2}
\right\}
\right)\) \\

& & & \(\chi^2\)
& \(\widetilde{\mathcal O}\!\left(
\max\left\{
\frac{H^2(d_{\mathrm{MADC}}^{\mathrm{rob},\max})^2}{\varepsilon^2},
\frac{H^7\Psi_{\chi^2}(\sigma)
\,\iota_{\mathcal F,\mathcal G}^{\delta}}{\varepsilon^2}
\right\}
\right)\) \\

& & & KL
& \(\widetilde{\mathcal O}\!\left(
\max\left\{
\frac{H^2(d_{\mathrm{MADC}}^{\mathrm{rob},\max})^2}{\varepsilon^2},
\frac{H^7\Psi_{\mathrm{KL}}(\sigma)
\,\iota_{\mathcal F,\mathcal G}^{\delta}}{\varepsilon^2}
\right\}
\right)\) \\

\cmidrule(l){2-5}
& Lower bound & -- & -- & N/A \\

\midrule

\multirow{7}{*}{\textbf{Tabular}}

& \makecell{Online  \citep{ma2023decentralized}}
& Yes & KL
& \(\widetilde{\mathcal O}\!\left(
\varepsilon^{-2}H^5S(\max_i A_i)^2
\right)\) with an oracle \\

\cmidrule(l){2-5}

& \multirow{2}{*}{\makecell{Online  \citep{farhat2026sampleefficient}}}
& \multirow{2}{*}{Yes}
& TV
& \(\widetilde{\mathcal O}\!\left(
\varepsilon^{-2}H^3S A_{\mathrm{joint}}\min\{\sigma_{\min}^{-1},H\}
\right)\) \\

& & & KL
& \(\widetilde{\mathcal O}\!\left(
\varepsilon^{-2}\sigma_{\min}^{-2}(P_{\min}^{\star})^{-1}
H^5e^{2H^2}S A_{\mathrm{joint}}
\right)\) \\

\cmidrule(l){2-5}

& \multirow{3}{*}{\makecell{\textbf{Online}\\\textbf{{\Algoname} (Ours)}}}
& \multirow{3}{*}{Yes}
& TV
& \(\widetilde{\mathcal O}\!\left(
\max\left\{
\frac{H^2S^2A_{\mathrm{joint}}^2}{\varepsilon^2},
\frac{C_{\mathrm{cov}}^2H^7\Psi_{\mathrm{TV}}(\sigma)
S A_{\mathrm{joint}}}{\varepsilon^2}
\right\}
\right)\) \\

& & & \(\chi^2\)
& \(\widetilde{\mathcal O}\!\left(
\max\left\{
\frac{H^2S^2A_{\mathrm{joint}}^2}{\varepsilon^2},
\frac{H^7\Psi_{\chi^2}(\sigma)
S A_{\mathrm{joint}}}{\varepsilon^2}
\right\}
\right)\) \\

& & & KL
& \(\widetilde{\mathcal O}\!\left(
\max\left\{
\frac{H^2S^2A_{\mathrm{joint}}^2}{\varepsilon^2},
\frac{H^7\Psi_{\mathrm{KL}}(\sigma)
S A_{\mathrm{joint}}}{\varepsilon^2}
\right\}
\right)\) \\

\cmidrule(l){2-5}

& \makecell{Lower bound (Generative)\\ \citep{Arxiv2024_SampleEfficientMARL_Shi}}
& Yes & TV
& \(\widetilde{\Omega}\!\left(
\varepsilon^{-2}H^3S(\max_i A_i)\min\{\sigma_{\min}^{-1},H\}
\right)\) \\

\midrule

\multirow{5}{*}{\makecell{\textbf{MABE}\\ \textbf{Dimension}}}

& \makecell{Online  \citep{Xiong2023SampleEfficientMR}}
& No & --
& \(\widetilde{\mathcal O}\!\left(
m^2H^3d_{\text{MABE}}^2\,\iota_{\mathcal F}^{\delta}
\,\varepsilon^{-2}
\right)\) \\

\cmidrule(l){2-5}
& \makecell{Online \citep{jin2022power}}
& No & --
& \(\widetilde{\mathcal O}\!\left(
H^3d_{\mathrm{MABE}}^2
\log(|\mathcal F|Hd_{\mathrm{BE}}/\delta)
\varepsilon^{-2}
\right)\) \\

\cmidrule(l){2-5}

& \multirow{3}{*}{\makecell{\textbf{Online}\\\textbf{{\Algoname} (Ours)}}}
& \multirow{3}{*}{Yes}
& TV
& \(\widetilde{\mathcal O}\!\left(
\max\left\{
\frac{H^4(d_{\text{MABE}}^{\text{rob},\max} )^2}{\varepsilon^2},
\frac{C_{\mathrm{cov}}^2H^7\Psi_{\mathrm{TV}}(\sigma)
\,\iota_{\mathcal F,\mathcal G}^{\delta}}{\varepsilon^2}
\right\}
\right)\) \\

& & & \(\chi^2\)
& \(\widetilde{\mathcal O}\!\left(
\max\left\{
\frac{H^4(d_{\text{MABE}}^{\text{rob},\max} )^2}{\varepsilon^2},
\frac{H^7\Psi_{\chi^2}(\sigma)
\,\iota_{\mathcal F,\mathcal G}^{\delta}}{\varepsilon^2}
\right\}
\right)\) \\

& & & KL
& \(\widetilde{\mathcal O}\!\left(
\max\left\{
\frac{H^4(d_{\text{MABE}}^{\text{rob},\max} )^2}{\varepsilon^2},
\frac{H^7\Psi_{\mathrm{KL}}(\sigma)
\,\iota_{\mathcal F,\mathcal G}^{\delta}}{\varepsilon^2}
\right\}
\right)\) \\

\midrule

\multirow{6}{*}{\textbf{Linear}}

& \makecell{Online \citep{Xiong2023SampleEfficientMR}}
& No & --
& \(\widetilde{\mathcal O}\!\left(
m^2H^3d_{\mathrm{lin}}^2\varepsilon^{-2}
\right)\) \\

\cmidrule(l){2-5}

& \makecell{Online  \citep{wang2023breaking}}
& No & --
& \(\widetilde{\mathcal O}\!\left(
m^2H^6d_{\mathrm{lin}}^4(\max_i A_i^5)\varepsilon^{-2}
\right)\) \\

\cmidrule(l){2-5}

& \makecell{Online  \citep{zheng2026distributionally}}
& Yes & TV
& \(\widetilde{\mathcal O}\!\left(
H^3d_{\mathrm{lin}}^2(\min\{H,1/\sigma_{\min}\})^2
\varepsilon^{-2}
\right)\) \\

\cmidrule(l){2-5}

& \multirow{3}{*}{\makecell{\textbf{Online}\\\textbf{{\Algoname} (Ours)}}}
& \multirow{3}{*}{Yes}
& TV
& \(\widetilde{\mathcal O}\!\left(
\max\left\{
\frac{\gamma^2C_{\mathrm{cov}}^2H^4d_{\mathrm{lin}}^2}{\varepsilon^2},
\frac{C_{\mathrm{cov}}^2H^7\Psi_{\mathrm{TV}}(\sigma)
\,\iota_{\mathrm{lin}}^{\delta}}{\varepsilon^2}
\right\}
\right)\) \\

& & & \(\chi^2\)
& \(\widetilde{\mathcal O}\!\left(
\max\left\{
\frac{\gamma^2H^4d_{\mathrm{lin}}^2}{\varepsilon^2},
\frac{H^7\Psi_{\chi^2}(\sigma)
\,\iota_{\mathrm{lin}}^{\delta}}{\varepsilon^2}
\right\}
\right)\) \\

& & & KL
& \(\widetilde{\mathcal O}\!\left(
\max\left\{
\frac{\gamma^2H^4d_{\mathrm{lin}}^2}{\varepsilon^2},
\frac{H^7\Psi_{\mathrm{KL}}(\sigma)
\,\iota_{\mathrm{lin}}^{\delta}}{\varepsilon^2}
\right\}
\right)\) \\

\cmidrule(l){2-5}

& \makecell{Lower bound  \citep{zheng2026distributionally}}
& Yes & TV
& \(\widetilde{\Omega}\!\left(
H^2d_{\mathrm{lin}}^2(\min\{H,1/\sigma_{\min}\})^2
\varepsilon^{-2}
\right)\) \\

\bottomrule
\end{tabular}
\end{table}

Table~\ref{tab:comparison} summarizes the sample-complexity guarantees of
{\Algoname} and positions them relative to existing results across four regimes:
general function approximation, tabular models, BE/MABE-type characterizations,
and linear structure. The comparison focuses on online model-free methods in
general-sum Markov games, while also including two-player zero-sum BE-dimension
benchmarks as reference points. The main message is that our guarantees extend
the intrinsic-complexity viewpoint of non-robust MARL to distributionally robust
MARL: exploration is governed by a robust decoupling coefficient, while
robustness contributes additional terms reflecting the geometry of the
uncertainty set and the mismatch between nominal data and worst-case evaluation.

\paragraph{General function approximation: robust extension of intrinsic complexity characterizations.}
In non-robust multi-agent RL with general function approximation, the sharpest
known guarantees are naturally phrased in terms of intrinsic complexity rather
than explicit dependence on \((S,A)\). In particular, the MAMEX framework shows
that sample complexity scales as
\(
\widetilde{\mathcal O}\!\left(
m^2H^3d_{\mathrm{MADC}}^2\,\iota_{\mathcal F}^{\delta}
\,\varepsilon^{-2}
\right),
\)
where \(d_{\mathrm{MADC}}\) controls how empirical Bellman errors translate into
value-prediction errors across agents \citep{Xiong2023SampleEfficientMR}. This
extends BE-style complexity ideas from single-agent and two-player zero-sum
settings to general-sum games, where there is no single Bellman operator and
agents' value functions are coupled through equilibrium constraints.

Our result extends this intrinsic characterization to the distributionally
robust setting. The revised sample-complexity bound has the generic form
\(
\widetilde{\mathcal O}\!\left(
\max\left\{
\frac{H^2(d_{\mathrm{MADC}}^{\mathrm{rob},\max})^2}{\varepsilon^2},
\frac{H^7\Psi_\phi(\sigma)
\,\iota_{\mathcal F,\mathcal G}^{\delta}}{\varepsilon^2}
\right\}
\right).
\)
The first term is the robust exploration term, governed by the robust MADC. It
is the analogue of the non-robust decoupling term, but now the prediction error
is measured against robust values induced by adversarial transition kernels.
The second term is unique to the robust setting: it is the cost of estimating
the plug-in robust Bellman operator from nominal interaction data. The factor
\(B_\phi(\sigma)\), equivalently summarized through \(\Psi_\phi(\sigma)\), comes
from the dual-loss envelope associated with the \(\phi\)-divergence uncertainty
set. Moreover, \(C_{\mathrm{cov}}\) captures the robust-to-nominal distribution
mismatch for TV-case. Thus, robustness does not simply multiply the MADC term; instead, it introduces a separate robust-estimation term in addition to the exploration complexity.

\paragraph{BE dimension, MABE, and robust MADC.}
A central conceptual distinction behind Table~\ref{tab:comparison} is the
difference between BE-dimension-based analyses and MADC-based analyses. In
two-player zero-sum Markov games, the analysis of \citep{jin2022power} relies on
a single Nash value function induced by a minimax Bellman operator. This unified
operator allows exploration to be controlled through one sequence of Bellman
errors, yielding sample complexity of order
\(
\widetilde{\mathcal O}\!\left(
H^3d_{\mathrm{MABE}}^2
\log(|\mathcal F|Hd_{\mathrm{MABE}}/\delta)\varepsilon^{-2}
\right).
\)
In general-sum games, however, this structure is absent: each agent has its own
value function, equilibrium behavior is defined through coupled unilateral
deviations, and there is no single Bellman fixed point that governs all agents.

This is precisely the role of MADC. It quantifies how multi-agent Bellman
prediction errors propagate into equilibrium-relevant value errors across
agents. Our robust extension preserves this structure but replaces the nominal
Bellman residual by the robust Bellman residual. The robust MADC therefore
captures the structural difficulty of using nominal Bellman training errors to
control robust value prediction under worst-case dynamics.

Theorem~\ref{thm:robust_thm54} further shows that low robust MABE implies low
robust MADC:
\(
d_{\mathrm{MADC}}^{\mathrm{rob},\max}
=
\widetilde{\mathcal O}\!\left(
H d_{\text{MABE}}^{\text{rob},\max} 
\right).
\)
Substituting this into the general theorem yields the BE/MABE rows in
Table~\ref{tab:comparison}. Importantly, the robust MABE-to-MADC implication is
a structural statement, where we introduce \(C_{\mathrm{cov}}\) for the TV case in the final sample complexity through the robust-estimation term in Theorem~\ref{thm:main-robust-regret-TV}.

\paragraph{Tabular setting: recovering uncertainty-dependent scaling.}
The tabular setting provides the clearest interpretation of how the uncertainty radius affects the sample complexity. Under full tabular function classes, the
robust MADC satisfies
\(
d_{\mathrm{MADC}}^{\mathrm{rob},\max}
=
\widetilde{\mathcal O}(S A_{\mathrm{joint}}),
\)
and the logarithmic complexity of the primal and dual classes scales as
\(
\log(|\mathcal F||\mathcal G|)
=
\widetilde{\mathcal O}(S A_{\mathrm{joint}}).
\)
Substituting these into the generic bound gives
\(
\widetilde{\mathcal O}\!\left(
\max\left\{
\frac{H^2S^2A_{\mathrm{joint}}^2}{\varepsilon^2},
\frac{H^7\Psi_\phi(\sigma)
S A_{\mathrm{joint}}}{\varepsilon^2}
\right\}
\right),
\)
which is reported in Table~\ref{tab:comparison}. For TV-case, we get the bound as \(
\widetilde{\mathcal O}\!\left(
\max\left\{
\frac{H^2S^2A_{\mathrm{joint}}^2}{\varepsilon^2},
\frac{C^2_{\mathrm{cov}}H^7\Psi_\phi(\sigma)
S A_{\mathrm{joint}}}{\varepsilon^2}
\right\}
\right).
\)

This specialization should be interpreted carefully. For the TV-case, the coefficient
\(C_{\mathrm{cov}}\) is a robust-to-nominal distribution-shift coefficient as mentioned in Definition \ref{ass:robust-nominal-concentrability}, not
a purely combinatorial quantity such as \(SA_{\mathrm{joint}}\). Therefore we
keep it explicit rather than replacing it by a tabular dimension. 
The key point is that the tabular dependence emerges from the same robust MADC framework used for general function approximation,
rather than from a separate tabular analysis. Moreover, the divergence factor becomes
\(\Psi_{\mathrm{TV}}(\sigma)=(\min\{H,1/\sigma_{\min}\})^4\), recovering the
same qualitative contraction dependence that appears in recent tabular robust
MDP and DRMG analyses. 

For \(\chi^2\) and KL uncertainty, the same framework
yields the factors \((1+\sqrt{\sigma_{\max}})^4\) and
\(\max\{1,\sigma_{\max}^4\}\), respectively. Our tabular bounds are not intended
to sharpen every parameter relative to specialized tabular algorithms; rather,
they demonstrate that the uncertainty-dependent behavior follows from a unified
Bellman-error-based robust MADC analysis.

The benefit of using robust MADC rather than the tabular quantity
\(S A_{\mathrm{joint}}\) is most visible under structured function approximation.
In the full tabular class, Bellman-error-based complexity measures may scale as
\(S A_{\mathrm{joint}}\), and our framework recovers this dependence. However,
this is only the worst-case specialization. Prior work on BE dimension shows that, for structured function classes such as low Bellman-rank, linear, generalized-linear, and kernel models, the Bellman-error complexity is governed by an intrinsic dimension \(d\), rather than by the number of state--joint-action pairs \citep{NeurIPS2021_BellmanEluderDim_Jin, jin2022power, ghosh2025scaling}. In particular, low Bellman rank \(d\) implies
\(
\dim_{\mathrm{MABE}}
\le
\widetilde{\mathcal O}(d),
\)
\citep[Proposition 11]{NeurIPS2021_BellmanEluderDim_Jin}, which can be much smaller than \(S A_{\mathrm{joint}}\). Our robust MABE-to-MADC result plays the same
role in the robust multi-agent setting by
Theorem~\ref{thm:robust_thm54},
where \(
d_{\mathrm{MADC}}^{\mathrm{rob},\max}
=
\widetilde{\mathcal O}\!\left(H d_{\text{MABE}}^{\text{rob},\max} \right).
\)
Thus, whenever the robust Bellman-error class has low robust MABE dimension
\(d_{\text{MABE}}^{\text{rob},\max} (\mathcal{F}, \Pi,1/K)\ll S A_{\mathrm{joint}}\), the robust MADC is also much smaller
than the tabular complexity up to horizon and logarithmic factors. This is the
precise sense in which our bound avoids explicit state--joint-action dependence:
it replaces the combinatorial tabular quantity by an intrinsic complexity
measure that captures how the chosen function class generalizes robust Bellman
errors across policies, agents, and adversarial transition dynamics.

\paragraph{Linear function approximation: reduction from the bilinear theory.}
The linear setting serves as a useful benchmark for comparing our general
intrinsic-complexity theory with specialized linear robust MARL analyses. In
the non-robust setting, MAMEX achieves
\(
\widetilde{\mathcal O}\!\left(
m^2H^3d_{\mathrm{lin}}^2\varepsilon^{-2}
\right),
\)
while decentralized linear Markov game methods such as
\citep{wang2023breaking} obtain bounds of order
\(
\widetilde{\mathcal O}\!\left(
m^2H^6d_{\mathrm{lin}}^4(\max_i A_i^5)\varepsilon^{-2}
\right).
\)
These methods exploit linear structure and, in decentralized settings, reduce
the burden of joint exploration, but they remain tied to linear realizability
and specialized algorithmic designs.

In the robust linear setting, the closest comparison is
\citep{zheng2026distributionally}, which studies online distributionally robust
general-sum Markov games under linear function approximation and TV uncertainty.
Their sample complexity is
\(
\widetilde{\mathcal O}\!\left(
H^3d_{\mathrm{lin}}^2(\min\{H,1/\sigma_{\min}\})^2
\varepsilon^{-2}
\right),
\)
with a nearly matching lower bound
\(
\widetilde{\Omega}\!\left(
H^2d_{\mathrm{lin}}^2(\min\{H,1/\sigma_{\min}\})^2
\varepsilon^{-2}
\right).
\)
These rates are sharper in \(H\) and in the TV-radius dependence because the
analysis is tailored to linear structure and TV uncertainty.

Our linear specialization follows instead from the general robust bilinear
framework. Theorem~\ref{thm:robust-bilinear-implies-low-madc} gives, under
robust-to-nominal feature coverage,
\(
d_{\mathrm{MADC}}^{\mathrm{rob},\max}
=
\widetilde{\mathcal O}\!\left(
\gamma^2H d_{\mathrm{lin}}
\right).
\)
Substituting this into Theorem~\ref{thm:main-robust-regret} yields the linear
rows in Table~\ref{tab:comparison}. In particular, under TV uncertainty,
\[
\widetilde{\mathcal O}\!\left(
\max\left\{
\frac{\gamma^2H^4d_{\mathrm{lin}}^2}{\varepsilon^2},
\frac{H^7(\min\{H,1/\sigma_{\min}\})^4
\,\iota_{\mathrm{lin}}^\delta}{\varepsilon^2}
\right\}
\right)
\]
is obtained up to logarithmic factors. Compared with
\citep{zheng2026distributionally}, this bound has worse horizon and uncertainty
dependence, but it is derived as a direct specialization of a much broader
general-function-approximation theory and applies uniformly to TV, \(\chi^2\),
and KL uncertainty sets. Thus our result should be viewed as complementary to
linear-specific analyses rather than as a minimax-optimal linear result.

\paragraph{Overall interpretation.}
The table highlights a trade-off between specialization and generality. Specialized tabular or linear algorithms achieve sharper rates by exploiting
problem-specific structure, often under a fixed divergence such as TV. In contrast, our framework gives a unified analysis for general-sum DRMGs with general function approximation and multiple \(\phi\)-divergence uncertainty sets. The resulting rates separate three effects: the intrinsic exploration complexity \(d_{\mathrm{MADC}}^{\mathrm{rob},\max}\), and the divergence geometry
\(\Psi_\phi(\sigma)\). This separation clarifies what is structural, what is statistical, and what is specific to the chosen uncertainty model.

\section{Numerical Validation}
\label{app:Experiments}

We provide a preliminary numerical evaluation of {\Algoname} to examine three aspects of the proposed framework: 
(i) whether robust training improves performance under transition-kernel shift relative to its non-robust counterpart;
(ii) whether the use of general function approximation preserves the robustness achieved by an exact tabular robust method; and
(iii) whether increasing the expressiveness of the dual function class reduces the approximation error underlying Assumption~2.
Our experiments therefore complement the theoretical analysis by evaluating the practical behavior of the two central components of {\Algoname}: robust fitted value learning and functional approximation of the dual robust Bellman operator.

\subsection{Experimental Setup}
\label{app:exp-setup}

\paragraph{Risky-Coordination DRMG.}
We consider a scalable general-sum DRMG, termed \emph{Risky-Coordination}, which extends the two-agent risky/safe coordination environment of \citet{farhat2026sampleefficient} to multiple agents and a tunable state space.
The environment consists of $m=4$ agents and horizon $H=2$.
At $h=0$, each agent $i$ observes an independently sampled context
$c_i\in\{0,\ldots,n_{\rm ctx}-1\}$ and chooses between a \emph{safe} and a \emph{risky} action.
The contexts are resampled in every episode and affect neither the transition dynamics nor rewards; instead, $n_{\rm ctx}$ controls the size of the state space and hence the difficulty of explicit tabular estimation.

Let $\rho$ denote the fraction of agents choosing the risky action.
The resulting joint action leads to a trap with probability
\[
    p_{\rm trap}(\rho)
    =
    (\texttt{trap\_base}+\sigma)\rho,
\]
where $\texttt{trap\_base}=0.6$ and $\sigma$ controls the transition-kernel shift.
Conditional on avoiding the trap, the reward level is sampled from
$\mathrm{Binomial}(n_{\rm reward}-1,\rho)$.
A fully safe joint action has zero trap probability for every $\sigma$, providing a shift-invariant alternative to the potentially higher-return risky action.
The outcome is realized through the transition at $h=0$ and the corresponding reward is received at $h=1$, after which the episode terminates.
We consider a TV-divergence ambiguity set around the nominal transition kernel, so that $\sigma$ directly controls the severity of the transition perturbation.

\paragraph{Baselines.}
We compare {\Algoname} with two complementary baselines.
\begin{itemize}
\item \textbf{MAMEX} \citep{Xiong2023SampleEfficientMR} provides the non-robust function-approximation comparison.
We use the same fitted-learning architecture as {\Algoname}, but set the training robustness radius to $\sigma_{\rm train}=0$.
Since the robust Bellman operator reduces to the nominal Bellman operator when $\sigma=0$, this comparison isolates the effect of robust training while keeping the function-approximation machinery otherwise fixed.

\item \textbf{MORNAVI} \citep{farhat2026sampleefficient} provides the robust tabular comparison.
It estimates the transition model explicitly and performs tabular robust planning using the closed-form TV worst-case expectation and an exploration bonus.
Both MORNAVI and {\Algoname} are trained with $\sigma_{\rm train}=0.3$.
This comparison examines whether replacing explicit tabular robust planning with the dual function-approximation machinery of {\Algoname} incurs a substantial loss in robustness.
\end{itemize}

\paragraph{Implementation and evaluation protocol.}
We instantiate {\Algoname} under TV uncertainty (\emph{RoMEX-TV}).
Each agent uses an independent double-Q network together with a dual network that estimates the local upper and lower value bounds required by the TV robust backup.
The resulting Bellman target takes the form
\[
    y(s,a)
    =
    r
    +
    \gamma(1-\mathrm{done})
    \left[
        v_{\rm next}
        -
        \sigma\bigl(g_{\max}-g_{\min}\bigr)
    \right],
\]
which agrees with the closed-form TV worst-case expectation under the calibration used in our experiments.
For computational tractability, we use independent per-agent learners for all methods rather than explicitly constructing the equilibrium oracle over enumerated joint policies.

Each method is trained once at a fixed robustness radius:
$\sigma_{\rm train}=0.3$ for {\Algoname} and MORNAVI, and
$\sigma_{\rm train}=0$ for MAMEX.
After training, we freeze the learned policies and evaluate them under transition shifts
\[
    \sigma\in\{0.0,0.1,0.2,0.3,0.4,0.5\},
\]
without further adaptation.
Results are averaged over three random seeds and $200$ evaluation episodes per seed; shaded regions indicate one standard deviation across seeds.
Thus, the evaluation measures robustness to deployment-time transition shift rather than continued online adaptation.

\subsection{Results}
\label{app:exp-results}

\begin{figure}[t]
    \centering
    \begin{subfigure}[t]{0.32\textwidth}
        \centering
        \includegraphics[width=\linewidth]
        {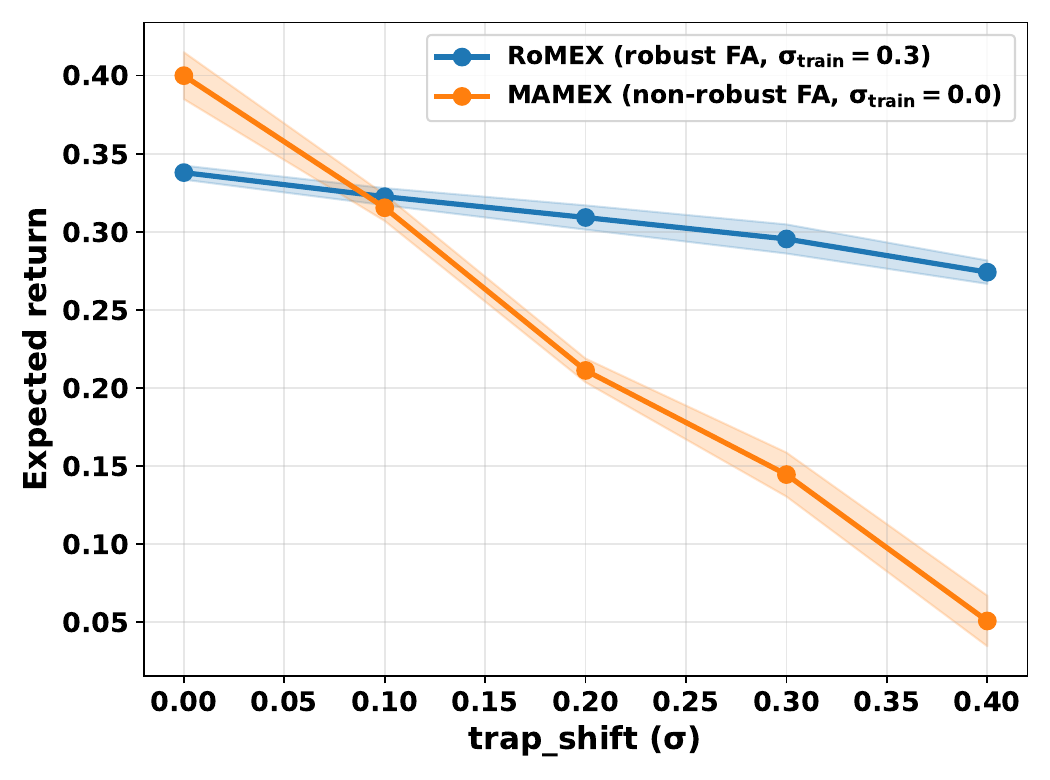}
        \caption{Robust vs.\ non-robust training.}
        \label{fig:romex_vs_mamex}
    \end{subfigure}
    \hfill
    \begin{subfigure}[t]{0.32\textwidth}
        \centering
        \includegraphics[width=\linewidth]
        {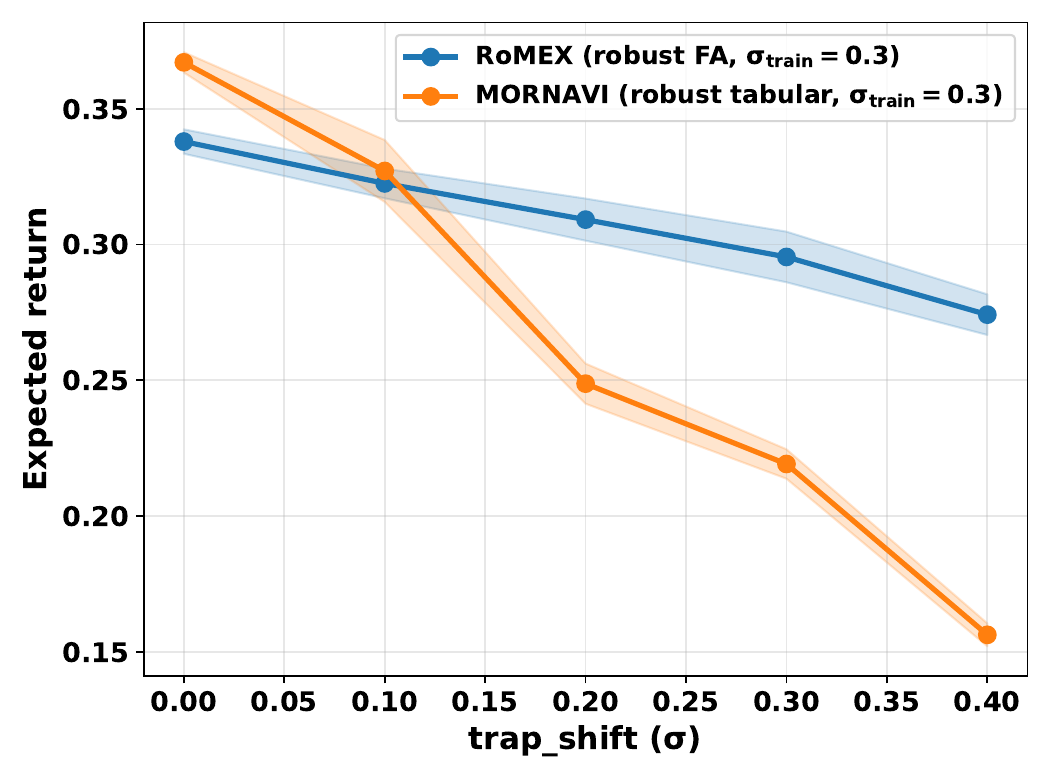}
        \caption{Function approximation vs.\ tabular.}
        \label{fig:romex_vs_mornavi}
    \end{subfigure}
    \hfill
    \begin{subfigure}[t]{0.32\textwidth}
        \centering
        \includegraphics[width=\linewidth]
        {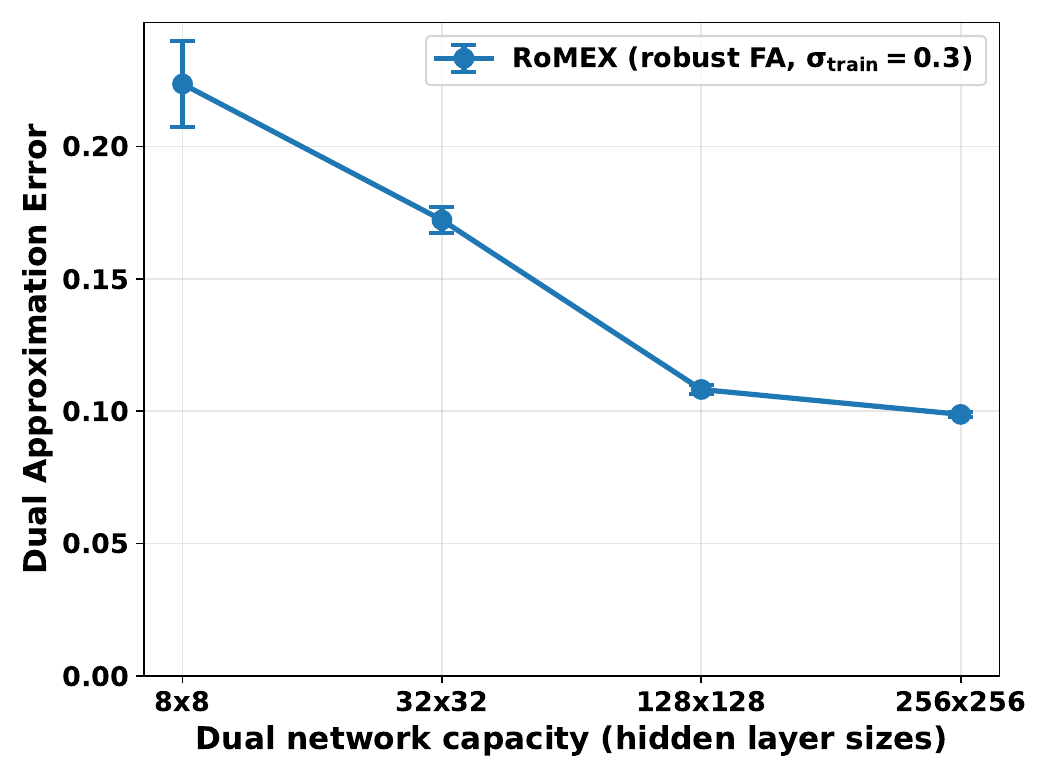}
        \caption{Dual approximation error.}
        \label{fig:assumption2}
    \end{subfigure}
    \caption{
    Performance of {\Algoname}.
    (a) Expected return under increasing transition shift for {\Algoname}
    ($\sigma_{\rm train}=0.3$) and its non-robust counterpart MAMEX
    ($\sigma_{\rm train}=0$).
    (b) Comparison with the robust tabular baseline MORNAVI, with both methods
    trained at $\sigma_{\rm train}=0.3$.
    (c) Empirical dual approximation error as the capacity of the dual function
    class increases.
    Shaded regions denote one standard deviation across three random seeds.
    }
    \label{fig:romex_performance}
\end{figure}

\paragraph{Does robust training improve resilience to transition shift?}
Figure~\ref{fig:romex_vs_mamex} compares {\Algoname} with the non-robust MAMEX baseline.
At the nominal kernel ($\sigma=0$), MAMEX achieves a higher return
($0.392$ vs.\ $0.338$), illustrating the expected trade-off between nominal performance and robustness.
As the transition shift increases, however, MAMEX deteriorates substantially faster.
The curves cross around $\sigma\approx0.1$, after which {\Algoname} consistently achieves higher return.
At $\sigma=0.4$, for example, MAMEX decreases from $0.392$ to $0.063$, whereas {\Algoname} decreases from $0.338$ to $0.277$.
These results indicate that optimizing the robust objective sacrifices some nominal performance but yields substantially greater stability under transition mismatch.

\paragraph{Does function approximation sacrifice robust performance?}
Figure~\ref{fig:romex_vs_mornavi} compares {\Algoname} against MORNAVI, with both methods trained using the same robustness radius $\sigma_{\rm train}=0.3$.
Despite replacing explicit tabular transition estimation and robust planning with learned value and dual functions, {\Algoname} remains competitive with MORNAVI around the training radius and achieves higher return for the larger shifts considered in our evaluation.
Across the shift sweep, MORNAVI's return decreases from approximately $0.367$ to $0.156$, whereas {\Algoname} exhibits substantially milder degradation.
Thus, in this environment, the use of function approximation does not introduce an observable loss in robustness relative to the tabular robust baseline.
This provides empirical support for the motivation behind our framework: robust value estimation can be performed through learned functions without explicitly enumerating the state space.

\paragraph{How expressive must the dual function class be?}
Finally, Figure~\ref{fig:assumption2} examines the approximation component underlying Assumption~2.
We increase the hidden-layer width of the dual network from
$8\times8$ to $256\times256$ and measure its empirical approximation error.
The error decreases from $0.224\pm0.032$ to $0.099\pm0.002$, while the variability across seeds also decreases.
This trend is consistent with Assumption~2: a richer dual function class reduces the empirical gap associated with approximating the functional dual solution.
We emphasize that this experiment does not establish Assumption~2 in general; rather, it provides evidence that the dual approximation error can be made small with a moderately expressive function class in the environment considered here.

\subsection{Ablation: Sensitivity to the Robustness Radius}
\label{app:exp-sigma-sensitivity}

The preceding experiments use $\sigma_{\rm train}=0.3$ for {\Algoname}. 
We next examine whether the observed robustness gain is sensitive to this particular choice and characterize how the training robustness radius affects nominal and shifted performance. 
We train {\Algoname} with
$\sigma_{\rm train}\in\{0.1,0.2,0.3,0.4,0.5\}$, using three random seeds and $n_{\rm ctx}=10$, while keeping the remaining experimental protocol unchanged from Section~\ref{app:exp-setup}. 
We evaluate each trained policy without further adaptation in the nominal environment ($\sigma=0$) and under a representative large transition shift ($\sigma=0.4$). 
For reference, $\sigma_{\rm train}=0$ corresponds to the non-robust MAMEX baseline.

\begin{table}[H]
\centering
\caption{Sensitivity of {\Algoname} to the training robustness radius $\sigma_{\rm train}$ on the Risky-Coordination DRMG ($n_{\rm ctx}=10$, three seeds, and 50 evaluation episodes per seed). The case $\sigma_{\rm train}=0$ corresponds to MAMEX.}
\label{tab:sigma-sensitivity}
\begin{tabular}{lcc}
\toprule
$\sigma_{\rm train}$ 
& Return at $\sigma=0$ 
& Return at $\sigma=0.4$ \\
\midrule
0.0 (MAMEX) & $0.366 \pm 0.444$ & $0.037 \pm 0.105$ \\
0.1 & $0.349 \pm 0.236$ & $0.214 \pm 0.160$ \\
0.2 & $0.338 \pm 0.139$ & $0.254 \pm 0.142$ \\
0.3 & $0.336 \pm 0.112$ & $0.262 \pm 0.137$ \\
0.4 & $0.334 \pm 0.119$ & $0.257 \pm 0.140$ \\
0.5 & $0.336 \pm 0.122$ & $0.259 \pm 0.138$ \\
\bottomrule
\end{tabular}
\end{table}

Table~\ref{tab:sigma-sensitivity} reveals a clear trade-off between nominal performance and robustness to transition shift. 
At nominal deployment ($\sigma=0$), introducing robustness results in a modest decrease in average return, from $0.366$ at $\sigma_{\rm train}=0$ to approximately $0.33$--$0.35$ for positive robustness radii. 
In contrast, under the larger shift $\sigma=0.4$, even a small robustness radius substantially improves performance: the average return increases from $0.037$ at $\sigma_{\rm train}=0$ to $0.214$ at $\sigma_{\rm train}=0.1$ and to $0.254$ at $\sigma_{\rm train}=0.2$. 
Beyond $\sigma_{\rm train}=0.2$, performance remains relatively stable, with average returns between $0.254$ and $0.262$; given the observed variability, these differences should not be interpreted as a meaningful ordering among the larger radii. 
Overall, the ablation shows that the robustness gain observed in Figure~\ref{fig:romex_vs_mamex} is not specific to the choice $\sigma_{\rm train}=0.3$: training with a positive robustness radius consistently improves performance under substantial transition shift, while increasing the radius beyond a moderate level yields diminishing empirical gains in this environment.

\subsection{Discussion and Limitations}
\label{app:exp-limits}

Taken together, the experiments support three observations.
First, robust training provides substantially greater resilience to transition-kernel shift than an otherwise comparable non-robust learner.
Second, the function-approximation implementation of {\Algoname} remains competitive with an explicit tabular robust method, suggesting that the scalability gained through functional approximation need not come at the expense of robust performance.
Third, increasing dual-network capacity systematically reduces the empirical dual approximation error, supporting the practical feasibility of the functional dual formulation.

These experiments are intended as a controlled validation of the proposed mechanism rather than a comprehensive benchmark study.
In particular, we use independent per-agent learners instead of the explicit equilibrium oracle in Algorithm~1, and our implementation focuses on TV uncertainty because its closed-form worst-case expectation enables direct validation of the learned dual approximation.
Extending the empirical study to $\chi^2$ and KL uncertainty requires the corresponding dual objectives and robust backups.
Finally, our implementation of MORNAVI uses a Hoeffding-style exploration bonus rather than the tighter Bernstein-type bonus appearing in its theoretical analysis.

\section{Conclusion}
\label{sec:Conclusion}
We study online distributionally robust MARL in general-sum Markov games under general function approximation. We propose {\Algoname}, a model-free framework that uses a dual reformulation of the robust Bellman operator to enable tractable worst-case evaluation from nominal data via an equilibrium oracle. We introduce the \emph{robust Multi-Agent Decoupling Coefficient} (robust MADC) as an intrinsic complexity measure and establish the first sublinear regret guarantees for learning robust equilibria. Our results decompose sample complexity into intrinsic exploration, governed by robust MADC, and robust Bellman estimation from nominal data, while avoiding explicit dependence on the state and joint action spaces. Robust MADC further provides a unified view connecting intrinsic complexity with structured function classes. An important direction for future work is extending the analysis to richer function classes, such as infinite or neural parameterizations, as well as to decentralized or partially observed multi-agent settings.


\bibliographystyle{abbrv}  
\bibliography{iclr2027_conference}

@book{fudenberg1991game,
  title={{Game Theory}},
  author={Fudenberg, Drew and Tirole, Jean},
  year={1991},
  publisher={MIT press}
}

@article{daskalakisnash09,
author = {Daskalakis, Constantinos and Goldberg, Paul W. and Papadimitriou, Christos H.},
title = {{The Complexity of Computing a Nash Equilibrium}},
year = {2009},
issue_date = {February 2009},
publisher = {Association for Computing Machinery},
address = {New York, NY, USA},
volume = {52},
number = {2},
issn = {0001-0782},
doi = {10.1145/1461928.1461951},
journal = {Commun. ACM},
month = feb,
pages = {89–97},
numpages = {9}
}

@InProceedings{pmlr-v139-liu21z,
  title = 	 {{A Sharp Analysis of Model-based Reinforcement Learning with Self-Play}},
  author =       {Liu, Qinghua and Yu, Tiancheng and Bai, Yu and Jin, Chi},
  booktitle = 	 {Proceedings of the 38th International Conference on Machine Learning (ICML)},
  pages = 	 {7001--7010},
  year = 	 {2021},
  editor = 	 {Meila, Marina and Zhang, Tong},
  volume = 	 {139},
  series = 	 {Proceedings of Machine Learning Research (PMLR)},
  month = 	 {18--24 Jul},
  publisher =    {PMLR},
}

@inproceedings{Xiong2023SampleEfficientMR,
  author       = {Nuoya Xiong and
                  Zhihan Liu and
                  Zhaoran Wang and
                  Zhuoran Yang},
  title        = {{Sample-Efficient Multi-Agent {RL:} An Optimization Perspective}},
  booktitle    = {The Twelfth International Conference on Learning Representations,
                  {ICLR} 2024, Vienna, Austria, May 7-11, 2024},
  publisher    = {OpenReview.net},
  year         = {2024},
  bibsource    = {dblp computer science bibliography, https://dblp.org}
}

@article{yangdual22,
author = {Yang, Wenhao and Zhang, Liangyu and Zhang, Zhihua},
year = {2022},
month = {12},
pages = {},
title = {{Toward Theoretical Understandings of Robust Markov Decision Processes: Sample Complexity and Asymptotics}},
volume = {50},
journal = {The Annals of Statistics},
doi = {10.1214/22-AOS2225}
}

@article{ghosh2025scaling,
  title={{Scaling Online Distributionally Robust Reinforcement Learning: Sample-Efficient Guarantees with General Function Approximation}},
  author={Ghosh, Debamita and Atia, George K and Wang, Yue},
  journal={arXiv preprint arXiv:2512.18957},
  year={2025}
}

@misc{ghosh2025provablynearoptimaldistributionallyrobust,
      title={{ORVIT: Near-Optimal Online Distributionally Robust Reinforcement Learning}}, 
      author={Debamita Ghosh and George K. Atia and Yue Wang},
      year={2025},
      eprint={2508.03768},
      archivePrefix={arXiv},
      primaryClass={cs.LG},
      url={https://arxiv.org/abs/2508.03768}, 
}

@article{INFORM2005_RobusDP_Iyengar,
  title={{Robust Dynamic Programming}},
  author={Iyengar, Garud N},
  journal={Mathematics of Operations Research},
  volume={30},
  number={2},
  pages={257--280},
  year={2005},
  publisher={INFORMS}
}

@article{NeuRIPS2023_CuriousPriceDRRLGenerativeMdel_Shi,
  title={{The Curious Price of Distributional Robustness in Reinforcement Learning with a Generative Model}},
  author={Shi, Laixi and Li, Gen and Wei, Yuting and Chen, Yuxin and Geist, Matthieu and Chi, Yuejie},
  journal={Advances in Neural Information Processing Systems},
  volume={36},
  pages={79903--79917},
  year={2023}
}

@article{INFORMS2013_RobustMDP_wiesemann,
  title={{Robust Markov Decision Processes}},
  author={Wiesemann, Wolfram and Kuhn, Daniel and Rustem, Ber{\c{c}}},
  journal={Mathematics of Operations Research},
  volume={38},
  number={1},
  pages={153--183},
  year={2013},
  publisher={INFORMS}
}

@inproceedings{PMLR2021_DROTabularRL_Zhou,
  title={{Finite-Sample Regret Bound for Distributionally Robust Offline Tabular Reinforcement Learning}},
  author={Zhou, Zhengqing and Zhou, Zhengyuan and Bai, Qinxun and Qiu, Linhai and Blanchet, Jose and Glynn, Peter},
  booktitle={International Conference on Artificial Intelligence and Statistics},
  pages={3331--3339},
  year={2021},
  organization={PMLR}
}

@article{NeuRIPS2023_DoublePessimismDROfflineRL_Blanchet,
  title={{Double Pessimism is Provably Efficient for Distributionally Robust Offline Reinforcement Learning: Generic Algorithm and Robust Partial Coverage}},
  author={Blanchet, Jose and Lu, Miao and Zhang, Tong and Zhong, Han},
  journal={Advances in Neural Information Processing Systems},
  volume={36},
  pages={66845--66859},
  year={2023}
}

@article{AnnalsStat2022_TheoreticalUnderstandingRMDP_Yang,
  title={{Toward Theoretical Understandings of Robust Markov Decision Processes: Sample Complexity and Asymptotics}},
  author={Yang, Wenhao and Zhang, Liangyu and Zhang, Zhihua},
  journal={The Annals of Statistics},
  volume={50},
  number={6},
  pages={3223--3248},
  year={2022},
  publisher={Institute of Mathematical Statistics}
}

@article{NeuRIPS2024_UnifiedPessimismOfflineRL_Yue,
  title={{A Unified Principle of Pessimism for Offline Reinforcement Learning under Model Mismatch}},
  author={Wang, Yue and Sun, Zhongchang and Zou, Shaofeng},
  journal={Advances in Neural Information Processing Systems},
  volume={37},
  pages={9281--9328},
  year={2024}
}

@article{Arxiv2023_TowardsMinimaxOptimalityRobustRL_Clavier,
  title={{Towards Minimax Optimality of Model-based Robust Reinforcement Learning}},
  author={Clavier, Pierre and Pennec, Erwan Le and Geist, Matthieu},
  journal={arXiv preprint arXiv:2302.05372},
  year={2023}
}

@inproceedings{liu2024distributionally,
  title={Distributionally robust off-dynamics reinforcement learning: Provable efficiency with linear function approximation},
  author={Liu, Zhishuai and Xu, Pan},
  booktitle={International Conference on Artificial Intelligence and Statistics},
  pages={2719--2727},
  year={2024},
  organization={PMLR}
}

@string{aaai="Proc. Conference on Artificial Intelligence (AAAI)"}

@string{colt="Proc. Annual Conference on Learning Theory (CoLT)"}

@string{iclr="Proc. International Conference on Learning Representations (ICLR)"}

@string{icml="Proc. International Conference on Machine Learning (ICML)"}

@string{nips="Proc. Advances in Neural Information Processing Systems (NIPS)"}

@string{nipsnew="Proc. Advances in Neural Information Processing Systems (NeurIPS)"}

@article{Arxiv2024_DRORLwithInteractiveData_Lu,
  title={Distributionally robust reinforcement learning with interactive data collection: Fundamental hardness and near-optimal algorithm},
  author={Lu, Miao and Zhong, Han and Zhang, Tong and Blanchet, Jose},
  journal={The Thirty-eighth Annual Conference on Neural Information Processing Systemss},
  year={2024}
}

@article{NeuRIPS2001_OnlineLearningAlgo_Cesa,
  title={{On the Generalization Ability of On-Line Learning Algorithms}},
  author={Cesa-Bianchi, Nicol{\'o} and Conconi, Alex and Gentile, Claudio},
  journal={Advances in neural information processing systems},
  volume={14},
  year={2001}
}

@inproceedings{ICML2025_OnlineDRMDPSampleComplexity_He,
  title={{Sample Complexity of Distributionally Robust Off-Dynamics Reinforcement Learning with Online Interaction}},
  author={He, Yiting and Liu, Zhishuai and Wang, Weixin and Xu, Pan},
  booktitle={Forty-second International Conference on Machine Learning},
  year={2025}
}

@article{Arxiv2020_OffDynRL_Eysenbach,
  title={{Off-Dynamics Reinforcement Learning: Training for Transfer with Domain Classifiers}},
  author={Eysenbach, Benjamin and Asawa, Swapnil and Chaudhari, Shreyas and Levine, Sergey and Salakhutdinov, Ruslan},
  journal={arXiv preprint arXiv:2006.13916},
  year={2020}
}

@article{Arxiv2024_UpperLowerDRRL_Liu,
  title={Upper and lower bounds for distributionally robust off-dynamics reinforcement learning},
  author={Liu, Zhishuai and Wang, Weixin and Xu, Pan},
  journal={arXiv preprint arXiv:2409.20521},
  year={2024}
}

@article{NeurIPS2024_MinimaxOptimalOfflineRL_Liu,
  title={{Minimax Optimal and Computationally Efficient Algorithms for Distributionally Robust Offline Reinforcement Learning}},
  author={Liu, Zhishuai and Xu, Pan},
  journal={Advances in Neural Information Processing Systems},
  volume={37},
  pages={86602--86654},
  year={2024}
}

@article{Arxiv2024_ModelFreeRobustRL_Panaganti,
  title={{Model-Free Robust $\phi$-Divergence Reinforcement Learning Using Both Offline and Online Data}},
  author={Panaganti, Kishan and Wierman, Adam and Mazumdar, Eric},
  journal={arXiv preprint arXiv:2405.05468},
  year={2024}
}

@article{NeurIPS2022_RobustRLOffline_Panaganti,
  title={{Robust Reinforcement Learning using Offline Data}},
  author={Panaganti, Kishan and Xu, Zaiyan and Kalathil, Dileep and Ghavamzadeh, Mohammad},
  journal={Advances in neural information processing systems},
  volume={35},
  pages={32211--32224},
  year={2022}
}

@article{Arxiv2024_SampleComplexityOfflineLinearDRMDP_Wang,
  title={{Sample Complexity of Offline Distributionally Robust Linear Markov Decision Processes}},
  author={Wang, He and Shi, Laixi and Chi, Yuejie},
  journal={arXiv preprint arXiv:2403.12946},
  year={2024}
}

@article{Arxiv2022_OfflineDRRLLinearFunctionApprox_Ma,
  title={{Distributionally Robust Offline Reinforcement Learning with Linear Function Approximation}},
  author={Ma, Xiaoteng and Liang, Zhipeng and Blanchet, Jose and Liu, Mingwen and Xia, Li and Zhang, Jiheng and Zhao, Qianchuan and Zhou, Zhengyuan},
  journal={arXiv preprint arXiv:2209.06620},
  year={2022}
}

@article{NeurIPS2021_OnlineRobustRLModelUncertainty_Wang,
  title={{Online Robust Reinforcement Learning with Model Uncertainty}},
  author={Wang, Yue and Zou, Shaofeng},
  journal={Advances in Neural Information Processing Systems},
  volume={34},
  pages={7193--7206},
  year={2021}
}

@article{NeurIPS2021_BellmanEluderDim_Jin,
  title={{Bellman Eluder Dimension: New Rich Classes of RL Problems, and Sample-Efficient Algorithms}},
  author={Jin, Chi and Liu, Qinghua and Miryoosefi, Sobhan},
  journal={Advances in neural information processing systems},
  volume={34},
  pages={13406--13418},
  year={2021}
}

@article{vinitsky2020robust,
  title={Robust Reinforcement Learning using Adversarial Populations},
  author={Vinitsky, Eugene and Du, Yuqing and Parvate, Kanaad and Jang, Kathy and Abbeel, Pieter and Bayen, Alexandre},
  journal={arXiv preprint arXiv:2008.01825},
  year={2020}
}

@article{abdullah2019wasserstein,
  title={{Wasserstein Robust Reinforcement Learning}},
  author={Abdullah, Mohammed Amin and Ren, Hang and Ammar, Haitham Bou and Milenkovic, Vladimir and Luo, Rui and Zhang, Mingtian and Wang, Jun},
  journal={arXiv preprint arXiv:1907.13196},
  year={2019}
}

@inproceedings{zhang2020robust,
  title={Robust Multi-Agent Reinforcement Learning with Model Uncertainty},
  author={Zhang, Kaiqing and Sun, Tao and Tao, Yunzhe and Genc, Sahika and Mallya, Sunil and Basar, Tamer},
  booktitle=nipsnew,
  volume={33},
  year={2020}
}

@article{kardecs2011discounted,
  title={Discounted robust stochastic games and an application to queueing control},
  author={Karde{\c{s}}, Erim and Ord{\'o}{\~n}ez, Fernando and Hall, Randolph W},
  journal={Operations research},
  volume={59},
  number={2},
  pages={365--382},
  year={2011},
  publisher={INFORMS}
}

@inproceedings{agarwal2014taming,
  title={{Taming the Monster: A Fast and Simple Algorithm for Contextual Bandits}},
  author={Agarwal, Alekh and Hsu, Daniel and Kale, Satyen and Langford, John and Li, Lihong and Schapire, Robert},
  booktitle=icml,
  pages={1638--1646},
  year={2014},
  organization={PMLR}
}

@inproceedings{jiang2017contextual,
  title={Contextual decision processes with low {B}ellman rank are {PAC}-learnable},
  author={Jiang, Nan and Krishnamurthy, Akshay and Agarwal, Alekh and Langford, John and Schapire, Robert E},
  booktitle=icml,
  pages={1704--1713},
  year={2017},
  organization={PMLR}
}

@inproceedings{abbasi2011improved,
  title={{Improved Algorithms for Linear Stochastic Bandits}},
  author={Abbasi-Yadkori, Yasin and P{\'a}l, D{\'a}vid and Szepesv{\'a}ri, Csaba},
  booktitle={NIPS},
  volume={11},
  pages={2312--2320},
  year={2011}
}

@article{du2021bilinear,
  title={Bilinear classes: A structural framework for provable generalization in rl},
  author={Du, Simon S and Kakade, Sham M and Lee, Jason D and Lovett, Shachar and Mahajan, Gaurav and Sun, Wen and Wang, Ruosong},
  journal={arXiv preprint arXiv:2103.10897},
  year={2021}
}

@inproceedings{liu2025distributionally,
title={Distributionally Robust Multi-Agent Reinforcement Learning for Dynamic Chute Mapping},
author={Guangyi Liu and Suzan Iloglu and Michael Caldara and Joseph W Durham and Michael M. Zavlanos},
booktitle=icml,
year={2025}
}

@inproceedings{zhao2020sim,
  title={Sim-to-real transfer in deep reinforcement learning for robotics: a survey},
  author={Zhao, Wenshuai and Queralta, Jorge Pe{\~n}a and Westerlund, Tomi},
  booktitle={2020 IEEE symposium series on computational intelligence (SSCI)},
  pages={737--744},
  year={2020},
  organization={IEEE}
}

@book{holla2021off,
  title={On the off-dynamics approach to reinforcement learning},
  author={Holla, Joshua Arvind},
  year={2021},
  publisher={McGill University (Canada)}
}

@inproceedings{
zhang2025modelfree,
title={Model-Free Offline Reinforcement Learning with Enhanced Robustness},
author={Chi Zhang and Zain Ulabedeen Farhat and George K. Atia and Yue Wang},
booktitle=iclr,
year={2025}
}

@inproceedings{peng2018sim,
  title={Sim-to-real transfer of robotic control with dynamics randomization},
  author={Peng, Xue Bin and Andrychowicz, Marcin and Zaremba, Wojciech and Abbeel, Pieter},
  booktitle={2018 IEEE international conference on robotics and automation (ICRA)},
  pages={3803--3810},
  year={2018},
  organization={IEEE}
}

@article{padakandla2020reinforcement,
  title={Reinforcement learning algorithm for non-stationary environments},
  author={Padakandla, Sindhu and KJ, Prabuchandran and Bhatnagar, Shalabh},
  journal={Applied Intelligence},
  volume={50},
  number={11},
  pages={3590--3606},
  year={2020},
  publisher={Springer}
}

@article{rajeswaran2016epopt,
  title={Epopt: Learning robust neural network policies using model ensembles},
  author={Rajeswaran, Aravind and Ghotra, Sarvjeet and Ravindran, Balaraman and Levine, Sergey},
  journal={arXiv preprint arXiv:1610.01283},
  year={2016}
}

@article{jiao_minimax-optimal_2024,
  title={Minimax-Optimal Multi-Agent Robust Reinforcement Learning},
  author={Jiao, Yuchen and Li, Gen},
  journal={arXiv preprint arXiv:2412.19873},
  year={2024}
}

@article{shi2024breaking,
  title={Breaking the Curse of Multiagency in Robust Multi-Agent Reinforcement Learning},
  author={Shi, Laixi and Gai, Jingchu and Mazumdar, Eric and Chi, Yuejie and Wierman, Adam},
  journal={arXiv preprint arXiv:2409.20067},
  year={2024}
}

@unpublished{li2025sample,
  title={Sample Efficient Robust Offline Self-Play for Model-based Reinforcement Learning},
  author={Li, Na and Zheng, Zewu and Ni, Wei and Shan, Hangguan and Zhang, Wenjie and Li, Xinyu},
  note={Manuscript, OpenReview preprint},
  year={2025},
  url={https://openreview.net/forum?id=3lXZjsir0e}
}

@article{ma2023decentralized,
  title={Decentralized robust v-learning for solving markov games with model uncertainty},
  author={Ma, Shaocong and Chen, Ziyi and Zou, Shaofeng and Zhou, Yi},
  journal={Journal of Machine Learning Research},
  volume={24},
  number={371},
  pages={1--40},
  year={2023}
}

@inproceedings{cisneros2023finite,
  title={{Finite-sample Guarantees for Nash Q-learning with Linear Function Approximation
}},
  author={Cisneros-Velarde, Pedro and Koyejo, Sanmi},
  booktitle={Uncertainty in Artificial Intelligence},
  pages={424--432},
  year={2023},
  organization={PMLR}
}

@inproceedings{jin2022power,
  title={The power of exploiter: Provable multi-agent rl in large state spaces},
  author={Jin, Chi and Liu, Qinghua and Yu, Tiancheng},
  booktitle=icml,
  pages={10251--10279},
  year={2022},
  organization={PMLR}
}

@article{shalev2016safe,
  title={Safe, multi-agent, reinforcement learning for autonomous driving},
  author={Shalev-Shwartz, Shai and others},
  journal={arXiv preprint arXiv:1610.03295},
  year={2016}
}

@inproceedings{lowe2017multi,
  title={Multi-agent actor-critic for mixed cooperative-competitive environments},
  author={Lowe, Ryan and Wu, Yi and Tamar, Aviv and Harb, Jean and Abbeel, Pieter and Mordatch, Igor},
  booktitle=nips,
  pages={6379--6390},
  year={2017}
}

@article{matignon2012independent,
  title={Independent reinforcement learners in cooperative markov games: a survey regarding coordination problems},
  author={Matignon, Laetitia and Laurent, Guillaume J and Le Fort-Piat, Nadine},
  journal={The Knowledge Engineering Review},
  volume={27},
  number={1},
  pages={1--31},
  year={2012},
  publisher={Cambridge University Press}
}

@article{papoudakis2019dealing,
title={Dealing with non-stationarity in multi-agent deep reinforcement learning},
author={Papoudakis, Georgios and Christianos, Filippos and Rahman, Arrasy and Albrecht, Stefano V},
journal={arXiv preprint arXiv:1906.04737},
year={2019}
}

@book{albrecht2024multi,
  title={{Multi-Agent Reinforcement Learning: Foundations and Modern Approaches}},
  author={Albrecht, Stefano V and Christianos, Filippos and Sch{\"a}fer, Lukas},
  year={2024},
  publisher={MIT Press}
}

@inproceedings{xie2020learning,
  title={Learning zero-sum simultaneous-move markov games using function approximation and correlated equilibrium},
  author={Xie, Qiaomin and Chen, Yudong and Wang, Zhaoran and Yang, Zhuoran},
  booktitle=colt,
  pages={3674--3682},
  year={2020},
  organization={PMLR}
}

@article{wang_reinforcement_2020,
	title = {Reinforcement {Learning} with {General} {Value} {Function} {Approximation}: {Provably} {Efficient} {Approach} via {Bounded} {Eluder} {Dimension}},
	shorttitle = {Reinforcement {Learning} with {General} {Value} {Function} {Approximation}},
	language = {en},
	urldate = {2022-03-08},
	journal = {arXiv:2005.10804 [cs, math, stat]},
	author = {Wang, Ruosong and Salakhutdinov, Ruslan and Yang, Lin F.},
	month = jun,
	year = {2020},
	note = {arXiv: 2005.10804},
}

@inproceedings{du_bilinear_2021,
	series = {Proceedings of {Machine} {Learning} {Research}},
	title = {Bilinear {Classes}: {A} {Structural} {Framework} for {Provable} {Generalization} in {RL}},
	volume = {139},
	booktitle = {Proceedings of the 38th {International} {Conference} on {Machine} {Learning}},
	publisher = {PMLR},
	author = {Du, Simon and Kakade, Sham and Lee, Jason and Lovett, Shachar and Mahajan, Gaurav and Sun, Wen and Wang, Ruosong},
	editor = {Meila, Marina and Zhang, Tong},
	month = jul,
	year = {2021},
	pages = {2826--2836},
}

@article{wong2023deep,
  title={Deep multiagent reinforcement learning: Challenges and directions},
  author={Wong, Annie and B{\"a}ck, Thomas and Kononova, Anna V and Plaat, Aske},
  journal={Artificial Intelligence Review},
  volume={56},
  number={6},
  pages={5023--5056},
  year={2023},
  publisher={Springer}
}

@article{canese2021multi,
  title={{Multi-Agent Reinforcement Learning: A Review of Challenges and Applications}},
  author={Canese, Lorenzo and Cardarilli, Gian Carlo and Di Nunzio, Luca and Fazzolari, Rocco and Giardino, Daniele and Re, Marco and Span{\`o}, Sergio},
  journal={Applied Sciences},
  volume={11},
  number={11},
  pages={4948},
  year={2021},
  publisher={MDPI}
}

@article{Arxiv2024_SampleEfficientMARL_Shi,
  title={{Sample-Efficient Robust Multi-Agent Reinforcement Learning in the Face of Environmental Uncertainty}},
  author={Shi, Laixi and Mazumdar, Eric and Chi, Yuejie and Wierman, Adam},
  journal={arXiv preprint arXiv:2404.18909},
  year={2024}
}

@article{article,
author = {Kardes, Erim and Ordóñez, Fernando and Hall, Randolph},
year = {2011},
month = {04},
pages = {365-382},
title = {Discounted Robust Stochastic Games and an Application to Queueing Control},
volume = {59},
journal = {Operations Research},
doi = {10.2307/23013175}
}

@InProceedings{pmlr-v167-chen22d,
  title = 	 {{Almost Optimal Algorithms for Two-player Zero-Sum Linear Mixture Markov Games}},
  author =       {Chen, Zixiang and Zhou, Dongruo and Gu, Quanquan},
  booktitle = 	 {Proceedings of The 33rd International Conference on Algorithmic Learning Theory},
  pages = 	 {227--261},
  year = 	 {2022},
  editor = 	 {Dasgupta, Sanjoy and Haghtalab, Nika},
  volume = 	 {167},
  series = 	 {Proceedings of Machine Learning Research},
  month = 	 {29 Mar--01 Apr},
  publisher =    {PMLR},
  url = 	 {https://proceedings.mlr.press/v167/chen22d.html}
}

@article{GO_DavidSilver,
author = {Silver, David and Huang, Aja and Maddison, Christopher and Guez, Arthur and Sifre, Laurent and Driessche, George and Schrittwieser, Julian and Antonoglou, Ioannis and Panneershelvam, Veda and Lanctot, Marc and Dieleman, Sander and Grewe, Dominik and Nham, John and Kalchbrenner, Nal and Sutskever, Ilya and Lillicrap, Timothy and Leach, Madeleine and Kavukcuoglu, Koray and Graepel, Thore and Hassabis, Demis},
year = {2016},
month = {01},
pages = {484-489},
title = {Mastering the game of Go with deep neural networks and tree search},
volume = {529},
journal = {Nature},
doi = {10.1038/nature16961}
}

@article{Vinyals2019GrandmasterLI,
  title={Grandmaster level in StarCraft II using multi-agent reinforcement learning},
  author={Oriol Vinyals and Igor Babuschkin and Wojciech M. Czarnecki and Micha{\"e}l Mathieu and Andrew Dudzik and Junyoung Chung and David Choi and Richard Powell and Timo Ewalds and Petko Georgiev and Junhyuk Oh and Dan Horgan and Manuel Kroiss and Ivo Danihelka and Aja Huang and L. Sifre and Trevor Cai and John P. Agapiou and Max Jaderberg and Alexander Sasha Vezhnevets and R{\'e}mi Leblond and Tobias Pohlen and Valentin Dalibard and David Budden and Yury Sulsky and James Molloy and Tom Le Paine and Caglar Gulcehre and Ziyun Wang and Tobias Pfaff and Yuhuai Wu and Roman Ring and Dani Yogatama and Dario W{\"u}nsch and Katrina McKinney and Oliver Smith and Tom Schaul and Timothy P. Lillicrap and Koray Kavukcuoglu and Demis Hassabis and Chris Apps and David Silver},
  journal={Nature},
  year={2019},
  volume={575},
  pages={350 - 354},
  url={https://api.semanticscholar.org/CorpusID:204972004}
}

@misc{hua2024multiagentreinforcementlearningconnected,
      title={Multi-Agent Reinforcement Learning for Connected and Automated Vehicles Control: Recent Advancements and Future Prospects}, 
      author={Min Hua and Dong Chen and Xinda Qi and Kun Jiang and Zemin Eitan Liu and Quan Zhou and Hongming Xu},
      year={2024},
      eprint={2312.11084},
      archivePrefix={arXiv},
      primaryClass={cs.RO},
      url={https://arxiv.org/abs/2312.11084}, 
}

@inproceedings{lu2020deepreinforcementlearningready,
  title={{Is deep reinforcement learning ready for practical applications in healthcare? A sensitivity analysis of duel-DDQN for hemodynamic management in sepsis patients}},
  author={Lu, MingYu and Shahn, Zachary and Sow, Daby and Doshi-Velez, Finale and Lehman, Li-wei H},
  booktitle={AMIA annual symposium proceedings},
  volume={2020},
  pages={773},
  year={2021}
}

@misc{alaa2023drlhealthcare,
  author       = {Alaa Eldin, Baraa},
  title        = {{Why applying Deep Reinforcement Learning in Healthcare is hard}},
  year         = {2023},
  howpublished = {\url{https://medium.com/@baraa.alaa.eldin/why-applying-deep-reinforcement-learning-in-healthcare-is-hard-ffc6e05ab7ca}},
  note         = {Accessed: 2025-07-28}
}

@misc{demontis2022surveyreinforcementlearningsecurity,
      title={A Survey on Reinforcement Learning Security with Application to Autonomous Driving}, 
      author={Ambra Demontis and Maura Pintor and Luca Demetrio and Kathrin Grosse and Hsiao-Ying Lin and Chengfang Fang and Battista Biggio and Fabio Roli},
      year={2022},
      eprint={2212.06123},
      archivePrefix={arXiv},
      primaryClass={cs.LG},
      url={https://arxiv.org/abs/2212.06123}, 
}

@article{daskalakis2013complexity,
  title={{On the Complexity of Approximating a Nash Equilibrium}},
  author={Daskalakis, Constantinos},
  journal={ACM Transactions on Algorithms (TALG)},
  volume={9},
  number={3},
  pages={1--35},
  year={2013},
  publisher={ACM New York, NY, USA}
}

@inproceedings{cui2023breaking,
  title={{Breaking the Curse of Multiagents in a Large State Space: RL in Markov Games with Independent Linear Function Approximation}},
  author={Cui, Qiwen and Zhang, Kaiqing and Du, Simon},
  booktitle={The Thirty Sixth Annual Conference on Learning Theory},
  pages={2651--2652},
  year={2023},
  organization={PMLR}
}

@book{rockafellar1998variational,
  title={Variational analysis},
  author={Rockafellar, R Tyrrell and Wets, Roger JB},
  year={1998},
  publisher={Springer}
}

@book{shalev2014understanding,
  title={Understanding machine learning: From theory to algorithms},
  author={Shalev-Shwartz, Shai and Ben-David, Shai},
  year={2014},
  publisher={Cambridge university press}
}

@article{dann2021provably,
  title={{A Provably Efficient Model-Free Posterior Sampling Method for Episodic Reinforcement Learning}},
  author={Dann, Christoph and Mohri, Mehryar and Zhang, Tong and Zimmert, Julian},
  journal={Advances in neural information processing systems},
  volume={34},
  pages={12040--12051},
  year={2021}
}

@article{russo2013eluder,
  title={Eluder dimension and the sample complexity of optimistic exploration},
  author={Russo, Daniel and Van Roy, Benjamin},
  journal={Advances in Neural Information Processing Systems},
  volume={26},
  year={2013}
}

@inproceedings{
farhat2026sampleefficient,
title={Sample-Efficient Distributionally Robust Multi-Agent Reinforcement Learning via Online Interaction},
author={Zain Ulabedeen Farhat and Debamita Ghosh and George K. Atia and Yue Wang},
booktitle={The Fourteenth International Conference on Learning Representations},
year={2026},
url={https://openreview.net/forum?id=hIBMCDOwMI}
}

@inproceedings{sun2019model,
  title={Model-based rl in contextual decision processes: Pac bounds and exponential improvements over model-free approaches},
  author={Sun, Wen and Jiang, Nan and Krishnamurthy, Akshay and Agarwal, Alekh and Langford, John},
  booktitle={Conference on learning theory},
  pages={2898--2933},
  year={2019},
  organization={PMLR}
}

@inproceedings{wang2023breaking,
  title={Breaking the curse of multiagency: Provably efficient decentralized multi-agent rl with function approximation},
  author={Wang, Yuanhao and Liu, Qinghua and Bai, Yu and Jin, Chi},
  booktitle={The Thirty Sixth Annual Conference on Learning Theory},
  pages={2793--2848},
  year={2023},
  organization={PMLR}
}

@article{foster2021statistical,
  title={The statistical complexity of interactive decision making},
  author={Foster, Dylan J and Kakade, Sham M and Qian, Jian and Rakhlin, Alexander},
  journal={arXiv preprint arXiv:2112.13487},
  year={2021}
}

@article{roch2025distributionally,
      title={Distributionally Robust Markov Games with Average Reward}, 
      author={Zachary Roch and Yue Wang},
      year={2025},
      journal={arXiv preprint arXiv:2508.03136}
}

@inproceedings{dai2024refined,
  title={{Refined Sample Complexity for Markov Games with Independent Linear Function Approximation}},
  author={Dai, Yan and Cui, Qiwen and Du, Simon S},
  booktitle={The Thirty Seventh Annual Conference on Learning Theory},
  pages={1260--1261},
  year={2024},
  organization={PMLR}
}

@inproceedings{he2023nearly,
  title={Nearly minimax optimal reinforcement learning for linear markov decision processes},
  author={He, Jiafan and Zhao, Heyang and Zhou, Dongruo and Gu, Quanquan},
  booktitle={International Conference on Machine Learning},
  pages={12790--12822},
  year={2023},
  organization={PMLR}
}

@article{huang2021towards,
  title={Towards general function approximation in zero-sum markov games},
  author={Huang, Baihe and Lee, Jason D and Wang, Zhaoran and Yang, Zhuoran},
  journal={arXiv preprint arXiv:2107.14702},
  year={2021}
}

@article{zhong2022posterior,
  title={A posterior sampling framework for interactive decision making},
  author={Zhong, Han and Xiong, Wei and Zheng, Sirui and Wang, Liwei and Wang, Zhaoran and Yang, Zhuoran and Zhang, Tong},
  journal={arXiv preprint arXiv:2211.01962},
  volume={2},
  number={3},
  year={2022}
}

@article{chen2022unified,
  title={{Unified Algorithms for RL with Decision-Estimation Coefficients: No-Regret, PAC, and Reward-Free Learning }},
  author={Chen, Fan and Mei, Song and Bai, Yu},
  year={2022}
}

@inproceedings{foster2023complexity,
  title={On the complexity of multi-agent decision making: From learning in games to partial monitoring},
  author={Foster, Dean and Foster, Dylan J and Golowich, Noah and Rakhlin, Alexander},
  booktitle={The Thirty Sixth Annual Conference on Learning Theory},
  pages={2678--2792},
  year={2023},
  organization={PMLR}
}

@article{zhan2022decentralized,
  title={Decentralized optimistic hyperpolicy mirror descent: Provably no-regret learning in markov games},
  author={Zhan, Wenhao and Lee, Jason D and Yang, Zhuoran},
  journal={arXiv preprint arXiv:2206.01588},
  year={2022}
}

@article{ni2022representation,
  title={Representation learning for general-sum low-rank markov games},
  author={Ni, Chengzhuo and Song, Yuda and Zhang, Xuezhou and Jin, Chi and Wang, Mengdi},
  journal={arXiv preprint arXiv:2210.16976},
  year={2022}
}

@inproceedings{zheng2026distributionally,
  title={Distributionally Robust Online Markov Game with Linear Function Approximation},
  author={Zheng, Zewu and Lin, Yuanyuan},
  booktitle={Proceedings of the AAAI Conference on Artificial Intelligence},
  volume={40},
  pages={28892--28900},
  year={2026}
}

\newpage
\appendix
\onecolumn
\label{sec:Appendix}

\section{Additional Preliminaries of DRMGs}
\label{app:DRMG_Prelim}

\section{Dual Formulation and Functional Optimization}
\label{app:dual}
\paragraph{Functional formulation.}
To avoid pointwise optimization in the dual Bellman operator, we consider a functional optimization problem over dual functions. Let $\mu_h^\pi$ denote the state–joint action visitation distribution induced by policy $\pi$ under the nominal kernel $P^\star$. 

We define the dual function space
\[
\mathcal{L}^1_i(\mu_h^\pi;\mathbb{R}^2)
:=
\left\{
g_i=(g_i^\lambda,g_i^\nu):
\mathcal{S}\times\mathcal{A}\to\mathbb{R}^2
\;\middle|\;
g_i^\lambda,g_i^\nu \in \mathcal{L}^1(\mu_h^\pi)
\right\},
\]
and the dual loss (restating \Cref{eq:dual_loss})
\[
\mathrm{LossDual}_{i,h,\phi}(g; f_i)
=
\mathbb{E}_{(s,\bm a)\sim \mu_h^\pi}
\left[
\mathbb{E}_{s'\sim P_h^\star}
\mathrm{Loss}_{i,h,\phi}(f_i; s,\bm a,s'; g(s,\bm a))
\right].
\]

\paragraph{Dual equivalence.}
The following lemma extend the results of single-agent in \citep{NeurIPS2022_RobustRLOffline_Panaganti} (which considers the TV-divergence set in an offline setting), \citep{ghosh2025provablynearoptimaldistributionallyrobust} (which consider $\phi$-divergence set in an online setting), and \citep{Arxiv2024_ModelFreeRobustRL_Panaganti} (which considers regularized MDPs in a hybrid setting), and show that for general-sum $\phi$-divergence DRMG, minimization of the dual loss is equivalent to the support function under some distribution for the multi-agent setup. 

\begin{lem}[Multi-agent Dual loss minimization]
\label{lem:equiv_loss_dual}
Let $\mathrm{LossDual}_{i,\phi}$ be defined as the dual loss function as in eq. \ref{eq:dual_loss} for agent $i \in \mathcal M$. Then, for any function $f_i:\mathcal S \times \mathcal A \to [0,H]$ for any $i \in \mathcal M$, we have 
\begin{align*}
\inf_{g\in \mathcal{L}^1_i(\mu;\mathbb R^2)} \mathrm{LossDual}_{i,h\phi}(g; f_i)=
\mathbb{E}_{(s,{\bm a})\sim \mu}\left[\mathbb{E}_{\mathcal{P}_{i,h,\phi}^{\sigma_i}(s, {\bm a})} \Big[V^{\pi,f}_{i,h+1}\Big]\right], \forall i \in \mathcal M.
\end{align*}
\end{lem}


\begin{proof}
We prove the result for an arbitrary but fixed agent $i\in\mathcal M$ and stage $h\in[H]$. However, once the agent $i$ and policy $\pi$ are fixed, the continuation value $V_{i,h+1}^{\pi,f}(s')$ is simply a scalar measurable function of $s'$, and the rest of the argument is identical to the single-agent functional dual-minimization proof \citep{ghosh2025provablynearoptimaldistributionallyrobust, NeurIPS2022_RobustRLOffline_Panaganti}. The detailed proof-lines are given below.

\medskip
\noindent
\textbf{Step 1: Define the pointwise continuation value and the pointwise dual objective.}

For each next state $s'\in\mathcal S$, recall the definition of the continuation value as 
\[
V_{i,h+1}^{\pi,f}(s')
:=
\mathbb E_{{\bm a}'\sim\pi_{h+1}(\cdot\mid s')}
\big[f_{i,h+1}(s',{\bm a}')\big].
\]
Since $\pi_{h+1}(\cdot\mid s')$ is a probability distribution on the joint action space
$\mathcal A$ and $f_{i,h+1}$ is measurable on $\mathcal S\times\mathcal A$, the mapping
$s'\mapsto V_{i,h+1}^{\pi,f}(s')$ is measurable. Now, for every joint state-action pair $(s,{\bm a})\in\mathcal S\times\mathcal A$, we define the pointwise scalar optimization objective as
\begin{align}
\label{eq:pointiwise_obj}
\Psi_{i,h}^{f}(s,{\bm a};\lambda,\nu)
:=
\mathbb E_{s'\sim P_h^\star(\cdot\mid s,{\bm a})}
\!\left[
\lambda\sigma_i-\nu
+
\lambda\phi^\star\!\left(
-\frac{V_{i,h+1}^{\pi,f}(s')-\nu}{\lambda}
\right)
\right],
\qquad \lambda>0,\ \nu\in\mathbb R.
\end{align}
Equivalently, if we define the pointwise integrand as $\mathrm{Loss}_{i,h,\phi}(f_i;s,{\bm a},s';\lambda,\nu)$ as before, then
\[
\Psi_{i,h}(f_i;s,{\bm a};\lambda,\nu) = \mathbb E_{s'\sim P_h^\star(\cdot\mid s,{\bm a})}\big[\mathrm{Loss}_{i,h,\phi}(f_i;s,{\bm a},s';\lambda,\nu)\big].\]

\medskip
\noindent
\textbf{Step 2: Pointwise dual representation of the robust support function.}

By the dual reformulation of the $\phi$-divergence robust Bellman operator used in the draft,
for every fixed $(s,{\bm a})$ we have
\[
\mathbb E_{\mathcal P^\sigma_{i,h,\phi}(s,{\bm a})}
\big[V_{i,h+1}^{\pi,f}\big]
=
\inf_{\lambda>0,\nu\in\mathbb R}
\Psi_{i,h}(f_i;s,{\bm a};\lambda,\nu).
\]
That is,
\begin{equation}
\label{eq:pointwise-dual-identity}
\inf_{\lambda>0,\nu\in\mathbb R}
\mathbb E_{s'\sim P_h^\star(\cdot\mid s,{\bm a})}
\!\left[
\lambda\sigma_i-\nu
+
\lambda\phi^\star\!\left(
-\frac{V_{i,h+1}^{\pi,f}(s')-\nu}{\lambda}
\right)
\right]
=
\mathbb E_{\mathcal P^\sigma_{i,h,\phi}(s,{\bm a})}
\big[V_{i,h+1}^{\pi,f}\big].
\end{equation}
This is exactly the same pointwise duality principle as in the single-agent robust MDP proof, except that here the action variable is the \emph{joint action} ${\bm a}$, and the ambiguity set is the agent-specific set $\mathcal P^\sigma_{i,h,\phi}(s,{\bm a})$. The multi-agent structure does not alter the
duality argument at this step; it only changes the domain from $(s,a_i)$ or $(s,a)$ in the single-agent \citep{ghosh2025provablynearoptimaldistributionallyrobust, NeurIPS2022_RobustRLOffline_Panaganti} to the present joint-action pair $(s,{\bm a})\in\mathcal S\times\mathcal A$.

\medskip
\noindent
\textbf{Step 3: Functional dual loss in terms of the pointwise objective.}

Recall that the dual function is $g_i=(g_i^\lambda,g_i^\nu)\in \mathcal L^1_i(\mu;\mathbb R^2)$, so at each $(s,{\bm a})$ the function $g_i$ selects a pair
\[
g_i(s,{\bm a})=\big(g_i^\lambda(s,{\bm a}),g_i^\nu(s,{\bm a})\big).
\]
By the definition of the dual loss as given in \eqref{eq:dual_loss}, we have
\[\mathrm{Loss}^{\mathrm{Dual}}_{i,h,\phi}(g_i;f_i)
= \mathbb E_{(s,{\bm a})\sim\mu}\Big[\mathbb E_{s'\sim P_h^\star(\cdot\mid s,{\bm a})}\big[\mathrm{Loss}_{i,h,\phi}(f_i;s,{\bm a},s';g_i)\big]\Big],\]
where
\[\mathrm{Loss}_{i,h,\phi}(f_i;s,{\bm a},s';g_i) = g_i^\lambda(s,{\bm a})\sigma_i-g_i^\nu(s,{\bm a}) + g_i^\lambda(s,{\bm a})\phi^\star\!\left(-\frac{V_{i,h+1}^{\pi,f}(s')-g_i^\nu(s,{\bm a})}{g_i^\lambda(s,{\bm a})}\right).\]
Therefore,
\begin{align}
\mathrm{Loss}^{\mathrm{Dual}}_{i,h,\phi}(g_i;f_i)
&= \mathbb E_{(s,{\bm a})\sim\mu}
\Big[
\Psi_{i,h}\big(f_i;s,a;g_i^\lambda(s,a),g_i^\nu(s,a)\big)
\Big].
\label{eq:functional-loss-as-integral}
\end{align}
Hence,
\begin{equation}
\label{eq:main-functional-inf}
\inf_{g_i\in \mathcal L^1_i(\mu;\mathbb R^2)}
\mathrm{Loss}^{\mathrm{Dual}}_{i,h,\phi}(g_i;f_i)
=
\inf_{g_i\in \mathcal L^1_i(\mu;\mathbb R^2)}
\mathbb E_{(s,a)\sim\mu}
\Big[
\Psi_{i,h}\big(f_i;s,a;g_i^\lambda(s,a),g_i^\nu(s,a)\big)
\Big].
\end{equation}

\medskip
\noindent
\textbf{Step 4: Verify that the minimization--integration interchange theorem applies.}

We now want to interchange the functional minimization over $g_i$ and the expectation over $(s,{\bm a})\sim\mu$. For this, define the integrand
\begin{align*}
\ell((s,{\bm a}),x)
&:=
\begin{cases}
    \Psi_{i,h}(f_i;s,{\bm a};x), \quad &\lambda>0\\
    +\infty \quad &\lambda\leq 0
\end{cases}
,\quad \text{where} \quad x=(\lambda,\nu)\in\mathbb R^2.
\end{align*}
We will apply the integral-interchange of \citep{rockafellar1998variational}[Theorem 14.60] as stated in Lemma \ref{lem:rockafellar_Thm14.60}, exactly as done in the single-agent robust offline RL proof \citep{ghosh2025provablynearoptimaldistributionallyrobust}. To justify this, we check the needed properties.

\begin{itemize}
\item \emph{The decision space is decomposable.}
The space $\mathcal L^1(\mu;\mathbb R^2)$ is decomposable. Indeed, if $g_0,g_1\in \mathcal L^1(\mu;\mathbb R^2)$ and $B\in\Sigma(\mathcal S\times\mathcal A)$ is measurable
with finite $\mu$-measure, define 
\[g(s,{\bm a}) := \mathbf 1_B(s,{\bm a})\,g_1(s,{\bm a}) + \mathbf 1_{B^c}(s,{\bm a})\,g_0(s,{\bm a}).\]
Then $g$ is measurable and integrable componentwise, and thus
$g\in \mathcal L^1(\mu;\mathbb R^2)$. Hence $\mathcal L^1(\mu;\mathbb R^2)$ is decomposable.

\item \emph{Measurability in $(s,{\bm a})$.}
For every fixed $x=(\lambda,\nu)$ with $\lambda>0$, the map
\(
(s,{\bm a})\mapsto \ell((s,{\bm a}),x)
\)
is measurable, because $(s,{\bm a})\mapsto P_h^\star(\cdot\mid s,{\bm a})$ is a measurable kernel, and the inner expression
\[
s'\mapsto
\lambda\sigma_i-\nu
+
\lambda\phi^\star\!\left(
-\frac{V_{i,h+1}^{\pi,f}(s')-\nu}{\lambda}
\right)
\]
is measurable in $s'$.

\item \emph{Lower semicontinuity in $x=(\lambda,\nu)$.}
For every fixed $(s,{\bm a})$, the map $(\lambda,\nu)\mapsto \ell((s,{\bm a}),(\lambda,\nu))$ is lower semicontinuous on the domain $\{(\lambda,\nu):\lambda>0,\nu\in\mathbb R\}$.
Indeed, the terms $\lambda\sigma_i-\nu$ are continuous, the convex conjugate $\phi^\star$
is convex and lower semicontinuous, the map
\[
(\lambda,\nu,z)\mapsto
\lambda\phi^\star\!\left(-\frac{z-\nu}{\lambda}\right)
\]
is lower semicontinuous for $\lambda>0$, and expectation with respect to $P_h^\star(\cdot\mid s,{\bm a})$ preserves lower semicontinuity under standard measurability and integrability conditions. Thus $\ell$ is a normal integrand.
\end{itemize}
Therefore, the assumptions of Rockafellar \& Wets integral-interchange lemma \ref{lem:rockafellar_Thm14.60} hold for the present problem.

\medskip
\noindent
\textbf{Step 5: Interchange the infimum and the expectation.}

Applying Rockafellar \& Wets Theorem (Lemma \ref{lem:rockafellar_Thm14.60}) yields
\begin{align}
\inf_{g\in \mathcal L^1_i(\mu;\mathbb R^2)}
\mathbb E_{(s,{\bm a})\sim\mu}\Big[\Psi_{i,h}\big(f_i;s,{\bm a};g_i^\lambda(s,{\bm a}),g_i^\nu(s,{\bm a})\big)\Big]
&=\mathbb E_{(s,{\bm a})\sim\mu}
\Big[\inf_{(\lambda,\nu)\in(0,\infty)\times\mathbb R} \Psi_{i,h}(f_i;s,a;\lambda,\nu)\Big].
\label{eq:interchange-step}
\end{align}
Combining \eqref{eq:main-functional-inf} and \eqref{eq:interchange-step}, we obtain
\begin{equation}
\label{eq:after-interchange}
\inf_{g_i\in L^1_i(\mu;\mathbb R^2)}
\mathrm{LossDual}_{i,h,\phi}(g_i;f_i)
=
\mathbb E_{(s,a)\sim\mu}
\Big[
\inf_{\lambda>0,\nu\in\mathbb R}
\Psi_{i,h}^{f_i}(s,a;\lambda,\nu)
\Big].
\end{equation}

\medskip
\noindent
\textbf{Step 6: Substitute the pointwise robust support-function identity.}

Now substitute the pointwise identity \eqref{eq:pointwise-dual-identity} into
\eqref{eq:after-interchange}. This gives
\begin{align*}
\inf_{g\in \mathcal L^1_i(\mu;\mathbb R^2)}
\mathrm{Loss}^{\mathrm{Dual}}_{i,h,\phi}(g;f_i)
&= \mathbb E_{(s,{\bm a})\sim\mu}\Big[\mathbb E_{\mathcal P^\sigma_{i,h,\phi}(s,{\bm a})}\big[V_{i,h+1}^{\pi,f}\big]\Big].
\end{align*}
Hence, the result follows.\qedhere
\end{proof}

\section{Proof of Lemma \ref{lem:phi_operator_approx}}
\label{subsec:Lemma2_proof}

\begin{proof}
Fix any agent $i\in\mathcal M$ and stage $h\in[H]$. Steps 1--4 and the first part
of Step 5 fix an arbitrary policy $\pi\in\Pi^{\rm pur}$; the dependence of
$\zeta_f$, $\widehat\zeta_f$ and $\widehat g_i^{\,f,\pi}$ on $\pi$ is suppressed
wherever there is no ambiguity, and is restored in the concentration argument of
Step 5, which is established uniformly over $\pi\in\Pi^{\rm pur}$. For notational
simplicity we likewise suppress the dependence on $(i,h)$ throughout.

\paragraph{Step 1: Define the population dual objective.}
For any $f_{i,h+1}\in\mathcal F_{i,h+1}$ and any dual function $g\in\mathcal G_{i,h}$, define
\[
\zeta_f(g) := \mathbb E_{(s,\bm a)\sim \bar\mu_h^{\,n}}
\Big[
\mathbb E_{s'\sim P_h^\star(\cdot\mid s,\bm a)}
\big[
\mathrm{Loss}_{i,h,\phi}(f_i;s,\bm a,s';g)
\big]
\Big].
\]
Since $\bar\mu_h^{\,n}=\frac1n\sum_{t=1}^n\mu_h^{\pi^t}$ is an average of probability
measures on $\mathcal S\times\mathcal A$, linearity of the integral in the measure gives
the equivalent form
\begin{align}
\label{eq:zeta-mixture-linearity}
\zeta_f(g)
=
\frac1n\sum_{t=1}^n
\mathbb E_{(s,\bm a)\sim \mu_h^{\pi^t}}
\Big[
\mathbb E_{s'\sim P_h^\star(\cdot\mid s,\bm a)}
\big[
\mathrm{Loss}_{i,h,\phi}(f_i;s,\bm a,s';g)
\big]
\Big],
\end{align}
which will be used in Step 5. By the definition of the dual Bellman operator, we have
\begin{align}
\label{eq:revised_dual_TV_g}
\bigl[\mathcal T^{\pi,\phi,\sigma_i}_{g} f_i\bigr](s,\bm a) = r_{i,h}(s,\bm a) - \mathbb E_{s'\sim P_h^\star(\cdot\mid s,\bm a)}\big[\mathrm{Loss}_{i,h,\phi}(f_i;s,\bm a,s';g)\big].
\end{align}
On the other hand, by the dual reformulation of the robust Bellman operator as given in \eqref{eq:dual_rob_bellman_operator}, we have
\[
\bigl[\mathcal T^{\pi,\phi,\sigma_i}_{i,h} f_i\bigr](s,\bm a) = r_{i,h}(s,\bm a) - \inf_{\lambda>0,\nu\in\mathbb R}
\mathbb E_{s'\sim P_h^\star(\cdot\mid s,\bm a)}
\big[
\mathrm{Loss}_{i,h,\phi}(f_i;s,\bm a,s';\lambda,\nu)
\big].
\]
Therefore, for every $(s,\bm a)$,
\begin{align*}
   \bigl[\mathcal T^{\pi,\phi,\sigma_i}_{i,h} f_i\bigr](s,\bm a) - \bigl[\mathcal T^{\pi,\phi,\sigma_i}_{g} f_i\bigr](s,\bm a) &=\mathbb E_{s'\sim P_h^\star(\cdot\mid s,\bm a)}\big[\mathrm{Loss}_{i,h,\phi}(f_i;s,\bm a,s';g)\big]\nonumber\\
   &- \inf_{\lambda>0,\nu\in\mathbb R}\mathbb E_{s'\sim P_h^\star(\cdot\mid s,\bm a)}\big[\mathrm{Loss}_{i,h,\phi}(f_i;s,\bm a,s';\lambda,\nu)\big]. 
\end{align*}
Since the second term is the infimum over all admissible dual parameters, the
right-hand side is nonnegative \emph{pointwise in $(s,\bm a)$}. Hence the absolute
value may be removed under any probability measure on $\mathcal S\times\mathcal A$,
and in particular under $\bar\mu_h^{\,n}$:
\begin{align}
\bigl\|
\mathcal T^{\pi,\phi,\sigma_i}_{i,h} f_i
-
\mathcal T^{\pi,\phi,\sigma_i}_{g} f_i
\bigr\|_{1,\bar\mu_h^{\,n}}
&=
\mathbb E_{(s,\bm a)\sim \bar\mu_h^{\,n}}
\Big[
\bigl[\mathcal T^{\pi,\phi,\sigma_i}_{i,h} f_i\bigr](s,\bm a) - \bigl[\mathcal T^{\pi,\phi,\sigma_i}_{g} f_i\bigr](s,\bm a) \Big]
\nonumber\\
&=
\zeta_f(g)
-
\mathbb E_{(s,\bm a)\sim \bar\mu_h^{\,n}}
\Big[
\inf_{\lambda>0,\nu\in\mathbb R}
\mathbb E_{s'\sim P_h^\star(\cdot\mid s,\bm a)}
\big[
\mathrm{Loss}_{i,h,\phi}(f_i;s,\bm a,s';\lambda,\nu)
\big]
\Big].
\label{eq:lemma2-step1}
\end{align}

\paragraph{Step 2: Use Lemma \ref{lem:equiv_loss_dual} to identify the infimum term.}
Lemma \ref{lem:equiv_loss_dual} holds for any probability measure on
$\mathcal S\times\mathcal A$; applying it with $\bar\mu_h^{\,n}$, the interchange rule
for the dual functional objective implies
\[\mathbb E_{(s,\bm a)\sim \bar\mu_h^{\,n}}\Big[\inf_{\lambda>0,\nu\in\mathbb R}\mathbb E_{s'\sim P_h^\star(\cdot\mid s,\bm a)}\big[\mathrm{Loss}_{i,h,\phi}(f_i;s,\bm a,s';\lambda,\nu)\big]\Big]
=\inf_{g\in \mathcal L^1_i(\bar\mu_h^{\,n};\mathbb R^2)} \zeta_f(g).
\]
Substituting into \eqref{eq:lemma2-step1} yields
\begin{equation}
\label{eq:lemma2-main-decomp}
\bigl\|\mathcal T^{\pi,\phi,\sigma_i}_{i,h} f_i - \mathcal T^{\pi,\phi,\sigma_i}_{g} f_i
\bigr\|_{1,\bar\mu_h^{\,n}} = \zeta_f(g)-\inf_{g'\in \mathcal L^1_i(\bar\mu_h^{\,n};\mathbb R^2)}\zeta_f(g').
\end{equation}

\paragraph{Step 3: Specialize to the empirical ERM minimizer.}
Recall $n=|\mathcal D_h|$, and write $(s_t,\bm a_t,s'_t)$ for the stage-$h$ transition
tuple contributed to $\mathcal D_h$ by episode $t$, $t=1,\dots,n$. Let
\[\widehat{g}_i^{\,f,\pi} \in \arg\min_{g\in\mathcal G_{i,h}}
\widehat{\zeta}_f(g)\]
be an empirical minimizer, where
\[\widehat{\zeta}_f(g) :=\frac1n\sum_{t=1}^n\mathrm{Loss}_{i,h,\phi}(f_i;s_t,\bm a_t,s'_t;g).
\]
Applying \eqref{eq:lemma2-main-decomp} with $g=\widehat{g}_i^{\,f,\pi}$ gives
\begin{align}
\bigl\|\mathcal T^{\pi,\phi,\sigma_i}_{i,h} f_i - \mathcal T^{\pi,\phi,\sigma_i}_{\widehat{g}_i^{\,f,\pi}} f_i
\bigr\|_{1,\bar\mu_h^{\,n}} &= \zeta_f(\widehat{g}_i^{\,f,\pi})-\inf_{g\in \mathcal L^1_i(\bar\mu_h^{\,n};\mathbb R^2)}\zeta_f(g) \nonumber\\
&=
\underbrace{\Bigl(\zeta_f(\widehat{g}_i^{\,f,\pi})-\inf_{g\in\mathcal G_{i,h}}\zeta_f(g)\Bigr)}_{(\mathrm{I})}
+
\underbrace{\Bigl(\inf_{g\in\mathcal G_{i,h}}\zeta_f(g)-\inf_{g\in \mathcal L^1_i(\bar\mu_h^{\,n};\mathbb R^2)}\zeta_f(g)\Bigr)}_{(\mathrm{II})}.
\label{eq:lemma2-I-II}
\end{align}

\paragraph{Step 4: Bound the approximation term $(\mathrm{II})$.}
Since $\bar\mu_h^{\,n}$ is an admissible stage-wise distribution, Assumption
\ref{ass:approx_dual_realizability} applies with $\bar\mu_h^{\,n}$ in place of
$\mu_h^\pi$ and gives
\[
\inf_{g\in\mathcal G_{i,h}}\zeta_f(g)-\inf_{g\in \mathcal L^1_i(\bar\mu_h^{\,n};\mathbb R^2)}\zeta_f(g)
\le \varepsilon_{\rm dual}^{\rm approx}.
\]
Hence,
\begin{equation}
\label{eq:lemma2-termII}
(\mathrm{II})\le \varepsilon_{\rm dual}^{\rm approx}.
\end{equation}

\paragraph{Step 5: Bound the ERM excess-risk term $(\mathrm{I})$.}
We now bound the excess-risk term
\[
(\mathrm{I})=\zeta_f(\widehat{g}_i^{\,f,\pi})-\inf_{g\in\mathcal G_{i,h}}\zeta_f(g).
\]
For any fixed $f_{i,h+1}\in\mathcal F_{i,h+1}$, let
$g_f^\star\in\arg\min_{g\in\mathcal G_{i,h}}\zeta_f(g)$ be a population minimizer
over $\mathcal G_{i,h}$. Since $\widehat{g}_i^{\,f,\pi}$ minimizes the empirical objective,
\[
\widehat{\zeta}_f(\widehat{g}_i^{\,f,\pi})\le \widehat{\zeta}_f(g_f^\star).
\]
Hence,
\begin{align}
\text{(I)}= \zeta_f(\widehat{g}_i^{\,f,\pi})-\zeta_f(g_f^\star)
&=
\bigl(\zeta_f(\widehat{g}_i^{\,f,\pi})-\widehat{\zeta}_f(\widehat{g}_i^{\,f,\pi})\bigr)
+
\bigl(\widehat{\zeta}_f(\widehat{g}_i^{\,f,\pi})-\widehat{\zeta}_f(g_f^\star)\bigr)
+
\bigl(\widehat{\zeta}_f(g_f^\star)-\zeta_f(g_f^\star)\bigr)
\nonumber\\
&\le
2\sup_{g\in\mathcal G_{i,h}}
\bigl|
\zeta_f(g)-\widehat{\zeta}_f(g)
\bigr|,\label{eq:I_reformulated}
\end{align}
where the middle term is nonpositive by optimality of $\widehat g_i^{\,f,\pi}$.
Taking suprema over $f_{i,h+1}\in\mathcal F_{i,h+1}$ and $\pi\in\Pi^{\rm pur}$ yields
\[
\sup_{\pi\in\Pi^{\rm pur}}\sup_{f_{i,h+1}\in\mathcal F_{i,h+1}} (\mathrm{I})
\le
2 \sup_{\pi\in\Pi^{\rm pur}}\sup_{f_{i,h+1}\in\mathcal F_{i,h+1}} \sup_{g\in\mathcal G_{i,h}}
|\zeta_f(g) - \widehat{\zeta}_f(g)|.
\]
It therefore suffices to control $|\zeta_f(g)-\widehat\zeta_f(g)|$ uniformly over
$(f_{i,h+1},g,\pi)$; note that the uniformity is what permits the subsequent
substitution of the data-dependent minimizer $\widehat g_i^{\,f,\pi}$.

By Assumption \ref{ass:multiplier-range-TV}, for every admissible $(f,g,\pi)$ and
every sample $(s,\bm a,s')$ we have
$\bigl|\mathrm{Loss}_{i,h,\phi}(f_i;s,\bm a,s';g)\bigr|\le B_\phi(\sigma_i)$.
Restoring the dependence on $\pi$, define the finite class
\[\mathcal H_i := \Bigl\{(s,\bm a,s')\mapsto \mathrm{Loss}_{i,h,\phi}(f_i;s,\bm a,s';g,\pi):\; f_{i,h+1}\in\mathcal F_{i,h+1},\ g\in\mathcal G_{i,h},\ \pi\in\Pi^{\rm pur}\Bigr\},
\]
so that all functions in $\mathcal H_i$ are uniformly bounded by $B_\phi(\sigma_i)$ and
$|\mathcal H_i|\le|\mathcal F_{i,h+1}||\mathcal G_{i,h}||\Pi^{\rm pur}|$.

\emph{Martingale structure.} Fix $(f_{i,h+1},g,\pi)\in\mathcal F_{i,h+1}\times\mathcal G_{i,h}\times\Pi^{\rm pur}$ and define, for $t=1,\dots,n$,
\[
Z_t(f_i,g,\pi):=\mathrm{Loss}_{i,h,\phi}(f_i;s_t,\bm a_t,s'_t;g)
-\mathbb E\bigl[\mathrm{Loss}_{i,h,\phi}(f_i;s_t,\bm a_t,s'_t;g)\,\big|\,\mathcal H_{t-1}\bigr].
\]
Because the behavior policy $\pi^t$ is $\mathcal H_{t-1}$-measurable, the conditional
law of $(s_t,\bm a_t)$ given $\mathcal H_{t-1}$ is $\mu_h^{\pi^t}$, and
$s'_t\sim P_h^\star(\cdot\mid s_t,\bm a_t)$; hence
\[
\mathbb E\bigl[\mathrm{Loss}_{i,h,\phi}(f_i;s_t,\bm a_t,s'_t;g)\,\big|\,\mathcal H_{t-1}\bigr]
=
\mathbb E_{(s,\bm a)\sim\mu_h^{\pi^t}}\Bigl[\mathbb E_{s'\sim P_h^\star(\cdot\mid s,\bm a)}\bigl[\mathrm{Loss}_{i,h,\phi}(f_i;s,\bm a,s';g)\bigr]\Bigr].
\]
Averaging over $t=1,\dots,n$ and comparing with \eqref{eq:zeta-mixture-linearity} gives
$\frac1n\sum_{t=1}^n\mathbb E[\,\cdot\mid\mathcal H_{t-1}]=\zeta_f(g)$, and therefore
\begin{equation}
\label{eq:zeta-gap-is-mds}
\widehat\zeta_f(g)-\zeta_f(g)=\frac1n\sum_{t=1}^n Z_t(f_i,g,\pi).
\end{equation}
By construction $\mathbb E[Z_t\mid\mathcal H_{t-1}]=0$, and $|Z_t|\le 2B_\phi(\sigma_i)$ by
Assumption \ref{ass:multiplier-range-TV}; thus $\{Z_t(f_i,g,\pi)\}_{t=1}^n$ is a bounded
martingale difference sequence with respect to the filtration $\{\mathcal H_t\}_{t\ge0}$.

\emph{Concentration.} The conditional variances satisfy
$\sum_{t=1}^n\mathbb E[Z_t^2\mid\mathcal H_{t-1}]\le 4nB_\phi(\sigma_i)^2$. By Freedman's
inequality (Lemma \ref{lem:Freedman}), for a fixed triple $(f_{i,h+1},g,\pi)$ and any
$\delta'\in(0,1)$, with probability at least $1-\delta'$,
\[
\Bigl|\sum_{t=1}^n Z_t(f_i,g,\pi)\Bigr|
\le
C_1 B_\phi(\sigma_i)\sqrt{n\log(1/\delta')}+C_2 B_\phi(\sigma_i)\log(1/\delta').
\]
If $\log(1/\delta')>n$ the claimed bound below is vacuous, since
$|\widehat\zeta_f(g)-\zeta_f(g)|\le 2B_\phi(\sigma_i)$ always; we may therefore assume
$\log(1/\delta')\le n$, in which case the second term is dominated by the first.
Combining with \eqref{eq:zeta-gap-is-mds},
\[
\bigl|\zeta_f(g)-\widehat\zeta_f(g)\bigr|
\le C_3 B_\phi(\sigma_i)\sqrt{\frac{\log(1/\delta')}{n}}.
\]
Taking a union bound over the finite class $\mathcal H_i$, i.e. over
$(f_{i,h+1},g,\pi)\in\mathcal F_{i,h+1}\times\mathcal G_{i,h}\times\Pi^{\rm pur}$ with
$\delta'=\delta/(|\mathcal F_{i,h+1}||\mathcal G_{i,h}||\Pi^{\rm pur}|)$, we obtain that with
probability at least $1-\delta$,
\[
\sup_{\pi\in\Pi^{\rm pur}}\sup_{f_{i,h+1}\in\mathcal F_{i,h+1}}\sup_{g\in\mathcal G_{i,h}}
\bigl|\zeta_f(g)-\widehat\zeta_f(g)\bigr|
\le
C\,B_\phi(\sigma_i)\sqrt{\frac{\log\bigl(|\mathcal F_{i,h+1}||\mathcal G_{i,h}||\Pi^{\rm pur}|/\delta\bigr)}{n}}.
\]
We emphasise that the union bound is taken over the hypothesis and policy classes
only, and not over the realizations of the behavior sequence $\{\pi^t\}_{t=1}^n$;
the adaptivity of that sequence is absorbed by the martingale structure above.
Substituting into \eqref{eq:I_reformulated} gives
\begin{equation}
\label{eq:lemma2-termI}
\sup_{\pi\in\Pi^{\rm pur}}\sup_{f_{i,h+1}\in\mathcal F_{i,h+1}} (\mathrm{I})
\le
2C\,B_\phi(\sigma_i)
\sqrt{\frac{\log\bigl(|\mathcal F_{i,h+1}||\mathcal G_{i,h}||\Pi^{\rm pur}|/\delta\bigr)}{n}}.
\end{equation}

\paragraph{Step 6: Combine the bounds.}
Combining \eqref{eq:lemma2-I-II}, \eqref{eq:lemma2-termII}, and
\eqref{eq:lemma2-termI}, we conclude that with probability at least
$1-\delta$,
\begin{align*}
\sup_{\pi\in\Pi^{\rm pur}}
\sup_{f_{i,h+1}\in\mathcal F_{i,h+1}}
\bigl\|
\mathcal T^{\pi,\phi,\sigma_i}_{i,h} f_{i,h+1}
-
\mathcal T^{\pi,\phi,\sigma_i}_{\widehat{g}_i^{\,f,\pi}} f_{i,h+1}
\bigr\|_{1,\bar\mu_h^{\,n}}
&\le
C' B_\phi(\sigma_i)
\sqrt{\frac{\log\bigl(|\mathcal F_{i,h+1}||\mathcal G_{i,h}||\Pi^{\rm pur}|/\delta\bigr)}{n}}
\\
&\qquad \qquad +\varepsilon_{\rm dual}^{\rm approx},
\end{align*}
for another absolute constant $C'>0$. This proves the result.\qedhere
\end{proof}

\section{Proof of Theorem \ref{thm:main-robust-regret}}
\label{app:regret_bound_proof}

\textbf{Notation.} For each agent $i\in\mathcal M$, we consider $\pi^k$ to be the output of the equilibrium oracle for a target equilibrium notion $\mathsf{Eq}\in\{\mathrm{NE},\mathrm{CCE},\mathrm{CE}\}$ and we define the corresponding unilateral deviation policy:
\begin{align}
\label{eq:best_response_i_all_agents}
\widetilde{\bm \pi}^{\dagger, k, \mathsf{Eq}}_i
=
\begin{cases}
\pi_{i,\mathrm{NE}}^{\dagger,k}(\pi_{-i}^k)\times \pi_{-i}^k, & \mathsf{Eq}=\mathrm{NE}, \\[6pt]
\pi_{i,\mathrm{CCE}}^{\dagger,k}(\pi_{-i}^k)\times \pi_{-i}^k, & \mathsf{Eq}=\mathrm{CCE}, \\[6pt]
\psi_i^{\dagger,k}\diamond \pi^k, & \mathsf{Eq}=\mathrm{CE}.
\end{cases}
\end{align}

\begin{proof}
We prove the result for robust NE regret. The robust CCE and robust CE cases
follow by replacing the surrogate Nash-deviation inequality with the corresponding
surrogate CCE/CE deviation inequality from Lemma~\ref{lem:surrogate_ne_avg}.

Let
\[
\iota_{\mathcal F, \mathcal G}^\delta :=
\log\!\left(
\frac{|\mathcal F||\mathcal G||\Pi^{\mathrm{pur}}|KHm}{\delta}
\right),
\qquad
B_i:=B_\phi(\sigma_i).
\]

\medskip
\noindent
\textbf{Step 1: Reduce robust regret to per-agent gaps.}
Let \(\pi_i^{\dagger,k}(\pi_{-i}^k)\) be a robust best response of agent \(i\)
against \(\pi_{-i}^k\), and define
\[
\widetilde{\bm \pi}_i^{\dagger,k}
:=
\pi_i^{\dagger,k}(\pi_{-i}^k)\times \pi_{-i}^k.
\]
For each episode \(k\) and agent \(i\), define
\[
\mathrm{Gap}_{i,k}^{\mathrm{NE}}
:=
V_{i,1}^{\widetilde{\bm \pi}_i^{\dagger,k},\sigma_i}(s_1^k)
-
V_{i,1}^{\pi^k,\sigma_i}(s_1^k).
\]
By Definition~\ref{def:robust_regret},
\begin{equation}
\mathrm{Reg}^{\mathrm{NE}}(K)
\le
\max_{i\in\mathcal M}
\sum_{k=1}^K
\mathrm{Gap}_{i,k}^{\mathrm{NE}}.
\label{eq:regret-reduction-main-alt}
\end{equation}
Thus, it suffices to bound
\(\sum_{k=1}^K\mathrm{Gap}_{i,k}^{\mathrm{NE}}\) for each fixed agent \(i\).

\medskip
\noindent
\textbf{Step 2: Surrogate decomposition.}
Fix an agent \(i\in\mathcal M\). For each episode \(k\), add and subtract the
surrogate payoff of the deviation policy and the played policy:
\begin{align}
\mathrm{Gap}_{i,k}^{\mathrm{NE}}
&=
\underbrace{
V_{i,1}^{\widetilde{\bm \pi}_i^{\dagger,k},\sigma_i}(s_1^k)
-
\mathbb E_{\widetilde{\bm \pi}_i^{\dagger,k}}[\widehat U_i^k]
}_{=:A_{i,k}}
+
\underbrace{
\mathbb E_{\widetilde{\bm \pi}_i^{\dagger,k}}[\widehat U_i^k]
-
\mathbb E_{\pi^k}[\widehat U_i^k]
}_{=:B_{i,k}}
\nonumber\\
&\qquad+
\underbrace{
\mathbb E_{\pi^k}[\widehat U_i^k]
-
V_{i,1}^{\pi^k,\sigma_i}(s_1^k)
}_{=:C_{i,k}} .
\label{eq:gap-decomposition-main-alt}
\end{align}
Since \(\pi^k\) is a Nash equilibrium of the surrogate normal-form game
\(\widehat{\bm \Gamma}^k\), Lemma~\ref{lem:surrogate_ne_avg} gives
\[
\mathbb E_{\pi^k}[\widehat U_i^k]
\ge
\mathbb E_{\widetilde{\bm \pi}_i^{\dagger,k}}[\widehat U_i^k].
\]
Hence \(B_{i,k}\le 0\), and therefore
\begin{equation}
\mathrm{Gap}_{i,k}^{\mathrm{NE}}
\le
A_{i,k}+C_{i,k}.
\label{eq:gap-AC-main-alt}
\end{equation}

\medskip
\noindent
\textbf{Step 3: Bound the deviation-side term \(A_{i,k}\).}
Apply Lemma~\ref{lem:benchmark-surrogate} with
\(\pi=\widetilde{\bm \pi}_i^{\dagger,k}\). Let
\[
f_i^{\dagger,k,\star}
:=
f_i^{\widetilde{\bm \pi}_i^{\dagger,k},\sigma_i}
\]
be the robust Bellman-consistent benchmark sequence for the deviation policy.
Then, for every \(k\),
\begin{align}
A_{i,k}
\le
&\;
\beta
\widehat{\mathcal K}_i^{k-1}
\!\left(
f_i^{\dagger,k,\star};
\widehat g_i^{f_i^{\dagger,k,\star},\widetilde{\bm \pi}_i^{\dagger,k}};
\widetilde{\bm \pi}_i^{\dagger,k}
\right)
+
\mathcal R^{A}_{i,k},
\label{eq:Aik-cor1}
\end{align}
where the residuals satisfy the cumulative bound
\begin{align}
\sum_{k=1}^K \mathcal R^{A}_{i,k}
\le
\mathcal O\!\left(
\beta H B_i^2 K\sqrt{\iota_{\mathcal F, \mathcal G}^\delta}
+ H B_i\sqrt{K\iota_{\mathcal F, \mathcal G}^\delta}
+
\varepsilon_{\mathrm{dual}}^{\mathrm{approx}}
\right).
\label{eq:Aik-residual-cumulative}
\end{align}
Here, \eqref{eq:Aik-residual-cumulative} is the cumulative form of the
benchmark-to-surrogate comparison: the \(C_{\mathrm{cov}}\)-term is summed using
\[
\sum_{k=1}^K
\sqrt{\frac{\iota_{\mathcal F, \mathcal G}^\delta}{|\mathcal D_h^{k-1}|\vee 1}}
\le
\mathcal O(\sqrt{K\iota_{\mathcal F, \mathcal G}^\delta}),
\]
and the dual-switch and plug-in approximation terms are controlled on the
uniform high-probability event. Moreover, by Lemma~\ref{lem:robust_optimal_concentration}, applied uniformly
over all \((i,k,\pi)\), and using the revised cumulative centered empirical
discrepancy, we have
\begin{align}
\sum_{k=1}^K
\widehat{\mathcal K}_i^{k-1}
\!\left(
f_i^{\dagger,k,\star};
\widehat g_i^{f_i^{\dagger,k,\star},\widetilde{\bm \pi}_i^{\dagger,k}};
\widetilde{\bm \pi}_i^{\dagger,k}
\right)
\le
\mathcal O\!\left(
H B_i^2 K\sqrt{\iota_{\mathcal F, \mathcal G}^\delta}
+
\varepsilon_{\mathrm{dual}}^{\mathrm{approx}}
\right).
\label{eq:Aik-Khat-sum}
\end{align}
Summing \eqref{eq:Aik-cor1} over \(k=1,\dots,K\), and using
\eqref{eq:Aik-residual-cumulative} and \eqref{eq:Aik-Khat-sum}, gives
\begin{align}
\sum_{k=1}^K A_{i,k}
\le
\mathcal O\!\left(
\beta H B_i^2 K\sqrt{\iota_{\mathcal F, \mathcal G}^\delta}
+ H B_i\sqrt{K\iota_{\mathcal F, \mathcal G}^\delta}
+
\varepsilon_{\mathrm{dual}}^{\mathrm{approx}}
\right).
\label{eq:Aik-final}
\end{align}

\medskip
\noindent
\textbf{Step 4: Bound the played-policy term \(C_{i,k}\).}
By the definition of the regularized robust payoff,
\[
\widehat U_i^k(\pi^k)
=
f_{i,1}^{k,\pi^k}(s_1^k)
-
\beta
\widehat{\mathcal K}_i^{k-1}
(f_i^{k,\pi^k};\widehat g_i^{k,\pi^k};\pi^k).
\]
Therefore,
\begin{align}
C_{i,k}
&=
V_{i,1}^{\pi^k,f_i^{k,\pi^k}}(s_1^k)
-
V_{i,1}^{\pi^k,\sigma_i}(s_1^k)
-
\beta
\widehat{\mathcal K}_i^{k-1}
(f_i^{k,\pi^k};\widehat g_i^{k,\pi^k};\pi^k),
\label{eq:Cik-expand}
\end{align}
where
\[
V_{i,h}^{\pi,f_i}(s)
:=
\mathbb E_{\bm a\sim\pi_h(\cdot|s)}
[f_{i,h}(s,\bm a)].
\]

Now apply Definition~\ref{def:robust_MADC} to the played sequence
\(f_i^k:=f_i^{k,\pi^k}\) and the played policies \(\{\pi^k\}_{k=1}^K\).
For any \(\mu>0\),
\begin{align}
\sum_{k=1}^K
\left(
V_{i,1}^{\pi^k,f_i^k}(s_1^k)
-
V_{i,1}^{\pi^k,\sigma_i}(s_1^k)
\right)
&\le
\frac{1}{\mu}
\sum_{k=1}^K\sum_{s=1}^{k-1}
e_{i}^s(f_i^k,\pi^k) +
\mu d_{i,\mathrm{MADC}}^{\mathrm{rob}}
+
cH d_{i,\mathrm{MADC}}^{\mathrm{rob}} .
\label{eq:madc-applied}
\end{align}
Combining \eqref{eq:Cik-expand} and \eqref{eq:madc-applied}, we get
\begin{align}
\sum_{k=1}^K C_{i,k}
\le
&\frac{1}{\mu}
\sum_{k=1}^K\sum_{s=1}^{k-1}
e_{i}^s(f_i^k,\pi^k)
-
\beta
\sum_{k=1}^K
\widehat{\mathcal K}_i^{k-1}
(f_i^k;\widehat g_i^k;\pi^k)
+
\mu d_{i,\mathrm{MADC}}^{\mathrm{rob}}
+
cH d_{i,\mathrm{MADC}}^{\mathrm{rob}} .
\label{eq:Cik-pre-lemma5}
\end{align}

Using Lemma~\ref{lem:robust_A1_corrected}, in its cumulative form under the
revised aggregate definition of \(\widehat{\mathcal K}\), we have
\begin{align}
\sum_{k=1}^K\sum_{s=1}^{k-1}
e_{i}^s(f_i^k,\pi^k)
\le
2
\sum_{k=1}^K
\widehat{\mathcal K}_i^{k-1}
(f_i^k;\widehat g_i^k;\pi^k)
+
\mathcal O\!\left(
H B_i^2 K\sqrt{\iota_{\mathcal F, \mathcal G}^\delta}
+
\varepsilon_{\mathrm{dual}}^{\mathrm{approx}}
\right).
\label{eq:Cik-pre-lemma5-bound}
\end{align}
Substituting \eqref{eq:Cik-pre-lemma5-bound} into
\eqref{eq:Cik-pre-lemma5}, we obtain
\begin{align}
\sum_{k=1}^K C_{i,k}
\le
&\left(\frac{2}{\mu}-\beta\right)
\sum_{k=1}^K
\widehat{\mathcal K}_i^{k-1}
(f_i^k;\widehat g_i^k;\pi^k)
+
\mathcal O\!\left(
\frac{1}{\mu}
H B_i^2 K\sqrt{\iota_{\mathcal F, \mathcal G}^\delta}
+
\frac{1}{\mu}\varepsilon_{\mathrm{dual}}^{\mathrm{approx}}
\right)\nonumber\\
&\qquad \qquad +
\mu d_{i,\mathrm{MADC}}^{\mathrm{rob}}
+
cH d_{i,\mathrm{MADC}}^{\mathrm{rob}} .
\label{eq:Cik-before-choice-alt}
\end{align}
Choose \(\beta\ge 2/\mu\). Since
\(\widehat{\mathcal K}_i^{k-1}(\cdot)\ge 0\), the first term on the
right-hand side of \eqref{eq:Cik-before-choice-alt} is nonpositive and can be
discarded. Thus,
\begin{align}
\sum_{k=1}^K C_{i,k}
\le
\mathcal O\!\left(
\frac{1}{\mu}
H B_i^2 K\sqrt{\iota_{\mathcal F, \mathcal G}^\delta}
+
\mu d_{i,\mathrm{MADC}}^{\mathrm{rob}}
+
cH d_{i,\mathrm{MADC}}^{\mathrm{rob}}
+
\varepsilon_{\mathrm{dual}}^{\mathrm{approx}}
\right).
\label{eq:Cik-final}
\end{align}

\medskip
\noindent
\textbf{Step 5: Combine the deviation and played-policy bounds.}
By \eqref{eq:gap-AC-main-alt},
\[
\sum_{k=1}^K\mathrm{Gap}_{i,k}^{\mathrm{NE}}
\le
\sum_{k=1}^K A_{i,k}
+
\sum_{k=1}^K C_{i,k}.
\]
Combining \eqref{eq:Aik-final} and \eqref{eq:Cik-final}, we obtain
\begin{align}
\sum_{k=1}^K\mathrm{Gap}_{i,k}^{\mathrm{NE}}
\le
&\mathcal O\!\bigg(
\beta H B_i^2 K\sqrt{\iota_{\mathcal F, \mathcal G}^\delta}
+
\frac{1}{\mu}H B_i^2 K\sqrt{\iota_{\mathcal F, \mathcal G}^\delta}
+
H B_i\sqrt{K\iota_{\mathcal F, \mathcal G}^\delta}\nonumber\\
&\qquad +
\mu d_{i,\mathrm{MADC}}^{\mathrm{rob}}
+
cH d_{i,\mathrm{MADC}}^{\mathrm{rob}}
+
\varepsilon_{\mathrm{dual}}^{\mathrm{approx}}
\bigg).
\label{eq:Gapik-before-beta-revised-alt}
\end{align}

\medskip
\noindent
\textbf{Step 6: Choose \(\beta\) and \(\mu\).}
Set
\(
\mu=\sqrt K,
\quad
\beta=\frac{2}{\mu}=\frac{2}{\sqrt K}.
\)
Then \(\beta\ge 2/\mu\), and substituting these choices into
\eqref{eq:Gapik-before-beta-revised-alt} yields
\begin{align}
\sum_{k=1}^K\mathrm{Gap}_{i,k}^{\mathrm{NE}}
\le
\mathcal O\!\left(
H B_i^2\sqrt{K\iota_{\mathcal F, \mathcal G}^\delta}
+
H B_i\sqrt{K\iota_{\mathcal F, \mathcal G}^\delta}
+
d_{i,\mathrm{MADC}}^{\mathrm{rob}}\sqrt K
+
cH d_{i,\mathrm{MADC}}^{\mathrm{rob}}
+
\varepsilon_{\mathrm{dual}}^{\mathrm{approx}}
\right).
\label{eq:Gapik-clean-bound}
\end{align}

\medskip
\noindent
\textbf{Step 7: Maximize over agents.}
Taking the maximum over \(i\in\mathcal M\), using
\[
d_{\mathrm{MADC}}^{\mathrm{rob},\max}
:=
\max_{i\in\mathcal M}
d_{i,\mathrm{MADC}}^{\mathrm{rob}},
\qquad
B_\phi^{\max}
:=
\max_{i\in\mathcal M}
B_\phi(\sigma_i),
\]
and applying \eqref{eq:regret-reduction-main-alt}, we obtain
\begin{align}
\label{eq:regret-bound-check}
\mathrm{Reg}^{\mathrm{NE}}(K)
\le
\mathcal O\!\left(
d_{\mathrm{MADC}}^{\mathrm{rob},\max}\sqrt K
+
H d_{\mathrm{MADC}}^{\mathrm{rob},\max}
+
H(B_\phi^{\max})^2\sqrt{K\iota_{\mathcal F, \mathcal G}^\delta}
+
HB_\phi^{\max}\sqrt{K\iota_{\mathcal F, \mathcal G}^\delta}
+
\varepsilon_{\mathrm{dual}}^{\mathrm{approx}}
\right).
\end{align}
Since \(B_\phi^{\max}\ge 1\), the two robust estimation terms can be combined as
\[
\mathcal{O}\left(H(B_\phi^{\max})^2\sqrt{K\iota_{\mathcal F, \mathcal G}^\delta}
+ H B_\phi^{\max}\sqrt{K\iota_{\mathcal F, \mathcal G}^\delta}\right)
\le
\mathcal O\!\left(
H(B_\phi^{\max})^2\sqrt{K\iota_{\mathcal F, \mathcal G}^\delta}
\right).
\]
Therefore,
\begin{align}
\mathrm{Reg}^{\mathrm{NE}}(K)
\le
\mathcal O\!\left(
d_{\mathrm{MADC}}^{\mathrm{rob},\max}\sqrt K
+
H d_{\mathrm{MADC}}^{\mathrm{rob},\max}
+
H(B_\phi^{\max})^2\sqrt{K\iota_{\mathcal F, \mathcal G}^\delta}
+
\varepsilon_{\mathrm{dual}}^{\mathrm{approx}}
\right).
\label{eq:NE-final-bound}
\end{align}
This completes the proof of the robust NE regret bound.

\paragraph{Robust CCE-regret bound.}
For each episode \(k\in[K]\) and agent \(i\in\mathcal M\), define the unilateral
coarse deviation
\[
\widetilde{\bm \pi}_{i,\mathrm{CCE}}^{\dagger,k}
:=
\pi_i^{\dagger,k,\mathrm{CCE}}(\pi_{-i}^k)\times \pi_{-i}^k,
\]
where
\[
\pi_{i,\mathrm{CCE}}^{\dagger,k}(\pi_{-i}^k)
\in
\arg\max_{\pi_i'}
V_{i,1}^{\pi_i'\times \pi_{-i}^k,\sigma_i}(s_1^k).
\]
Define the per-episode robust CCE gap
\[
\mathrm{Gap}_{i,k}^{\mathrm{CCE}}
:=
V_{i,1}^{\widetilde{\bm \pi}_{i,\mathrm{CCE}}^{\dagger,k},\sigma_i}(s_1^k)
-
V_{i,1}^{\pi^k,\sigma_i}(s_1^k).
\]
By Definition~\ref{def:robust_regret},
\[
\mathrm{Reg}^{\mathrm{CCE}}(K)
\le
\max_{i\in\mathcal M}
\sum_{k=1}^K
\mathrm{Gap}_{i,k}^{\mathrm{CCE}}.
\]
If the regret definition instead sums over agents, the same argument gives the
same bound with an additional factor \(m\).

We decompose each gap as
\[
\mathrm{Gap}_{i,k}^{\mathrm{CCE}}
=
A_{i,k}^{\mathrm{CCE}}
+
B_{i,k}^{\mathrm{CCE}}
+
C_{i,k}^{\mathrm{CCE}},
\]
where
\begin{align*}
A_{i,k}^{\mathrm{CCE}}
&=
V_{i,1}^{\widetilde{\bm \pi}_{i,\mathrm{CCE}}^{\dagger,k},\sigma_i}(s_1^k)
-
\mathbb E_{\widetilde{\bm \pi}_{i,\mathrm{CCE}}^{\dagger,k}}[\widehat U_i^k],\\
B_{i,k}^{\mathrm{CCE}}
&=
\mathbb E_{\widetilde{\bm \pi}_{i,\mathrm{CCE}}^{\dagger,k}}[\widehat U_i^k]
-
\mathbb E_{\pi^k}[\widehat U_i^k],\\
C_{i,k}^{\mathrm{CCE}}
&=
\mathbb E_{\pi^k}[\widehat U_i^k]
-
V_{i,1}^{\pi^k,\sigma_i}(s_1^k).
\end{align*}

\paragraph{Key step: surrogate CCE inequality.}
Since \(\pi^k\) is a CCE of the surrogate game
\(\widehat{\bm \Gamma}^k\), Lemma~\ref{lem:surrogate_ne_avg} gives
\[
\mathbb E_{\pi^k}[\widehat U_i^k]
\ge
\mathbb E_{\widetilde{\bm \pi}_{i,\mathrm{CCE}}^{\dagger,k}}[\widehat U_i^k].
\]
Thus
\[
B_{i,k}^{\mathrm{CCE}}\le 0.
\]

\paragraph{Bounding \(A_{i,k}^{\mathrm{CCE}}\) and \(C_{i,k}^{\mathrm{CCE}}\).}
The terms \(A_{i,k}^{\mathrm{CCE}}\) and \(C_{i,k}^{\mathrm{CCE}}\) are bounded
identically to the Nash case. Specifically, \(A_{i,k}^{\mathrm{CCE}}\) is bounded
by applying Lemma~\ref{lem:benchmark-surrogate} to the coarse deviation policy
\(\widetilde{\bm \pi}_{i,\mathrm{CCE}}^{\dagger,k}\), and then using
Lemma~\ref{lem:robust_optimal_concentration} to control the cumulative
benchmark discrepancy. The term \(C_{i,k}^{\mathrm{CCE}}\) is the same
played-policy term as in the Nash proof and is bounded by applying the robust
MADC inequality together with Lemma~\ref{lem:robust_A1_corrected}. Therefore, using \(B_i\ge 1\)
for each fixed agent \(i\), we have
\[
\sum_{k=1}^K
\mathrm{Gap}_{i,k}^{\mathrm{CCE}}
\le
\mathcal O\!\left(
d_{i,\mathrm{MADC}}^{\mathrm{rob}}\sqrt K
+
H d_{i,\mathrm{MADC}}^{\mathrm{rob}}
+
H B_i^2\sqrt{K\iota_{\mathcal F, \mathcal G}^\delta}
+
\varepsilon_{\mathrm{dual}}^{\mathrm{approx}}
\right).
\]
Maximizing over \(i\) and using the definitions of
\(d_{\mathrm{MADC}}^{\mathrm{rob},\max}\) and \(B_\phi^{\max}\), we obtain
\[
\mathrm{Reg}^{\mathrm{CCE}}(K)
\le
\mathcal O\!\left(
d_{\mathrm{MADC}}^{\mathrm{rob},\max}\sqrt K
+
H d_{\mathrm{MADC}}^{\mathrm{rob},\max}
+
H(B_\phi^{\max})^2\sqrt{K\iota_{\mathcal F, \mathcal G}^\delta}
+
\varepsilon_{\mathrm{dual}}^{\mathrm{approx}}
\right),
\]
up to a factor \(m\) if the CCE regret is defined as a sum over agents rather
than a maximum over agents.

\paragraph{Robust CE-regret bound.}
For each episode \(k\in[K]\) and agent \(i\in\mathcal M\), define the
strategy-modification deviation
\[
\widetilde{\bm \pi}_{i}^{\dagger,k,\mathrm{CE}}
:=
\psi_i^{\dagger,k}\diamond \pi^k,
\]
where
\[
\psi_i^{\dagger,k}
\in
\arg\max_{\psi_i\in\Psi_i}
V_{i,1}^{\psi_i\diamond \pi^k,\sigma_i}(s_1^k).
\]
Define the per-episode robust CE gap
\[
\mathrm{Gap}_{i,k}^{\mathrm{CE}}
:=
V_{i,1}^{\widetilde{\bm \pi}_{i}^{\dagger,k,\mathrm{CE}},\sigma_i}(s_1^k)
-
V_{i,1}^{\pi^k,\sigma_i}(s_1^k).
\]
By Definition~\ref{def:robust_regret},
\[
\mathrm{Reg}^{\mathrm{CE}}(K)
\le
\max_{i\in\mathcal M}
\sum_{k=1}^K
\mathrm{Gap}_{i,k}^{\mathrm{CE}}.
\]
Again, if the CE regret is defined as a sum over agents, the same argument gives
an additional factor \(m\).

We decompose
\[
\mathrm{Gap}_{i,k}^{\mathrm{CE}}
=
A_{i,k}^{\mathrm{CE}}
+
B_{i,k}^{\mathrm{CE}}
+
C_{i,k}^{\mathrm{CE}},
\]
where
\begin{align*}
A_{i,k}^{\mathrm{CE}}
&=
V_{i,1}^{\widetilde{\bm \pi}_{i}^{\dagger,k,\mathrm{CE}},\sigma_i}(s_1^k)
-
\mathbb E_{\widetilde{\bm \pi}_{i}^{\dagger,k,\mathrm{CE}}}[\widehat U_i^k],\\
B_{i,k}^{\mathrm{CE}}
&=
\mathbb E_{\widetilde{\bm \pi}_{i}^{\dagger,k,\mathrm{CE}}}[\widehat U_i^k]
-
\mathbb E_{\pi^k}[\widehat U_i^k],\\
C_{i,k}^{\mathrm{CE}}
&=
\mathbb E_{\pi^k}[\widehat U_i^k]
-
V_{i,1}^{\pi^k,\sigma_i}(s_1^k).
\end{align*}

\paragraph{Key step: surrogate CE inequality.}
Since \(\pi^k\) is a CE of the surrogate game \(\widehat{\bm \Gamma}^k\),
Lemma~\ref{lem:surrogate_ne_avg} gives
\[
\mathbb E_{\pi^k}[\widehat U_i^k]
\ge
\mathbb E_{\widetilde{\bm \pi}_{i}^{\dagger,k,\mathrm{CE}}}[\widehat U_i^k].
\]
Therefore,
\[
B_{i,k}^{\mathrm{CE}}\le 0.
\]

\paragraph{Bounding \(A_{i,k}^{\mathrm{CE}}\) and \(C_{i,k}^{\mathrm{CE}}\).}
The bounds on \(A_{i,k}^{\mathrm{CE}}\) and \(C_{i,k}^{\mathrm{CE}}\) are
identical to the Nash case. The term \(A_{i,k}^{\mathrm{CE}}\) is handled by
Lemma~\ref{lem:benchmark-surrogate} applied to the strategy-modification
deviation \(\widetilde{\bm \pi}_{i}^{\dagger,k,\mathrm{CE}}\), followed by
Lemma~\ref{lem:robust_optimal_concentration}. The term \(C_{i,k}^{\mathrm{CE}}\)
is the same played-policy term and is controlled by the robust MADC inequality
and Lemma~\ref{lem:robust_A1_corrected}. Hence, using \(B_i\ge 1\) for each fixed agent \(i\), we have
\[
\sum_{k=1}^K
\mathrm{Gap}_{i,k}^{\mathrm{CE}}
\le
\mathcal O\!\left(
d_{i,\mathrm{MADC}}^{\mathrm{rob}}\sqrt K
+
H d_{i,\mathrm{MADC}}^{\mathrm{rob}}
+
H B_i^2\sqrt{K\iota_{\mathcal F, \mathcal G}^\delta}
+
\varepsilon_{\mathrm{dual}}^{\mathrm{approx}}
\right).
\]
Maximizing over \(i\) gives
\[
\mathrm{Reg}^{\mathrm{CE}}(K)
\le
\mathcal O\!\left(
d_{\mathrm{MADC}}^{\mathrm{rob},\max}\sqrt K
+
H d_{\mathrm{MADC}}^{\mathrm{rob},\max}
+
H(B_\phi^{\max})^2\sqrt{K\iota_{\mathcal F, \mathcal G}^\delta}
+
\varepsilon_{\mathrm{dual}}^{\mathrm{approx}}
\right),
\]
again up to a factor \(m\) under a sum-over-agents regret definition.

Combining the NE, CCE, and CE cases proves the theorem for
\(\mathsf{Eq}\in\{\mathrm{NE},\mathrm{CCE},\mathrm{CE}\}\).
\end{proof}

\subsection{Regret bound under $\chi^2$ and KL}
\label{app:cor_sample_complexity_phi}

\paragraph{Special cases: $\chi^2$, and KL}
For concreteness, we recall the resulting one-dimensional variational forms for three choices frequently used in robust RL-- TV, $\chi^2$ and KL; detailed derivations can be found in
\citep{AnnalsStat2022_TheoreticalUnderstandingRMDP_Yang, ICML2025_OnlineDRMDPSampleComplexity_He}. Here, we focus on $\chi$ and KL divergence case, and fro TV-divergence refer to Appendix \ref{app:TV_regret_bound}. Under the $\mathcal{S} \times \mathcal{A}$-rectangularity assumption and eq. \ref{eq:dual_rob_bellman_operator}, the robust expectation for any $V:\mathcal{S}\to[0,H]$ and $P^{\star}_h$ admits the following equivalent forms:

\begin{itemize}

\item \textbf{$\chi^2$-divergence} ($\phi(t)=(t-1)^2$).
One obtains a variance-sensitive form:
\begin{align}
\label{eq:dual_chi}
\mathbb{E}_{\mathcal{P}_{i,h,\chi^2}^{\sigma_i}(s, {\bm a})}[f]
=
&-\inf_{\nu \in [0,H]}
\bigg\{\sqrt{\sigma\,\mathrm{Var}_{P_h^\star(\cdot|s,a)}\big(\nu-\max_{a'}f(s',a')\big)_{+}}\nonumber\\
&\qquad + \mathbb E_{s'\sim \mathbb{P}^\star_h(\cdot|s,a)}\big[\max_{a'}f(s',a')-\nu\big]_{+}
\bigg\}.
\end{align}

\item \textbf{KL-divergence} ($\phi(t)=t\log t$).
The robust expectation can be written as
\begin{align}
\label{eq:dual_KL}
\mathbb{E}_{\mathcal{P}_{i,h,\mathrm{KL}}^{\sigma_i}(s, {\bm a})}[f]
=
-\inf_{\nu \in [\underline{\nu},\,H]}
\left\{\nu \log\!\Big( \mathbb E_{s'\sim \mathbb{P}^\star_h(\cdot|s,a)}\big[\exp\{-\max_{a'}f(s',a')/\nu\}\big]\Big)+\nu\sigma\right\},
\end{align}
where $\underline{\nu}>0$ is a regularity bound on the optimal dual variable, as commonly assumed in
\citep{NeuRIPS2023_DoublePessimismDROfflineRL_Blanchet, ICML2025_OnlineDRMDPSampleComplexity_He}.
\end{itemize}

\begin{cor}
\label{cor:sample_complexity_phi}
Consider the same setup as Theorem~\ref{thm:main-robust-regret}. 
Define \(
\sigma_{\max}:=\max_{i\in\mathcal M}\sigma_i.
\)
Then, with probability at least \(1-\delta\),  \Algoname  obtains an \(\varepsilon\)-robust-\(\mathsf{Eq}\) equilibrium with sample complexity of 
\[
T=KH \asymp
\begin{cases}

\max\left\{H^2\bigl(d_{\mathrm{MADC}}^{\mathrm{rob},\max}\bigr)^2\big/\varepsilon^2,
\;H^7 (1+\sqrt{\sigma_{\max}})^4
\iota_{\mathcal F,\mathcal G}^{\delta}\big/\varepsilon^2
\right\},
& \phi=\chi^2,
\\[3ex]
\max\left\{H^2\bigl(d_{\mathrm{MADC}}^{\mathrm{rob},\max}\bigr)^2\big/\varepsilon^2,
\;H^7 \max\{1,\sigma_{\max}^4\}
\iota_{\mathcal F,\mathcal G}^{\delta}\big/\varepsilon^2
\right\},
& \phi=\mathrm{KL}.
\end{cases}
\]
\end{cor}

\begin{proof}
We provide the proof for the robust NE case. The CCE and CE cases follow
analogously by replacing the NE gaps with the corresponding equilibrium gaps
and using the surrogate CCE/CE deviation inequalities.

For each agent \(i\in\mathcal M\) and episode \(k\in[K]\), define the
per-episode robust NE gap
\[
\mathrm{Gap}^{\mathrm{NE}}_{i,k}
:=
V^{\dagger,\pi_k^{-i},\sigma_i}_{i,1}(s_1^k)
-
V^{\pi^k,\sigma_i}_{i,1}(s_1^k).
\]
Recall that
\(
\mathrm{Reg}^{\mathrm{NE}}(K)
:=
\max_{i\in\mathcal M}
\sum_{k=1}^K
\mathrm{Gap}^{\mathrm{NE}}_{i,k}.
\)
Let the output policy be the uniform mixture
\(
\pi^{\mathrm{out}}
:=
\frac{1}{K}\sum_{k=1}^K \pi^k .
\)

\medskip
\noindent
\textbf{Step 1: Relating output gap to cumulative regret.}
Fix any agent \(i\in\mathcal M\). By the standard online-to-batch averaging
argument, the robust gap of the uniform output policy is bounded by the average
of the per-episode robust fogaps:
\[
\mathrm{Gap}^{\mathrm{NE}}_i(\pi^{\mathrm{out}})
\le
\frac{1}{K}
\sum_{k=1}^K
\mathrm{Gap}^{\mathrm{NE}}_{i,k}.
\]
Taking the maximum over agents gives
\[
\mathrm{Gap}^{\mathrm{NE}}(\pi^{\mathrm{out}})
=
\max_{i\in\mathcal M}
\mathrm{Gap}^{\mathrm{NE}}_i(\pi^{\mathrm{out}})
\le
\frac{1}{K}
\mathrm{Reg}^{\mathrm{NE}}(K).
\]

\medskip
\noindent
\textbf{Step 2: Applying the revised regret bound.}
By Theorem~\ref{thm:main-robust-regret}, with
\(\varepsilon_{\mathrm{dual}}^{\mathrm{approx}}=0\), with probability at least
\(1-\delta\),
\[
\mathrm{Reg}^{\mathrm{NE}}(K)
\le
C\left(
d_{\mathrm{MADC}}^{\mathrm{rob},\max}\sqrt K
+
H d_{\mathrm{MADC}}^{\mathrm{rob},\max}
+
H(B_\phi^{\max})^2
\sqrt{K\iota_{\mathcal F,\mathcal G}^{\delta}}
\right),
\]
where \(C>0\) is a universal constant,
\[
d_{\mathrm{MADC}}^{\mathrm{rob},\max}
:=
\max_{i\in\mathcal M}
d_{i,\mathrm{MADC}}^{\mathrm{rob}},
\qquad
B_\phi^{\max}
:=
\max_{i\in\mathcal M}
B_\phi(\sigma_i).
\]
Therefore,
\[
\mathrm{Gap}^{\mathrm{NE}}(\pi^{\mathrm{out}})
\le
C\left(
\frac{d_{\mathrm{MADC}}^{\mathrm{rob},\max}}{\sqrt K}
+
\frac{H d_{\mathrm{MADC}}^{\mathrm{rob},\max}}{K}
+
H(B_\phi^{\max})^2
\sqrt{
\frac{\iota_{\mathcal F,\mathcal G}^{\delta}}{K}}
\right).
\]

\medskip
\noindent
\textbf{Step 3: Choosing \(K\).}
To ensure
\(\mathrm{Gap}^{\mathrm{NE}}(\pi^{\mathrm{out}})\le\varepsilon\), it is
sufficient that
\[
\frac{d_{\mathrm{MADC}}^{\mathrm{rob},\max}}{\sqrt K}
\lesssim \varepsilon,
\qquad
\frac{H d_{\mathrm{MADC}}^{\mathrm{rob},\max}}{K}
\lesssim \varepsilon,
\qquad
H(B_\phi^{\max})^2
\sqrt{
\frac{\iota_{\mathcal F,\mathcal G}^{\delta}}{K}}
\lesssim \varepsilon.
\]
Equivalently,
\[
K
\gtrsim
\max\left\{
\frac{(d_{\mathrm{MADC}}^{\mathrm{rob},\max})^2}{\varepsilon^2},
\frac{H d_{\mathrm{MADC}}^{\mathrm{rob},\max}}{\varepsilon},
\frac{H^2(B_\phi^{\max})^4
\iota_{\mathcal F,\mathcal G}^{\delta}}{\varepsilon^2}
\right\}.
\]
Multiplying by \(H\), the sample complexity \(T=KH\) satisfies
\[
T
=
\mathcal O\!\left(
\max\left\{
\frac{H(d_{\mathrm{MADC}}^{\mathrm{rob},\max})^2}{\varepsilon^2},
\frac{H^2 d_{\mathrm{MADC}}^{\mathrm{rob},\max}}{\varepsilon},
\frac{H^3(B_\phi^{\max})^4
\iota_{\mathcal F,\mathcal G}^{\delta}}{\varepsilon^2}
\right\}
\right).
\]
Since \(H\ge1\), \(d_{\mathrm{MADC}}^{\mathrm{rob},\max}\ge1\), and
\(\varepsilon\in(0,1]\), the first two terms are jointly dominated by
\[
\frac{H^2(d_{\mathrm{MADC}}^{\mathrm{rob},\max})^2}{\varepsilon^2}.
\]
Thus, it is enough to take
\[
T
=
\mathcal O\!\left(
\max\left\{
\frac{H^2(d_{\mathrm{MADC}}^{\mathrm{rob},\max})^2}{\varepsilon^2},
\frac{H^3(B_\phi^{\max})^4
\iota_{\mathcal F,\mathcal G}^{\delta}}{\varepsilon^2}
\right\}
\right).
\]
Under this choice, \(\mathrm{Gap}^{\mathrm{NE}}(\pi^{\mathrm{out}})\le
\varepsilon\), so \(\pi^{\mathrm{out}}\) is an
\(\varepsilon\)-robust NE policy.

\medskip
\noindent
\textbf{Step 4: Divergence-specific sample-complexity bounds.}
Finally, substitute the divergence-specific bounds from Assumption~\ref{ass:multiplier-range-TV} and by using \eqref{eq:dual_chi} (for $\chi^2$) and \eqref{eq:dual_KL} (for KL), we get
\[
B_{\chi^2}(\sigma_i)
=
\mathcal O\!\left(
H(1+\sqrt{\sigma_i})
\right),
\quad \text{and} \quad 
B_{\mathrm{KL}}(\sigma_i)
=
\mathcal O\!\left(
H\max\{1,\sigma_i\}
\right).
\]
For heterogeneous radii, define
\[
\sigma_{\min}:=\min_{i\in\mathcal M}\sigma_i,
\qquad
\sigma_{\max}:=\max_{i\in\mathcal M}\sigma_i.
\]
Then
\[
B_{\chi^2}^{\max}
=
\mathcal O\!\left(
H(1+\sqrt{\sigma_{\max}})
\right), \quad \text{and} \quad 
B_{\mathrm{KL}}^{\max}
=
\mathcal O\!\left(
H\max\{1,\sigma_{\max}\}
\right).
\]

Substituting these bounds into
\[
T
=
\mathcal O\!\left(
\max\left\{
\frac{H^2(d_{\mathrm{MADC}}^{\mathrm{rob},\max})^2}{\varepsilon^2},
\frac{H^3(B_\phi^{\max})^4
\iota_{\mathcal F,\mathcal G}^{\delta}}{\varepsilon^2}
\right\}
\right)
\]
gives
\[
\phi=\chi^2:\qquad
T
=
\mathcal O\!\left(
\max\left\{
\frac{H^2(d_{\mathrm{MADC}}^{\mathrm{rob},\max})^2}{\varepsilon^2},
\frac{H^7
(1+\sqrt{\sigma_{\max}})^4
\iota_{\mathcal F,\mathcal G}^{\delta}}
{\varepsilon^2}
\right\}
\right),
\]
and
\[
\phi=\mathrm{KL}:\qquad
T
=
\mathcal O\!\left(
\max\left\{
\frac{H^2(d_{\mathrm{MADC}}^{\mathrm{rob},\max})^2}{\varepsilon^2},
\frac{H^7
\max\{1,\sigma_{\max}^4\}
\iota_{\mathcal F,\mathcal G}^{\delta}}
{\varepsilon^2}
\right\}
\right).
\]

The CCE and CE cases follow identically from the same online-to-batch argument
using the robust CCE and robust CE regret bounds from
Theorem~\ref{thm:main-robust-regret}. Hence, under the stated choice of \(T\),
the output policy is an \(\varepsilon\)-robust-\(\mathsf{Eq}\) policy for
\(\mathsf{Eq}\in\{\mathrm{NE},\mathrm{CCE},\mathrm{CE}\}\).
\end{proof}

\section{Theoretical Guarantee of TV-divergence Uncertainty Set}
\label{app:TV_regret_bound}
Focusing on DRMGs with TV-distance uncertainty sets, \citep{Arxiv2024_DRORLwithInteractiveData_Lu,Arxiv2024_UpperLowerDRRL_Liu}. showed that efficient online learning is achievable under suitable assumptions; in particular, \citep{liu2024distributionally} introduced the following fail-states assumption. 
\begin{assu}[Failure States]
\label{ass:vanisin_minimal}
For any agent $i \in \mathcal M$, there exists an (agent-specified) set of failure states $\mathcal{S}_f^i\subseteq\mathcal{S}$, such that $r_i(s,\bm{a})=0$,  and $P^\star_h(s'|s,\bm{a})=0$, $\forall \bm{a}\in\mathcal{A},\forall s\in\mathcal{S}^i_f,\forall s'\notin\mathcal{S}^i_f$. 
\end{assu}
The paper \citep{ICML2025_OnlineDRMDPSampleComplexity_He} highlights that existing efficient learning results are tailored to TV-distance uncertainty sets and do not extend to more general $\phi$-divergence settings (e.g., KL or $\chi^2$). In these broader settings, key structural properties break down, making online learning in robust MDPs fundamentally challenging without additional assumptions. To address this, the authors introduce an intrinsic difficulty measure based on the ratio between nominal and worst-case visitation distributions, and assume this ratio is bounded (polynomially), enabling tractable analysis. Motivated by this, we apply the following definition to tackle the issue of the TV case, as follows:
\begin{defi}[Robust-to-nominal concentrability]
\label{ass:robust-nominal-concentrability}
For every agent \(i\in\mathcal M\), stage \(h\in[H]\), episode \(k\in[K]\), and
policy \(\pi\in\Pi^{\mathrm{pur}}\), let \(\nu_{i,h}^{\pi}\) denote the
stage-\(h\) state--joint-action occupancy induced by \(\pi\) under the
worst-case transition kernels attaining the robust Bellman recursion for agent
\(i\). Let
\(
\bar\mu_h^{k-1}
:=
\frac{1}{k-1}\sum_{j=1}^{k-1}\mu_h^{\pi^j}
\)
denote the historical nominal data-generating mixture distribution at stage
\(h\), with the convention that the first episode is handled separately. We denote the  finite constant \(C_{\mathrm{cov}}\ge 1\) as the smallest value such that
\(
\left\|
\frac{d\nu_{i,h}^{\pi}}{d\bar\mu_h^{k-1}}
\right\|_{\infty}
\le
C_{\mathrm{cov}} .
\)
This value exists due to the support inclusion Definition \ref{def:f_divergence_uncertainty}.
\end{defi}

\begin{thm}[Unified robust equilibrium-regret bound for {\Algoname}]
\label{thm:main-robust-regret-TV}
Consider the same setup as in Theorem~\ref{thm:main-robust-regret}. Then, under Definition  \ref{ass:robust-nominal-concentrability} with probability at least $1-\delta$, the corresponding robust equilibrium regret for TV-divergence set satisfies
\begin{align*}
\mathrm{Reg}^{\mathsf{Eq}}_{\mathrm{TV}}(K) \le \mathcal O\!\left(
d_{\mathrm{MADC}}^{\mathrm{rob},\max}\sqrt K
+
H d_{\mathrm{MADC}}^{\mathrm{rob},\max}
+
C_{\mathrm{cov}}H(B_\phi^{\max})^2\sqrt{K\iota_{\mathcal F, \mathcal G}^\delta}
+
\varepsilon_{\mathrm{dual}}^{\mathrm{approx}}
\right).
\end{align*}
Consequently, for \(\varepsilon^{\mathrm{approx}}_{\mathrm{dual}}=0\), and assuming Assumption~\ref{ass:vanisin_minimal} for TV-divergence, with probability at least \(1-\delta\), the sample complexity of
{\AlgonameTV} to obtain an \(\varepsilon\)-robust-\(\mathsf{Eq}\) policy is
\(T=KH\), where
\[
T= \mathcal O\!\left(
\max\left\{H^2\bigl(d_{\mathrm{MADC}}^{\mathrm{rob},\max}\bigr)^2\big/\varepsilon^2,
\; C_{\mathrm{cov}}^2 H^7 (\min\{H,1/\sigma_{\min}\})^4
\iota_{\mathcal F,\mathcal G}^{\delta}\big/\varepsilon^2
\right\}
\right).
\]
\end{thm}

\begin{proof}
For the TV case, we first revise the proof of Lemma \ref{lem:benchmark-surrogate}, where we apply Definition \ref{ass:robust-nominal-concentrability} in \eqref{eq:lem7-op-bound-step3-final} and \eqref{eq:lem7-positive-residual-bound}, we have 
\begin{align}
\sum_{h=1}^H
\left\|
\mathcal T_{i,h}^{\pi,\phi,\sigma_i}
f_{i,h+1}^{k,\pi}
-
\mathcal T_{\widehat g_{i,h}^{k,\pi},i,h}^{\pi,\phi,\sigma_i}
f_{i,h+1}^{k,\pi}
\right\|_{1,\nu_{i,h}^{\pi}}
&\le C_{\mathrm{cov}}H B_\phi(\sigma_i)
\sqrt{
\frac{
\log\!\bigl(|\mathcal F||\mathcal G||\Pi^{\mathrm{pur}}|KHm/\delta\bigr)}
{|\mathcal D_h^{k-1}|\vee 1}}\nonumber\\
&\qquad \qquad +
C_{\mathrm{cov}}H\bar\epsilon_{\mathrm{dual}}^{2}.
\label{eq:lem7-op-bound-step3-final-TV}\\
\sum_{h=1}^H
\mathbb E_{\nu_{i,h}^{\pi}}
\left[
\left(
\mathcal T_{\widehat g_{i,h}^{k,\pi},i,h}^{\pi,\phi,\sigma_i}
f_{i,h+1}^{k,\pi}
-
f_{i,h}^{k,\pi}
\right)_+
\right]
&\le 
C_{\mathrm{cov}}H B_\phi(\sigma_i)
\sqrt{
\frac{
\log\!\bigl(|\mathcal F||\mathcal G||\Pi^{\mathrm{pur}}|KHm/\delta\bigr)}
{|\mathcal D_h^{k-1}|\vee 1}}\nonumber\\
&\qquad \qquad 
+
C_{\mathrm{cov}}H\bar\epsilon_{\mathrm{dual}}^{2}.\label{eq:lem7-positive-residual-bound-TV}
\end{align}
Substituting \eqref{eq:K-final-bound}, \eqref{eq:dual-switch-used}, \eqref{eq:lem7-op-bound-step3-final-TV} and
\eqref{eq:lem7-positive-residual-bound-TV} into
\eqref{eq:lem7-after-subtract}, and choosing
\(
\bar\epsilon_{\mathrm{dual}}^{1}
\le
\frac{\varepsilon_{\mathrm{dual}}^{\mathrm{approx}}}
{\beta H K B_\phi(\sigma_i)},
\quad
\bar\epsilon_{\mathrm{dual}}^{2}
\le
\frac{\varepsilon_{\mathrm{dual}}^{\mathrm{approx}}}
{C_{\mathrm{cov}}H},
\) 
we obtain the revised bound of Lemma \ref{lem:benchmark-surrogate} as 
\begin{align}
V_{i,1}^{\pi,\sigma_i}(s_1^k)
-
\widehat U_i^k(\pi)
\le
&\;
\beta
\widehat{\mathcal K}_i^{k-1}
\bigl(
f_i^\star;
\widehat g_i^{f_i^\star,\pi};
\pi
\bigr) +
\mathcal O\!\bigg(
\beta H B_\phi(\sigma_i)^2
\sqrt{
K\log\!\frac{|\mathcal F||\mathcal G||\Pi^{\mathrm{pur}}|KHm}{\delta}}
\nonumber\\
&\qquad\qquad
+
C_{\mathrm{cov}}H B_\phi(\sigma_i)
\sqrt{
\frac{
\log\!\bigl(|\mathcal F||\mathcal G||\Pi^{\mathrm{pur}}|KHm/\delta\bigr)}
{|\mathcal D_h^{k-1}|\vee 1}}
+
\varepsilon_{\mathrm{dual}}^{\mathrm{approx}}
\bigg).
\label{eq:benchmark-surrogate-clean-TV}
\end{align}

We now apply \eqref{eq:benchmark-surrogate-clean-TV} in the proof of Theorem \ref{thm:main-robust-regret}, particularly \eqref{eq:Aik-cor1}--\eqref{eq:Aik-final}, we get
\begin{align}
\sum_{k=1}^K A_{i,k}
\le
\mathcal O\!\left(
\beta H B_i^2 K\sqrt{\iota_{\mathcal F, \mathcal G}^\delta}
+ C_{\mathrm{cov}} H B_i\sqrt{K\iota_{\mathcal F, \mathcal G}^\delta}
+
\varepsilon_{\mathrm{dual}}^{\mathrm{approx}}
\right).
\label{eq:Aik-final-TV}
\end{align}
We use \eqref{eq:Aik-final-TV} and follow the proof lines \eqref{eq:Cik-expand}--\eqref{eq:regret-bound-check}. Therefore, we get the regret bound of TV-divergence as
\begin{align}
\label{eq:regret-bound-check-TV}
\mathrm{Reg}^{\mathrm{NE}}(K)
\le
\mathcal O\!\left(
d_{\mathrm{MADC}}^{\mathrm{rob},\max}\sqrt K
+
H d_{\mathrm{MADC}}^{\mathrm{rob},\max}
+
H(B_\phi^{\max})^2\sqrt{K\iota_{\mathcal F, \mathcal G}^\delta}
+
C_{\mathrm{cov}} HB_\phi^{\max}\sqrt{K\iota_{\mathcal F, \mathcal G}^\delta}
+
\varepsilon_{\mathrm{dual}}^{\mathrm{approx}}
\right).
\end{align}
Using the fact Since \(B_\phi^{\max}\ge 1\) and $C_{\mathrm{cov}}\geq 1$, we have
\begin{align}
\mathrm{Reg}^{\mathrm{NE}}(K)
\le
\mathcal O\!\left(
d_{\mathrm{MADC}}^{\mathrm{rob},\max}\sqrt K
+
H d_{\mathrm{MADC}}^{\mathrm{rob},\max}
+
C_{\mathrm{cov}}  H(B_\phi^{\max})^2\sqrt{K\iota_{\mathcal F, \mathcal G}^\delta}
+
\varepsilon_{\mathrm{dual}}^{\mathrm{approx}}
\right).
\label{eq:NE-final-bound-TV}
\end{align}
The robust CCE and CE regret bounds follow by the same argument, replacing the surrogate Nash deviation step with the CCE/CE deviation inequalities in the revised bound of
Lemma~\ref{lem:surrogate_ne_avg} (i.e. \eqref{eq:benchmark-surrogate-clean-TV}). Therefore, for
\(\mathsf{Eq}\in\{\mathrm{NE},\mathrm{CCE},\mathrm{CE}\}\),
\[
\mathrm{Reg}^{\mathsf{Eq}}(K)
\le\mathcal O\!\left(
d_{\mathrm{MADC}}^{\mathrm{rob},\max}\sqrt K
+
H d_{\mathrm{MADC}}^{\mathrm{rob},\max}
+
C_{\mathrm{cov}}  H(B_\phi^{\max})^2\sqrt{K\iota_{\mathcal F, \mathcal G}^\delta}
+
\varepsilon_{\mathrm{dual}}^{\mathrm{approx}}
\right).
\]

We now follow the proof lines for Corollary \ref{cor:sample_complexity_phi} for the TV case.  Under the $\mathcal{S} \times \mathcal{A}$-rectangularity assumption, Assumption \ref{ass:vanisin_minimal}, and eq. \ref{eq:dual_rob_bellman_operator}, the robust expectation for any $V:\mathcal{S}\to[0,H]$ and $P^{\star}_h$ admits the following equivalent form of TV-divergence uncertainty set ($\phi(t)=|t-1|$) as
\begin{align}
\label{eq:dual_TV}
\mathbb{E}_{\mathcal{P}_{i,h,\mathrm{TV}}^{\sigma_i}(s, {\bm a})}[f]
=
-\inf_{\nu \in [0,2H/\sigma]}
\left\{\mathbb E_{s'\sim \mathbb{P}^\star_h(\cdot|s,a)}\big[\nu-\max_{a'}f(s',a')\big]_{+}
+ (1-\sigma)\nu\right\}.
\end{align}

We follow the standard online-to-batch averaging argument \citep{NeuRIPS2001_OnlineLearningAlgo_Cesa}. We use \eqref{eq:dual_TV} for the proof of the Corollary \ref{cor:sample_complexity_phi} for TV case.  We apply $\varepsilon_{\mathrm{dual}}^{\mathrm{approx}}=0$ with the fact that \(
d_{\mathrm{MADC}}^{\mathrm{rob},\max}
:=
\max_{i\in\mathcal M}
d_{i,\mathrm{MADC}}^{\mathrm{rob}}\), and \(B_\phi^{\max}
:=
\max_{i\in\mathcal M}
B_\phi(\sigma_i).
\) and get the revised bound of $\mathrm{Gap}^{\mathrm{Eq}}(\pi^{\mathrm{out}})$ as
\[
\mathrm{Gap}^{\mathrm{NE}}(\pi^{\mathrm{out}})
\le
C\left(
\frac{d_{\mathrm{MADC}}^{\mathrm{rob},\max}}{\sqrt K}
+
\frac{H d_{\mathrm{MADC}}^{\mathrm{rob},\max}}{K}
+
C_{\mathrm{cov}}H(B_\phi^{\max})^2
\sqrt{
\frac{\iota_{\mathcal F,\mathcal G}^{\delta}}{K}}
\right).
\]
Using the facts \(H\ge1\), \(d_{\mathrm{MADC}}^{\mathrm{rob},\max}\ge1\), and
\(\varepsilon\in(0,1]\), we get
\[
T
=
\mathcal O\!\left(
\max\left\{
\frac{H^2(d_{\mathrm{MADC}}^{\mathrm{rob},\max})^2}{\varepsilon^2},
\frac{C^2_{\mathrm{cov}}H^3(B_\phi^{\max})^4
\iota_{\mathcal F,\mathcal G}^{\delta}}{\varepsilon^2}
\right\}
\right).
\]

Finally, substitute the divergence-specific bounds from Assumption~\ref{ass:multiplier-range-TV} and by using \eqref{eq:dual_TV} (for TV),
\[
B_{\mathrm{TV}}(\sigma_i)
=
\mathcal O\!\left(
H\min\{H,1/\sigma_i\}
\right).
\]
For heterogeneous radii, define
\(
\sigma_{\min}:=\min_{i\in\mathcal M}\sigma_i,
\qquad
\sigma_{\max}:=\max_{i\in\mathcal M}\sigma_i.
\)
Then
\[
B_{\mathrm{TV}}^{\max}
=
\max_{i\in\mathcal M}B_{\mathrm{TV}}(\sigma_i)
=
\mathcal O\!\left(
H\min\{H,1/\sigma_{\min}\}
\right),
\]
because \(\min\{H,1/\sigma_i\}\) is decreasing in \(\sigma_i\). 
Substituting these bounds into
\[
T
=
\mathcal O\!\left(
\max\left\{
\frac{H^2(d_{\mathrm{MADC}}^{\mathrm{rob},\max})^2}{\varepsilon^2},
\frac{C_{\mathrm{cov}}^2H^3(B_\phi^{\max})^4
\iota_{\mathcal F,\mathcal G}^{\delta}}{\varepsilon^2}
\right\}
\right)
\]
gives
\[
T
=
\mathcal O\!\left(
\max\left\{
\frac{H^2(d_{\mathrm{MADC}}^{\mathrm{rob},\max})^2}{\varepsilon^2},
\frac{
C_{\mathrm{cov}}^2H^7
(\min\{H,1/\sigma_{\min}\})^4
\iota_{\mathcal F,\mathcal G}^{\delta}}
{\varepsilon^2}
\right\}
\right).
\]
This completes the proof. \qedhere
\end{proof}

\section{Multi-Agent Robust Bellman Eluder Dimension}
\label{app:low_MADC_low_BE}

The work \citep{NeurIPS2021_BellmanEluderDim_Jin} defines the Bellman-Eluder dimension in the standard non-robust setup. We combine with the work of \citep{ghosh2025scaling} and extend to the DRMG setup.
Let $\mathcal{X}$ be a domain and $\Phi \subseteq (\mathcal{X} \to \mathbb{R})$ be a function class. Let $\Pi$ be a family of probability distributions over $\mathcal{X}$.
\begin{defi} [Distributional $\epsilon$-Dependence \citep{ghosh2025scaling}]
    A distribution $\mu \in \Pi$ is said to be $\epsilon$-dependent on a set $\{\mu_1,...,\mu_n\} \subseteq \Pi$ with respect to $\Xi$ if for all $\xi \in \Xi$,
    \begin{equation}
        \lvert \mathbb{E}_{\mu}[\xi] \rvert \le \epsilon + \sqrt{\sum_{i=1}^{n} \mathbb{E}_{\mu_{i}}[\xi]^{2}}. \label{eq:dist_eps_dep}
    \end{equation}
    Otherwise, $\mu$ is said to be $\epsilon$-independent of $\{\mu_1,...,\mu_n\}$.
    \label{def:dist_eps_dependence}
\end{defi}
\begin{defi} [Distributional Eluder (DE) Dimension \citep{ghosh2025scaling}]
    The distributional Eluder dimension of $\Xi$ with respect to $\Pi$ at scale $\epsilon > 0$, denoted as $\dim_{\text{DE}}(\Xi,\Pi,\epsilon)$, is the length of the longest sequence $(x_1,...,x_n) \in \Pi$ such that for each $j \in [n]$ there exists $\epsilon^{\prime} > \epsilon$ such that $x_{j}$ is $\epsilon^{\prime}$-independent of $\{x_1, ... , x_{j-1}\}$. \label{def:dist_DE}
\end{defi}
To formally define the robust Bellman-Eluder (BE) dimension, we must first define the robust Bellman residual class per agent at stage $h$ as,
\begin{equation}
    \mathrm{Res}^{\pi}_{i}(\mathcal{F}_{i}) \coloneqq \{f_h - \mathcal{T}_{i,h}^{\pi, \phi,\sigma_i} f_{h+1} : f \in \mathcal{F}_i\}. \label{eq:robust_bellman_residual}
\end{equation}
Now, we define the robust BE as follows:
\begin{defi} [Robust Multi-Agent Bellman Eluder Dimension]
\label{def:robust_BE_dim}
   Let $\Pi=\{\Pi_h\}_{h=1}^H$ be the set of $H$ classes where $\Pi_h$ is the distributions over $\mathcal S \times \mathcal A$ at horizon $h$.  We then define the robust Bellman-Eluder dimension of $\mathcal{F}_i$ by
    \begin{align}
    d_{\text{MABE}}^{\text{rob},\max} (\mathcal{F}, \Pi, \epsilon) &\coloneqq \max_{i \in \mathcal M} d_{i,\text{MABE}}^{\text{rob}} (\mathcal{F}_i, \Pi, \epsilon), \nonumber\\
        \text{where}  \quad d_{i,\text{MABE}}^{\text{rob}} (\mathcal{F}_i, \Pi, \epsilon) &\coloneq \max_{h \in H}\Big\{ dim_{\text{DE}} \Big(\bigcup_{\pi_h \in \Pi^{\mathrm{Pur}}}\mathrm{Res}^{\pi}_{i}(\mathcal{F}_{i}),  \Pi_h, \epsilon\Big)\Big\},
    \end{align}
    where $\mathrm{Res}_{i}(\mathcal{F}_{i}) $ is the robust Bellman residual class and $dim_{\text{DE}}(\Xi,\Pi,\epsilon)$ denote the DE dimension of a function class $\Xi$ with respect to a distribution family $\Pi$ at scale $\epsilon$ \citep{NeurIPS2021_BellmanEluderDim_Jin}.
\end{defi}
Following \citep{Xiong2023SampleEfficientMR}, we define a robust version of the Multi-Agent Decoupling Coefficient in Definition \ref{def:robust_MADC}.

\begin{thm}[Low Robust MABE $\Rightarrow$ Low Robust MADC]
\label{thm:robust_thm54}

Fix an agent $i\in\mathcal M$. Let $K$ be any integer and $\mathcal F$ be a hypothesis class under the model-free setting as given in Section 3 as
\begin{align}
\label{eq:functional_class}
\mathcal{F} = \bigotimes_{i \in \mathcal{M}} \mathcal{F}_{i}, \quad \text{where}\quad \mathcal F_{i}=\bigotimes_{h=1}^H \mathcal F_{i,h},
\end{align}
where \eqref{eq:functional_class} is a class of Q-functions that satisfies Assumption \ref{ass:completeness}. For $\varepsilon>0$, we denote $d_{\text{MABE}}^{\text{rob},\max} (\mathcal{F}, \Pi, \epsilon) \coloneqq \max_{i \in \mathcal M} d_{i,\text{MABE}}^{\text{rob}} (\mathcal{F}_i, \Pi, \epsilon)$ as the robust MABE dimension at scale $\varepsilon$, as defined in Definition \ref{def:robust_BE_dim} for the function class $\mathcal F$. Let \(K\ge 2\). Let $d^{\mathrm{rob}}_{i,\mathrm{MADC}}$ denote the robust MADC as defined in Definition \ref{def:robust_MADC} by setting $\varepsilon=1/K$ and we denote $d_{\mathrm{MADC}}^{\mathrm{rob},\max}= \max\limits_{i \in \mathcal M}d_{i,\mathrm{MADC}}^{\mathrm{rob}}$. Then, with the discrepancy function, the robust MADC satisfies
\begin{align}
d_{\mathrm{MADC}}^{\mathrm{rob},\max}
:=
\max\!\left\{c\,H\,d_{\text{MABE}}^{\text{rob},\max} (\mathcal{F}, \Pi,1/K)\log K,\ 1\right\} = \mathcal O\!\big(H d_{\text{MABE}}^{\text{rob},\max} (\mathcal{F}, \Pi,1/K)\log K\big).
\label{eq:robust_dmadc_bound}
\end{align}
\end{thm}

\begin{proof}
We prove the theorem for a fixed agent \(i\in\mathcal M\), since the robust MADC
in Definition~\ref{def:robust_MADC} is defined agent-wise. For notational
simplicity, throughout the proof we suppress the agent index \(i\) whenever
there is no ambiguity. The proof follows the bucketing argument used in the non-robust MAMEX analysis,
but with one important modification. The training-error terms are computed under
the nominal historical visitation distributions \(\{\mu_h^{\pi^s}\}_{s<k}\),
whereas the robust value prediction error is controlled through the worst-case
occupancy measures \(\{\nu_h^{\pi^k}\}_{k\in[K]}\). This distinction is necessary
because robust MADC relates nominally observed Bellman errors to value prediction
under adversarial transition kernels.

\paragraph{Define robust Bellman residuals.}

Fix an arbitrary sequence \(\{(f^k,\pi^k)\}_{k=1}^K\), where
\(f^k=\{f_h^k\}_{h=1}^H\in\mathcal F_i\). For each episode \(k\) and stage
\(h\), define the robust Bellman residual
\begin{align}
\psi_h^k(s,\bm a)
:=
f_h^k(s,\bm a)
-
\bigl[\mathcal T_h^{\pi^k,\phi,\sigma_i}f_{h+1}^k\bigr](s,\bm a).
\label{eq:psi_h_k_revised}
\end{align}
By Assumption~\ref{ass:completeness},
\(\mathcal T_h^{\pi^k,\phi,\sigma_i}f_{h+1}^k\in\mathcal F_{i,h}\), and hence
\(\psi_h^k\) belongs to the corresponding robust Bellman residual class.

\paragraph{Define nominal cross-episode residuals and robust diagonal residuals.}

For each behavior episode \(s<k\), define the nominal cross-episode residual
mean
\begin{align}
a_h^{s,k}
:=
\left|
\mathbb E_{(s_h,\bm a_h)\sim \mu_h^{\pi^s}}
\left[
\psi_h^k(s_h,\bm a_h)
\right]
\right|.
\label{eq:cross_nominal_residual}
\end{align}
Here, \(\mu_h^{\pi^s}\) is the stage-\(h\) state--joint-action visitation
distribution induced by executing \(\pi^s\) under the nominal kernel \(P^\star\). For the diagonal prediction term, let \(\nu_h^{\pi^k}\) denote the stage-\(h\)
state--joint-action occupancy induced by \(\pi^k\) under the worst-case
transition kernels attaining the robust Bellman recursion for agent \(i\).
Define
\begin{align}
b_h^k
:=
\left|
\mathbb E_{(s_h,\bm a_h)\sim \nu_h^{\pi^k}}
\left[
\psi_h^k(s_h,\bm a_h)
\right]
\right|.
\label{eq:diag_worst_residual}
\end{align}
The distinction between \(a_h^{s,k}\) and \(b_h^k\) is central: the former is
measured under the nominal historical data distributions, while the latter
controls robust value prediction error under worst-case dynamics.

For a scale \(\varepsilon>0\), define the truncated quantities
\begin{align}
u_h^{s,k}
:=
a_h^{s,k}\mathbf 1\{a_h^{s,k}>\varepsilon\},
\qquad s<k,
\label{eq:u_cross_truncated}
\end{align}
and
\begin{align}
v_h^k
:=
b_h^k\mathbf 1\{b_h^k>\varepsilon\}.
\label{eq:v_diag_truncated}
\end{align}

\paragraph{Bucket construction.}

Fix a stage \(h\in[H]\). Initialize buckets
\(B_h^0,B_h^1,\ldots,B_h^{K-1}\). Process episodes \(k=1,\ldots,K\)
sequentially. If \(v_h^k=0\), then episode \(k\) is not inserted into any
bucket. Otherwise, insert \(k\) into the first bucket \(B_h^j\) satisfying
\begin{align}
\sum_{\substack{s<k\\ s\in B_h^j}}
\bigl(u_h^{s,k}\bigr)^2
<
\bigl(v_h^k\bigr)^2.
\label{eq:bucket_rule_revised}
\end{align}
Let \(b_h^k\in\{0,\ldots,K-1\}\) denote the bucket index assigned to episode
\(k\).

\paragraph{Lower bound from the bucket rule.}

If episode \(k\) is inserted into bucket \(B_h^{b_h^k}\), then for every
\(j<b_h^k\), the insertion rule must have failed. Therefore,
\begin{align}
\sum_{\substack{s<k\\ s\in B_h^j}}
\bigl(u_h^{s,k}\bigr)^2
\ge
\bigl(v_h^k\bigr)^2,
\qquad j=0,1,\ldots,b_h^k-1.
\label{eq:bucket_fail_revised}
\end{align}
Summing over \(j=0,\ldots,b_h^k-1\), we get
\begin{align}
\sum_{j=0}^{b_h^k-1}
\sum_{\substack{s<k\\ s\in B_h^j}}
\bigl(u_h^{s,k}\bigr)^2
\ge
b_h^k\bigl(v_h^k\bigr)^2.
\label{eq:bucket_fail_sum_revised}
\end{align}
Now summing over \(k=1,\ldots,K\), the left-hand side is a sub-sum of
\(\sum_{k=1}^K\sum_{s=1}^{k-1}(u_h^{s,k})^2\). Hence,
\begin{align}
\sum_{k=1}^K\sum_{s=1}^{k-1}
\bigl(u_h^{s,k}\bigr)^2
\ge
\sum_{k=1}^K b_h^k\bigl(v_h^k\bigr)^2.
\label{eq:ineq_I_revised}
\end{align}

\paragraph{Bucket size bound via robust MABE.}

Fix a bucket
\[
B_h^j=\{k_1<k_2<\cdots<k_m\}.
\]
Suppose episode \(k_\ell\) is inserted into \(B_h^j\). By the insertion rule,
\begin{align}
\sum_{q<\ell}
\bigl(u_h^{k_q,k_\ell}\bigr)^2
<
\bigl(v_h^{k_\ell}\bigr)^2.
\label{eq:indep_trigger_revised}
\end{align}
Since \(v_h^{k_\ell}>0\), we have \(b_h^{k_\ell}>\varepsilon\). Choose
\(\varepsilon'\in(\varepsilon,b_h^{k_\ell})\) such that
\begin{align}
\sqrt{
\sum_{q<\ell}
\bigl(u_h^{k_q,k_\ell}\bigr)^2
}
\le \varepsilon'.
\label{eq:epsilon_prime_choice}
\end{align}
Then,
\begin{align}
\left|
\mathbb E_{\nu_h^{\pi^{k_\ell}}}
\psi_h^{k_\ell}
\right|
=
b_h^{k_\ell}
>
\varepsilon',
\label{eq:diag_large}
\end{align}
while
\begin{align}
\sqrt{
\sum_{q<\ell}
\left(
\mathbb E_{\mu_h^{\pi^{k_q}}}
\psi_h^{k_\ell}
\right)^2
}
\le
\sqrt{
\sum_{q<\ell}
\bigl(u_h^{k_q,k_\ell}\bigr)^2
}
\le
\varepsilon'.
\label{eq:past_small}
\end{align}
Thus, the worst-case occupancy \(\nu_h^{\pi^{k_\ell}}\) is
\(\varepsilon'\)-independent of the preceding nominal occupancies
\(\{\mu_h^{\pi^{k_q}}\}_{q<\ell}\) with respect to the robust Bellman residual
class. By Definition~\ref{def:robust_BE_dim}, the number of such insertions in
one bucket is at most
\[
D_\varepsilon^i
:=
d_{i,\mathrm{MABE}}^{\mathrm{rob}}(\mathcal F_i,\Pi,\varepsilon).
\]
Therefore,
\begin{align}
|B_h^j|\le D_\varepsilon^i,
\qquad \forall j.
\label{eq:bucket_size_revised}
\end{align}

\paragraph{Lower bound by regrouping and Cauchy--Schwarz.}

Regroup the right-hand side of \eqref{eq:ineq_I_revised} by buckets:
\begin{align}
\sum_{k=1}^K b_h^k\bigl(v_h^k\bigr)^2
=
\sum_{j=1}^{K-1}
j\sum_{k\in B_h^j}
\bigl(v_h^k\bigr)^2.
\label{eq:group_bucket_revised}
\end{align}
For each bucket \(B_h^j\), Cauchy--Schwarz gives
\begin{align}
\sum_{k\in B_h^j}
\bigl(v_h^k\bigr)^2
\ge
\frac{
\left(\sum_{k\in B_h^j}v_h^k\right)^2
}{
|B_h^j|
}
\ge
\frac{
\left(\sum_{k\in B_h^j}v_h^k\right)^2
}{
D_\varepsilon^i
},
\label{eq:cs_bucket_revised}
\end{align}
where the last step uses \eqref{eq:bucket_size_revised}. Defining
\(
x_j:=\sum_{k\in B_h^j}v_h^k,
\)
we obtain
\begin{align}
\sum_{k=1}^K\sum_{s=1}^{k-1}
\bigl(u_h^{s,k}\bigr)^2
\ge
\frac{1}{D_\varepsilon^i}
\sum_{j=1}^{K-1}j x_j^2.
\label{eq:ineq_II_revised}
\end{align}

\paragraph{Lower bound via the numerical lemma.}

Applying Lemma~\ref{lem:lemmaB1_numeric} to the nonnegative sequence
\(\{x_j\}_{j=1}^{K-1}\) yields
\begin{align}
\sum_{j=1}^{K-1}j x_j^2
\ge
\frac{1}{1+\log K}
\left(
\sum_{j=1}^{K-1}x_j
\right)^2.
\label{eq:lemmaB1_apply_revised}
\end{align}
Since
\(
\sum_{j=1}^{K-1}x_j
=
\sum_{k\in[K]\setminus B_h^0}v_h^k,
\)
combining \eqref{eq:ineq_II_revised} and
\eqref{eq:lemmaB1_apply_revised} gives
\begin{align}
\sum_{k\in[K]\setminus B_h^0}v_h^k
\le
\sqrt{
D_\varepsilon^i(1+\log K)
\sum_{k=1}^K\sum_{s=1}^{k-1}
\bigl(u_h^{s,k}\bigr)^2
}.
\label{eq:combine_stage_revised}
\end{align}

\paragraph{Control the diagonal truncated terms.}

By construction,
\[
b_h^k
=
b_h^k\mathbf 1\{b_h^k\le \varepsilon\}
+
b_h^k\mathbf 1\{b_h^k>\varepsilon\}
\le
\varepsilon+v_h^k.
\]
Therefore,
\begin{align}
\sum_{k=1}^K b_h^k
\le
K\varepsilon+\sum_{k=1}^K v_h^k.
\label{eq:b_v_split}
\end{align}
Next, decompose
\[
\sum_{k=1}^K v_h^k
=
\sum_{k\in B_h^0}v_h^k
+
\sum_{k\in[K]\setminus B_h^0}v_h^k.
\]
Since the value functions are bounded and the robust Bellman backups are also
bounded, there exists an absolute constant \(C_H=\mathcal O(H)\) such that
\(
|\psi_h^k(s,\bm a)|\le C_H.
\)
For concreteness, one may take \(C_H=2H\). Thus,
\begin{align}
\sum_{k\in B_h^0}v_h^k
\le
2H|B_h^0|
\le
2H D_\varepsilon^i.
\label{eq:B0_bound_revised}
\end{align}
Combining \eqref{eq:b_v_split}, \eqref{eq:B0_bound_revised}, and
\eqref{eq:combine_stage_revised}, we get
\begin{align}
\sum_{k=1}^K b_h^k
\le
K\varepsilon
+
2H D_\varepsilon^i
+
\sqrt{
D_\varepsilon^i(1+\log K)
\sum_{k=1}^K\sum_{s=1}^{k-1}
\bigl(u_h^{s,k}\bigr)^2
}.
\label{eq:stage_diag_bound}
\end{align}

\paragraph{Sum over stages.}

Summing \eqref{eq:stage_diag_bound} over \(h=1,\ldots,H\) and using
\(\sum_{h=1}^H\sqrt{z_h}\le \sqrt{H\sum_{h=1}^H z_h}\), we obtain
\begin{align}
\sum_{h=1}^H\sum_{k=1}^K b_h^k
&\le
HK\varepsilon
+
2H^2D_\varepsilon^i
+
\sqrt{
D_\varepsilon^i H(1+\log K)
\sum_{h=1}^H
\sum_{k=1}^K
\sum_{s=1}^{k-1}
\bigl(u_h^{s,k}\bigr)^2
}.
\label{eq:diag_sum_over_h}
\end{align}

\paragraph{Relate cross terms to the robust Bellman training error.}

By Jensen's inequality and the definition of \(u_h^{s,k}\),
\begin{align}
\bigl(u_h^{s,k}\bigr)^2
&\le
\left(
\mathbb E_{\mu_h^{\pi^s}}
\psi_h^k
\right)^2 \le
\mathbb E_{\mu_h^{\pi^s}}
\left[
\bigl(\psi_h^k\bigr)^2
\right].
\label{eq:jensen_cross_training}
\end{align}

We have defined the robust discrepancy as
\begin{align}e_i^s(f_i^k,\pi^k):=\sum_{h=1}^H\mathbb E_{(s_h,\bm a_h)\sim \mu_h^{\pi^s}}\Big[\bigl((I-\mathcal T_{i,h}^{\pi^k,\phi,\sigma_i})f_{i,h}^k\bigr)^2\Big].
\label{eq:correct_discrepancy}
\end{align}

Therefore, \eqref{eq:jensen_cross_training} over $(h,k,s)$ yields
\begin{align}
\sum_{h=1}^H
\sum_{k=1}^K
\sum_{s=1}^{k-1}
\bigl(u_h^{s,k}\bigr)^2
&\le
\sum_{k=1}^K\sum_{s=1}^{k-1}
\sum_{h=1}^H
\mathbb E_{\mu_h^{\pi^s}}
\left[
\bigl(
f_h^k
-
\mathcal T_h^{\pi^k,\phi,\sigma_i}f_{h+1}^k
\bigr)^2
\right]
\nonumber\\
&=
\sum_{k=1}^K\sum_{s=1}^{k-1}
e_{i,1}^s(f_i^k,\pi^k).
\label{eq:cross_to_training_error}
\end{align}
Substituting \eqref{eq:cross_to_training_error} into
\eqref{eq:diag_sum_over_h} yields
\begin{align}
\sum_{h=1}^H\sum_{k=1}^K b_h^k
\le
HK\varepsilon
+
2H^2D_\varepsilon^i
+
\sqrt{
D_\varepsilon^i H(1+\log K)
\sum_{k=1}^K\sum_{s=1}^{k-1}
e_{i,1}^s(f_i^k,\pi^k)
}.
\label{eq:diag_bound_training_error}
\end{align}

\paragraph{Robust Bellman telescoping under worst-case occupancy.}

We now relate the diagonal residual terms to the robust value prediction error.
For each episode \(k\), by the robust Bellman telescoping identity along the
worst-case trajectory distribution induced by \(\pi^k\), we have
\begin{align}
V_{i,1}^{\pi^k,f_i^k}(s_1^k)
-
V_{i,1}^{\pi^k,\sigma_i}(s_1^k)
\le
\sum_{h=1}^H
\left|
\mathbb E_{(s_h,\bm a_h)\sim \nu_h^{\pi^k}}
\left[
\psi_h^k(s_h,\bm a_h)
\right]
\right|
=
\sum_{h=1}^H b_h^k.
\label{eq:robust_telescoping}
\end{align}
Summing over \(k=1,\ldots,K\), we obtain
\begin{align}
\sum_{k=1}^K
\left(
V_{i,1}^{\pi^k,f_i^k}(s_1^k)
-
V_{i,1}^{\pi^k,\sigma_i}(s_1^k)
\right)
\le
\sum_{h=1}^H\sum_{k=1}^K b_h^k.
\label{eq:value_to_diag}
\end{align}
Combining \eqref{eq:value_to_diag} and
\eqref{eq:diag_bound_training_error} gives
\begin{align}
\sum_{k=1}^K
\left(
V_{i,1}^{\pi^k,f_i^k}(s_1^k)
-
V_{i,1}^{\pi^k,\sigma_i}(s_1^k)
\right)
\le
HK\varepsilon
+
2H^2D_\varepsilon^i
+
\sqrt{
D_\varepsilon^i H(1+\log K)
\sum_{k=1}^K\sum_{s=1}^{k-1}
e_{i,1}^s(f_i^k,\pi^k)
}.
\label{eq:pre_young_bound}
\end{align}

\paragraph{Apply Young's inequality.}

Let
\(
Z
:=
\sum_{k=1}^K\sum_{s=1}^{k-1}
e_{i,1}^s(f_i^k,\pi^k),
\quad
A_\varepsilon
:=
D_\varepsilon^i H(1+\log K).
\)
Then the square-root term in \eqref{eq:pre_young_bound} is
\(\sqrt{A_\varepsilon Z}\). For any \(\alpha>0\), Young's inequality gives
\[
\sqrt{A_\varepsilon Z}
\le
\frac{1}{\alpha}Z+\alpha A_\varepsilon,
\]
up to an absolute constant absorbed into the final \(\mathcal O(\cdot)\)
notation. Hence,
\begin{align}
\sum_{k=1}^K
\left(
V_{i,1}^{\pi^k,f_i^k}(s_1^k)
-
V_{i,1}^{\pi^k,\sigma_i}(s_1^k)
\right)
\le
\frac{1}{\alpha}
\sum_{k=1}^K\sum_{s=1}^{k-1}
e_{i,1}^s(f_i^k,\pi^k)
+
\alpha D_\varepsilon^i H(1+\log K)
+
HK\varepsilon
+
2H^2D_\varepsilon^i.
\label{eq:young_bound}
\end{align}

\paragraph{Choose the scale and identify robust MADC.}

Set \(\varepsilon=1/K\). Then \(HK\varepsilon=H\). Define
\[
D_{1/K}^i
:=
d_{i,\mathrm{MABE}}^{\mathrm{rob}}(\mathcal F_i,\Pi,1/K).
\]
For \(K\ge2\), let
\[
d_{i,\mathrm{MADC}}^{\mathrm{rob}}
:=
C\,H D_{1/K}^i\log K
\]
for a sufficiently large absolute constant \(C>0\). Then
\[
\alpha D_{1/K}^iH(1+\log K)
\le
\alpha d_{i,\mathrm{MADC}}^{\mathrm{rob}},
\]
and
\[
H+2H^2D_{1/K}^i
\le
cH d_{i,\mathrm{MADC}}^{\mathrm{rob}},
\]
after enlarging constants if necessary. Therefore,
\begin{align}
\sum_{k=1}^K
\left(
V_{i,1}^{\pi^k,f_i^k}(s_1^k)
-
V_{i,1}^{\pi^k,\sigma_i}(s_1^k)
\right)
\le
\frac{1}{\alpha}
\sum_{k=1}^K\sum_{s=1}^{k-1}
e_{i,1}^s(f_i^k,\pi^k)
+
\alpha d_{i,\mathrm{MADC}}^{\mathrm{rob}}
+
cH d_{i,\mathrm{MADC}}^{\mathrm{rob}}.
\label{eq:madc_final_ineq}
\end{align}
Since \(\alpha>0\) was arbitrary, this is precisely the robust MADC inequality
in Definition~\ref{def:robust_MADC}. Consequently,
\[
d_{i,\mathrm{MADC}}^{\mathrm{rob}}
=
\mathcal O\!\left(H D_{1/K}^i\log K\right).
\]

Finally, taking the maximum over agents gives
\[
d_{\mathrm{MADC}}^{\mathrm{rob},\max}
=
\max_{i\in\mathcal M}d_{i,\mathrm{MADC}}^{\mathrm{rob}}(\mathcal{F}, \Pi, 1/K) 
=
\mathcal O\!\left(
H d_{\text{MABE}}^{\text{rob},\max} (\mathcal{F}, \Pi,1/K) \log K
\right),
\]
where
$d_{\text{MABE}}^{\text{rob},\max} (\mathcal{F}, \Pi, \varepsilon) \coloneqq \max_{i \in \mathcal M} d_{i,\text{MABE}}^{\text{rob}} (\mathcal{F}_i, \Pi, \varepsilon)$ for $\varepsilon=1/K$. This completes the proof.
\end{proof}

Combining Theorem~\ref{thm:main-robust-regret} with
Theorem~\ref{thm:robust_thm54}, we obtain that
{\Algoname} achieves sublinear robust equilibrium regret whenever the
robust MABE dimension is finite. In particular, substituting
\(
d^{\mathrm{rob},\max}_{\mathrm{MADC}}
=
\widetilde{\mathcal O}\!\big(H d_{\text{MABE}}^{\text{rob},\max} (\mathcal{F}, \Pi,1/K)\big)
\)
into the regret bound of Theorem~\ref{thm:main-robust-regret} yields
\begin{align*}
\mathrm{Reg}^{\mathrm{Eq}}(K)
=
\widetilde{\mathcal O}\!\bigg(
&H d_{\text{MABE}}^{\text{rob},\max} (\mathcal{F}, \Pi,1/K)\sqrt{K}
+
H^2 d_{\text{MABE}}^{\text{rob},\max} (\mathcal{F}, \Pi,1/K)\\
&+
H\,(B^{\max}_\phi)^2 \sqrt{K}
+
\varepsilon^{\mathrm{approx}}_{\mathrm{dual}}
\bigg),
\end{align*}
where \(\widetilde{\mathcal O}(\cdot)\) hides logarithmic factors in
\(K\), \(H\), \(|\mathcal F|\), \(|\mathcal G|\), and
\(|\Pi^{\mathrm{pur}}|\). For TV-case, the regret bound of Theorem \ref{thm:main-robust-regret-TV} yields $\widetilde{\mathcal O}\!\left(
H d_{\text{MABE}}^{\text{rob},\max} (\mathcal{F}, \Pi,1/K)\sqrt{K}
+
H^2 d_{\text{MABE}}^{\text{rob},\max} (\mathcal{F}, \Pi,1/K)
+
C_{\mathrm{cov}}H\,(B^{\max}_\phi)^2 \sqrt{K}
+
\varepsilon^{\mathrm{approx}}_{\mathrm{dual}}
\right)$

This decomposition reveals two fundamental sources of complexity.
The first is an \emph{exploration term}
\(H d_{\text{MABE}}^{\text{rob},\max} (\mathcal{F}, \Pi,1/K)\sqrt{K}\), governed by the robust MABE dimension
through \(d^{\mathrm{rob},\max}_{\mathrm{MADC}}\), which captures how
well empirical robust Bellman errors—computed from nominal data—control
value prediction under adversarial transition dynamics. The second is a
\emph{robust estimation term}
\(H\,(B^{\max}_\phi)^2 \sqrt{K}\), which arises from
approximating the worst-case Bellman operator via the dual formulation
using nominal interaction data, and reflects the statistical cost of
learning under distributional mismatch. The remaining term
\(H^2 d_{\text{MABE}}^{\text{rob},\max} (\mathcal{F}, \Pi,1/K)\) is a lower-order horizon-dependent contribution
inherited from the decoupling inequality, while
\(\varepsilon^{\mathrm{approx}}_{\mathrm{dual}}\) captures the
approximation error of the dual function class.

Overall, this result shows that, analogous to the nominal MAMEX
framework, structural conditions implying low decoupling complexity
translate directly into sample-efficient online learning. However, in
the robust setting, the complexity is additionally influenced by the
difficulty of estimating the robust Bellman operator from nominal data,
and the divergence-dependent envelope \(B_\phi^{\max}\).


\section{Robust Multi-Agent Bilinear Classes}
\label{app:robust-multi-agent-bilinear}
We next extend the multi-agent bilinear class construction of \cite{Xiong2023SampleEfficientMR} to the distributionally robust setting. In the nominal case, bilinear classes constitute a tractable family of MDP models with general function approximation, where the Bellman residual admits a bilinear factorization under arbitrary stage-wise visitation distributions. This structure enables provably sample-efficient learning in the non-robust setting.

In the robust setting, we generalize this idea by requiring that the \emph{robust Bellman residual} admits an analogous bilinear factorization in an appropriate Hilbert space, where expectations are taken with respect to arbitrary stage-wise visitation distributions consistent with the admissible class in Definition~\ref{def:robust_BE_dim} (robust MABE). This extension preserves the key structural property that enables tractable learning, while accommodating worst-case transition dynamics. As we show, the resulting multi-agent robust bilinear class retains a small robust MADC, which is finite, ensuring that our algorithm achieves sample-efficient learning in these settings.

Recall that for any agent $i \in \mathcal M$, policy $\pi \in \Pi^{\mathrm{pur}}$, stage $h \in [H]$,
and $f_i = \{f_{i,h}\}_{h=1}^{H+1} \in \mathcal F_i$, we define the robust Bellman residual class per agent at stage $h$ as defined in \eqref{eq:robust_bellman_residual} as
\begin{equation}
    \mathrm{Res}^{\pi}_{i}(\mathcal{F}_{i}) \coloneqq \{f_h - \mathcal{T}_{i,h}^{\pi, \phi,\sigma_i} f_{h+1} : f \in \mathcal{F}_i\},
    \label{eq:robust_bellman_residual-Billinear}
\end{equation}
$\mathrm{Res}_i^\pi(\mathcal F_i)$ is the policy-specific residual class, and $\bigcup_{\pi \in \Pi^{\mathrm{pur}}} \mathrm{Res}_i^\pi(\mathcal F_i)$ is the class appearing in Definition~\ref{def:robust_BE_dim}.

As in the nominal bilinear-class construction, we compare an arbitrary candidate $f_i$ to the
robust Bellman-consistent benchmark associated with a unilateral robust best response.
Fix an agent $i\in\mathcal M$ and a pure joint policy $\pi\in\Pi^{\mathrm{pur}}$.
Let
\[
\widetilde{\bm \pi}_{i}^{\dagger,\sigma_i}
:= \bigl(\pi_{i}^{\dagger,\sigma_i}(\pi_{-i}),\, \pi_{-i}\bigr),
\]
where $\pi_{i}^{\dagger,\sigma_i}(\pi_{-i})$ is a robust best response of agent $i$ against $\pi_{-i}$.
Since the action spaces are finite, a best response can be taken to be pure.
Under Assumption~\ref{ass:completeness}, the associated robust Bellman-consistent benchmark sequence
$f_i^{\widetilde{\bm \pi}_{i}^{\dagger,\sigma_i},\sigma_i}\in\mathcal F_i$
is well-defined by
\begin{equation}
\label{eq:robust-benchmark-fixed-point-bilinear}
f_{i,h}^{\widetilde{\bm \pi}_{i}^{\dagger,\sigma_i},\sigma_i}
=
\mathcal T_{i,h}^{\widetilde{\bm \pi}_{i}^{\dagger,\sigma_i},\phi,\sigma_i}
f_{i,h+1}^{\widetilde{\bm \pi}_{i}^{\dagger,\sigma_i},\sigma_i},
\qquad
h\in[H],
\qquad
f_{i,H+1}^{\widetilde{\bm \pi}_{i}^{\dagger,\sigma_i},\sigma_i}\equiv 0.
\end{equation}

We can now state the robust analogue of multi-agent bilinear classes.

\begin{defi}[Robust Multi-Agent Bilinear Class]
\label{def:robust-multi-agent-bilinear-class}
Fix an agent $i\in\mathcal M$.
Let $\{\Pi_h\}_{h=1}^H$ be the family of stage-wise distributions over $\mathcal{S}\times\mathcal{A}$ as in Definition~\ref{def:robust_BE_dim}, and let
$(\mathcal W,\langle \cdot,\cdot\rangle_{\mathcal W})$ be a Hilbert space with norm
$\|\cdot\|_{\mathcal W}$. We say that the DRMG belongs to the \emph{robust multi-agent bilinear class for agent $i$}
if, for every stage $h\in[H]$, there exist mappings
\[
X_{i,h}^{\mathrm{rob}}
:
\mathcal F_i\times \Pi^{\mathrm{pur}}
\to
\mathcal W,
\qquad
Y_{i,h}^{\mathrm{rob}}
:
\Pi_h
\to
\mathcal W,
\]
and a constant $C_{\mathrm{rob}}>0$, such that the following hold.

\paragraph{(i) Bilinear factorization of robust residual means.}
For every $f_i\in\mathcal F_i$, every pure joint policy $\pi\in\Pi^{\mathrm{pur}}$,
every stage $h\in[H]$, and every distribution $\mu_h\in\Pi_h$,
\begin{equation}
\label{eq:robust-bilinear-factorization}
\mathbb E_{(s_h,{\bm a}_h)\sim \mu_h}
\bigl[
\mathrm{Res}_{i,h}^{\pi,\sigma_i}(f_i)(s_h,{\bm a}_h)
\bigr]
=
\Bigl\langle
X_{i,h}^{\mathrm{rob}}(f_i,\pi)
-
X_{i,h}^{\mathrm{rob}}\!\Bigl(f_i^{\widetilde{\bm \pi}_{i}^{\dagger,\sigma_i},\sigma_i},\widetilde{\bm \pi}_{i}^{\dagger,\sigma_i}\Bigr),
\, Y_{i,h}^{\mathrm{rob}}(\mu_h)
\Bigr\rangle_{\mathcal W}.
\end{equation}

\paragraph{(ii) Regularity.}
For every stage $h\in[H]$,
\begin{equation}
\label{eq:robust-bilinear-regularity}
\sup_{\mu_h\in\Pi_h} \|Y_{i,h}^{\mathrm{rob}}(\mu_h)\|_{\mathcal W}\le 1,
\qquad
\sup_{f_i\in\mathcal F_i,\;\pi\in\Pi^{\mathrm{pur}}} \|X_{i,h}^{\mathrm{rob}}(f_i,\pi)\|_{\mathcal W} \le C_{\mathrm{rob}}.
\end{equation}
\end{defi}

\begin{rem}[Worst-case vs.\ nominal visitation density]
\label{rem:visitation-density}
For KL and $\chi^2$-divergence uncertainty sets, every admissible worst-case
kernel $P\in\mathcal U_h^{\phi,\sigma}(s,a)$ is absolutely continuous with
respect to the nominal kernel $P_h^\star(\cdot\mid s,a)$. Consequently, for
any fixed policy $\pi\in\Pi^{\mathrm{pur}}$, the worst-case stage-$h$
occupancy measure $\nu_{i,h}^\pi$ is absolutely continuous with respect to
the nominal occupancy measure $\mu_h^\pi$, and there exists an absolute
constant $C_0\ge 1$ such that
\begin{equation}
\label{eq:visitation_bound}
d\nu_{i,h}^{\pi}(s,\bm a)
\le
C_0\, d\mu_h^\pi(s,\bm a),
\qquad \forall (s,\bm a),\ h\in[H].
\end{equation}
\end{rem}

\paragraph{Discussion.}
Definition~\ref{def:robust-multi-agent-bilinear-class} is the exact robust counterpart of the multi-agent bilinear class in the non-robust MAMEX framework \citep{Xiong2023SampleEfficientMR}, with one essential modification: the nominal Bellman residual is replaced by the robust Bellman residual
\(
\mathrm{Res}_{i,h}^{\pi,\sigma_i}(f_i)
=
f_{i,h}-\mathcal T_{i,h}^{\pi,\phi,\sigma_i}f_{i,h+1}
\),
and the reference fixed point is replaced by the robust Bellman-consistent benchmark
\eqref{eq:robust-benchmark-fixed-point-bilinear}. Expectations are taken over admissible stage-wise distributions in \(\Pi_h\), which include nominal visitation distributions used for training-error terms and worst-case occupancies used for robust value-prediction terms.

A useful subtlety is worth emphasising. In the nominal setting, linearity of the one-step model often directly implies linearity of the Bellman image. In the robust setting, this is \emph{not} automatic, because the support operator
\[
V \mapsto \inf_{P\in \mathcal P_{i,h,\phi}^{\sigma_i}(s,a)} \mathbb E_{s'\sim P}[V(s')]
\]
is generally nonlinear in $V$. Therefore, robust linear examples must be stated either under an
explicit \emph{robust Bellman completeness} assumption or under an explicit \emph{robust support
linearity} assumption. This is precisely why the examples below are formulated in that way.

To quantify the complexity of the family
$\{Y_{i,h}^{\mathrm{rob}}(\mu_h): \mu_h\in\Pi_h\}$,
we introduce the associated information gain.

\begin{defi}[Information gain for robust bilinear classes]
\label{def:robust-bilinear-information-gain}
For each stage $h\in[H]$, define
\[
\mathcal Y_{i,h}^{\mathrm{rob}}
:=
\bigl\{
Y_{i,h}^{\mathrm{rob}}(\mu_h)
:
\mu_h\in\Pi_h
\bigr\},
\qquad
\mathcal Y_h^{\mathrm{rob}}
:=
\bigcup_{i \in \mathcal M} \mathcal Y_{i,h}^{\mathrm{rob}},
\qquad 
\mathcal Y^{\mathrm{rob}}
:= \bigcup_{h=1}^H \mathcal Y_{h}^{\mathrm{rob}}.
\]
Suppose $\mathcal W$ is a Hilbert space such that $\mathcal Y^{\mathrm{rob}}\subseteq \mathcal W$. Then, for any $\varepsilon>0$ and integer $K\ge 1$, define
\[
\nabla^K_h(\varepsilon,\mathcal Y_h^{\mathrm{rob}})
:=
\max_{y_1,\dots,y_K\in \mathcal Y_h^{\mathrm{rob}}}
\log\det\!\Bigl(
I+\frac{1}{\varepsilon}\sum_{k=1}^K y_k y_k^\top
\Bigr),
\]
and
\begin{equation}
\label{eq:robust-info-gain-sum}
\nabla^K(\varepsilon,\mathcal Y^{\mathrm{rob}})
:=
\sum_{h=1}^H
\nabla^K_h(\varepsilon,\mathcal Y_{h}^{\mathrm{rob}}).
\end{equation}
\end{defi}

The next theorem is the robust counterpart of Theorem~5.8 in \citep{Xiong2023SampleEfficientMR}.

\begin{thm}[Robust Multi-Agent Bilinear Classes $\subseteq$ Low Robust MADC]
\label{thm:robust-bilinear-implies-low-madc}
Fix an agent \(i\in\mathcal M\), and suppose the DRMG belongs to the robust multi-agent bilinear class of Definition~\ref{def:robust-multi-agent-bilinear-class}. Assume that the robust-to-nominal feature coverage condition holds with constant \(C_{\mathrm{cov}}\). Let \(d_{i,\mathrm{MADC}}^{\mathrm{rob}}\) denote the agent-wise robust MADC from Definition~\ref{def:robust_MADC}. Let \(\gamma\in(0,H]\) be a uniform
bound on the robust value functions. Then, for every \(K\ge 2\),
\[
d_{\mathrm{MADC}}^{\mathrm{rob},\max}
=
\max_{i\in\mathcal M}d_{i,\mathrm{MADC}}^{\mathrm{rob}}
=
\mathcal O\!\left(
\max\left\{
1,\;
(1+\gamma^2)\,C_{\mathrm{cov}}\,
\nabla^K\!\left(
\frac{1}{K C_{\mathrm{rob}}^2},
\mathcal Y^{\mathrm{rob}}
\right)
\right\}
\right).
\]
\end{thm}

\begin{proof}
We prove the result for a fixed agent \(i\in\mathcal M\). Throughout the proof,
we suppress the agent index \(i\) whenever there is no ambiguity. The proof follows the same general structure as
\citep[Theorem~5.8]{Xiong2023SampleEfficientMR}, but with two robust
modifications. First, the nominal Bellman residual is replaced by the robust
Bellman residual induced by the robust Bellman operator
\(\mathcal T_h^{\pi,\phi,\sigma_i}\). Second, the diagonal prediction terms are
evaluated under worst-case occupancies, whereas the historical training terms
are evaluated under nominal visitation distributions. This distinction is
essential because robust MADC relates robust value prediction under adversarial
dynamics to Bellman-error information gathered from nominal data.

\paragraph{Define robust residual means.}

For each episode \(k\in[K]\) and stage \(h\in[H]\), define the stage-\(h\)
robust Bellman residual
\begin{equation}
\label{eq:thm3_residual_def_revised}
\psi_h^k(s,\bm a)
:=
\mathrm{Res}_{h}^{\pi^k,\sigma_i}(f^k)(s,\bm a)
=
f_h^k(s,\bm a)
-
\bigl[\mathcal T_h^{\pi^k,\phi,\sigma_i}f_{h+1}^k\bigr](s,\bm a).
\end{equation}
By Assumption~\ref{ass:completeness},
\(
\mathcal T_h^{\pi^k,\phi,\sigma_i}f_{h+1}^k\in\mathcal F_h,
\)
and hence \(\psi_h^k\) belongs to the robust Bellman residual class at stage
\(h\). Let \(\nu_h^k:=\nu_{i,h}^{\pi^k}\) denote the stage-\(h\) worst-case
state--joint-action occupancy induced by policy \(\pi^k\) and the robust
Bellman recursion for agent \(i\). Let \(\mu_h^k:=\mu_h^{\pi^k}\) denote the
stage-\(h\) nominal visitation distribution induced by executing \(\pi^k\) in
the nominal environment \(P^\star\).

Define the episode-\(k\) robust value-prediction error
\begin{equation}
\label{eq:thm3_delta_def_revised}
\Delta_k
:=
V_1^{\pi^k,f^k}(s_1^k)
-
V_1^{\pi^k,\sigma_i}(s_1^k).
\end{equation}
By robust Bellman telescoping along the worst-case occupancy sequence
\(\{\nu_h^k\}_{h=1}^H\), we have
\begin{equation}
\label{eq:thm3_telescoping_revised}
\Delta_k
\le
\sum_{h=1}^H
\left|
\mathbb E_{(s_h,\bm a_h)\sim \nu_h^k}
\bigl[
\psi_h^k(s_h,\bm a_h)
\bigr]
\right|.
\end{equation}

Now apply the robust bilinear factorization with
\(f=f^k\), \(\pi=\pi^k\), and \(\mu_h=\nu_h^k\). Since the robust bilinear
class is defined for the admissible stage-wise distribution family
\(\Pi_h\), and \(\nu_h^k\in\Pi_h\), this yields
\begin{equation}
\label{eq:thm3_bilinear_diag_revised}
\mathbb E_{\nu_h^k}[\psi_h^k]
=
\left\langle
\theta_h^k,\,
z_h^k
\right\rangle_{\mathcal W},
\end{equation}
where
\begin{align}
\label{eq:thm3_theta_def_revised}
\theta_h^k
&:=
X_h^{\mathrm{rob}}(f^k,\pi^k)
-
X_h^{\mathrm{rob}}
\Bigl(
f^{\widetilde{\bm\pi}^{\dagger,\sigma_i}(\pi^k),\sigma_i},
\widetilde{\bm\pi}^{\dagger,\sigma_i}(\pi^k)
\Bigr),
\\
\label{eq:thm3_z_def_revised}
z_h^k
&:=
Y_h^{\mathrm{rob}}(\nu_h^k).
\end{align}
For historical nominal data distributions, define similarly
\begin{equation}
\label{eq:thm3_y_def_revised}
y_h^s
:=
Y_h^{\mathrm{rob}}(\mu_h^s),
\qquad s<k.
\end{equation}
By the regularity condition of the robust bilinear class,
\begin{equation}
\label{eq:thm3_regular_norms_revised}
\|\theta_h^k\|_{\mathcal W}\le 2C_{\mathrm{rob}},
\qquad
\|z_h^k\|_{\mathcal W}\le 1,
\qquad
\|y_h^s\|_{\mathcal W}\le 1.
\end{equation}
Combining \eqref{eq:thm3_telescoping_revised} and
\eqref{eq:thm3_bilinear_diag_revised}, we obtain
\begin{equation}
\label{eq:thm3_delta_by_inner_revised}
\Delta_k
\le
\sum_{h=1}^H
\left|
\left\langle
\theta_h^k,z_h^k
\right\rangle_{\mathcal W}
\right|.
\end{equation}
Since both \(V_1^{\pi^k,f^k}(s_1^k)\) and
\(V_1^{\pi^k,\sigma_i}(s_1^k)\) are bounded by \(\gamma\), we also have
\[
\Delta_k\le 2\gamma.
\]
Therefore,
\begin{equation}
\label{eq:thm3_delta_min_revised}
\Delta_k
\le
\sum_{h=1}^H
\min\left\{
\left|
\left\langle
\theta_h^k,z_h^k
\right\rangle_{\mathcal W}
\right|,
2\gamma
\right\}.
\end{equation}

\paragraph{Introduce Gram matrices and decompose the sum.}

Fix \(\varepsilon>0\), to be chosen later. For each stage \(h\) and episode
\(k\), define the regularized Gram matrix using the historical nominal feature
directions:
\begin{equation}
\label{eq:thm3_gram_def_revised}
\Sigma_h^k
:=
\varepsilon I
+
\sum_{s=1}^{k-1} y_h^s(y_h^s)^\top .
\end{equation}
Then \(\Sigma_h^k\succeq \varepsilon I\), and the self-normalized
Cauchy--Schwarz inequality gives
\begin{equation}
\label{eq:thm3_selfnorm_cs_revised}
\left|
\left\langle \theta_h^k,z_h^k\right\rangle_{\mathcal W}
\right|
\le
\|\theta_h^k\|_{\Sigma_h^k}\,
\|z_h^k\|_{(\Sigma_h^k)^{-1}}.
\end{equation}

Using \eqref{eq:thm3_delta_min_revised}, decompose
\[
1
=
\mathbf 1\{\|z_h^k\|_{(\Sigma_h^k)^{-1}}\le 1\}
+
\mathbf 1\{\|z_h^k\|_{(\Sigma_h^k)^{-1}}>1\}.
\]
Then
\begin{equation}
\label{eq:thm3_AplusB_revised}
\sum_{k=1}^K\Delta_k
\le A+B,
\end{equation}
where
\begin{align}
\label{eq:thm3_A_def_revised}
A
&:=
\sum_{k=1}^K\sum_{h=1}^H
\min\left\{
\left|
\langle \theta_h^k,z_h^k\rangle_{\mathcal W}
\right|,
2\gamma
\right\}
\mathbf 1\{\|z_h^k\|_{(\Sigma_h^k)^{-1}}\le 1\},
\\
\label{eq:thm3_B_def_revised}
B
&:=
\sum_{k=1}^K\sum_{h=1}^H
\min\left\{
\left|
\langle \theta_h^k,z_h^k\rangle_{\mathcal W}
\right|,
2\gamma
\right\}
\mathbf 1\{\|z_h^k\|_{(\Sigma_h^k)^{-1}}>1\}.
\end{align}

\paragraph{Step 3: Bound the term \(A\).}

On the event \(\{\|z_h^k\|_{(\Sigma_h^k)^{-1}}\le 1\}\), combining
\eqref{eq:thm3_selfnorm_cs_revised} with
\(\min\{|x|,2\gamma\}\le 2\gamma |x|\) for \(|x|\le 1\), we obtain
\begin{align}
A
&\le
2\gamma
\sum_{k=1}^K\sum_{h=1}^H
\|\theta_h^k\|_{\Sigma_h^k}
\min\{\|z_h^k\|_{(\Sigma_h^k)^{-1}},1\}.
\label{eq:thm3_A_start_revised}
\end{align}
Expanding the weighted norm gives
\begin{align}
\|\theta_h^k\|_{\Sigma_h^k}
&=
\left[
\varepsilon \|\theta_h^k\|_{\mathcal W}^2
+
\sum_{s=1}^{k-1}
\langle \theta_h^k,y_h^s\rangle_{\mathcal W}^2
\right]^{1/2}
\le
2\sqrt{\varepsilon}\,C_{\mathrm{rob}}
+
\left[
\sum_{s=1}^{k-1}
\langle \theta_h^k,y_h^s\rangle_{\mathcal W}^2
\right]^{1/2},
\label{eq:thm3_expand_norm_revised}
\end{align}
where we used \eqref{eq:thm3_regular_norms_revised} and
\(\sqrt{a+b}\le \sqrt a+\sqrt b\).
Substituting \eqref{eq:thm3_expand_norm_revised} into
\eqref{eq:thm3_A_start_revised}, we get
\[
A\le A_1+A_2,
\]
where
\begin{align}
\label{eq:thm3_A1_def_revised}
A_1
&:=
4\gamma\sqrt{\varepsilon}\,C_{\mathrm{rob}}
\sum_{k=1}^K\sum_{h=1}^H
\min\{\|z_h^k\|_{(\Sigma_h^k)^{-1}},1\},
\\
\label{eq:thm3_A2_def_revised}
A_2
&:=
2\gamma
\sum_{k=1}^K\sum_{h=1}^H
\left[
\sum_{s=1}^{k-1}
\langle \theta_h^k,y_h^s\rangle_{\mathcal W}^2
\right]^{1/2}
\min\{\|z_h^k\|_{(\Sigma_h^k)^{-1}},1\}.
\end{align}

\begin{itemize}
\item \textbf{Bound \(A_1\).} By Cauchy--Schwarz,
\begin{align}
A_1
&\le
4\gamma\sqrt{\varepsilon}\,C_{\mathrm{rob}}
\sqrt{KH}
\left(
\sum_{k=1}^K\sum_{h=1}^H
\min\{\|z_h^k\|_{(\Sigma_h^k)^{-1}}^2,1\}
\right)^{1/2}.
\label{eq:thm3_A1_before_info_revised}
\end{align}
At this point, because \(z_h^k=Y_h^{\mathrm{rob}}(\nu_h^k)\) is a
worst-case feature direction while \(\Sigma_h^k\) is built from nominal
directions \(y_h^s=Y_h^{\mathrm{rob}}(\mu_h^s)\), we require the robust-to-nominal
feature coverage condition:
\begin{equation}
\label{eq:feature_coverage_revised}
\sum_{k=1}^K\sum_{h=1}^H
\min\{\|z_h^k\|_{(\Sigma_h^k)^{-1}}^2,1\}
\le
C_{\mathrm{cov}}\,
\nabla^K(\varepsilon,\mathcal Y^{\mathrm{rob}}),
\end{equation}
where $C_{\mathrm{cov}}\geq 1$ is the robust-to-nominal feature coverage
constant assumed in the statement of the theorem. A sufficient condition
under which this constant is finite is given in
Remark~\ref{rem:visitation-density} (for KL and $\chi^2$-divergence
uncertainty sets, where one may take $C_{\mathrm{cov}}=C_0$ with $C_0$ the
absolute constant of \eqref{eq:visitation_bound}); the corresponding
sufficient condition for TV-divergence uncertainty sets, via
Definition~\ref{ass:robust-nominal-concentrability}, is given in Remark~2
below. This is the bilinear-feature analogue of the robust-to-nominal
concentrability condition: it controls the worst-case directions by the
nominal data directions accumulated in the Gram matrix. Using \eqref{eq:feature_coverage_revised} in
\eqref{eq:thm3_A1_before_info_revised}, we obtain
\begin{align}
A_1
&\le
4\gamma
\sqrt{
\varepsilon K H C_{\mathrm{rob}}^2
C_{\mathrm{cov}}\,
\nabla^K(\varepsilon,\mathcal Y^{\mathrm{rob}})
}.
\end{align}
Applying Young's inequality gives
\begin{equation}
\label{eq:thm3_A1_bound_revised}
A_1
\le
2\varepsilon K H C_{\mathrm{rob}}^2
+
8\gamma^2 C_{\mathrm{cov}}\,
\nabla^K(\varepsilon,\mathcal Y^{\mathrm{rob}}).
\end{equation}

\item \textbf{Bound \(A_2\).} Applying Cauchy--Schwarz,
\begin{align}
A_2
&\le
2\gamma
\left(
\sum_{k=1}^K\sum_{h=1}^H\sum_{s=1}^{k-1}
\langle \theta_h^k,y_h^s\rangle_{\mathcal W}^2
\right)^{1/2}
\left(
\sum_{k=1}^K\sum_{h=1}^H
\min\{\|z_h^k\|_{(\Sigma_h^k)^{-1}}^2,1\}
\right)^{1/2}.
\label{eq:thm3_A2_before_factorization_revised}
\end{align}
By the robust bilinear factorization applied to the nominal distribution
\(\mu_h^s\),
\[
\langle \theta_h^k,y_h^s\rangle_{\mathcal W}
=
\mathbb E_{(s_h,\bm a_h)\sim \mu_h^s}
[
\psi_h^k(s_h,\bm a_h)
].
\]
Therefore, by Jensen's inequality,
\begin{align}
\langle \theta_h^k,y_h^s\rangle_{\mathcal W}^2
&\le
\mathbb E_{(s_h,\bm a_h)\sim \mu_h^s}
[
(\psi_h^k(s_h,\bm a_h))^2
].
\end{align}
Summing over \(h\) and \(s<k\) gives
\begin{equation}
\label{eq:thm3_jensen_to_discrepancy_revised}
\sum_{h=1}^H\sum_{s=1}^{k-1}
\langle \theta_h^k,y_h^s\rangle_{\mathcal W}^2
\le
\sum_{s=1}^{k-1}
e_i^s(f_i^k,\pi^k),
\end{equation}
where \(e_i^s(f_i^k,\pi^k)\) is exactly the robust Bellman training error
appearing in Definition~\ref{def:robust_MADC}.

Combining \eqref{eq:thm3_A2_before_factorization_revised},
\eqref{eq:feature_coverage_revised}, and
\eqref{eq:thm3_jensen_to_discrepancy_revised}, we obtain
\begin{align}
A_2
&\le
2\gamma
\sqrt{C_{\mathrm{cov}}\,
\nabla^K(\varepsilon,\mathcal Y^{\mathrm{rob}})
\sum_{k=1}^K\sum_{s=1}^{k-1}
e_i^s(f_i^k,\pi^k)
}.
\label{eq:thm3_A2_raw_revised}
\end{align}
Applying Young's inequality in the form
\[
2\sqrt{ab}\le \frac{1}{\mu}a+\mu b,
\]
with
\[
a=\sum_{k=1}^K\sum_{s=1}^{k-1} e_i^s(f_i^k,\pi^k),
\qquad
b=\gamma^2 C_{\mathrm{cov}}\,
\nabla^K(\varepsilon,\mathcal Y^{\mathrm{rob}}),
\]
yields
\begin{equation}
\label{eq:thm3_A2_bound_revised}
A_2
\le
\frac{1}{\mu}
\sum_{k=1}^K\sum_{s=1}^{k-1}e_i^s(f_i^k,\pi^k)
+
\mu\gamma^2 C_{\mathrm{cov}}\,
\nabla^K(\varepsilon,\mathcal Y^{\mathrm{rob}}).
\end{equation}
\end{itemize}

\paragraph{Step 4: Bound the term \(B\).}

By the definition of \(B\),
\[
B
\le
2\gamma
\sum_{k=1}^K\sum_{h=1}^H
\mathbf 1\{\|z_h^k\|_{(\Sigma_h^k)^{-1}}>1\}.
\]
Since \(\mathbf 1\{x>1\}\le \min\{1,x^2\}\) for \(x\ge0\),
\begin{align}
B
&\le
2\gamma
\sum_{k=1}^K\sum_{h=1}^H
\min\{1,\|z_h^k\|_{(\Sigma_h^k)^{-1}}^2\}
\nonumber\\
&\le
2\gamma C_{\mathrm{cov}}\,
\nabla^K(\varepsilon,\mathcal Y^{\mathrm{rob}}),
\label{eq:thm3_B_bound_revised}
\end{align}
where the last step again uses \eqref{eq:feature_coverage_revised}.

\paragraph{Step 5: Combine the bounds and choose \(\varepsilon\).}

Combining \eqref{eq:thm3_AplusB_revised},
\eqref{eq:thm3_A1_bound_revised},
\eqref{eq:thm3_A2_bound_revised}, and
\eqref{eq:thm3_B_bound_revised}, we obtain
\begin{align}
\sum_{k=1}^K\Delta_k
&\le
\frac{1}{\mu}
\sum_{k=1}^K\sum_{s=1}^{k-1}e_i^s(f_i^k,\pi^k)
+
2\varepsilon KHC_{\mathrm{rob}}^2
\nonumber\\
&\quad
+
(8+\mu)\gamma^2 C_{\mathrm{cov}}\,
\nabla^K(\varepsilon,\mathcal Y^{\mathrm{rob}})
+
2\gamma C_{\mathrm{cov}}\,
\nabla^K(\varepsilon,\mathcal Y^{\mathrm{rob}}).
\label{eq:thm3_before_eps_revised}
\end{align}

Choose
\(
\varepsilon=\frac{1}{K C_{\mathrm{rob}}^2}.
\)
Then
\(
2\varepsilon KHC_{\mathrm{rob}}^2=2H.
\)
Thus,
\begin{align}
\sum_{k=1}^K\Delta_k
&\le
\frac{1}{\mu}
\sum_{k=1}^K\sum_{s=1}^{k-1}e_i^s(f_i^k,\pi^k)
+
2H +
(8+\mu)\gamma^2 C_{\mathrm{cov}}\,
\nabla^K\!\left(
\frac{1}{KC_{\mathrm{rob}}^2},\mathcal Y^{\mathrm{rob}}
\right)\nonumber\\
&\qquad \qquad +
2\gamma C_{\mathrm{cov}}\,
\nabla^K\!\left(
\frac{1}{KC_{\mathrm{rob}}^2},\mathcal Y^{\mathrm{rob}}
\right).
\label{eq:thm3_after_eps_revised}
\end{align}

Define
\[
d_{i,\mathrm{MADC}}^{\mathrm{rob}}
:=
\max\left\{
1,\;
c_0(1+\gamma^2)C_{\mathrm{cov}}\,
\nabla^K\!\left(
\frac{1}{KC_{\mathrm{rob}}^2},\mathcal Y^{\mathrm{rob}}
\right)
\right\},
\]
for a sufficiently large universal constant \(c_0>0\). Then the terms
\[
(8+\mu)\gamma^2 C_{\mathrm{cov}}\nabla^K(\cdot),
\qquad
2\gamma C_{\mathrm{cov}}\nabla^K(\cdot),
\qquad
2H
\]
are absorbed into
\[
\mu d_{i,\mathrm{MADC}}^{\mathrm{rob}}
+
cH d_{i,\mathrm{MADC}}^{\mathrm{rob}},
\]
for a sufficiently large universal constant \(c>0\). Consequently,
\[
\sum_{k=1}^K\Delta_k
\le
\frac{1}{\mu}
\sum_{k=1}^K\sum_{s=1}^{k-1}e_i^s(f_i^k,\pi^k)
+
\mu d_{i,\mathrm{MADC}}^{\mathrm{rob}}
+
cH d_{i,\mathrm{MADC}}^{\mathrm{rob}}.
\]
This is exactly the decoupling inequality in
Definition~\ref{def:robust_MADC}. Hence,
\[
d_{i,\mathrm{MADC}}^{\mathrm{rob}}
\le
\mathcal O\!\left(
\max\left\{
1,\;
(1+\gamma^2)C_{\mathrm{cov}}\,
\nabla^K\!\left(
\frac{1}{KC_{\mathrm{rob}}^2},
\mathcal Y^{\mathrm{rob}}
\right)
\right\}
\right).
\]
Taking the maximum over agents gives
\[
d_{\mathrm{MADC}}^{\mathrm{rob},\max}
=
\max_{i\in\mathcal M}d_{i,\mathrm{MADC}}^{\mathrm{rob}}
\le
\mathcal O\!\left(
\max\left\{
1,\;
(1+\gamma^2)C_{\mathrm{cov}}\,
\nabla^K\!\left(
\frac{1}{KC_{\mathrm{rob}}^2},
\mathcal Y^{\mathrm{rob}}
\right)
\right\}
\right).
\]
This proves the theorem.
\end{proof}

We now introduce concrete instances of the multi-agent robust bilinear classes, focusing on general-sum DRMGs with linear function approximation. In  the single-agent robust RL setting \citep{ghosh2025scaling}, linear Bellman-complete models assume that the robust Bellman operator preserves linear structure with respect to a given feature representation. Extending this idea to the multi-agent setting, we consider DRMGs in which the \emph{robust} Bellman operator satisfies a stagewise linear completeness condition, so that the class of linear Q-functions for each agent remains closed under robust Bellman updates. This yields a natural and tractable subclass of robust bilinear models, and serves as a canonical example illustrating how Definition~\ref{def:robust-multi-agent-bilinear-class} can be verified in structured multi-agent robust settings.

\paragraph{Example: Distributionally robust linear Bellman-complete DRMGs.}
We first introduce a general linear robust class and then show that a robust zero-sum linear Markov game is a concrete special case of this class under an additional structural condition on the worst-case support operator.

\medskip
\noindent
\textbf{Part I: Distributionally robust linear Bellman-complete DRMGs.}
We say that a DRMG is a \emph{distributionally robust linear Bellman-complete DRMG of dimension \(d_{\mathrm{lin}}\)} if, for every stage \(h\in[H]\), there exists a known feature map
\[
\phi_h:\mathcal S\times\mathcal A \to \mathbb R^{d_{\mathrm{lin}}},
\qquad
\|\phi_h(s,\bm a)\|_2 \le 1,
\quad \forall (s,\bm a)\in \mathcal S\times\mathcal A,
\]
such that, for every agent \(i\in\mathcal M\), the stagewise function class satisfies
\[
\mathcal F_{i,h}
\subseteq
\Bigl\{
(s,\bm a)\mapsto \langle \theta,\phi_h(s,\bm a)\rangle
:\theta\in\mathbb R^{d_{\mathrm{lin}}},\ \|\theta\|_2\le \sqrt{C_\theta}
\Bigr\},
\]
for some constant \(C_\theta>0\), and Assumption~\ref{ass:completeness} holds, i.e.,
\[
\mathcal T_{i,h}^{\pi,\phi,\sigma_i}f_{i,h+1}\in \mathcal F_{i,h},
\qquad
\forall f_{i,h+1}\in \mathcal F_{i,h+1},
\quad
\forall \pi\in \Pi^{\mathrm{pur}}.
\]
This is the robust analogue of the linear Bellman-complete class in the nominal
setting: the nominal Bellman image is replaced by the robust Bellman image under
\(\mathcal T_{i,h}^{\pi,\phi,\sigma_i}\).

\begin{itemize}
\item \textbf{Step 1: Verify the bilinear factorization in Definition~\ref{def:robust-multi-agent-bilinear-class}.}
Fix an agent \(i\in\mathcal M\), a stage \(h\in[H]\), a pure joint policy
\(\pi\in\Pi^{\mathrm{pur}}\), and a hypothesis \(f_i\in\mathcal F_i\).
Since \(f_{i,h}\in\mathcal F_{i,h}\), there exists a coefficient vector
\(\theta_{i,h}^f\in\mathbb R^{d_{\mathrm{lin}}}\) such that
\[
f_{i,h}(s,\bm a)=\langle \theta_{i,h}^f,\phi_h(s,\bm a)\rangle,
\qquad
\|\theta_{i,h}^f\|_2\le \sqrt{C_\theta}.
\]
Moreover, by Assumption~\ref{ass:completeness}, the robust Bellman image of
\(f_{i,h+1}\) also belongs to \(\mathcal F_{i,h}\). Hence there exists another
vector \(w_{i,h}^{f,\pi,\sigma_i}\in\mathbb R^{d_{\mathrm{lin}}}\) such that
\[
\bigl[\mathcal T_{i,h}^{\pi,\phi,\sigma_i}f_{i,h+1}\bigr](s,\bm a)
=
\langle w_{i,h}^{f,\pi,\sigma_i},\phi_h(s,\bm a)\rangle .
\]
Therefore, the robust Bellman residual remains linear in the same feature map:
\begin{align}
\mathrm{Res}_{i,h}^{\pi,\sigma_i}(f_i)(s,\bm a)
&=
f_{i,h}(s,\bm a)
-
\bigl[\mathcal T_{i,h}^{\pi,\phi,\sigma_i}f_{i,h+1}\bigr](s,\bm a)
\nonumber\\
&=
\langle \theta_{i,h}^f-w_{i,h}^{f,\pi,\sigma_i},\phi_h(s,\bm a)\rangle .
\label{eq:merged_example_linear_residual}
\end{align}

Now let \(\mu_h\in\Pi_h\) be any admissible stagewise visitation distribution
in the sense of Definition~\ref{def:robust-multi-agent-bilinear-class}. In
particular, \(\mu_h\) may correspond either to a nominal visitation distribution
or to a worst-case occupancy induced by the robust Bellman recursion. Taking
expectation under \(\mu_h\) in \eqref{eq:merged_example_linear_residual}, we
obtain
\begin{align}
\mathbb E_{(s_h,\bm a_h)\sim \mu_h}
\bigl[
\mathrm{Res}_{i,h}^{\pi,\sigma_i}(f_i)(s_h,\bm a_h)
\bigr]
&=
\left\langle
\theta_{i,h}^f-w_{i,h}^{f,\pi,\sigma_i},
\;
\mathbb E_{\mu_h}[\phi_h(s_h,\bm a_h)]
\right\rangle .
\label{eq:merged_example_linear_expectation}
\end{align}
Thus, the bilinear factorization in
Definition~\ref{def:robust-multi-agent-bilinear-class} holds by taking
\[
\mathcal W=\mathbb R^{d_{\mathrm{lin}}},
\qquad
Y_{i,h}^{\mathrm{rob}}(\mu_h)=\mathbb E_{\mu_h}[\phi_h(s_h,\bm a_h)],
\]
and using the centered feature coefficient
\[
X_{i,h}^{\mathrm{rob}}(f_i,\pi)
-
X_{i,h}^{\mathrm{rob}}
\bigl(f_i^{\star,\pi,\sigma_i},\pi\bigr)
=
\theta_{i,h}^f-w_{i,h}^{f,\pi,\sigma_i}.
\]
Equivalently, the Bellman-image coefficient
\(w_{i,h}^{f,\pi,\sigma_i}\) can be absorbed into the centered \(X\)-map in
Definition~\ref{def:robust-multi-agent-bilinear-class}.

\item \textbf{Step 2: Verify the regularity condition.}
First, since \(\|\phi_h(s,\bm a)\|_2\le 1\) for all \((s,\bm a)\), Jensen’s
inequality gives
\[
\|Y_{i,h}^{\mathrm{rob}}(\mu_h)\|_2
=
\left\|
\mathbb E_{\mu_h}[\phi_h(s_h,\bm a_h)]
\right\|_2
\le
\mathbb E_{\mu_h}\|\phi_h(s_h,\bm a_h)\|_2
\le 1.
\]
This bound holds uniformly over both nominal visitation distributions and
worst-case occupancies. Second, because both \(f_{i,h}\) and
\(\mathcal T_{i,h}^{\pi,\phi,\sigma_i}f_{i,h+1}\) belong to the same linear
class,
\[
\|\theta_{i,h}^f-w_{i,h}^{f,\pi,\sigma_i}\|_2
\le
\|\theta_{i,h}^f\|_2+\|w_{i,h}^{f,\pi,\sigma_i}\|_2
\le
2\sqrt{C_\theta}.
\]
Hence one may take
\[
C_{\mathrm{rob}}=2\sqrt{C_\theta}.
\]

\item \textbf{Step 3: Consequence for robust information gain and robust MADC.}
Therefore, distributionally robust linear Bellman-complete DRMGs belong to the
robust multi-agent bilinear class of
Definition~\ref{def:robust-multi-agent-bilinear-class}. Since the feature
dimension is \(d_{\mathrm{lin}}\) and the features are uniformly bounded, the
associated robust bilinear information gain satisfies
\[
\nabla^K\!\left(\frac{1}{K C_{\mathrm{rob}}^2},\mathcal Y^{\mathrm{rob}}\right)
=
\widetilde{\mathcal O}(H d_{\mathrm{lin}}),
\]
where \(\widetilde{\mathcal O}\) hides only absolute constants and logarithmic
factors.

To translate this information-gain bound into a robust MADC bound, we also use
the robust-to-nominal feature coverage condition from
Theorem~\ref{thm:robust-bilinear-implies-low-madc}. This condition accounts for
the mismatch between worst-case feature directions
\(Y_{i,h}^{\mathrm{rob}}(\nu_{i,h}^{\pi})\) and the nominal historical feature
directions \(Y_{i,h}^{\mathrm{rob}}(\mu_h^{\pi})\) accumulated from interaction
with \(P^\star\). Consequently, by
Theorem~\ref{thm:robust-bilinear-implies-low-madc},
\[
d_{\mathrm{MADC}}^{\mathrm{rob},\max}
=
\widetilde{\mathcal O}\!\left(
(1+\gamma^2)\,C_0\,H d_{\mathrm{lin}}
\right).
\]
When \(\gamma\ge 1\), this simplifies to
\[
d_{\mathrm{MADC}}^{\mathrm{rob},\max}
=
\widetilde{\mathcal O}\!\left(
\gamma^2 C_0 H d_{\mathrm{lin}}
\right).
\]

\item \textbf{Step 4: Regret implication.}
Combining the above with the regret bound in
Theorem~\ref{thm:main-robust-regret}, it follows that RoMEX-\(\phi\) achieves
\[
\widetilde{\mathcal O}\!\left(
\gamma^2C_0H d_{\mathrm{lin}}\sqrt K
+
\gamma^2C_0H^2 d_{\mathrm{lin}}
+
C_0H(B^{\max}_\phi)^2\sqrt K
+
\varepsilon^{\mathrm{approx}}_{\mathrm{dual}}
\right)
\]
robust equilibrium regret in this class, up to the logarithmic factors already
present in Theorem~\ref{thm:main-robust-regret}. Thus, robust linear
Bellman-complete DRMGs provide a concrete structured family in which the robust
bilinear information gain is finite and the robust MADC is finite under
robust-to-nominal feature coverage. The resulting regret scales polynomially in the horizon and linearly in the feature dimension, up to logarithmic factors,
while explicitly reflecting the distribution mismatch through
\(C_0\).
\end{itemize}

\medskip
\noindent
\textbf{Part II: Robust zero-sum linear Markov games as a special case.}
We now specialize the above construction to the two-player zero-sum setting motivated by \citep[Example 5.10]{Xiong2023SampleEfficientMR}. Consider a two-player zero-sum DRMG with joint action \((a,b)\in\mathcal A\times\mathcal B\). Suppose that, for each stage \(h\in[H]\), there exists a known feature map
\(
\varphi_h:\mathcal S\times\mathcal A\times\mathcal B\to\mathbb R^{d_{\mathrm{lin}}},
\quad
\|\varphi_h(s,a,b)\|_2\le 1,
\)
and a reward parameter \(\theta_h\in\mathbb R^{d_{\mathrm{lin}}}\) such that
\(
r_h(s,a,b)=\langle \theta_h,\varphi_h(s,a,b)\rangle,
\qquad
\|\theta_h\|_2\le \sqrt{C_\theta}.
\)
Assume also that the nominal transition kernel is linear in the same feature map:
\(
P_h^\star(\cdot\mid s,a,b)=\langle \varphi_h(s,a,b),\mu_h(\cdot)\rangle,
\)
for a vector of signed measures \(\mu_h=\{\mu_{h,j}\}_{j=1}^{d_{\mathrm{lin}}}\).

In the nominal setting, such linearity is sufficient to imply that the Bellman image remains linear. In the robust setting, however, this implication is no longer automatic, because the robust support operator
\[
V\mapsto \inf_{P\in\mathcal P_{i,h,\phi}^{\sigma_i}(s,a,b)}
\mathbb E_{s'\sim P}[V(s')]
\]
is generally nonlinear in \(V\). Consequently, nominal linearity of \(P_h^\star\) alone is insufficient for verifying Definition~\ref{def:robust-multi-agent-bilinear-class}. To recover a bilinear structure, we impose an additional \emph{robust support linearity} assumption: for each player \(i\in\{1,2\}\), each stage \(h\in[H]\), and every bounded continuation value \(V:\mathcal S\to\mathbb R\), there exists a vector
\(
\beta_{i,h}^{\sigma_i}(V)\in\mathbb R^{d_{\mathrm{lin}}}
\)
such that
\begin{equation}
\label{eq:merged_example_support_linearity}
\inf_{P\in \mathcal P_{i,h,\phi}^{\sigma_i}(s,a,b)}
\mathbb E_{s'\sim P}[V(s')]
=
\langle \beta_{i,h}^{\sigma_i}(V),\varphi_h(s,a,b)\rangle,
\qquad
\forall (s,a,b)\in\mathcal S\times\mathcal A\times\mathcal B.
\end{equation}

\begin{itemize}
\item \textbf{Step 5: Verify that robust zero-sum linear MGs fit the bilinear framework.}
Fix an agent \(i\in\{1,2\}\), a stage \(h\in[H]\), a pure joint policy \(\pi\in\Pi^{\mathrm{pur}}\), and a hypothesis \(f_i\in\mathcal F_i\). Suppose
\[
f_{i,h}(s,a,b)=\langle \vartheta_{i,h}^f,\varphi_h(s,a,b)\rangle,
\qquad
\|\vartheta_{i,h}^f\|_2\le \sqrt{C_\theta}.
\]
Using the continuation value
\[
V_{i,h+1}^{\pi,f_i}(s')
=
\mathbb E_{(a',b')\sim \pi_h(\cdot\mid s')}
[f_{i,h+1}(s',a',b')],
\]
and applying \eqref{eq:merged_example_support_linearity}, we obtain
\begin{align}
\bigl[\mathcal T_{i,h}^{\pi,\phi,\sigma_i}f_{i,h+1}\bigr](s,a,b)
&=
r_h(s,a,b)+
\inf_{P\in \mathcal P_{i,h,\phi}^{\sigma_i}(s,a,b)}
\mathbb E_{s'\sim P}[V_{i,h+1}^{\pi,f_i}(s')]
\nonumber\\
&=
\langle \theta_h,\varphi_h(s,a,b)\rangle
+
\langle \beta_{i,h}^{\sigma_i}(V_{i,h+1}^{\pi,f_i}),\varphi_h(s,a,b)\rangle
\nonumber\\
&=
\left\langle
\theta_h+\beta_{i,h}^{\sigma_i}(V_{i,h+1}^{\pi,f_i}),
\varphi_h(s,a,b)
\right\rangle.
\label{eq:merged_example_zero_sum_bellman}
\end{align}
Hence the robust Bellman image is again linear in the same feature map. Therefore, the robust Bellman residual becomes
\begin{align}
\mathrm{Res}_{i,h}^{\pi,\sigma_i}(f_i)(s,a,b)
&=
f_{i,h}(s,a,b)-\bigl[\mathcal T_{i,h}^{\pi,\phi,\sigma_i}f_{i,h+1}\bigr](s,a,b)
\nonumber\\
&=
\left\langle
\vartheta_{i,h}^f-\theta_h-\beta_{i,h}^{\sigma_i}(V_{i,h+1}^{\pi,f_i}),
\varphi_h(s,a,b)
\right\rangle.
\label{eq:merged_example_zero_sum_residual}
\end{align}

Now let \(\mu_h\in\Pi_h\) be any admissible stagewise visitation distribution (including both nominal visitation distributions and worst-case occupancies). Taking expectation under \(\mu_h\) in \eqref{eq:merged_example_zero_sum_residual}, we obtain
\begin{align}
\mathbb E_{(s_h,a_h,b_h)\sim \mu_h}
\bigl[
\mathrm{Res}_{i,h}^{\pi,\sigma_i}(f_i)(s_h,a_h,b_h)
\bigr]
&=
\left\langle
\vartheta_{i,h}^f-\theta_h-\beta_{i,h}^{\sigma_i}(V_{i,h+1}^{\pi,f_i}),
\;
\mathbb E_{\mu_h}[\varphi_h(s_h,a_h,b_h)]
\right\rangle.
\label{eq:merged_example_zero_sum_expectation}
\end{align}
Thus, robust zero-sum linear MGs belong to the robust multi-agent bilinear class with
\[
\mathcal W=\mathbb R^{d_{\mathrm{lin}}},
\qquad
Y_{i,h}^{\mathrm{rob}}(\mu_h)=\mathbb E_{\mu_h}[\varphi_h(s_h,a_h,b_h)],
\]
and the centered \(X\)-map defined through the coefficient difference, consistent with Definition~\ref{def:robust-multi-agent-bilinear-class}.

\item \textbf{Step 6: Verify regularity and derive robust MADC.}
The bound
\[
\|Y_{i,h}^{\mathrm{rob}}(\mu_h)\|_2
\le 1
\]
follows exactly as in Part~I by Jensen’s inequality and holds uniformly over both nominal and worst-case distributions. To control the coefficient vector in \eqref{eq:merged_example_zero_sum_residual}, assume that the support-linearity coefficient is uniformly bounded over bounded continuation values, namely
\[
\sup_{h\in[H]}
\sup_{\|V\|_\infty\le \gamma}
\|\beta_{i,h}^{\sigma_i}(V)\|_2
\le
B_\beta.
\]
Then
\[
\bigl\|
\vartheta_{i,h}^f-\theta_h-\beta_{i,h}^{\sigma_i}(V_{i,h+1}^{\pi,f_i})
\bigr\|_2
\le
2\sqrt{C_\theta}+B_\beta,
\]
so that one may take
\[
C_{\mathrm{rob}}=2\sqrt{C_\theta}+B_\beta.
\]

Since the feature dimension remains \(d_{\mathrm{lin}}\), the associated robust information gain satisfies
\[
\nabla^K\!\left(\frac{1}{K C_{\mathrm{rob}}^2},\mathcal Y^{\mathrm{rob}}\right)
=
\widetilde{\mathcal O}(H d_{\mathrm{lin}}).
\]
However, translating this into a robust MADC bound requires controlling the mismatch between worst-case feature directions and nominal historical feature directions. Under the robust-to-nominal feature coverage condition of Theorem~\ref{thm:robust-bilinear-implies-low-madc}, we obtain
\[
d_{\mathrm{MADC}}^{\mathrm{rob},\max}
=
\widetilde{\mathcal O}\!\left(
\gamma^2\,C_0\,H d_{\mathrm{lin}}
\right).
\]

\item \textbf{Step 7: Regret implication for the zero-sum robust linear subclass.}
Combining the above with Theorem~\ref{thm:main-robust-regret}, we conclude that RoMEX-\(\phi\) achieves
\begin{align}
\label{eq:regret-bilinear}
  \widetilde{\mathcal O}\!\left(
\gamma^2 C_0 H d_{\mathrm{lin}}\sqrt K
+
\gamma^2 C_0 H^2 d_{\mathrm{lin}}
+
C_0 H(B^{\max}_\phi)^2\sqrt K
+
\varepsilon^{\mathrm{approx}}_{\mathrm{dual}}
\right)  
\end{align}
robust equilibrium regret in this subclass, up to logarithmic factors. In particular, this provides a concrete structured two-player zero-sum robust class in which the robust MADC is finite and the regret remains sublinear in \(K\), with an explicit dependence on the distribution mismatch through \(C_0\).
\end{itemize}
\begin{rem}[Bilinear Regret for TV]
    For TV-case, we apply Definition \ref{ass:robust-nominal-concentrability} to handle the support-shift issue and we apply $C_{\mathrm{cov}}\geq 1$ in  \eqref{eq:regret-bilinear} instead of $C_0$ to get the revised regret bound as $  \widetilde{\mathcal O}\!\left(
\gamma^2 C_{\mathrm{cov}} H d_{\mathrm{lin}}\sqrt K
+
\gamma^2 C_{\mathrm{cov}} H^2 d_{\mathrm{lin}}
+
C_{\mathrm{cov}} H(B^{\max}_\phi)^2\sqrt K
+
\varepsilon^{\mathrm{approx}}_{\mathrm{dual}}
\right)$.
\end{rem}

This merged example highlights the key structural distinction between the nominal and robust settings. In the nominal case, linearity of the transition model directly implies linearity of the Bellman operator and hence immediately yields a bilinear factorization. In the robust case, this implication fails because the worst-case support operator is nonlinear in the continuation value. The general robust linear Bellman-complete DRMG class addresses this by assuming closure of the robust Bellman image, while the robust zero-sum linear MG specialization makes explicit the additional support-linearity condition under which this closure can be verified constructively. 

In both cases, the resulting information gain scales as \(\widetilde{\mathcal O}(H d_{\mathrm{lin}})\). However, in the robust setting this alone is not sufficient to control value prediction, due to the mismatch between nominal training distributions and worst-case occupancies. When combined with the robust-to-nominal feature coverage condition, this yields a bounded robust MADC of order 
\(\widetilde{\mathcal O}(\gamma^2 H d_{\mathrm{lin}})\). Consequently, by Theorem~\ref{thm:main-robust-regret}, this structure ensures sublinear robust equilibrium regret in bilinear classes.


\section{Auxiliary Lemmas}
\label{app:aux_lemma}
We first prove several auxiliary lemmas.
\begin{lem}[Unified surrogate deviation inequality]
\label{lem:surrogate_ne_avg}
Fix an episode $k\in[K]$, and let $\widehat{\bm \Gamma}^k$ be the surrogate normal-form game induced by the regularized robust payoffs $\{\widehat U_i^k\}_{i\in\mathcal M}$ defined in~(26). Let $\pi^k$ be the output of the equilibrium oracle for a target equilibrium notion $\mathsf{Eq}\in\{\mathrm{NASH},\mathrm{CCE},\mathrm{CE}\}$. For each agent $i\in\mathcal M$, define the corresponding surrogate unilateral deviation policy $\widetilde{\bm \pi}^{\dagger, k, \mathsf{Eq}}_i$ as defined in \eqref{eq:best_response_i_all_agents}. Then, for every agent $i\in\mathcal M$, the surrogate payoff satisfies
\[
\mathbb E_{\pi^k}\!\left[\widehat U_i^k\right]
\ge
\mathbb E_{\widetilde{\bm \pi}^{\dagger, k, \mathsf{Eq}}_i}\!\left[\widehat U_i^k\right].
\]

\end{lem}

\begin{proof}
Fix $k\in[K]$ and $i\in\mathcal M$. By construction, $\widehat{\bm \Gamma}^k$ is a normal-form game whose pure strategies are elements of $\Pi^{\mathrm{pur}}$, and whose payoff to agent $i$ at pure profile $\pi$ is $\widehat U_i^k(\pi)$ as defined in \eqref{eq:reg_robust_payoff}. Depending on the target equilibrium notion, the oracle returns either
(i) a product mixed strategy over $\{\Pi_i^{\mathrm{pur}}\}_{i\in\mathcal M}$ in the NASH case,
or (ii) a correlated distribution over $\Pi^{\mathrm{pur}}$ in the CCE/CE cases. Since $\pi^k$ is returned by the Equilibrium oracle applied to the surrogate normal-form game $\widehat{\bm \Gamma}^k$ under the target equilibrium notion $\mathsf{Eq}=\{\mathrm{NASH}/\mathrm{CCE}/\mathrm{CE}\}$, it satisfies the unilateral deviation property: for any agent $i \in \mathcal{M}$ and any mixed strategy $\pi'_i \in \Delta(\Pi_i^{\mathrm{pur}})$,
\begin{equation}\label{eq:ne_deviation_property}
\mathbb{E}_{\pi^k}\!\left[\widehat U_i^k\right]
\;\ge\;
\mathbb{E}_{\pi'_i \times \pi_{-i}^k}\!\left[\widehat U_i^k\right].
\end{equation}

\begin{itemize}
\item \textbf{Case 1: $\bm{\mathsf{Eq}=\mathrm{NASH}}$.}
Since $\pi^k$ is a Nash equilibrium of $\widehat{\bm \Gamma}^k$, no agent can improve its surrogate payoff by unilateral deviation. Hence,
\[
\mathbb E_{\pi^k}\!\left[\widehat U_i^k\right]
\ge
\mathbb E_{\pi_i'\times \pi_{-i}^k}\!\left[\widehat U_i^k\right]
\quad \forall \pi_i'.
\]
Taking $\pi_i'=\pi_{i,\mathrm{NASH}}^{\dagger,k}(\pi_{-i}^k)$ gives the result.

\item \textbf{Case 2: $\bm {\mathsf{Eq}=\mathrm{CCE}}$.}
Since $\pi^k$ is a coarse correlated equilibrium of $\widehat{\bm \Gamma}^k$, no agent can improve its surrogate payoff by deviating to any fixed policy chosen independently of the recommendation. Therefore,
\[
\mathbb E_{\pi^k}\!\left[\widehat U_i^k\right]
\ge
\mathbb E_{\pi_i'\times \pi_{-i}^k}\!\left[\widehat U_i^k\right]
\quad \forall \pi_i'.
\]
Taking $\pi_i'=\pi_{i,\mathrm{CCE}}^{\dagger,k}(\pi_{-i}^k)$ yields the result. Here $\pi_{-i}^k$ denotes the marginal distribution over the opponents' pure strategies
induced by the correlated distribution $\pi^k$ in the surrogate game.

\item \textbf{Case 3: $\bm{\mathsf{Eq}=\mathrm{CE}}$.} Since $\pi^k$ is a correlated equilibrium of $\widehat{\Gamma}^k$, no agent can improve
its surrogate payoff by applying any strategy modification
$\psi_i \in \widehat{\Psi}_i$ to its recommendation. Hence,
\[
\mathbb E_{\pi^k}[\widehat U_i^k]
\ge
\mathbb E_{\psi_i \diamond \pi^k}[\widehat U_i^k],
\qquad \forall \psi_i \in \widehat{\Psi}_i,
\]
where we define $\widehat{\Psi}_i :=\{\psi_i: \Pi^{\mathrm{pur}}_i \to \Pi^{\mathrm{pur}}_i\}$. Taking $\psi_i = \psi_i^{\dagger,k}$ gives the result.
\end{itemize}
Combining the three cases proves the lemma. This argument follows the same outer normal-form equilibrium step used in \citep[Appendix A.2] {Xiong2023SampleEfficientMR}; see, e.g., Eqs.~(A.13)--(A.14) for an analogous deviation inequality.
\end{proof}

\begin{lem}[Robust payoff domination by plug-in surrogate value]
\label{cor:robust-payoff-domination}
Fix an episode \(k\), a pure joint policy \(\pi \in \Pi^{\mathrm{pur}}\), and an agent
\(i \in \mathcal M\). Let \(\{f_{i,h}^{k,\pi}\}_{h=1}^{H+1}\) be the backward fitted sequence
constructed by the algorithm, with \(f_{i,H+1}^{k,\pi}\equiv 0\), and let
\(\{\widehat g_{i,h}^{k,\pi}\}_{h=1}^H\) be the corresponding empirical dual minimizers.

Let \(\nu_{i,h}^{\pi}\) denote the stage-\(h\) state--joint-action occupancy induced by policy
\(\pi\) and the worst-case transition kernels attaining the robust Bellman recursion for agent \(i\).
Then, for every episode \(k\), pure joint policy \(\pi\), and agent \(i\),
\begin{align*}
V_{i,1}^{\pi,\sigma_i}(s_1^k)
\le\;&
\widehat U_i^k(\pi)
+
\beta \widehat{\mathcal K}_i^{k-1}
\!\left(f_i^{k,\pi};\widehat g_i^{k,\pi};\pi\right)
+
\Delta_{i,k}^{\mathrm{op}}(\pi)
+
\sum_{h=1}^H
\mathbb E_{\nu_{i,h}^{\pi}}
\Big[
\big(
\mathcal T_{\widehat g_{i,h}^{k,\pi},i,h}^{\pi,\phi,\sigma_i}
f_{i,h+1}^{k,\pi}
-
f_{i,h}^{k,\pi}
\big)_+
\Big],
\end{align*}
where the robust-occupancy operator residual is defined as
\[
\Delta_{i,k}^{\mathrm{op}}(\pi)
:=
\sum_{h=1}^H
\left\|
\mathcal T_{i,h}^{\pi,\phi,\sigma_i} f_{i,h+1}^{k,\pi}
-
\mathcal T_{\widehat g_{i,h}^{k,\pi},i,h}^{\pi,\phi,\sigma_i}
f_{i,h+1}^{k,\pi}
\right\|_{1,\nu_{i,h}^{\pi}} .
\]
\end{lem}

\begin{proof}
Fix \(k\), \(\pi\), and \(i\). For notational simplicity, write
\[
V_h := V_{i,h}^{\pi,\sigma_i},\quad
Q_h := Q_{i,h}^{\pi,\sigma_i},\quad
f_h := f_{i,h}^{k,\pi},\quad
\widehat g_h := \widehat g_{i,h}^{k,\pi},
\]
and
\(
\mathcal T_h := \mathcal T_{i,h}^{\pi,\phi,\sigma_i},
\quad
\widehat{\mathcal T}_h
:=
\mathcal T_{\widehat g_{i,h}^{k,\pi},i,h}^{\pi,\phi,\sigma_i}.
\)
Let
\(
f_h^\pi(s)
:=
\mathbb E_{{\bm a}\sim\pi_h(\cdot\mid s)}
[f_h(s,{\bm a})].
\)
Since \(\pi\in\Pi^{\mathrm{pur}}\), this reduces to
\(
f_h^\pi(s)=f_h(s,\pi_h(s)).
\)

\paragraph{Step 1: Robust Bellman recursion.}
By the robust Bellman equation,
\[
Q_h(s,{\bm a})
=
r_{i,h}(s,{\bm a})
+
\inf_{P\in\mathcal P_{i,h}^{\sigma_i,\phi}(s,{\bm a})}
\mathbb E_{s'\sim P}[V_{h+1}(s')],
\]
and
\[
V_h(s)
=
\mathbb E_{{\bm a}\sim\pi_h(\cdot\mid s)}
[Q_h(s,{\bm a})].
\]
Equivalently,
\[
Q_h=\mathcal T_h V_{h+1},
\qquad
V_h(s)
=
\mathbb E_{{\bm a}\sim\pi_h(\cdot\mid s)}
[(\mathcal T_hV_{h+1})(s,{\bm a})].
\]
Also,
\[
V_{H+1}\equiv 0,\qquad f_{H+1}\equiv 0.
\]

\paragraph{Step 2: Define the robust occupancy distribution.}
For the fixed policy \(\pi\) and agent \(i\), let
\(\nu_{i,h}^{\pi}\) denote the stage-\(h\) state-action occupancy induced by
\(\pi\) and the worst-case transition kernels attaining the robust Bellman
recursion for \(V_{i,h}^{\pi,\sigma_i}\). That is,
\(
(s_h,{\bm a}_h)\sim \nu_{i,h}^{\pi}
\)
when the trajectory starts from \(s_1^k\), follows policy \(\pi\), and evolves
under the worst-case kernels associated with the robust value of agent \(i\).

\paragraph{Step 3: Expected robust performance-difference identity.}
For any function sequence \(f=\{f_h\}_{h=1}^{H+1}\) with \(f_{H+1}\equiv 0\),
the robust Bellman telescoping identity gives
\begin{align}
V_1(s_1^k)-f_1^\pi(s_1^k)
=
\sum_{h=1}^H
\mathbb E_{\nu_{i,h}^{\pi}}
\left[
(\mathcal T_h f_{h+1})(s_h,{\bm a}_h)
-
f_h(s_h,{\bm a}_h)
\right].
\label{eq:expected-domination-core}
\end{align}
This follows by adding and subtracting \(f_h\) along the robust trajectory:
\[
V_h(s_h)
-
f_h(s_h,{\bm a}_h)
=
(\mathcal T_h f_{h+1})(s_h,{\bm a}_h)
-
f_h(s_h,{\bm a}_h)
+
\mathbb E[
V_{h+1}(s_{h+1})-f_{h+1}^\pi(s_{h+1})
\mid s_h,{\bm a}_h],
\]
and then summing over \(h=1,\dots,H\). The terminal term vanishes because
\(V_{H+1}=f_{H+1}=0\).

\paragraph{Step 4: Add and subtract the dual plug-in operator.}
For each stage \(h\),
\[
\mathcal T_h f_{h+1}-f_h
=
\bigl(\mathcal T_h f_{h+1}
-
\widehat{\mathcal T}_h f_{h+1}\bigr)
+
\bigl(\widehat{\mathcal T}_h f_{h+1}-f_h\bigr).
\]
Using \(x\le |x|\) and \(x\le x_+\), we obtain
\begin{align}
(\mathcal T_h f_{h+1}-f_h)(s,{\bm a})
\le
&
\left|
(\mathcal T_h f_{h+1})(s,{\bm a})
-
(\widehat{\mathcal T}_h f_{h+1})(s,{\bm a})
\right|
\nonumber\\
&+
\left(
(\widehat{\mathcal T}_h f_{h+1})(s,{\bm a})
-
f_h(s,{\bm a})
\right)_+ .
\label{eq:add-subtract-plugin}
\end{align}

Substituting \eqref{eq:add-subtract-plugin} into
\eqref{eq:expected-domination-core} gives
\begin{align}
V_1(s_1^k)-f_1^\pi(s_1^k)
\le
&
\sum_{h=1}^H
\left\|
\mathcal T_h f_{h+1}
-
\widehat{\mathcal T}_h f_{h+1}
\right\|_{1,\nu_{i,h}^{\pi}}
+
\sum_{h=1}^H
\mathbb E_{\nu_{i,h}^{\pi}}
\left[
\left(
\widehat{\mathcal T}_h f_{h+1}
-
f_h
\right)_+
\right].
\label{eq:step7}
\end{align}

\paragraph{Step 5: Substitute the fitted functions and payoff definition.}
Taking \(f_h=f_{i,h}^{k,\pi}\), we have
\[
f_1^\pi(s_1^k)=f_{i,1}^{k,\pi}(s_1^k),
\]
because \(\pi\) is pure. Therefore,
\begin{align}
V_{i,1}^{\pi,\sigma_i}(s_1^k)
\le
&f_{i,1}^{k,\pi}(s_1^k)
+
\sum_{h=1}^H
\left\|
\mathcal T_{i,h}^{\pi,\phi,\sigma_i}f_{i,h+1}^{k,\pi}
-
\mathcal T_{\widehat g_{i,h}^{k,\pi},i,h}^{\pi,\phi,\sigma_i}
f_{i,h+1}^{k,\pi}
\right\|_{1,\nu_{i,h}^{\pi}}
\nonumber\\
&+
\sum_{h=1}^H
\mathbb E_{\nu_{i,h}^{\pi}}
\left[
\left(
\mathcal T_{\widehat g_{i,h}^{k,\pi},i,h}^{\pi,\phi,\sigma_i}
f_{i,h+1}^{k,\pi}
-
f_{i,h}^{k,\pi}
\right)_+
\right].
\label{eq:robust-payoff-domination-compact}
\end{align}

Here \(\widehat{\mathcal K}_i^{k-1}\) is the cumulative centered empirical
robust discrepancy defined in \eqref{eq:robust_empirical_discrepancy_centered}. By the definition of the regularized robust payoff,
\[
\widehat U_i^k(\pi)
=
f_{i,1}^{k,\pi}(s_1^k)
-
\beta
\widehat{\mathcal K}_i^{k-1}
(f_i^{k,\pi};\widehat g_i^{k,\pi};\pi).
\]
Equivalently,
\[
f_{i,1}^{k,\pi}(s_1^k)
=
\widehat U_i^k(\pi)
+
\beta
\widehat{\mathcal K}_i^{k-1}
(f_i^{k,\pi};\widehat g_i^{k,\pi};\pi).
\]
Substituting this into \eqref{eq:robust-payoff-domination-compact}, we obtain
\begin{align}
V_{i,1}^{\pi,\sigma_i}(s_1^k)
\le
&\widehat U_i^k(\pi)
+
\beta
\widehat{\mathcal K}_i^{k-1}
(f_i^{k,\pi};\widehat g_i^{k,\pi};\pi) +
\sum_{h=1}^H
\left\|
\mathcal T_{i,h}^{\pi,\phi,\sigma_i}f_{i,h+1}^{k,\pi}
-
\mathcal T_{\widehat g_{i,h}^{k,\pi},i,h}^{\pi,\phi,\sigma_i}
f_{i,h+1}^{k,\pi}
\right\|_{1,\nu_{i,h}^{\pi}}
\nonumber\\
&\qquad +
\sum_{h=1}^H
\mathbb E_{\nu_{i,h}^{\pi}}
\left[
\left(
\mathcal T_{\widehat g_{i,h}^{k,\pi},i,h}^{\pi,\phi,\sigma_i}
f_{i,h+1}^{k,\pi}
-
f_{i,h}^{k,\pi}
\right)_+
\right].
\end{align}
This proves the claim.
\end{proof}

\begin{lem}[Robust Empirical Discrepancy Concentration]
\label{lem:robust_A1_corrected}
Assume Assumptions~\ref{ass:completeness}--\ref{ass:multiplier-range-TV} hold and define
\(
\Lambda_i
:=
\log\!\left(
\frac{
|\mathcal F_i|
|\mathcal G_i|
|\Pi^{\mathrm{pur}}|
KHm
}{\delta}
\right).
\)
Fix any agent $i \in \mathcal M$, any pure joint policy $\pi \in \Pi^{\mathrm{pur}}$, any function $f_i \in \mathcal F_i$, and any $k \in [K]$. Then, for \(K\ge \Lambda_i\), with probability at least $1-\delta$, we have
\begin{align*}
\sum_{s=1}^{k-1} e_{i,1}^s(f_i,\pi)
\le
2\,\widehat{\mathcal K}_i^{k-1}(f_i;\widehat{g}_i^{f_i,\pi};\pi)
+ \mathcal{O}\left(H B_\phi(\sigma_i)^2\sqrt{K \log\!\Big(\frac{|\mathcal F_i||\mathcal G_i||\Pi^{\mathrm{pur}}|KHm}{\delta}\Big)} + \varepsilon^{\mathrm{approx}}_{\mathrm{dual}}\right).
\end{align*}
\end{lem}

\begin{proof}
We follow the same overall structure as the proof of Lemma~A.1 in
\citep{Xiong2023SampleEfficientMR}, but with one additional step to control
the error induced by approximating the robust Bellman operator through the
dual plug-in operator.

For each stage $h\in[H]$ and episode $k\ge 2$, recall the historical
occupancy mixture from Lemma \ref{lem:phi_operator_approx},
\begin{align}
\bar\mu_h^{k-1}
:=
\frac{1}{k-1}\sum_{s=1}^{k-1}\mu_h^{s,\pi}.
\label{eq:historical_occupancy_mixture_lemma5}
\end{align}
Equivalently, for every integrable function $\varphi$,
\begin{align}
\sum_{s=1}^{k-1}
\mathbb E_{\mu_h^{s,\pi}}[\varphi]
=
(k-1)\,
\mathbb E_{\bar\mu_h^{k-1}}[\varphi].
\label{eq:historical_mixture_identity_lemma5}
\end{align}

\paragraph{Step 1: Expand the cumulative population training error.}
By definition \eqref{eq:correct_discrepancy},
\begin{align}
\sum_{s=1}^{k-1} e_{i,1}^s(f_i,\pi)
=
\sum_{h=1}^H \sum_{s=1}^{k-1}
\mathbb{E}_{\mu_h^{s,\pi}}
\left[
\left(
f_{i,h}
-
\mathcal{T}_{i,h}^{\pi,\phi,\sigma_i} f_{i,h+1}
\right)^2
\right].
\label{eq:step1_expand_population}
\end{align}

\paragraph{Step 2: Add and subtract the plug-in robust Bellman operator.}
For each stage $h$, let
$\widehat g_{i,h}^{f_i,\pi}$ denote the empirical dual minimizer
computed from the historical dataset $\mathcal D_h^{k-1}$, and define the
corresponding plug-in robust dual-Bellman operator
\(
\big[
\mathcal{T}_{
\widehat g_{i,h}^{f_i,\pi},i,h
}^{\pi,\phi,\sigma_i}
f_{i,h+1}
\big](s,\bm a)
\)
as in \eqref{eq:revised_dual_TV_g}. Then,
\begin{align}
f_{i,h}
-
\mathcal{T}_{i,h}^{\pi,\phi,\sigma_i}f_{i,h+1}
&=
\underbrace{
f_{i,h}
-
\mathcal{T}_{
\widehat g_{i,h}^{f_i,\pi},i,h
}^{\pi,\phi,\sigma_i}
f_{i,h+1}
}_{\text{plug-in residual}}
\nonumber\\
&\quad+
\underbrace{
\mathcal{T}_{
\widehat g_{i,h}^{f_i,\pi},i,h
}^{\pi,\phi,\sigma_i}
f_{i,h+1}
-
\mathcal{T}_{i,h}^{\pi,\phi,\sigma_i}
f_{i,h+1}
}_{\text{operator approximation error}}.
\label{eq:add_subtract_plugin}
\end{align}
Using $(x+y)^2\le 2x^2+2y^2$, we obtain
\begin{align}
\sum_{s=1}^{k-1} e_{i,1}^s(f_i,\pi)
\le
2A_{i,k}(f_i,\pi)+2B_{i,k}(f_i,\pi),
\label{eq:A_B_decomposition}
\end{align}
where
\begin{align}
A_{i,k}(f_i,\pi)
&:=
\sum_{h=1}^H\sum_{s=1}^{k-1}
\mathbb E_{\mu_h^{s,\pi}}
\left[
\left(
f_{i,h}
-
\mathcal{T}_{
\widehat g_{i,h}^{f_i,\pi},i,h
}^{\pi,\phi,\sigma_i}
f_{i,h+1}
\right)^2
\right],
\label{eq:Aik_def}
\\
B_{i,k}(f_i,\pi)
&:=
\sum_{h=1}^H\sum_{s=1}^{k-1}
\mathbb E_{\mu_h^{s,\pi}}
\left[
\left(
\mathcal{T}_{
\widehat g_{i,h}^{f_i,\pi},i,h
}^{\pi,\phi,\sigma_i}
f_{i,h+1}
-
\mathcal{T}_{i,h}^{\pi,\phi,\sigma_i}
f_{i,h+1}
\right)^2
\right].
\label{eq:Bik_def}
\end{align}

\paragraph{Step 3: Control of the operator approximation term
$B_{i,k}$.}
For notational convenience, define
\begin{equation}
\Delta_{i,h}^{f_i,\pi}(s,\bm a)
:=
\Big(
\mathcal T_{
\widehat g_{i,h}^{f_i,\pi},i,h
}^{\pi,\phi,\sigma_i}
f_{i,h+1}
-
\mathcal T_{i,h}^{\pi,\phi,\sigma_i}
f_{i,h+1}
\Big)(s,\bm a).
\label{eq:Delta_def}
\end{equation}
Then, by \eqref{eq:historical_mixture_identity_lemma5} applied with
$\varphi=(\Delta_{i,h}^{f_i,\pi})^2$,
\begin{align}
B_{i,k}(f_i,\pi)
&=
\sum_{h=1}^H\sum_{s=1}^{k-1}
\mathbb E_{\mu_h^{s,\pi}}
\left[
\left(
\Delta_{i,h}^{f_i,\pi}
\right)^2
\right]
\nonumber\\
&=
(k-1)\sum_{h=1}^H
\mathbb E_{\bar\mu_h^{k-1}}
\left[
\left(
\Delta_{i,h}^{f_i,\pi}
\right)^2
\right].
\label{eq:Bik_mixture_form}
\end{align}

By Assumption~\ref{ass:multiplier-range-TV}, the dual integrand is
uniformly bounded:
\[
\left|
\mathrm{Loss}_{i,h,\phi}
(f_{i,h+1};s,\bm a,s';\lambda,\nu)
\right|
\le
B_\phi(\sigma_i)
\qquad
\forall (s,\bm a,s'),\quad
\forall(\lambda,\nu)\in\Theta_\phi.
\]
Since both the exact robust Bellman operator and the plug-in robust
dual-Bellman operator are obtained from expectations of such bounded
quantities, and their reward terms cancel, it follows that
\begin{equation}
\left|
\Delta_{i,h}^{f_i,\pi}(s,\bm a)
\right|
\le
2B_\phi(\sigma_i)
\qquad
\forall(s,\bm a).
\label{eq:Delta_uniform_bound}
\end{equation}
Consequently,
\begin{align}
\left(
\Delta_{i,h}^{f_i,\pi}(s,\bm a)
\right)^2
\le
2B_\phi(\sigma_i)
\left|
\Delta_{i,h}^{f_i,\pi}(s,\bm a)
\right|.
\end{align}
Taking expectation under $\bar\mu_h^{k-1}$ gives
\begin{equation}
\mathbb E_{\bar\mu_h^{k-1}}
\left[
\left(
\Delta_{i,h}^{f_i,\pi}
\right)^2
\right]
\le
2B_\phi(\sigma_i)
\left\|
\Delta_{i,h}^{f_i,\pi}
\right\|_{1,\bar\mu_h^{k-1}}.
\label{eq:L2_to_L1_reduction}
\end{equation}

Lemma~\ref{lem:phi_operator_approx}, applied with $\mathcal D_h=\mathcal D_h^{k-1}$
(so $n=k-1$), already guarantees a bound on
$\|\Delta_{i,h}^{f_i,\pi}\|_{1,\bar\mu_h^{k-1}}$ holding uniformly over
$\pi\in\Pi^{\mathrm{pur}}$ and $f_{i,h+1}\in\mathcal F_{i,h+1}$, for a fixed
agent $i$ and stage $h$. Taking a further union bound over the remaining
indices $i\in[m]$, $h\in[H]$, and $k\in\{2,\ldots,K\}$ (i.e., only over the
indices not already ranged over inside Lemma~\ref{lem:phi_operator_approx}),
we obtain that, with probability at least $1-\delta$, simultaneously for all
such choices and all $f_i\in\mathcal F_i$,
\begin{align}
\left\|
\Delta_{i,h}^{f_i,\pi}
\right\|_{1,\bar\mu_h^{k-1}}
\le
c\left(
B_\phi(\sigma_i)
\sqrt{
\frac{
\log\!\left(
\frac{
|\mathcal F_i|
|\mathcal G_i|
|\Pi^{\mathrm{pur}}|
KHm
}{\delta}
\right)
}{
k-1
}
}
+
\bar\varepsilon_{\mathrm{dual}}
\right).
\label{eq:lemma2_applied}
\end{align}
Substituting \eqref{eq:lemma2_applied} into
\eqref{eq:L2_to_L1_reduction} yields
\begin{align}
\mathbb E_{\bar\mu_h^{k-1}}
\left[
\left(
\Delta_{i,h}^{f_i,\pi}
\right)^2
\right]
&\le
c_1B_\phi(\sigma_i)^2
\sqrt{
\frac{
\log\!\left(
\frac{
|\mathcal F_i|
|\mathcal G_i|
|\Pi^{\mathrm{pur}}|
KHm
}{\delta}
\right)
}{
k-1
}
}
+
c_2B_\phi(\sigma_i)
\bar\varepsilon_{\mathrm{dual}}.
\label{eq:Delta_square_bound_correct}
\end{align}

Combining \eqref{eq:Bik_mixture_form} and
\eqref{eq:Delta_square_bound_correct}, and summing over
$h\in[H]$, gives
\begin{align}
B_{i,k}(f_i,\pi)
&\le
c_1H B_\phi(\sigma_i)^2
\sqrt{
(k-1)
\log\!\left(
\frac{
|\mathcal F_i|
|\mathcal G_i|
|\Pi^{\mathrm{pur}}|
KHm
}{\delta}
\right)
}
\nonumber\\
&\quad+
c_2H(k-1)B_\phi(\sigma_i)
\bar\varepsilon_{\mathrm{dual}}.
\label{eq:Bik_preliminary_bound}
\end{align}
Since $k-1\le K$ and
\(
\bar\varepsilon_{\mathrm{dual}}
\le
\frac{
\varepsilon_{\mathrm{dual}}^{\mathrm{approx}}
}{
HK B_\phi(\sigma_i)
},
\)
we conclude that
\begin{align}
B_{i,k}(f_i,\pi)
\le
\mathcal O\left(
H B_\phi(\sigma_i)^2
\sqrt{
K
\log\!\left(
\frac{
|\mathcal F_i|
|\mathcal G_i|
|\Pi^{\mathrm{pur}}|
KHm
}{\delta}
\right)
}
+
\varepsilon_{\mathrm{dual}}^{\mathrm{approx}}
\right).
\label{eq:Bik_correct_bound}
\end{align}

\paragraph{Step 4: Control the plug-in residual term \(A_{i,k}\) using
the aggregate empirical discrepancy.}

For each stage \(h\), function \(f_i\in\mathcal F_i\), and fixed dual
function \(g_{i,h}\in\mathcal G_{i,h}\), define the aggregate empirical
squared loss
\begin{align}
\widehat L_{i,h}^{k-1}(f_i;g_{i,h},\pi)
:=
\sum_{j=1}^{k-1}
\widehat\delta_{i,h}
\!\left(
f_i;g_{i,h};\pi;
s_h^j,\bm a_h^j,s_{h+1}^j
\right)^2,
\label{eq:Lhat_A_def}
\end{align}
and its corresponding conditional population loss
\begin{align}
L_{i,h}^{k-1}(f_i;g_{i,h},\pi)
:=
\sum_{j=1}^{k-1}
\mathbb E\!\left[
\widehat\delta_{i,h}
\!\left(
f_i;g_{i,h};\pi;
s_h^j,\bm a_h^j,s_{h+1}^j
\right)^2
\middle|\mathcal H_{j-1}
\right].
\label{eq:L_A_def}
\end{align}
Here \(g_{i,h}\) is held fixed when taking the conditional expectation.

With the revised definition of the centered empirical robust discrepancy,
\begin{align}
\widehat{\mathcal K}_i^{k-1}
\bigl(
f_i;
\widehat g_i^{f_i,\pi};
\pi
\bigr)
=
\sum_{h=1}^H
\Bigg[
\widehat L_{i,h}^{k-1}
\bigl(
f_i;
\widehat g_{i,h}^{f_i,\pi},
\pi
\bigr)
-
\inf_{f'_i\in\mathcal F_i}
\widehat L_{i,h}^{k-1}
\bigl(
f'_i;
\widehat g_{i,h}^{f_i,\pi},
\pi
\bigr)
\Bigg].
\label{eq:Khat_as_emp_excess}
\end{align}

By Assumption~\ref{ass:multiplier-range-TV}, the plug-in residual is
uniformly bounded:
\begin{align}
\left|
\widehat\delta_{i,h}
\!\left(
f_i;g_{i,h};\pi;s,\bm a,s'
\right)
\right|
\le
C B_\phi(\sigma_i)
\label{eq:plugin_residual_uniform_bound}
\end{align}
for every admissible \(f_i,g_{i,h},\pi\) and transition
\((s,\bm a,s')\).

For every fixed
\((f_i,g_{i,h},\pi)\in
\mathcal F_i\times\mathcal G_{i,h}\times\Pi^{\mathrm{pur}}\),
define
\begin{align}
Z_j(f_i,g_{i,h},\pi)
:={}&
\widehat\delta_{i,h}
\!\left(
f_i;g_{i,h};\pi;
s_h^j,\bm a_h^j,s_{h+1}^j
\right)^2
\nonumber\\
&-
\mathbb E\!\left[
\widehat\delta_{i,h}
\!\left(
f_i;g_{i,h};\pi;
s_h^j,\bm a_h^j,s_{h+1}^j
\right)^2
\middle|\mathcal H_{j-1}
\right].
\label{eq:plugin_squared_mds}
\end{align}
Since the behavior policy used in episode \(j\) is
\(\mathcal H_{j-1}\)-measurable,
\(\{Z_j(f_i,g_{i,h},\pi)\}_{j\ge1}\) is a bounded
martingale-difference sequence. Moreover,
\[
|Z_j(f_i,g_{i,h},\pi)|
\le
C B_\phi(\sigma_i)^2.
\]

Applying Freedman's inequality
(Lemma \ref{lem:Freedman}) and taking a union bound over
\[
i\in[m],\quad h\in[H],\quad k\in[K],\quad
\pi\in\Pi^{\mathrm{pur}},\quad
f_i\in\mathcal F_i,\quad
g_{i,h}\in\mathcal G_{i,h},
\]
we obtain that, with probability at least \(1-\delta\), simultaneously
for all such choices,
\begin{align}
\left|
\widehat L_{i,h}^{k-1}(f_i;g_{i,h},\pi)
-
L_{i,h}^{k-1}(f_i;g_{i,h},\pi)
\right|
\le
\epsilon_{i,k},
\label{eq:uniform_L_concentration_A}
\end{align}
where, writing
\(
\Lambda
:=
\log\!\left(
\frac{
|\mathcal F_i|
|\mathcal G_{i,h}|
|\Pi^{\mathrm{pur}}|
KHm
}{\delta}
\right),
\)
we may take
\begin{align}
\epsilon_{i,k}
=
\mathcal O\!\left(
B_\phi(\sigma_i)^2
\left[
\sqrt{(k-1)\Lambda}+\Lambda
\right]
\right).
\label{eq:epsilon_ik_def}
\end{align}
In particular, since \(k-1\le K\),
\begin{align}
\epsilon_{i,k}
=
\mathcal O\!\left(
B_\phi(\sigma_i)^2
\left[
\sqrt{K\Lambda}+\Lambda
\right]
\right).
\label{eq:epsilon_ik_K}
\end{align}

Because \eqref{eq:uniform_L_concentration_A} holds uniformly over
\(g_{i,h}\in\mathcal G_{i,h}\), we may now substitute the
data-dependent empirical minimizer
\[
g_{i,h}
=
\widehat g_{i,h}^{f_i,\pi}.
\]

We next pass from uniform loss concentration to excess-loss
concentration. For brevity, fix
\[
\widehat g_{i,h}
:=
\widehat g_{i,h}^{f_i,\pi},
\]
and let
\[
f_{i,h}^{\mathrm{pop}}
\in
\arg\min_{f'_i\in\mathcal F_i}
L_{i,h}^{k-1}(f'_i;\widehat g_{i,h},\pi).
\]
Then,
\begin{align}
&
L_{i,h}^{k-1}(f_i;\widehat g_{i,h},\pi)
-
\inf_{f'_i\in\mathcal F_i}
L_{i,h}^{k-1}(f'_i;\widehat g_{i,h},\pi)
\nonumber\\
&\le
\widehat L_{i,h}^{k-1}(f_i;\widehat g_{i,h},\pi)
-
\inf_{f'_i\in\mathcal F_i}
\widehat L_{i,h}^{k-1}(f'_i;\widehat g_{i,h},\pi)
+
2\epsilon_{i,k}.
\label{eq:population_excess_to_empirical_excess_A}
\end{align}

We now relate this population excess loss to the plug-in residual appearing
in \(A_{i,k}\). For a fixed \(g_{i,h}\), the conditional
bias--variance decomposition gives
\begin{align}
L_{i,h}^{k-1}(f_i;g_{i,h},\pi)
={}&
\sum_{j=1}^{k-1}
\mathbb E_{\mu_h^{j,\pi}}
\left[
\left(
f_{i,h}
-
\mathcal T_{g_{i,h},i,h}^{\pi,\phi,\sigma_i}
f_{i,h+1}
\right)^2
\right]
+
\operatorname{Var}_{i,h}^{k-1}
(f_{i,h+1};g_{i,h},\pi),
\label{eq:L_bias_variance_A}
\end{align}
where the variance term depends on the next-stage function
\(f_{i,h+1}\), the fixed dual function \(g_{i,h}\), and the evaluation
policy \(\pi\), but not on the current-stage coordinate \(f_{i,h}\).

For each stage \(h\), define a comparator
\(\bar f_i^{h,f_i,\pi}\in\mathcal F_i\) whose relevant coordinates satisfy
\begin{align}
\bar f_{i,h}^{h,f_i,\pi}
&=
\mathcal T_{
\widehat g_{i,h}^{f_i,\pi},i,h
}^{\pi,\phi,\sigma_i}
f_{i,h+1},
\\
\bar f_{i,h+1}^{h,f_i,\pi}
&=
f_{i,h+1},
\end{align}
with the remaining coordinates chosen arbitrarily from their corresponding
classes. By the product structure of
\(\mathcal F_i=\prod_{h=1}^H\mathcal F_{i,h}\) and
Assumption~\ref{ass:completeness}, such a comparator belongs to
\(\mathcal F_i\).

Because \(\bar f_i^{h,f_i,\pi}\) uses the same next-stage function and the
same plug-in dual function as \(f_i\), the two losses have the same
conditional variance term, while the comparator's conditional-mean
residual is zero. Therefore,
\begin{align}
L_{i,h}^{k-1}
\left(
f_i;
\widehat g_{i,h}^{f_i,\pi},
\pi
\right)
-
L_{i,h}^{k-1}
\left(
\bar f_i^{h,f_i,\pi};
\widehat g_{i,h}^{f_i,\pi},
\pi
\right)
=
\sum_{j=1}^{k-1}
\mathbb E_{\mu_h^{j,\pi}}
\left[
\left(
f_{i,h}
-
\mathcal T_{
\widehat g_{i,h}^{f_i,\pi},i,h
}^{\pi,\phi,\sigma_i}
f_{i,h+1}
\right)^2
\right].
\label{eq:A_population_excess_identity}
\end{align}
Since
\begin{align}
\inf_{f'_i\in\mathcal F_i}
L_{i,h}^{k-1}
\left(
f'_i;
\widehat g_{i,h}^{f_i,\pi},
\pi
\right)
\le
L_{i,h}^{k-1}
\left(
\bar f_i^{h,f_i,\pi};
\widehat g_{i,h}^{f_i,\pi},
\pi
\right),
\end{align}
it follows that
\begin{align}
&
\sum_{j=1}^{k-1}
\mathbb E_{\mu_h^{j,\pi}}
\left[
\left(
f_{i,h}
-
\mathcal T_{
\widehat g_{i,h}^{f_i,\pi},i,h
}^{\pi,\phi,\sigma_i}
f_{i,h+1}
\right)^2
\right]
\nonumber\\
&\le
L_{i,h}^{k-1}
\left(
f_i;
\widehat g_{i,h}^{f_i,\pi},
\pi
\right)
-
\inf_{f'_i\in\mathcal F_i}
L_{i,h}^{k-1}
\left(
f'_i;
\widehat g_{i,h}^{f_i,\pi},
\pi
\right).
\label{eq:A_leq_population_excess_A}
\end{align}

Combining
\eqref{eq:population_excess_to_empirical_excess_A} and
\eqref{eq:A_leq_population_excess_A}, and then summing over
\(h\in[H]\), yields
\begin{align}
A_{i,k}(f_i,\pi)
\le{}&
\widehat{\mathcal K}_i^{k-1}
\bigl(
f_i;
\widehat g_i^{f_i,\pi};
\pi
\bigr)
\nonumber\\
&+
\mathcal O\!\left(
H B_\phi(\sigma_i)^2
\left[
\sqrt{
K
\log\!\left(
\frac{
|\mathcal F_i|
|\mathcal G_i|
|\Pi^{\mathrm{pur}}|
KHm
}{\delta}
\right)
}
+
\log\!\left(
\frac{
|\mathcal F_i|
|\mathcal G_i|
|\Pi^{\mathrm{pur}}|
KHm
}{\delta}
\right)
\right]
\right).
\label{eq:Aik_final_bound}
\end{align}

\paragraph{Step 5: Combine the two bounds.}
Substituting \eqref{eq:Aik_final_bound} and
\eqref{eq:Bik_correct_bound} into
\eqref{eq:A_B_decomposition}, we obtain
\begin{align}
\sum_{s=1}^{k-1}e_{i,1}^s(f_i,\pi)
\le{}&
2\,
\widehat{\mathcal K}_i^{k-1}
\bigl(
f_i;
\widehat g_i^{f_i,\pi};
\pi
\bigr)
+
\mathcal O\!\left(
H B_\phi(\sigma_i)^2
\left[
\sqrt{K\Lambda}
+
\Lambda
\right]
+
\varepsilon_{\mathrm{dual}}^{\mathrm{approx}}
\right),
\label{eq:robust_A1_raw_final}
\end{align}
where
\(\Lambda
:=
\log\!\left(
\frac{
|\mathcal F_i|
|\mathcal G_i|
|\Pi^{\mathrm{pur}}|
KHm
}{\delta}
\right).
\)
Under the standard regime \(K\ge \Lambda\), we have
\[
\Lambda
\le
\sqrt{K\Lambda},
\]
and hence the additive logarithmic term can be absorbed into the leading
square-root term. Therefore,
\begin{align}
\sum_{s=1}^{k-1}e_{i,1}^s(f_i,\pi)
\le{}&
2\,
\widehat{\mathcal K}_i^{k-1}
\bigl(
f_i;
\widehat g_i^{f_i,\pi};
\pi
\bigr)
+
\mathcal O\!\left(
H B_\phi(\sigma_i)^2
\sqrt{
K
\log\!\left(
\frac{
|\mathcal F_i|
|\mathcal G_i|
|\Pi^{\mathrm{pur}}|
KHm
}{\delta}
\right)
}
+
\varepsilon_{\mathrm{dual}}^{\mathrm{approx}}
\right).
\label{eq:robust_A1_final}
\end{align}
This proves the claim.
\qedhere
\end{proof}


\begin{lem}[Optimal concentration for Bellman-consistent robust benchmark]
\label{lem:robust_optimal_concentration}
Assume Assumptions~\ref{ass:completeness}--\ref{ass:multiplier-range-TV} hold, and
define $\Lambda_i:=\log\!\left(\frac{|\mathcal F_i||\mathcal G_i||\Pi^{\mathrm{pur}}|KHm}{\delta}\right)$.
Fix any agent $i \in \mathcal M$, any pure joint policy
$\pi \in \Pi^{\mathrm{pur}}$, and let
$f_i^{\pi,\sigma_i}=\{f_{i,h}^{\pi,\sigma_i}\}_{h=1}^{H+1}\in\mathcal F_i$
be a robust Bellman-consistent sequence satisfying
\begin{align}
f_{i,h}^{\pi,\sigma_i}
=
\mathcal T_{i,h}^{\pi,\phi,\sigma_i} f_{i,h+1}^{\pi,\sigma_i},
\qquad h\in[H],
\qquad \text{with}\qquad 
f_{i,H+1}^{\pi,\sigma_i}\equiv 0.
\label{eq:robust_fixed_point_sequence}
\end{align}
Then, for $K\ge\Lambda_i$, with probability at least $1-\delta$, for every $k\in[K]$,
\begin{align*}
\widehat{\mathcal K}_i^{k-1}(f_i^{\pi,\sigma_i};\widehat{g}_i^{\,f_i^{\pi,\sigma_i},\pi};\pi)
\le
\mathcal O\!\left(
H B_\phi(\sigma_i)^2
\sqrt{
K \log\!\Big(
\frac{|\mathcal F||\mathcal G|\,|\Pi^{\mathrm{pur}}|KHm}{\delta}
\Big)} + \varepsilon^{\mathrm{approx}}_{\mathrm{dual}}\right),
\end{align*}
where $\widehat {\mathcal K}_i^{k-1}$ denotes the cumulative centered empirical robust discrepancy as defined in \eqref{eq:robust_empirical_discrepancy_centered}.
\end{lem}

\begin{proof}
We fix an agent $i\in\mathcal M$, a pure joint policy $\pi\in\Pi^{\mathrm{pur}}$, and an index
$k\in[K]$. Throughout the proof, we abbreviate
\[
f_i^\star := f_i^{\pi,\sigma_i},
\qquad
\Lambda := \log\!\left(\frac{|\mathcal F||\mathcal G||\Pi^{\mathrm{pur}}|KHm}{\delta}\right).
\]
By assumption, $f_i^\star$ is Bellman-consistent with respect to the \emph{true} robust Bellman
operator, namely \eqref{eq:robust_fixed_point_sequence} holds, i.e.,
\[
f_{i,h}^\star
=
\mathcal T_{i,h}^{\pi,\phi,\sigma_i}f_{i,h+1}^\star,
\qquad h\in[H].
\]

\paragraph{Step 1: Define empirical and population aggregate losses.}

For each stage \(h\in[H]\), and for any fixed
\((f_i,g_i,\pi)\in\mathcal F_i\times\mathcal G_i\times\Pi^{\mathrm{pur}}\), define
\begin{align}
\widehat L_{i,h}^{k-1}(f_i;g_i)
:=
\sum_{j=1}^{k-1}
\widehat\delta_{i,h}
\!\left(
f_i;g_{i,h};\pi;
s_h^j,\bm a_h^j,s_{h+1}^j
\right)^2 ,
\label{eq:empirical-loss-Lhat}
\end{align}
and
\begin{align}
L_{i,h}^{k-1}(f_i;g_i)
:=
\sum_{j=1}^{k-1}
\mathbb E\!\left[
\widehat\delta_{i,h}
\!\left(
f_i;g_{i,h};\pi;
s_h^j,\bm a_h^j,s_{h+1}^j
\right)^2
\middle|\mathcal H_{j-1}
\right].
\label{eq:population-loss-L}
\end{align}
The concentration argument below is first proved uniformly over fixed
\((f_i,g_i,\pi)\). Since the resulting event holds simultaneously for all
\(g_i\in\mathcal G_i\), we may then substitute the data-dependent empirical dual
minimizer \(g_i=\widehat g_i^{\,f_i^\star,\pi}\). On this uniform event, taking \(g_i=\widehat g_i^{\,f_i^\star,\pi}\), the 
centered empirical robust discrepancy satisfies \eqref{eq:robust_empirical_discrepancy_centered}, as
\begin{align}
\widehat{\mathcal K}_i^{k-1}
(f_i^\star;\widehat g_i^{\,f_i^\star,\pi};\pi)
=
\sum_{h=1}^H
\left[
\widehat L_{i,h}^{k-1}(f_i^\star;\widehat g_i^{\,f_i^\star,\pi})
-
\inf_{f'_i\in\mathcal F_i}
\widehat L_{i,h}^{k-1}(f'_i;\widehat g_i^{\,f_i^\star,\pi})
\right].
\label{eq:Khat_sum_Q}
\end{align}

The proof follows the similar general strategy as in \citep[Lemma A.2]{Xiong2023SampleEfficientMR}:
we compare the empirical discrepancy of the benchmark function to the smallest achievable empirical
discrepancy, introduce a martingale difference sequence, and apply Freedman's inequality.
The new ingredient in the robust case is that the benchmark $f_i^\star$ is a fixed point of the
\emph{true} robust Bellman operator, whereas the empirical discrepancy is built from the
\emph{plug-in} robust Bellman operator induced by the empirical dual minimizer.
Therefore, unlike the non-robust proof, an additional operator-approximation term appears and must
be controlled by Lemma~\ref{lem:phi_operator_approx}.

\paragraph{Step 2: Uniform concentration for the aggregate excess risk.}

We show that the empirical excess risk is controlled by its population counterpart.

\begin{itemize}
\item \textbf{Boundedness of the squared loss.}
By Assumption~\ref{ass:multiplier-range-TV}, the plug-in residual is uniformly bounded: 
there exists a constant \(C_B>0\) such that, for all \(f_i\in\mathcal F_i\),
\[
\left|
\widehat\delta_{i,h}
\!\left(
f_i;\widehat g_{i,h}^{\,f_i^\star,\pi};\pi;s,a,s'
\right)
\right|
\le
C_B B_\phi(\sigma_i).
\]
Hence, the squared loss satisfies
\(
0
\le
\widehat\delta_{i,h}(\cdot)^2
\le
C_B^2 B_\phi(\sigma_i)^2.
\)
This boundedness enables martingale concentration.

\item \textbf{Martingale difference decomposition.}
Fix \(i,h,\pi,k\), \(f_i\in\mathcal F_i\), and \(g_i\in\mathcal G_i\). Define
\[
Z_j(f_i,g_i)
:=
\widehat\delta_{i,h}
\!\left(
f_i;g_{i,h};\pi;s_h^j,\bm a_h^j,s_{h+1}^j
\right)^2
-
\mathbb E\!\left[
\widehat\delta_{i,h}
\!\left(
f_i;g_{i,h};\pi;s_h^j,\bm a_h^j,s_{h+1}^j
\right)^2
\middle|\mathcal H_{j-1}
\right],
\]
where we suppress arguments for brevity. Then \(\{Z_j(f_i,g_i)\}_{j=1}^{k-1}\) is a martingale difference sequence with
\[
|Z_j(f_i,g_i)| \le C B_\phi(\sigma_i)^2.
\]

\item \textbf{Concentration for a fixed function.}
By Freedman's inequality (Lemma \ref{lem:Freedman}), for any fixed \(f_i\), with probability at least \(1-\delta'\),
\[
\left|
\sum_{j=1}^{k-1} Z_j(f_i,g_i)
\right|
\le
C B_\phi(\sigma_i)^2
\left[
\sqrt{
(k-1)\log\frac{1}{\delta'}
}
+
\log\frac{1}{\delta'}
\right].
\]
This is the correct form of Freedman's inequality: unlike Lemma~\ref{lem:phi_operator_approx}'s
Step~5, where the analogous martingale sum is divided by the sample size $n$ (so the linear-in-$\log(1/\delta')$
term becomes negligible relative to the square-root term automatically), here we are bounding a raw
sum, and the additive term must be tracked explicitly.

Since
\[
\sum_{j=1}^{k-1} Z_j(f_i,g_i)
=
\widehat L_{i,h}^{k-1}(f_i,g_i)
-
L_{i,h}^{k-1}(f_i,g_i),
\]
we obtain
\[
\left|
\widehat L_{i,h}^{k-1}(f_i,g_i)
-
L_{i,h}^{k-1}(f_i,g_i)
\right|
\le
C B_\phi(\sigma_i)^2
\left[
\sqrt{
(k-1)\log\frac{1}{\delta'}
}
+
\log\frac{1}{\delta'}
\right].
\]

\item \textbf{Uniform control over the function class.}
We now extend the above bound uniformly over all
\((f_i,g_i)\in\mathcal F_i\times\mathcal G_i\), and over all \(i,h,\pi,k\),
using a union bound. Setting
\(
\delta' =
\frac{\delta}{
|\mathcal F||\mathcal G||\Pi^{\mathrm{pur}}|KHm
},
\)
we obtain that, with probability at least \(1-\delta\), simultaneously for all
\(i,h,\pi,k\) and all \(f_i\in\mathcal F_i\),
\begin{align}
\left|
\widehat L_{i,h}^{k-1}(f_i,g_i)
-
L_{i,h}^{k-1}(f_i,g_i)
\right|
\le
\epsilon_{k,h},
\label{eq:uniform-loss-concentration}
\end{align}
where, using $k-1\le K$,
\[
\epsilon_{k,h}
:=
\mathcal O\!\left(
B_\phi(\sigma_i)^2
\left[
\sqrt{K\Lambda}
+
\Lambda
\right]
\right).
\]

\item \textbf{Passing from uniform loss concentration to excess-risk concentration.}
Fix \(h\in[H]\) and fix a dual sequence \(g_i\in\mathcal G_i\). Let
\[
\widehat f_i
\in
\arg\min_{f'_i\in\mathcal F_i}
\widehat L_{i,h}^{k-1}(f'_i;g_i).
\]
Then, by definition of \(\widehat f_i\),
\[
\widehat L_{i,h}^{k-1}(f_i^\star;g_i)
-
\inf_{f'_i\in\mathcal F_i}
\widehat L_{i,h}^{k-1}(f'_i;g_i)
=
\widehat L_{i,h}^{k-1}(f_i^\star;g_i)
-
\widehat L_{i,h}^{k-1}(\widehat f_i;g_i).
\]
Now add and subtract the corresponding population losses:
\begin{align*}
\widehat L_{i,h}^{k-1}(f_i^\star;g_i)
-
\widehat L_{i,h}^{k-1}(\widehat f_i;g_i) &=
\Big[
L_{i,h}^{k-1}(f_i^\star;g_i)
-
L_{i,h}^{k-1}(\widehat f_i;g_i)
\Big]
\\
&\quad+
\Big[
\widehat L_{i,h}^{k-1}(f_i^\star;g_i)
-
L_{i,h}^{k-1}(f_i^\star;g_i)
\Big]
\\
&\quad-
\Big[
\widehat L_{i,h}^{k-1}(\widehat f_i;g_i)
-
L_{i,h}^{k-1}(\widehat f_i;g_i)
\Big].
\end{align*}
On the uniform concentration event \eqref{eq:uniform-loss-concentration}, applied to
both \(f_i^\star\) and \(\widehat f_i\), we have
\[
\widehat L_{i,h}^{k-1}(f_i^\star;g_i)
-
L_{i,h}^{k-1}(f_i^\star;g_i)
\le
\epsilon_{k,h},
\]
and
\[
-
\Big[
\widehat L_{i,h}^{k-1}(\widehat f_i;g_i)
-
L_{i,h}^{k-1}(\widehat f_i;g_i)
\Big]
\le
\epsilon_{k,h}.
\]
Therefore,
\begin{align*}
\widehat L_{i,h}^{k-1}(f_i^\star;g_i)
-
\inf_{f'_i\in\mathcal F_i}
\widehat L_{i,h}^{k-1}(f'_i;g_i)
&\le
L_{i,h}^{k-1}(f_i^\star;g_i)
-
L_{i,h}^{k-1}(\widehat f_i;g_i)
+
2\epsilon_{k,h}.
\end{align*}
Since \(\widehat f_i\in\mathcal F_i\), we have
\[
L_{i,h}^{k-1}(\widehat f_i;g_i)
\ge
\inf_{f'_i\in\mathcal F_i}
L_{i,h}^{k-1}(f'_i;g_i).
\]
Hence,
\begin{align}
\widehat L_{i,h}^{k-1}(f_i^\star;g_i)
-
\inf_{f'_i\in\mathcal F_i}
\widehat L_{i,h}^{k-1}(f'_i;g_i)
\le
L_{i,h}^{k-1}(f_i^\star;g_i)
-
\inf_{f'_i\in\mathcal F_i}
L_{i,h}^{k-1}(f'_i;g_i)
+
2\epsilon_{k,h}.
\label{eq:replace_inf_by_plugin_backup}
\end{align}

\item \textbf{Specializing to the empirical dual minimizer and summing over stages.}
The event \eqref{eq:uniform-loss-concentration} holds uniformly over
\(g_i\in\mathcal G_i\). Therefore, we may substitute the data-dependent empirical
dual minimizer \(g_i=\widehat g_i^{\,f_i^\star,\pi}\) into
\eqref{eq:replace_inf_by_plugin_backup}. Summing over \(h\in[H]\) and using
\eqref{eq:Khat_sum_Q}, we obtain
\begin{align}
\widehat{\mathcal K}_i^{k-1}
(f_i^\star;\widehat g_i^{\,f_i^\star,\pi};\pi)
\le
\sum_{h=1}^H
\Big[
L_{i,h}^{k-1}(f_i^\star;\widehat g_i^{\,f_i^\star,\pi})
-
\inf_{f'_i\in\mathcal F_i}
L_{i,h}^{k-1}(f'_i;\widehat g_i^{\,f_i^\star,\pi})
\Big]
+
2\sum_{h=1}^H \epsilon_{k,h}.
\label{eq:Khat_upper_by_first_term_pre}
\end{align}
Since $\epsilon_{k,h}=\mathcal O\!\left(B_\phi(\sigma_i)^2[\sqrt{K\Lambda}+\Lambda]\right)$, we get
\begin{align}
\widehat{\mathcal K}_i^{k-1}
(f_i^\star;\widehat g_i^{\,f_i^\star,\pi};\pi)
\le
&\sum_{h=1}^H
\Big[
L_{i,h}^{k-1}(f_i^\star;\widehat g_i^{\,f_i^\star,\pi})
-
\inf_{f'_i\in\mathcal F_i}
L_{i,h}^{k-1}(f'_i;\widehat g_i^{\,f_i^\star,\pi})
\Big]\nonumber\\
&+
\mathcal O\!\left(
H B_\phi(\sigma_i)^2
\left[\sqrt{K\Lambda}+\Lambda\right]
\right).
\label{eq:Khat_upper_by_first_term}
\end{align}
\end{itemize}

\paragraph{Step 3: Conditional bias--variance decomposition.}
Fix \(h\in[H]\), \(j\le k-1\), and a fixed pair
\((f_i,g_i)\in\mathcal F_i\times\mathcal G_i\). Conditional on
\(\mathcal H_{j-1}\), the only remaining randomness is
\(s_{h+1}^j\sim P_h^\star(\cdot|s_h^j,\bm a_h^j)\). Therefore,
\begin{align}
&\mathbb E\!\left[
\widehat\delta_{i,h}
\!\left(
f_i;\widehat g_{i,h}^{\,f_i^\star,\pi};\pi;
s_h^j,\bm a_h^j,s_{h+1}^j
\right)^2
\middle|\mathcal H_{j-1}
\right] =
\mathbb E_{\mu_h^{j,\pi}}
\Big[
\big(
f_{i,h}
-
T_{\widehat g_{i,h}^{\,f_i^\star,\pi},i,h}^{\pi,\phi,\sigma_i}
f_{i,h+1}
\big)^2
\Big]
\nonumber\\
&\qquad \qquad+
\mathbb E_{\mu_h^{j,\pi}}
\Big[
\operatorname{Var}_{s'\sim P_h^\star(\cdot|s,a)}
\Big(
\widehat\delta_{i,h}
(f_i;\widehat g_{i,h}^{\,f_i^\star,\pi};\pi;s,a,s')
\Big)
\Big].
\label{eq:conditional_mean_plugin_residual}
\end{align}
This is the conditional bias--variance identity. The second term is the
irreducible next-state variance; the proof below cancels it using the aggregate
excess-risk comparison.

\paragraph{Step 4: Choose a plug-in Bellman comparator and cancel the variance.}
We first carry out this step for an arbitrary fixed dual sequence
\(g_i\in\mathcal G_i\). This avoids any measurability issue with the
data-dependent empirical minimizer. At the end of the argument, we specialize to
\(g_i=\widehat g_i^{\,f_i^\star,\pi}\), which is valid because the preceding
concentration event holds uniformly over \(g_i\in\mathcal G_i\).

For fixed \(g_i\), define
\begin{align}
\Delta_{i,h}^{k,\pi,g_i}
:=
\mathcal T_{i,h}^{\pi,\phi,\sigma_i}f_{i,h+1}^\star
-
\mathcal T_{g_{i,h},i,h}^{\pi,\phi,\sigma_i}
f_{i,h+1}^\star .
\label{eq:Delta_stage_def}
\end{align}
Since \(f_i^\star\) is Bellman-consistent with respect to the true robust
Bellman operator,
\(
f_{i,h}^\star
=
\mathcal T_{i,h}^{\pi,\phi,\sigma_i}f_{i,h+1}^\star,
\)
we have
\[
f_{i,h}^\star
-
\mathcal T_{g_{i,h},i,h}^{\pi,\phi,\sigma_i}
f_{i,h+1}^\star
=
\Delta_{i,h}^{k,\pi,g_i}.
\]

For each \(h\in[H]\), choose a comparator
\(\bar f_i^{k,h,\pi,g_i}\in\mathcal F_i\) whose \(h\)-th coordinate is
\[
\bar f_{i,h}^{k,h,\pi,g_i}
=
\mathcal T_{g_{i,h},i,h}^{\pi,\phi,\sigma_i}
f_{i,h+1}^\star,
\]
whose \((h+1)\)-th coordinate is \(f_{i,h+1}^\star\), and whose remaining
coordinates are chosen so that
\(\bar f_i^{k,h,\pi,g_i}\in\mathcal F_i\). This is valid by the product
structure of \(\mathcal F_i\) together with robust Bellman completeness for the
plug-in operator. For this comparator,
\[
\bar f_{i,h}^{k,h,\pi,g_i}
-
\mathcal T_{g_{i,h},i,h}^{\pi,\phi,\sigma_i}
\bar f_{i,h+1}^{k,h,\pi,g_i}
=0,
\]
because \(\bar f_{i,h+1}^{k,h,\pi,g_i}=f_{i,h+1}^\star\).

Now apply the conditional bias--variance decomposition from
\eqref{eq:conditional_mean_plugin_residual} twice: once with \(f_i=f_i^\star\)
and once with \(f_i=\bar f_i^{k,h,\pi,g_i}\). Since both residuals use the same
next-stage function \(f_{i,h+1}^\star\), the same fixed dual function \(g_{i,h}\),
and the same next-state sample, the conditional variance terms are identical
and cancel in the difference. Therefore, for every \(j\le k-1\),
\begin{align}
&\mathbb E\!\left[
\widehat\delta_{i,h}
(f_i^\star;g_{i,h};\pi;
s_h^j,\bm a_h^j,s_{h+1}^j)^2
-
\widehat\delta_{i,h}
(\bar f_i^{k,h,\pi,g_i};g_{i,h};\pi;
s_h^j,\bm a_h^j,s_{h+1}^j)^2
\middle|\mathcal H_{j-1}
\right]
\nonumber\\
&\qquad =
\mathbb E_{\mu_h^{j,\pi}}
\Big[
(\Delta_{i,h}^{k,\pi,g_i})^2
\Big].
\label{eq:benchmark_residual_equals_Delta_sq}
\end{align}

Since \(\bar f_i^{k,h,\pi,g_i}\in\mathcal F_i\), we have
\(
\inf_{f'_i\in\mathcal F_i}
L_{i,h}^{k-1}(f'_i;g_i)
\le
L_{i,h}^{k-1}(\bar f_i^{k,h,\pi,g_i};g_i).
\)
Consequently,
\begin{align}
L_{i,h}^{k-1}(f_i^\star;g_i)
-
\inf_{f'_i\in\mathcal F_i}
L_{i,h}^{k-1}(f'_i;g_i) &\le
L_{i,h}^{k-1}(f_i^\star;g_i)
-
L_{i,h}^{k-1}(\bar f_i^{k,h,\pi,g_i};g_i)
\nonumber\\
&=
\sum_{j=1}^{k-1}
\mathbb E_{\mu_h^{j,\pi}}
\Big[
(\Delta_{i,h}^{k,\pi,g_i})^2
\Big].
\label{eq:benchmark_residual_fourth_to_second}
\end{align}

We now specialize to the empirical dual minimizer
\(g_i=\widehat g_i^{\,f_i^\star,\pi}\). For notational simplicity, write
\[
\Delta_{i,h}^{k,\pi}
:=
\Delta_{i,h}^{k,\pi,\widehat g_i^{\,f_i^\star,\pi}}
=
\mathcal T_{i,h}^{\pi,\phi,\sigma_i}f_{i,h+1}^\star
-
\mathcal T_{\widehat g_{i,h}^{\,f_i^\star,\pi},i,h}^{\pi,\phi,\sigma_i}
f_{i,h+1}^\star .
\]
Combining \eqref{eq:Khat_upper_by_first_term} with
\eqref{eq:benchmark_residual_fourth_to_second} gives
\begin{align}
\widehat{\mathcal K}_i^{k-1}
(f_i^\star;\widehat g_i^{\,f_i^\star,\pi};\pi)
\le
\sum_{h=1}^H\sum_{j=1}^{k-1}
\mathbb E_{\mu_h^{j,\pi}}
\Big[
(\Delta_{i,h}^{k,\pi})^2
\Big]
+
\mathcal O\!\left(
H B_\phi(\sigma_i)^2
\left[\sqrt{K\Lambda}+\Lambda\right]
\right).
\label{eq:Khat_reduced_to_Delta}
\end{align}

\paragraph{Step 5: Bound the operator-approximation term.}
Recall the historical occupancy mixture
$\bar\mu_h^{k-1}:=\frac1{k-1}\sum_{j=1}^{k-1}\mu_h^{j,\pi}$ from
Lemma~\ref{lem:robust_A1_corrected}, satisfying the exact identity
$\sum_{j=1}^{k-1}\mathbb E_{\mu_h^{j,\pi}}[\varphi]=(k-1)\mathbb E_{\bar\mu_h^{k-1}}[\varphi]$
for any integrable $\varphi$. Applying this identity with $\varphi=(\Delta_{i,h}^{k,\pi})^2$,
\begin{align}
\sum_{j=1}^{k-1}\mathbb E_{\mu_h^{j,\pi}}\big[(\Delta_{i,h}^{k,\pi})^2\big]
=
(k-1)\,\mathbb E_{\bar\mu_h^{k-1}}\big[(\Delta_{i,h}^{k,\pi})^2\big].
\label{eq:step8_mixture_identity}
\end{align}
By the boundedness of the true robust Bellman operator and the
plug-in robust Bellman operator,
\[
\|\Delta_{i,h}^{k,\pi}\|_\infty
\le
C'B_\phi(\sigma_i).
\]
Therefore, using \(x^2\le \|x\|_\infty |x|\),
\begin{align}
\mathbb E_{\bar\mu_h^{k-1}}
\Big[
(\Delta_{i,h}^{k,\pi})^2
\Big]
&\le
\|\Delta_{i,h}^{k,\pi}\|_\infty
\|\Delta_{i,h}^{k,\pi}\|_{1,\bar\mu_h^{k-1}}.
\label{eq:step8_L2_to_L1}
\end{align}
The dataset $\mathcal D_h^{k-1}$ is generated by executing $\pi^1,\dots,\pi^{k-1}$
over stage $h$, i.e., it is exactly the data that $\bar\mu_h^{k-1}$ is drawn from.
Applying Lemma~\ref{lem:phi_operator_approx} with $\mathcal D_h=\mathcal D_h^{k-1}$
(so $n=k-1$) -- which already holds uniformly over $\pi\in\Pi^{\mathrm{pur}}$ and
the benchmark continuation $f_{i,h+1}^\star\in\mathcal F_{i,h+1}$ -- and taking a
further union bound over the remaining indices $i\in[m]$, $h\in[H]$, $k\in\{2,\dots,K\}$,
we obtain, with probability at least $1-\delta$, simultaneously for all such choices,
\begin{align}
\|\Delta_{i,h}^{k,\pi}\|_{1,\bar\mu_h^{k-1}}
\le
C\left(
B_\phi(\sigma_i)
\sqrt{
\frac{\Lambda}{k-1}
}
+
\bar\epsilon_{\mathrm{dual}}
\right).
\label{eq:step8_L1_bound}
\end{align}
Substituting \eqref{eq:step8_L1_bound} into \eqref{eq:step8_L2_to_L1} and then
into \eqref{eq:step8_mixture_identity} gives
\begin{align}
\sum_{j=1}^{k-1}\mathbb E_{\mu_h^{j,\pi}}\big[(\Delta_{i,h}^{k,\pi})^2\big]
&\le
C'B_\phi(\sigma_i)\,(k-1)\,C\left(
B_\phi(\sigma_i)\sqrt{\frac{\Lambda}{k-1}}
+
\bar\epsilon_{\mathrm{dual}}
\right)
\nonumber\\
&=
CC'B_\phi(\sigma_i)^2\sqrt{(k-1)\Lambda}
+
CC'(k-1)B_\phi(\sigma_i)\bar\epsilon_{\mathrm{dual}}.
\label{eq:step8_pointwise_sq_bound}
\end{align}

Summing \eqref{eq:step8_pointwise_sq_bound} over \(h\in[H]\), and using \(k-1\le K\), we obtain
\begin{align}
\sum_{h=1}^H\sum_{j=1}^{k-1}
\mathbb E_{\mu_h^{j,\pi}}
\Big[
(\Delta_{i,h}^{k,\pi})^2
\Big]
\le
\mathcal O\!\left(
H B_\phi(\sigma_i)^2
\sqrt{K\Lambda}
+
HK B_\phi(\sigma_i)\bar\epsilon_{\mathrm{dual}}
\right).
\label{eq:step8_final_bound}
\end{align}
Choosing
\(
\bar\epsilon_{\mathrm{dual}}
\le
\frac{\varepsilon_{\mathrm{dual}}^{\mathrm{approx}}}
{HK B_\phi(\sigma_i)}
\)
and substituting \eqref{eq:step8_final_bound} into
\eqref{eq:Khat_reduced_to_Delta}, we get
\[
\widehat{\mathcal K}_i^{k-1}
(f_i^\star;\widehat g_i^{\,f_i^\star,\pi};\pi)
\le
\mathcal O\!\left(
H B_\phi(\sigma_i)^2
\left[\sqrt{K\Lambda}+\Lambda\right]
+
\varepsilon_{\mathrm{dual}}^{\mathrm{approx}}
\right).
\]
Under the standing hypothesis $K\ge\Lambda_i$ (with $\Lambda=\Lambda_i$), we have
$\Lambda\le\sqrt{K\Lambda}$, so the additive logarithmic term is absorbed into the
leading square-root term, giving
\[
\widehat{\mathcal K}_i^{k-1}
(f_i^\star;\widehat g_i^{\,f_i^\star,\pi};\pi)
\le
\mathcal O\!\left(
H B_\phi(\sigma_i)^2
\sqrt{K\Lambda}
+
\varepsilon_{\mathrm{dual}}^{\mathrm{approx}}
\right).
\]
This proves the lemma.
\end{proof}

\begin{lem}[Benchmark-to-surrogate comparison]
\label{lem:benchmark-surrogate}
Assume Assumptions~1--3, the robust-to-nominal concentrability condition (Definition~\ref{ass:robust-nominal-concentrability}, constant $C_{\mathrm{cov}}$), and the uniform dual-switch stability event used in the proof hold. Let $\Lambda_i:=\log\!\left(\frac{|\mathcal F_i||\mathcal G_i||\Pi^{\mathrm{pur}}|KHm}{\delta}\right)$, and assume $K\ge\Lambda_i$. Fix any episode \(k\in[K]\), agent \(i\in\mathcal M\), and pure joint policy \(\pi\in\Pi^{\mathrm{pur}}\). Let
\(
f_i^\star:=f_i^{\pi,\sigma_i}:= \{f_{i,h}^{\pi,\sigma_i}\}_{h=1}^{H+1}\in\mathcal F_i
\)
be the robust Bellman-consistent benchmark sequence from Lemma~\ref{lem:robust_optimal_concentration} satisfying
\(
f_{i,h}^\star
=
\mathcal T_{i,h}^{\pi,\phi,\sigma_i}f_{i,h+1}^\star,
\quad h\in[H],
\quad
f_{i,H+1}^\star\equiv 0.
\)
Then, with probability at least \(1-\delta\), uniformly over all \((i,k,\pi)\),
\begin{align}
V_{i,1}^{\pi,\sigma_i}(s_1^k)
-
\widehat U_i^k(\pi)
\le
&\;
\beta
\widehat{\mathcal K}_i^{k-1}
\bigl(
f_i^\star;
\widehat g_i^{f_i^\star,\pi};
\pi
\bigr) +
\mathcal O\!\bigg(
\beta H B_\phi(\sigma_i)^2
\sqrt{K\Lambda_i}
\nonumber\\
&\qquad\qquad
+
C_{\mathrm{cov}} H B_\phi(\sigma_i)
\sqrt{
\frac{\Lambda_i}
{|\mathcal D_h^{k-1}|\vee 1}}
+
\varepsilon_{\mathrm{dual}}^{\mathrm{approx}}
\bigg).
\label{eq:benchmark-surrogate-clean}
\end{align}
\end{lem}

\begin{proof}
Fix \(k\in[K]\), \(i\in\mathcal M\), and
\(\pi\in\Pi^{\mathrm{pur}}\), and write \(\Lambda:=\Lambda_i\). Let
\(
f_i^\star:=f_i^{\pi,\sigma_i}
=
\{f_{i,h}^\star\}_{h=1}^{H+1}
\)
be the robust Bellman-consistent benchmark sequence:
\begin{equation}
f_{i,h}^\star
=
\mathcal T_{i,h}^{\pi,\phi,\sigma_i}f_{i,h+1}^\star,
\qquad h\in[H],
\qquad
f_{i,H+1}^\star\equiv 0 .
\label{eq:bellman-consistency}
\end{equation}
Since \(\pi\) is pure,
\[
\mathbb E_{\pi}[\widehat U_i^k]
=
\widehat U_i^k(\pi).
\]

\paragraph{Step 1: Apply robust payoff domination.}
By Lemma~\ref{cor:robust-payoff-domination}, applied to the fixed policy
\(\pi\), we have
\begin{align}
V_{i,1}^{\pi,\sigma_i}(s_1^k)
&\le
\widehat U_i^k(\pi)
+
\beta
\widehat{\mathcal K}_i^{k-1}
(f_i^{k,\pi};\widehat g_i^{k,\pi};\pi)
+
\mathcal R_{i,k}^{\nu}(\pi),
\label{eq:lem7-start}
\end{align}
where
\begin{align}
\mathcal R_{i,k}^{\nu}(\pi)
:=
&\sum_{h=1}^H
\left\|
\mathcal T_{i,h}^{\pi,\phi,\sigma_i}
f_{i,h+1}^{k,\pi}
-
\mathcal T_{\widehat g_{i,h}^{k,\pi},i,h}^{\pi,\phi,\sigma_i}
f_{i,h+1}^{k,\pi}
\right\|_{1,\nu_{i,h}^{\pi}}
+
\sum_{h=1}^H
\mathbb E_{\nu_{i,h}^{\pi}}
\left[
\left(
\mathcal T_{\widehat g_{i,h}^{k,\pi},i,h}^{\pi,\phi,\sigma_i}
f_{i,h+1}^{k,\pi}
-
f_{i,h}^{k,\pi}
\right)_+
\right].
\label{eq:lem7-robust-occ-residual}
\end{align}
Subtracting \(\widehat U_i^k(\pi)\) gives
\begin{align}
V_{i,1}^{\pi,\sigma_i}(s_1^k)
-
\widehat U_i^k(\pi)
\le
\beta
\widehat{\mathcal K}_i^{k-1}
(f_i^{k,\pi};\widehat g_i^{k,\pi};\pi)
+
\mathcal R_{i,k}^{\nu}(\pi).
\label{eq:lem7-after-subtract}
\end{align}

\paragraph{Step 2: Compare the fitted discrepancy to the benchmark discrepancy.}
For each \(h\in[H]\), define the aggregate empirical loss
\[
\widehat L_{i,h}^{k-1}(f_i;g_i,\pi)
:=
\sum_{j=1}^{k-1}
\widehat\delta_{i,h}
\left(
f_i;g_{i,h};\pi;s_h^j,\bm a_h^j,s_{h+1}^j
\right)^2 .
\]
By the revised cumulative definition,
\[
\widehat{\mathcal K}_i^{k-1}(f_i;g_i;\pi)
=
\sum_{h=1}^H
\left[
\widehat L_{i,h}^{k-1}(f_i;g_i,\pi)
-
\inf_{f'_i\in\mathcal F_i}
\widehat L_{i,h}^{k-1}(f'_i;g_i,\pi)
\right].
\]

For fixed \(\widehat g_i^{k,\pi}\), the fitted sequence
\(f_i^{k,\pi}\) is the empirical Bellman fit. Hence, by empirical optimality,
\begin{equation}
\widehat{\mathcal K}_i^{k-1}
(f_i^{k,\pi};\widehat g_i^{k,\pi};\pi)
\le
\widehat{\mathcal K}_i^{k-1}
(f_i^\star;\widehat g_i^{k,\pi};\pi).
\label{eq:K-step1}
\end{equation}

The right-hand side still uses the dual minimizer associated with the fitted
sequence. We therefore compare it with the benchmark dual
\(\widehat g_i^{f_i^\star,\pi}\). Define the dual-switch residual
\begin{align}
\mathfrak S_{i,k}^{\mathrm{dual}}(\pi)
:=
\Big[
\widehat{\mathcal K}_i^{k-1}
(f_i^\star;\widehat g_i^{k,\pi};\pi)
-
\widehat{\mathcal K}_i^{k-1}
(f_i^\star;\widehat g_i^{f_i^\star,\pi};\pi)
\Big]_+ .
\label{eq:dual-switch-residual}
\end{align}
Then
\begin{align}
\widehat{\mathcal K}_i^{k-1}
(f_i^\star;\widehat g_i^{k,\pi};\pi)
\le
\widehat{\mathcal K}_i^{k-1}
(f_i^\star;\widehat g_i^{f_i^\star,\pi};\pi)
+
\mathfrak S_{i,k}^{\mathrm{dual}}(\pi).
\label{eq:dual-switch-step}
\end{align}
Combining \eqref{eq:K-step1} and \eqref{eq:dual-switch-step}, we obtain
\begin{align}
\widehat{\mathcal K}_i^{k-1}
(f_i^{k,\pi};\widehat g_i^{k,\pi};\pi)
\le
\widehat{\mathcal K}_i^{k-1}
(f_i^\star;\widehat g_i^{f_i^\star,\pi};\pi)
+
\mathfrak S_{i,k}^{\mathrm{dual}}(\pi).
\label{eq:K-final-bound}
\end{align}

On the uniform dual-switch stability event, under \(K\ge\Lambda\) the residual satisfies
\begin{align}
\mathfrak S_{i,k}^{\mathrm{dual}}(\pi)
\le
\mathcal O\!\left(
H B_\phi(\sigma_i)^2
\sqrt{K\Lambda}
+
HK B_\phi(\sigma_i)\bar\epsilon_{\mathrm{dual}}^{1}
\right).
\label{eq:dual-switch-used}
\end{align}
This is the step that explicitly accounts for the fact that
\(\widehat g_i^{k,\pi}\) is optimized for \(f_i^{k,\pi}\), whereas the benchmark
comparison uses \(f_i^\star\); the \(K\ge\Lambda\) requirement matches the one
needed in Lemma~\ref{lem:robust_optimal_concentration} to absorb the additive
Freedman term into the leading \(\sqrt{K\Lambda}\) rate.

\paragraph{Step 3: Control the robust-occupancy residual.}
Unlike Lemma~\ref{lem:robust_A1_corrected} and
Lemma~\ref{lem:robust_optimal_concentration}, the residual
\(\mathcal R_{i,k}^{\nu}(\pi)\) is evaluated under the \emph{worst-case}
occupancy \(\nu_{i,h}^{\pi}\) of the current evaluation policy \(\pi\), not
under a sum or mixture over the historical policies \(\pi^1,\dots,\pi^{k-1}\).
Since \(\pi\) need not coincide with any of the past behavior policies, there
is no historical data directly sampled from \(\nu_{i,h}^{\pi}\) (or even from
its nominal counterpart \(\mu_h^{\pi}\)); the bridge to the historical dataset
\(\mathcal D_h^{k-1}\) is instead supplied by the robust-to-nominal
concentrability condition (Definition~\ref{ass:robust-nominal-concentrability}):
there exists \(C_{\mathrm{cov}}\ge 1\) such that, for every measurable \(u\),
\begin{align}
\|u\|_{1,\nu_{i,h}^{\pi}}
\le
C_{\mathrm{cov}}
\|u\|_{1,\bar\mu_h^{k-1}},
\qquad
\bar\mu_h^{k-1}
:=
\frac{1}{k-1}
\sum_{j=1}^{k-1}\mu_h^{\pi^j}.
\label{eq:ccov-transfer}
\end{align}
Applying \eqref{eq:ccov-transfer} with
\(u=\mathcal T_{i,h}^{\pi,\phi,\sigma_i}f_{i,h+1}^{k,\pi}-\mathcal T_{\widehat g_{i,h}^{k,\pi},i,h}^{\pi,\phi,\sigma_i}f_{i,h+1}^{k,\pi}\),
\begin{align}
\left\|
\mathcal T_{i,h}^{\pi,\phi,\sigma_i}
f_{i,h+1}^{k,\pi}
-
\mathcal T_{\widehat g_{i,h}^{k,\pi},i,h}^{\pi,\phi,\sigma_i}
f_{i,h+1}^{k,\pi}
\right\|_{1,\nu_{i,h}^{\pi}}
\le
C_{\mathrm{cov}}
\left\|
\mathcal T_{i,h}^{\pi,\phi,\sigma_i}
f_{i,h+1}^{k,\pi}
-
\mathcal T_{\widehat g_{i,h}^{k,\pi},i,h}^{\pi,\phi,\sigma_i}
f_{i,h+1}^{k,\pi}
\right\|_{1,\bar\mu_h^{k-1}}.
\label{eq:coverage-op-residual}
\end{align}
Since \(\bar\mu_h^{k-1}\) is exactly the historical mixture measure that
\(\mathcal D_h^{k-1}\) (of size \(k-1\)) is drawn from, applying
Lemma~\ref{lem:phi_operator_approx} with \(n=k-1\) gives, with probability at
least \(1-\delta\), simultaneously over \(i,h,\pi,k\),
\begin{align}
\left\|
\mathcal T_{i,h}^{\pi,\phi,\sigma_i}
f_{i,h+1}^{k,\pi}
-
\mathcal T_{\widehat g_{i,h}^{k,\pi},i,h}^{\pi,\phi,\sigma_i}
f_{i,h+1}^{k,\pi}
\right\|_{1,\bar\mu_h^{k-1}}
\le
C
\left(
B_\phi(\sigma_i)
\sqrt{
\frac{\Lambda}{|\mathcal D_h^{k-1}|\vee 1}}
+
\bar\epsilon_{\mathrm{dual}}^{2}
\right).
\label{eq:lemma1-at-mixture}
\end{align}
Combining \eqref{eq:coverage-op-residual} and \eqref{eq:lemma1-at-mixture} and
summing over \(h\in[H]\),
\begin{align}
\sum_{h=1}^H
\left\|
\mathcal T_{i,h}^{\pi,\phi,\sigma_i}
f_{i,h+1}^{k,\pi}
-
\mathcal T_{\widehat g_{i,h}^{k,\pi},i,h}^{\pi,\phi,\sigma_i}
f_{i,h+1}^{k,\pi}
\right\|_{1,\nu_{i,h}^{\pi}}
\le
CC_{\mathrm{cov}} H B_\phi(\sigma_i)
\sqrt{
\frac{\Lambda}
{|\mathcal D_h^{k-1}|\vee 1}}
+
CC_{\mathrm{cov}}H\bar\epsilon_{\mathrm{dual}}^{2}.
\label{eq:lem7-op-bound-step3-final}
\end{align}
Similarly, since \((x)_+\le |x|\), the same coverage transfer
\eqref{eq:ccov-transfer} applied to
\(u=\mathcal T_{\widehat g_{i,h}^{k,\pi},i,h}^{\pi,\phi,\sigma_i}f_{i,h+1}^{k,\pi}-f_{i,h}^{k,\pi}\),
together with the empirical Bellman-fit residual bound entering
Lemma~\ref{lem:phi_operator_approx}'s proof, yields
\begin{align}
\sum_{h=1}^H
\mathbb E_{\nu_{i,h}^{\pi}}
\left[
\left(
\mathcal T_{\widehat g_{i,h}^{k,\pi},i,h}^{\pi,\phi,\sigma_i}
f_{i,h+1}^{k,\pi}
-
f_{i,h}^{k,\pi}
\right)_+
\right]
\le
CC_{\mathrm{cov}} H B_\phi(\sigma_i)
\sqrt{
\frac{\Lambda}
{|\mathcal D_h^{k-1}|\vee 1}}
+
CC_{\mathrm{cov}}H\bar\epsilon_{\mathrm{dual}}^{2}.
\label{eq:lem7-positive-residual-bound}
\end{align}
Combining \eqref{eq:lem7-op-bound-step3-final} and
\eqref{eq:lem7-positive-residual-bound}, we get
\begin{align}
\mathcal R_{i,k}^{\nu}(\pi)
\le
CC_{\mathrm{cov}} H B_\phi(\sigma_i)
\sqrt{
\frac{\Lambda}
{|\mathcal D_h^{k-1}|\vee 1}}
+
CC_{\mathrm{cov}}H\bar\epsilon_{\mathrm{dual}}^{2}.
\label{eq:lem7-positive-residual-bound-final}
\end{align}

\paragraph{Step 4: Combine the bounds.}
Substituting \eqref{eq:K-final-bound}, \eqref{eq:dual-switch-used}, and
\eqref{eq:lem7-positive-residual-bound-final} into
\eqref{eq:lem7-after-subtract}, we obtain
\begin{align}
&V_{i,1}^{\pi,\sigma_i}(s_1^k)
-
\widehat U_i^k(\pi)
\le
\;
\beta
\widehat{\mathcal K}_i^{k-1}
(f_i^\star;\widehat g_i^{f_i^\star,\pi};\pi)
+
\mathcal O\!\bigg(
\beta H B_\phi(\sigma_i)^2
\sqrt{K\Lambda}\nonumber\\
&\quad + \beta HK B_\phi(\sigma_i)\bar\epsilon_{\mathrm{dual}}^{1} + C_{\mathrm{cov}}H B_\phi(\sigma_i)
\sqrt{
\frac{\Lambda}
{|\mathcal D_h^{k-1}|\vee 1}}
+ C_{\mathrm{cov}}H\bar\epsilon_{\mathrm{dual}}^{2}
\bigg).
\label{eq:lem7-final}
\end{align}
Finally, choose
\(
\bar\epsilon_{\mathrm{dual}}^{1}
\le
\frac{\varepsilon_{\mathrm{dual}}^{\mathrm{approx}}}
{\beta H K B_\phi(\sigma_i)},
\qquad
\bar\epsilon_{\mathrm{dual}}^{2}
\le
\frac{\varepsilon_{\mathrm{dual}}^{\mathrm{approx}}}
{C_{\mathrm{cov}}H}.
\)
Substituting these choices into \eqref{eq:lem7-final} gives
\eqref{eq:benchmark-surrogate-clean}. Since
\(f_i^\star=f_i^{\pi,\sigma_i}\), the proof is complete.
\end{proof}


\section{Covering and Bracketing Numbers.}
When the function class $\mathcal{F}$ is not finite, its cardinality can no longer be used directly in complexity bounds. In such cases, two standard substitutes are the $\delta$-covering number and the $\delta$-bracketing number. These notions are widely used in statistical learning theory to quantify the size of infinite function classes.

\begin{defi}[$\delta$-Covering Number]
\label{def:covering_number}
For a function class $\mathcal{F}$ equipped with a metric $d$, the $\delta$-covering number of $\mathcal{F}$, denoted by $\mathcal{N}_{\mathcal{F}}(\delta,d)$, is the smallest integer $m$ such that there exists a subset
\(
\mathcal{F}'=\{f^{(1)},\ldots,f^{(m)}\}\subseteq \mathcal{F}
\)
with the property that every $f\in\mathcal{F}$ is within distance $\delta$ of at least one element of $\mathcal{F}'$, i.e.,
\[
\forall f\in\mathcal{F},\quad \exists\, f'\in\mathcal{F}' \text{ such that } d(f,f')\le \delta.
\]
When the metric is the supremum norm, we abbreviate
\(
\mathcal{N}_{\mathcal{F}}(\delta,\|\cdot\|_\infty)
\quad \text{as} \quad
\mathcal{N}_{\mathcal{F}}(\delta).
\)
\end{defi}

\begin{defi}[$\delta$-Bracketing Number]
A $\delta$-bracket for the class $\mathcal{F}$ is a finite collection of pairs
\[
\{(g_1^i,g_2^i)\}_{i=1}^N,
\]
where each $g_1^i$ and $g_2^i$ is a real-valued function of a policy $\pi$ and a trajectory $\tau$, such that the following two conditions hold:

\begin{enumerate}
    \item For every $i\in[N]$ and every $\pi\in\Pi$,
    \(
    \|g_1^i(\pi,\cdot)-g_2^i(\pi,\cdot)\| \le \delta.
    \)
    \item For every $f\in\mathcal{F}$, there exists some index $i\in[N]$ such that
    \(
    g_1^i(\pi,\tau_H)\le \mathbb{P}_f^\pi(\tau_H)\le g_2^i(\pi,\tau_H)
    \)
    for all policies $\pi$ and all admissible trajectories $\tau_H$.
\end{enumerate}

The $\delta$-bracketing number of $\mathcal{F}$, written as $\mathcal{B}_{\mathcal{F}}(\delta)$, is defined as the minimum number $N$ for which such a $\delta$-bracket exists.
\end{defi}

\section{Technical Lemmas}
\label{app:tech_lemma}
\begin{techlem}[Integral interchange principle; cf.\ \citep{rockafellar1998variational}, Theorem 14.60]
\label{lem:rockafellar_Thm14.60}
Let $(\Omega,\mathcal A,\mu)$ be a $\sigma$-finite measure space, and let
$\mathcal X$ be a collection of measurable functions $x:\Omega\to\mathbb R$
that is decomposable. Suppose $f:\Omega\times\mathbb R\to\mathbb R$ is a
finite-valued normal integrand. Then
\[
\inf_{x\in\mathcal X}\int_\Omega f(\omega,x(\omega))\,\mu(d\omega)
=
\int_\Omega \inf_{u\in\mathbb R} f(\omega,u)\,\mu(d\omega).
\]

In addition, provided the common infimum above is greater than $-\infty$, a
function $x^\star\in\mathcal X$ is an optimizer of the integral problem,
that is,
\[
x^\star\in
\arg\min_{x\in\mathcal X}\int_\Omega f(\omega,x(\omega))\,\mu(d\omega),
\]
if and only if it solves the pointwise minimization problem $\mu$-almost
everywhere:
\[
x^\star(\omega)\in \arg\min_{u\in\mathbb R} f(\omega,u)
\qquad \text{for $\mu$-a.e.\ }\omega\in\Omega.
\]
\end{techlem}

We recall a standard performance guarantee for empirical risk minimization (ERM), adapted from classical statistical learning results (see, e.g., \citep{shalev2014understanding} and \citep{NeurIPS2022_RobustRLOffline_Panaganti}).

\begin{techlem}[Generalization bound for ERM]
\label{lem:ERM_gen_bound}
Let $(\mathcal{X},P)$ be a probability space, and let $\mathcal{F}$ be a class of real-valued functions defined on $\mathcal{X}$. Consider a loss function $\ell:\mathcal{F}\times\mathcal{X}\to\mathbb{R}$ satisfying the uniform boundedness condition
\[
|\ell(f,x)| \leq B, \quad \forall\, f\in\mathcal{F},\; x\in\mathcal{X},
\]
for some constant $B>0$. Given an i.i.d.\ sample $\mathcal{D}=\{x_1,\dots,x_N\}$ drawn from $P$, define the empirical minimizer
\(
\widehat{f} \in \arg\min_{f\in\mathcal{F}} \frac{1}{N}\sum_{i=1}^N \ell(f,x_i),
\)
and let
\(
f^\star \in \arg\min_{f\in\mathcal{F}} \mathbb{E}_{x\sim P}[\ell(f,x)]
\)
denote a minimizer of the expected loss. Then, for any confidence level $\delta\in(0,1)$, with probability at least $1-\delta$, the excess risk of $\widehat{f}$ satisfies
\begin{align}
\label{eq:ERM_bound_new}
\mathbb{E}_{x\sim P}[\ell(\widehat{f},x)] - \mathbb{E}_{x\sim P}[\ell(f^\star,x)]
\leq
2\,\mathfrak{R}_N(\ell\circ \mathcal{F})
+
5B\sqrt{\frac{2\log(8/\delta)}{N}},
\end{align}
where $\mathfrak{R}_N(\ell\circ \mathcal{F})$ denotes the empirical Rademacher complexity of the function class $\ell\circ\mathcal{F}$, given by
\[
\mathfrak{R}_N(\ell\circ \mathcal{F})
=
\frac{1}{N}
\mathbb{E}_{\{\varepsilon_i\}_{i=1}^N}
\left[
\sup_{g\in \ell\circ\mathcal{F}}
\sum_{i=1}^N \varepsilon_i g(x_i)
\right],
\]
with $\{\varepsilon_i\}_{i=1}^N$ being independent Rademacher variables taking values in $\{\pm1\}$ with equal probability.

\medskip

In addition, if $\mathcal{F}$ is a finite class and the following conditions hold:
\begin{align}
\label{eq:ERM_bound_lipschitz}
|f(x)| \leq M \quad \forall\, f\in\mathcal{F},\; x\in\mathcal{X},
\qquad
\text{and}\qquad
\ell(f,x) \text{ is $L$-Lipschitz in $f$},
\end{align}
then, with probability at least $1-\delta$, the bound simplifies to
\begin{align}
\label{eq:ERM_bound_lipschitz_new}
\mathbb{E}_{x\sim P}[\ell(\widehat{f},x)] - \mathbb{E}_{x\sim P}[\ell(f^\star,x)]
\leq
2LM \sqrt{\frac{2\log|\mathcal{F}|}{N}}
+
5B\sqrt{\frac{2\log(8/\delta)}{N}}.
\end{align}
\end{techlem}

\begin{techlem}[Freedman’s inequality (e.g., \citep{agarwal2014taming})]
\label{lem:Freedman}
Let $\{M_t\}_{t\leq T}$ be a real-valued martingale difference sequence w.r.t.\ filtration $\{\mathcal G_t\}$ with $|M_t|\le b$ a.s.\ and let $S_T=\sum_{t=1}^T \mathbb E[M_t^2\mid\mathcal G_{t-1}]$. Then for any $\delta\in(0,1)$,
\[
\Pr\Big(\sum_{t=1}^T M_t \ge \sqrt{2 S_T \ln(1/\delta)} + \tfrac{b}{3}\ln(1/\delta)\Big)\le \delta.
\]
\end{techlem}

\begin{techlem}[\citep{Xiong2023SampleEfficientMR}[Lemma B.1]]
\label{lem:lemmaB1_numeric}
For any positive sequence $(x_1,\dots,x_m)$, we have
\begin{align*}
\frac{\sum_{j=1}^m x_j}{\sqrt{\sum_{j=1}^m j x_j^2}}
~\le~
\sqrt{1+\log m}.
\end{align*}
\end{techlem}

\begin{techlem}[Elliptical Potential Inequality \citep{Xiong2023SampleEfficientMR}]
\label{lem:elliptical_potential}
Consider a sequence of vectors $\{x_k\}_{k=1}^K \subset \mathbb{R}^d$. 
Let $\epsilon > 0$ be a regularization parameter, and define the matrix sequence
\[
\Sigma_k \;:=\; \epsilon I + \sum_{s=1}^{k-1} x_s x_s^\top,
\qquad k = 1, \dots, K+1.
\]
Then, the cumulative self-normalized squared norms satisfy
\[
\sum_{k=1}^K \min\!\left\{1,\; \|x_k\|_{\Sigma_k^{-1}}^2 \right\}
\;\le\;
2 \log \!\left( \frac{\det(\Sigma_{K+1})}{\det(\Sigma_1)} \right)
\;\le\;
2 \log \det\!\left( I + \frac{1}{\epsilon} \sum_{k=1}^K x_k x_k^\top \right).
\]
\end{techlem}

\begin{proof}
This result follows from standard properties of self-normalized processes and the monotonicity of the log-determinant function applied to positive definite matrices. A complete argument can be found in classical analyses of regularized least-squares estimators and linear bandit methods. A detailed proof is given in \citep[Lemma 11]{abbasi2011improved}.\qedhere
\end{proof}




\end{document}